\ifdefined\XeTeXversion\else\ifdefined\pdfoutput\pdfoutput=1\fi\fi
\documentclass[11pt]{article}

\usepackage[letterpaper,margin=1in,headheight=12pt,headsep=25pt,footskip=30pt]{geometry}
\usepackage[T1]{fontenc}
\usepackage[utf8]{inputenc}
\usepackage{microtype}
\usepackage{amsmath,amssymb,amsthm,mathtools,bm}
\usepackage{booktabs,array,multirow}
\usepackage{tabularx}
\usepackage{enumitem}
\usepackage{graphicx}
\usepackage{diagbox}
\usepackage{authblk}
\usepackage{algorithm}
\usepackage{algpseudocode}
\usepackage{mathrsfs}
\usepackage{float}
\usepackage{caption}
\usepackage{placeins}
\usepackage{tikz}
\usetikzlibrary{arrows.meta,calc,positioning,shapes.geometric}
\usepackage[round,compress,authoryear]{natbib}
\usepackage{xcolor}
\usepackage{hyperref}
\usepackage[bf,big]{titlesec}
\usepackage[sc]{mathpazo}
\usepackage{setspace}
\usepackage{lmodern}

\definecolor{darkblue}{rgb}{0.0,0.0,0.65}
\definecolor{darkred}{rgb}{0.68,0.05,0.0}
\definecolor{darkgreen}{rgb}{0.0,0.29,0.29}
\definecolor{darkpurple}{rgb}{0.47,0.09,0.29}
\hypersetup{colorlinks=true,citecolor=darkblue,linkcolor=darkred,
  filecolor=darkblue,urlcolor=darkblue,
  pdftitle={Minimax Regret for Sobolev Kernel-Based Reinforcement Learning},
  pdfauthor={Haichen Hu and David Simchi-Levi}}

\allowdisplaybreaks

\newtheoremstyle{mytheoremstyle}{\topsep}{\topsep}{\normalfont}{}
  {\bfseries}{.}{.5em}{}
\theoremstyle{mytheoremstyle}
\newtheorem{theorem}{Theorem}[section]
\newtheorem{lemma}{Lemma}[section]
\newtheorem{proposition}{Proposition}[section]

\newtheorem{assumption}{Assumption}[section]
\newtheorem{definition}{Definition}[section]

\numberwithin{equation}{section}

\makeatletter
\renewcommand\AB@authnote[1]{\textsuperscript{#1}\hspace{5pt}}
\def\maketag@@@#1{\hbox{\m@th\normalfont\normalsize#1}}
\def\@fnsymbol#1{\ensuremath{\ifcase#1\or\natural\or\dagger\or *\or\ddagger\or
  \mathsection\or\mathparagraph\or\|\or **\or\dagger\dagger\or\ddagger\ddagger
  \else\@ctrerr\fi}}
\renewcommand{\thefootnote}{\fnsymbol{footnote}}
\newlength\aftertitskip
\newlength\beforetitskip
\newlength\interauthorskip
\newlength\aftermaketitskip
\renewcommand{\maketitle}{%
  \par\begingroup
  \def\thefootnote{\fnsymbol{footnote}}%
  \def\@makefnmark{\hbox to 4pt{$^{\@thefnmark}$\hss}}%
  \@maketitle\@thanks
  \endgroup
  \setcounter{footnote}{0}%
  \let\maketitle\relax\let\@maketitle\relax
  \gdef\@thanks{}\gdef\@author{}\gdef\@title{}\let\thanks\relax
}
\def\@startauthor{\large\bfseries}
\def\@endauthor{}
\def\@maketitle{\vbox{\hsize\textwidth
  \linewidth\hsize\vskip\beforetitskip
  {\begin{center}\LARGE\bfseries\@title\par\end{center}}\vskip\aftertitskip
  {\begin{center}
    \def\and{\unskip\enspace{\rm and}\enspace}%
    \def\addr{\normalsize\color{black!95}\it}%
    \def\email{\hfill\small\bf\color{darkpurple}\tt}%
    \def\name{\normalsize\color{darkgreen}\bf}%
    \def\AND{\@endauthor\rm\hss\vskip\interauthorskip\@startauthor}%
    \@startauthor\@author\@endauthor
  \end{center}}%
}}
\makeatother
\renewenvironment{abstract}
  {\vskip0.250in\centerline{\large\bfseries Abstract}\vspace{0.5ex}\begin{quote}}
  {\par\end{quote}\vskip1ex}

\renewcommand{\tilde}[1]{\widetilde{#1}}

\usepackage{notation-2}
\renewcommand{\hat}[1]{\widehat{#1}}

\newcommand{\reg}{\text{Reg}}

\newcommand{\pll}{\kern0.3em/\kern-0.9em/\kern0.3em}

\title{\normalfont Minimax Regret for Sobolev Kernel-Based Reinforcement Learning}
\author[1]{Haichen Hu\thanks{{Email: \texttt{huhc@mit.edu}}}}
\author[2]{David Simchi-Levi\thanks{{Email: \texttt{dslevi@mit.edu}}}}
\affil[1]{Center for Computational Science and Engineering, MIT}
\affil[2]{Institute for Data, Systems, and Society, MIT}
\date{}

\begin{document}
\maketitle
\singlespacing

\begin{abstract}
We study the minimax regret rate for episodic kernel-based reinforcement learning, where the expected reward and transition density belong to a Sobolev reproducing kernel Hilbert space (RKHS), with continuous state and action spaces, and adversarial initial states. For Sobolev smoothness parameter $\nu>0$, state dimension $d_S$, and action dimension $d_A$, we prove that the minimax expected regret over $T$ episodes is
$\widetilde\Theta(T^{\max\{e_A,e_B\}})$, where $e_A=\frac{\nu+d_A}{2\nu+d_A}$, $e_B=\frac{\nu+d_A+3d_S/2}{2\nu+d_A+2d_S}$
with the other model parameters fixed. The two exponents reflect the distinct difficulties of action selection and learning across states. We introduce Sobolev Optimal Value Iteration (\texttt{SOVI}), which attains this rate from observed rewards and transitions. For fixed smoothness and reference-kernel lengthscale, its regret bound has a prefactor polynomial in the state-action dimension $d=d_S+d_A$, horizon $H$, reward and transition RKHS norm bounds $B_r,B_p$, and the norm-comparison factor $\Lambda$, up to logarithmic factors; $\Lambda=1$ for normalized Mat\'ern kernels. The algorithm separates initial action selection from continuation-value estimation and combines local polynomial approximation with the shared RKHS structure of the approximation errors. Our matching upper and lower bounds determine the optimal dependence on the number of episodes for Sobolev kernel-based reinforcement learning regret minimization. Thus, we partially address the open problem in \citet{vakili2024} under Sobolev kernel RL.

\end{abstract}

\section{Introduction}
\label{sec:introduction}
Reinforcement learning (RL) provides a framework for learning to make
sequential decisions in unknown environments, with applications in
recommendation systems \citep{zhao2018page,ie2019slateq},
inventory management \citep{oroojlooyjadid2022beer},
and healthcare \citep{komorowski2018clinician}.
Motivated by such applications, a substantial body of work studies
the statistical and computational foundations of RL and related
contextual decision-making problems
\citep[e.g.,][]{azar2017,jin2018,jin2020,jia2020value,
foster2020beyond,jin2021bellman,mou2022markovian,xia2023krylov,
jia2024offline,qian2024offline,hu2025contextual,hu2026model}. \footnote{The authors used GPT and Claude to assist with checking
mathematical arguments and illustrative constructions, revising proofs, and polishing the manuscript. The authors take full responsibility
for all mathematical proofs, references, and other content
of this paper.}

Function approximation enables reinforcement learning in environments with large or continuous state spaces. Its theoretical foundations extend from tabular Markov decision processes (MDPs) \citep{azar2017,jin2018} to linear models \citep{jin2020} and kernel-based models \citep{yang2020provably}. Kernel methods provide rich nonlinear representations while retaining a mathematical structure suitable for analysis. Their connection to wide neural networks through the neural tangent kernel \citep{jacot2018} further motivates understanding their statistical limits. A central question is how the complexity of a kernel model determines the regret achievable through online interaction \citep{vakili2024}.

We consider episodic MDPs with unknown transition densities and unknown expected rewards. A known positive definite kernel $k$ specifies a reproducing kernel Hilbert space (RKHS) $\cH_k$ on the state-action space. For every next state $s'$, the section $P^h(s'\mid\cdot)$ has RKHS norm at most $B_p$, and each expected reward $r^h$ has RKHS norm at most $B_r$. The learner observes reward samples in $[0,1]$ with conditional mean $r^h(s,a)$. We study Sobolev-type kernels, including Mat\'ern kernels, on a state space of dimension $d_S$ and an action space of dimension $d_A$. Write $d:=d_S+d_A$.

A central difficulty is that Bellman regression estimates $r^h+P^hV$ using continuation values $V$ learned from the same observations. Confidence bounds must therefore account for both reward noise and the data dependence of these continuation values. \citet{yang2020provably} address the latter through a covering argument over value functions. For Mat\'ern kernels, their regret bound has order $\widetilde O(T^{\frac{\nu+3d/2}{2\nu+d}})$ at a fixed horizon \citep{maran2024a}, which is not sublinear when $d\ge2\nu$. \citet{chowdhury2023} obtain tighter confidence intervals under optimistic closure, which requires their class of optimistic state-value functions to lie in an RKHS with a uniform norm bound. Using local approximation in a broader smooth-MDP framework, \citet{maran2024a} establish the bound $\widetilde O(T^{(\nu+2d)/(2\nu+2d)})$, which is sublinear for every $\nu>0$. These results motivate a sharper characterization of the regret in kernel-based reinforcement learning.

\citet{vakili2024} formulated this question: does the kernel-based transition assumption permit no-regret learning, and what is the optimal regret growth in the number of episodes and the horizon? The kernel-bandit order $T^{\frac{\nu+d}{2\nu+d}}$ provides a natural candidate when $d$ is the dimension of the entire decision domain \citep{scarlett2017}. In their discussion of lower bounds, \citet{maran2024a} propose this order and $T^{\frac{\nu+2d}{2\nu+2d}}$ as two possibilities for continuous kernel-based RL. In an MDP, however, the learner selects actions while subsequent states are generated by the transition law. Determining the minimax rate therefore requires separating the difficulty of action selection from that of learning across states. In this paper, we narrow down to Sobolev-type RKHSs, including those induced by Matérn kernels, and ask the following question: 

\begin{center}\itshape
What is the minimax regret of Sobolev kernel based reinforcement learning, and how do the state and action dimensions determine this rate?
\end{center}

\paragraph{Our Contributions.}
We prove that the minimax expected regret for Sobolev-kernel RL is
\[
\inf_{\mathsf{Alg}}
\sup_{M\in\mathcal M(B_p,B_r)}
\sup_{\mathsf{Adv}}
\mathbb E^{M,\mathsf{Alg},\mathsf{Adv}}[\reg(T)]
=\widetilde\Theta\!\left(T^{\max(e_A,e_B)}\right),
\]
with $H$ and the other model parameters fixed, where
\[
e_A:=\frac{\nu+d_A}{2\nu+d_A},\ 
e_B:=\frac{\nu+d_A+3d_S/2}{2\nu+d_A+2d_S}.
\]
Here $\mathcal M(B_p,B_r)$ is the model class with parameter conditions
specified in Definition \ref{def:model-class}; $\mathsf{Alg}$ ranges over nonanticipating algorithms,
and $\mathsf{Adv}$ ranges over admissible initial-state adversaries.

Therefore, our result rules out both candidate orders proposed by \citet{maran2024a} for this Sobolev RKHS model and partially solves the open problem in \citet{vakili2024}.

We introduce Sobolev Optimal Value Iteration (\texttt{SOVI}), which attains this rate with unknown rewards, continuous actions, and adversarially chosen initial states. It jointly estimates reward and continuation contributions through Bellman regression, with polynomial dependence on $H$.


\section{Related Works}
\label{sec:related}

\paragraph{Kernel-Based RL.}
A central difficulty in kernel-based RL is that the regression targets involve value functions estimated from the same data. Under Bellman completeness, \citet{yang2020provably} address this issue by covering the class of possible value functions. However, their Mat\'ern regret bound is $\widetilde O(T^{(\nu+3d/2)/(2\nu+d)})$ \citep{maran2024a}, which is not sublinear when $d\geq2\nu$. Here and below, we compare rates at fixed horizon and write $d=d_S+d_A$. \citet[Assumption~1]{chowdhury2023} obtain tighter confidence intervals through kernel mean embeddings under optimistic closure. As pointed out by \citet[Appendix~F]{maran2024a}, the confidence-bound application in \citet[Section~4.2, Eq.~(19)]{vakili2024kernelized} has a gap: it uses a cell-local cover of continuation values, whereas the Bellman targets evaluate these values at next states that need not lie in the current cell's state region. The same issue is discussed by \citet[Appendix~C]{pavlovic2026}.

Other kernel-based regret guarantees use assumptions on mean dynamics \citep{chowdhury2019,curi2020,kakade2020} or transition models in a tensor-product RKHS over the current state--action pair and next state \citep{yangWang2020}, which differ from our pointwise RKHS constraint on transition densities. Related results concern average-reward RL \citep{vakiliOlkhovskaya2024}, environments generated from Gaussian-process priors \citep{bayrooti2025,flynn2026}, and trajectory-preference feedback \citep{pavlovic2026}. Meanwhile, \citet{yeh2023,kayal2025} study sample complexity rather than cumulative online regret. We determine the minimax regret for Sobolev kernel-based RL with unknown reward functions and transition densities, addressing the Sobolev case of the open problem in \citet{vakili2024}.

\paragraph{Smooth MDPs.}
Lipschitz and adaptive-discretization methods exploit regularity of the state-action space \citep{song2019,sinclair2019,cao2020,sinclair2020,touati2020,domingues2021,sinclair2023}, with average-reward extensions studied by \citet{kar2025}. For smoother MDPs, \citet{maran2024c} study regret, while \citet{maran2024b} give sample-complexity guarantees with a generative model. In particular, \citet{maran2024a} establish the regret bound $\widetilde O(T^{(\nu+2d)/(2\nu+2d)})$ for all Mat\'ern kernels through their Mildly Smooth MDP framework. Their framework permits unknown smooth rewards and guarantees sublinear regret, but its upper bound does not match the bandit lower-bound benchmark. For our Sobolev RKHS model, we exploit the norm constraints and bounded state-occupancy densities to obtain matching upper and lower bounds with a smaller regret.

\paragraph{Kernel Bandits and the Minimax Benchmark.}
Kernel-bandit algorithms and information-gain bounds have been studied by \citet{valko2013,vakili2021,whitehouse2023}. For Mat\'ern kernel bandits, the lower bound of \citet{scarlett2017} and matching upper bounds \citep{salgia2021,liScarlett2022} establish the exponent $(\nu+d)/(2\nu+d)$. This rate has also been used as an RL benchmark: \citet{maran2024a} discuss it and their own upper-bound exponent as two candidates for the optimal rate. Neither candidate gives the minimax exponent for our Sobolev RKHS model. Specifically, for $d_S\geq1$ and $H\geq2$, our matching upper and lower bounds have exponent $\max(e_A,e_B)<(\nu+d)/(2\nu+d)$ (Theorems~\ref{thm:regret-upper} and~\ref{thm:regret-lower}). Our minimax rate is therefore polynomially smaller than the full-domain kernel-bandit benchmark. Unlike a kernel bandit, the learner chooses only actions, while subsequent state distributions have bounded densities; the bandit lower bound cannot be transferred unchanged to this model.

\paragraph{Notations.}
We write $[n]:=\{1,\ldots,n\}$ and $\mathbb N_0:=\{0,1,\ldots\}$. For integers $a\le b$, let $[a:b]:=\{a,a+1,\ldots,b\}$. For a finite set $\mathcal X$, $|\mathcal X|$ denotes its cardinality; for a vector $x\in\mathbb R^n$, $|x|=\|x\|_2$ denotes its Euclidean norm. We write $I_n$ for the $n\times n$ identity matrix. For symmetric matrices $A$ and $B$, $A\preceq B$ (or $B\succeq A$) means that $B-A$ is positive semidefinite.
For a set $E$, $\mathbf1_E$ denotes its indicator function. We write $\|f\|_\infty$ for the supremum norm and $\|f\|_{L^p(\mu)}:=(\int |f|^p\,d\mu)^{1/p}$ for $p\in\{1,2\}$. For a measurable space $\mathcal X$, $\Delta(\mathcal X)$ denotes the set of probability distributions on $\mathcal X$. For random variables $X_1,\ldots,X_n$, $\sigma(X_1,\ldots,X_n)$ denotes the $\sigma$-field they generate. The Gamma function is $\Gamma(x):=\int_0^\infty t^{x-1}e^{-t}\,dt$ for $x>0$.
We use $O(\cdot)$ for upper bounds up to constant factors. The notation $\widetilde O(\cdot)$ additionally suppresses logarithmic factors, while $\widetilde\Theta(\cdot)$ denotes matching upper and lower bounds up to constant and logarithmic factors.

\section{Sobolev Kernel-Based Reinforcement Learning}
\label{sec:setup}

In this section, we describe the interaction protocol and the Sobolev RKHS assumptions on the unknown rewards and transitions.

An episodic Markov decision process (MDP) is specified by $M=(\cS,\cA,H,\{P^h\}_{h=1}^{H-1},\{r^h\}_{h=1}^{H})$, where $\cS$ is the state space, $\cA$ is the action space, and $H\ge2$ is the number of steps in each episode. Write $\cZ:=\cS\times\cA$ for the state-action space. At step $h<H$, $P^h(\cdot\mid s,a)$ specifies the distribution of the next state after the decision maker (DM) takes action $a$ in state $s$. The function $r^h:\cZ\to[0,1]$ specifies the expected reward at step $h$. Both the reward functions and transition kernels are unknown to the DM.

The state space $\cS$ is equipped with a finite reference measure $\mu$. The action space $\cA$ is equipped with its Borel $\sigma$-field, and $\cZ$ carries the product $\sigma$-field. Each transition kernel admits a jointly measurable density with respect to $\mu$, also denoted by $P^h(s'\mid s,a)$. Thus, for every $h\in[H-1]$, $z\in\cZ$, and measurable $E\subseteq\cS$,
\[
\int_{\cS}P^h(s'\mid z)\mu(ds')=1,\ P^h(E\mid z)=\int_E P^h(s'\mid z)\mu(ds').
\]
For any bounded measurable $V:\cS\to\mathbb R$, we write $(P^hV)(z):=\int_{\cS}V(s')P^h(s'\mid z)\mu(ds'),\ z\in\cZ$,
and set $P^HV\equiv0$.

A randomized Markov policy $\pi=(\pi^1,\ldots,\pi^H)$ is a sequence of measurable stochastic kernels from $\cS$ to $\cA$, where $\pi^h(\cdot\mid s)\in\Delta(\cA)$ specifies the action distribution at state $s$ and step $h$. A history-dependent policy may additionally depend on the states, actions, and rewards observed earlier in the episode. Let $\Pi$ denote the class of randomized, history-dependent policies, which contains all Markov policies.

The DM interacts with the MDP over $T\ge2$ episodes. At the beginning of episode $t$, an adaptive adversary selects an initial state $s_t^1$. The DM observes $s_t^1$ and chooses a policy $\pi_t\in\Pi$. At each step $h\in[H]$, the DM observes $s_t^h$, selects $a_t^h$ according to $\pi_t$, and receives a reward $R_t^h\in[0,1]$. For a Markov policy, the action is sampled from $\pi_t^h(\cdot\mid s_t^h)$. If $h<H$, the DM also observes the next state $s_t^{h+1}$. For simplicity, we write the joint state action pairs as
\[
z_t^h:=(s_t^h,a_t^h),\ z_t^{1:H}:=(z_t^1,\ldots,z_t^H).
\]

Let $\mathcal I_t^h$ denote the information available after choosing $a_t^h$ and before observing the reward and next state. We assume
\begin{equation}
\mathbb E[R_t^h\mid\mathcal I_t^h]=r^h(z_t^h),\ h\in[H],\ s_t^{h+1}\mid\mathcal I_t^h\sim P^h(\cdot\mid z_t^h),\ h<H.
\label{eq:observation-law}
\end{equation}
The observation mechanism is fixed for each model and depends only on the current stage and state-action pair; we leave its distribution implicit in $M$. The reward and next state may be dependent. Deterministic unknown rewards are included by taking $R_t^h=r^h(z_t^h)$.

For a model $M$, a policy $\pi\in\Pi$, and an initial state $s$, write $\mathbb P_s^{M,\pi}$ and $\mathbb E_s^{M,\pi}$ for probability and expectation under the process initialized at $s^1=s$. Let $\cH_1$ be the trivial $\sigma$-field. The history before episode $t$ is $\cH_t:=\sigma\left(z_i^h,R_i^h:i<t,\ h\in[H]\right)$. The adversary selects $s_t^1$ using this history and its own private randomization, without access to the DM's private randomness or future observations. The initial states may therefore be fixed, i.i.d., or chosen adaptively.

Now, we are ready to introduce the more specific kernel-based RL setup. Throughout the paper, we take $\cS=[0,1]^{d_S}$ with $d_S\ge1$ and use Lebesgue measure for $\mu$. We take $\cA=[0,1]^{d_A}$ with $d_A\ge1$ as our continuous action space. Thus, $\cZ=[0,1]^{d_S+d_A}$. 

Let $k:\cZ\times\cZ\to\mathbb R$ be a measurable positive definite kernel satisfying $k(z,z)\le1$ for every $z\in\cZ$. Write $k_z:=k(z,\cdot)$ and define $\cH_{k,0}:=\operatorname{span}\{k_z:z\in\cZ\}$. It induces an inner product on $\cH_{k,0}$:
\[
\left\langle\sum_{i=1}^{m}a_i k_{z_i},\sum_{j=1}^{n}b_j k_{w_j}\right\rangle_{\cH_{k,0}}:=\sum_{i=1}^{m}\sum_{j=1}^{n}a_i b_j k(z_i,w_j),
\]
for all $m,n\ge1$, $a_i,b_j\in\mathbb R$, and $z_i,w_j\in\cZ$. The associated reproducing kernel Hilbert space (RKHS) $\cH_k$ is the completion of $\cH_{k,0}$ under the induced norm $\|f\|_{\cH_{k,0}}:=\sqrt{\langle f,f\rangle_{\cH_{k,0}}}$. Its elements are identified with functions on $\cZ$ through pointwise limits. We denote the extended inner product and norm by $\langle\cdot,\cdot\rangle_{\cH_k}$ and $\|\cdot\|_{\cH_k}$. For every $f\in\cH_k$ and $z,z'\in\cZ$, we have
\[
k_z\in\cH_k,\ \langle k_z,k_{z'}\rangle_{\cH_k}=k(z,z'),\ f(z)=\langle f,k_z\rangle_{\cH_k}.
\]
Consequently, we know that $|f(z)|\le\|f\|_{\cH_k}\sqrt{k(z,z)}\le\|f\|_{\cH_k}$.
Now, we impose the following RKHS condition on the unknown transitions \citep{vakili2024,vakili2024kernelized,yeh2023}.

\begin{assumption}
\label{ass:kernel-transition}
There is a known constant $B_p<\infty$ such that $\forall h\in[H-1],\ \forall s'\in\cS$, we have
\[
P^h(s'\mid\cdot)\in\cH_k,\ \|P^h(s'\mid\cdot)\|_{\cH_k}\le B_p.
\]
\end{assumption}

This condition is an RKHS analog of the linear MDP \citep{jin2020} and is widely adopted in kernel-based RL \citep{vakili2024}. It constrains the dependence of the transition density on the current state-action pair, uniformly over the next state $s'$.

We assume that the expected reward functions belong to the same RKHS.

\begin{assumption}
\label{ass:kernel-reward}
There is a constant $0\le B_r<\infty$ such that $r^h\in\cH_k,\ \|r^h\|_{\cH_k}\le B_r,\ \forall h\in[H]$.
\end{assumption}

We furthermore retain the following regularity condition.

\begin{assumption}
\label{ass:action-regularity}
The map $(a,a')\mapsto k((s,a),(s,a'))$ is continuous on $\cA\times\cA$, and $a\mapsto r^h(s,a)$ is continuous on $\cA$, for every $s\in\cS$ and $h\in[H]$.
\end{assumption}

In this paper, we specifically focus on the Sobolev-type kernels, whose RKHSs are norm-equivalent to Sobolev spaces \citep{adamsFournier2003}. For $n\ge1$ and $\tau>n/2$, we define
\[
\begin{aligned}
H^\tau(\mathbb R^n)&:=\left\{F\in L^2(\mathbb R^n):\|F\|_{H^\tau(\mathbb R^n)}<\infty\right\},\ \|F\|_{H^\tau(\mathbb R^n)}^2&:=\int_{\mathbb R^n}(1+4\pi^2|\varpi|^2)^\tau|\widehat F(\varpi)|^2\ d\varpi.
\end{aligned}
\]
Here $\widehat F$ denotes the $L^2$ Fourier transform, with the convention $\widehat F(\varpi)=\int_{\mathbb R^n}F(x)e^{-2\pi i x\cdot\varpi}\ dx$ for $F\in L^1(\mathbb R^n)\cap L^2(\mathbb R^n)$. Since $\tau>n/2$, each element of $H^\tau(\mathbb R^n)$ has a continuous representative, which we use throughout. For $\mathcal X\subseteq\mathbb R^n$, we define
\[
H^\tau(\mathcal X):=\{F|_{\mathcal X}:F\in H^\tau(\mathbb R^n)\},\ \|g\|_{H^\tau(\mathcal X)}:=\inf_{\substack{F\in H^\tau(\mathbb R^n)\\F|_{\mathcal X}=g}}\|F\|_{H^\tau(\mathbb R^n)}.
\]
The restriction norm makes $H^\tau(\mathcal X)$ a Hilbert space. Since $\tau>n/2$, point evaluations are bounded \citep{adamsFournier2003}, so $H^\tau(\mathcal X)$ is also an RKHS \citep{aronszajn1950}. This motivates considering kernels whose RKHSs coincide with $H^\tau(\mathcal X)$ as sets of functions with equivalent norms.

\begin{definition}
\label{def:sobolev-kernel}
A positive definite kernel $k$ on $\mathcal X\subseteq\mathbb R^n$ is of Sobolev type of order $\tau>n/2$ if $\cH_k(\mathcal X)=H^\tau(\mathcal X)$ as sets of functions, and there are constants $0<c_1\le c_2<\infty$ such that
\[
c_1\|g\|_{H^\tau(\mathcal X)}\le\|g\|_{\cH_k(\mathcal X)}\le c_2\|g\|_{H^\tau(\mathcal X)},\ g\in H^\tau(\mathcal X).
\]
We use $\nu:=\tau-n/2>0$ as the smoothness parameter. Thus, smoothness $\nu$ corresponds to Sobolev order $\nu+n/2$ in dimension $n$.
\end{definition}
For a fixed $\nu>0$, let $d:=d_S+d_A$ denote the dimension of the joint state-action variable $z=(s,a)$. We take $k$ on $\cZ$ to be a Sobolev type kernel of order $\nu+d/2$, with $k(z,z)\le1$. The smoothness condition, therefore, concerns the joint state-action variables.
For these Sobolev kernels, Assumption~\ref{ass:action-regularity}
holds automatically: the kernel and every RKHS function are continuous,
as shown by the Sobolev argument following
Lemma~\ref{lem:continuity-compactness} in
Appendix~\ref{app:mathematical-tools}.

\begin{example}
\label{ex:Matern}
\label{lem:matern-sobolev-type}
Let $n\ge1$, $\nu>0$, $\ell>0$, and $m:=\nu+n/2$. The normalized Mat\'ern kernel is $k(x,x'):=\kappa_\nu(x-x')$, where
\[
\kappa_\nu(x)=\frac{2^{1-\nu}}{\Gamma(\nu)}\left(\frac{\sqrt{2\nu}|x|}{\ell}\right)^\nu K_\nu\!\left(\frac{\sqrt{2\nu}|x|}{\ell}\right),\ x\ne0,
\]
$K_\nu$ is the modified Bessel function of the second kind, and $\kappa_\nu(0)=1$ by continuity. This kernel is continuous and positive definite, with spectral density \citep{rasmussen2006}
\[
\widehat\kappa_\nu(\varpi)=C_{\nu,n}\left(2\nu/\ell^2+4\pi^2|\varpi|^2\right)^{-m},\ C_{\nu,n}:=\frac{2^n\pi^{n/2}\Gamma(m)(2\nu)^\nu}{\Gamma(\nu)\ell^{2\nu}}.
\]
Set $b:=\sqrt{2\nu}/\ell$. Since $\min\{1,b^2\}(1+4\pi^2|\varpi|^2)\le b^2+4\pi^2|\varpi|^2\le\max\{1,b^2\}(1+4\pi^2|\varpi|^2)$, Lemma~\ref{lem:fourier-criterion} and \ref{lem:rkhs-uniqueness-restriction} imply that the Mat\'ern kernel $k$ is a Sobolev kernel of order $m=\nu+n/2$ on every $\mathcal X\subseteq\mathbb R^n$, with norm-equivalence constants $c_1=\sqrt{\frac{\min\{1,b^2\}^{m}}{C_{\nu,n}}},\ c_2=\sqrt{\frac{\max\{1,b^2\}^{m}}{C_{\nu,n}}}$.
\end{example}

Our model, therefore, includes Mat\'ern kernel-based reinforcement learning \citep{yang2020provably,vakili2024kernelized} as a special case.

We abstract the set of all the MDPs that satisfies the conditions above as our model class.
\begin{definition}
\label{def:model-class}
For $B_p>0$ and $B_r\ge0$, let $\mathcal M(B_p,B_r)$ be the class of Sobolev kernel-based MDPs satisfying
Assumptions~\ref{ass:kernel-transition}-\ref{ass:action-regularity}
with norm bounds $B_p$ and $B_r$.
\end{definition}
If $\mathcal M(B_p,B_r)$ is nonempty, transition normalization and
Lemma~\ref{lem:bellman-completeness} imply
$B_p\ge\|\mathbf1_{\cZ}\|_{\cH_k}\ge1$; the proof of that lemma
verifies both inequalities.
In this paper, the learner knows that the underlying MDP lies in $\cM(B_p,B_r)$, but does not know which one is the true model. He needs to learn the unknown reward and transition probability kernel from the sequential interaction process with the environment described above.

\paragraph{Main Result.}
We prove that our algorithm \texttt{SOVI} achieves a
high-probability regret upper bound of order
$\widetilde O(T^{\max\{e_A,e_B\}})$ in Theorem~\ref{thm:regret-upper}. For fixed smoothness and lengthscale of the Mat\'ern reference kernel introduced in Section~\ref{sec:skovi}, the prefactor is polynomial in
$d,H,B_p,B_r$ and the kernel norm-comparison factor $\Lambda$
defined in Section~\ref{sec:skovi}, up to logarithmic factors.

Conversely, from Theorem~\ref{thm:regret-lower}, if $B_r>0$ and $B_p>\|\mathbf1_{\cZ}\|_{\cH_k}$,
every algorithm incurs the worst-case expected regret of order
$\Omega(T^{\max\{e_A,e_B\}})$. Consequently, under these
conditions, and with all model parameters other than $T$ fixed, we show that
\[
\inf_{\mathsf{Alg}}
\sup_{M\in\mathcal M(B_p,B_r)}
\sup_{\mathsf{Adv}}
\mathbb E^{M,\mathsf{Alg},\mathsf{Adv}}[\reg(T)]
=\widetilde\Theta\!\left(T^{\max(e_A,e_B)}\right),
\]
where $e_A=\frac{\nu+d_A}{2\nu+d_A}$ and
$e_B=\frac{\nu+d_A+3d_S/2}{2\nu+d_A+2d_S}$.
The dominant exponent changes at $d_A=2\nu$, since
\[
e_A-e_B=
\frac{d_S(d_A-2\nu)}
{2(2\nu+d_A)(2\nu+d_A+2d_S)}.
\]
When $d_A>2\nu$, $e_A$ dominates and the exponent of $T$
is independent of $d_S$. When $d_A<2\nu$, $e_B$ dominates,
so both state and action dimensions enter the exponent.
At $d_A=2\nu$, the two exponents coincide at $3/4$. Formally setting $d_S=0$, both exponents reduce to $(\nu+d_A)/(2\nu+d_A)$, the kernel-bandit exponent.

This distinction reflects the different roles of states and actions.
A small action region may contain the optimal action, so its volume
alone does not make it negligible. From step two onward, every policy
visits a state region $E$ with probability at most $B_p\mu(E)$
(Lemma~\ref{lem:state-occupancy-bound}). Thus, a bounded error confined
to $E$ contributes to expected regret in proportion to its visitation
probability, allowing continuation errors to be averaged over states.

The shared Sobolev norm also couples the difficulty across states:
spreading independent local perturbations over more state regions
requires smaller amplitudes. The two exponents capture searching for
good actions at a fixed state and learning which actions are good
across many states. Their maximum determines the minimax rate.

Now, we return to the reinforcement learning setup. We define the value functions and regret in reinforcement learning \citep{puterman2014}. For a Markov policy $\pi$, set $V^{H+1}(\cdot;\pi)\equiv0$, and for $h=H,\ldots,1$, define
\[
Q^h(s,a;\pi):=r^h(s,a)+(P^hV^{h+1}(\cdot;\pi))(s,a),\ V^h(s;\pi):=\int_{\cA}Q^h(s,a;\pi)\pi^h(da\mid s).
\]
These functions give the expected total reward from step $h$ onward. For a history-dependent policy $\pi\in\Pi$, define its value at the beginning of an episode by
\[
V^1(s;\pi):=\mathbb E_s^{M,\pi}\left[\sum_{h=1}^H r^h(s^h,a^h)\right].
\]
By \eqref{eq:observation-law}, this is also the expected sum of observed rewards.

The optimal values are defined by $V_*^{H+1}\equiv0$ and the Bellman optimality equations
\begin{equation}
Q_*^h(s,a):=r^h(s,a)+(P^hV_*^{h+1})(s,a),\ V_*^h(s):=\max_{a\in\cA}Q_*^h(s,a),\ h=H,\ldots,1.
\label{eq:bellman-optimality}
\end{equation}
The following lemma establishes that the maxima are attained and identifies the optimal value over history-dependent policies.

\begin{lemma}
\label{lem:optimal-values}
Under Assumptions~\ref{ass:kernel-transition}, \ref{ass:kernel-reward}, and \ref{ass:action-regularity}, the functions $Q_*^h$ and $V_*^h$ are measurable. There exists a deterministic Markov policy $\pi_*$ such that
\[
V^h(s;\pi_*)=V_*^h(s),\ s\in\cS,\ h\in[H].
\]
Moreover, we have $V_*^1(s)=\sup_{\pi\in\Pi}V^1(s;\pi),\ s\in\cS$.
\end{lemma}

We measure the DM's performance by the cumulative regret
\begin{equation}
\reg(T):=\sum_{t=1}^T\left[V_*^1(s_t^1)-V^1(s_t^1;\pi_t)\right].
\label{eq:regret}
\end{equation}
This quantity is random through the selected policies and initial states. Its expectation $\mathbb E[\reg(T)]$ is taken over the reward and transition observations as well as the randomization of the DM and the adversary.

For a bounded measurable $V:\cS\to\mathbb R$ and the Bellman operator $\cB^h$, we define the Bellman backup
\begin{equation}
(\mathcal B^hV)(z):=r^h(z)+(P^hV)(z),\ z\in\cZ,\ h\in[H].
\label{eq:bellman-backup}
\end{equation}
In particular, $\mathcal B^H0=r^H$. The RKHS assumptions control both the transition term and the full backup.

\begin{lemma}
\label{lem:bellman-completeness}
Under Assumptions~\ref{ass:kernel-transition} and~\ref{ass:kernel-reward}, for every bounded measurable $V:\cS\to\mathbb R$ and $h\in[H-1]$,
\begin{equation}
P^hV\in\cH_k,\ \|P^hV\|_{\cH_k}\le B_p\int_{\cS}|V(s')|\mu(ds')\le B_p\|V\|_\infty.
\label{eq:bellman-completeness}
\end{equation}
For every $h\in[H]$, we also have $\mathcal B^hV\in\cH_k$ and
\begin{equation}
\|\mathcal B^hV\|_{\cH_k}\le B_r+B_p\int_{\cS}|V(s')|\mu(ds')\le B_r+B_p\|V\|_\infty.
\label{eq:bellman-backup-norm}
\end{equation}
\end{lemma}

Since the expected reward at each step lies in $[0,1]$, the continuation value at stage $h+1$ is at most $H-h$. Consequently, for every Markov policy $\pi$ and $h\in[H]$, both $Q_*^h$ and $Q^h(\cdot;\pi)$ belong to $\cH_k$, with
\begin{equation}
\max\left\{\|Q_*^h\|_{\cH_k},\ \|Q^h(\cdot;\pi)\|_{\cH_k}\right\}\le B_r+B_p(H-h).
\label{eq:q-rkhs-bound}
\end{equation}

The transition assumption also bounds the state-occupancy densities after the first step.

\begin{lemma}
\label{lem:state-occupancy-bound}
Under Assumption~\ref{ass:kernel-transition} and the normalization $k(z,z)\le1$,
\[
0\le P^h(s'\mid z)\le B_p,\ h\in[H-1],\ z\in\cZ,\ s'\in\cS.
\]
For every policy $\pi\in\Pi$, initial state $s\in\cS$, and $h\in\{2,\ldots,H\}$, the distribution of $s^h$ under $\mathbb P_s^{M,\pi}$ admits a density with respect to $\mu$ bounded by $B_p$. Equivalently, for every measurable $E\subseteq\cS$,
\begin{equation}
\mathbb P_s^{M,\pi}(s^h\in E)\le B_p\mu(E).
\label{eq:state-occupancy-bound}
\end{equation}
The same bound holds conditionally during interaction $\mathbb P(s_t^h\in E\mid\mathcal I_t^1)\le B_p\mu(E), h\in[2:H]$.
\end{lemma}

\section{\texttt{SOVI}: Sobolev Optimal Value Iteration Overview}
\label{sec:skovi}

For the continuous-action model in Section~\ref{sec:setup}, with $H\ge2$, we introduce Sobolev Optimal Value Iteration (\texttt{SOVI}), which learns from observed rewards and transitions. In each episode, \texttt{SOVI} retains a version space of candidate Q-functions consistent with past observations. Given the initial state, it selects a first action and a candidate Q-function, then follows the candidate's greedy policy from step two onward.

Our algorithm separates the loss from first-action selection under adversarial initial states from the loss due to value function estimation. The first-step selection rule controls the former, and value iteration controls the latter. We therefore describe these components separately.

This section introduces the candidate Q-functions, the decision
objective, and the local polynomial estimation tools shared by the
two components of our algorithm.
Section~\ref{sec:sovi-version-space} constructs the version space
from past observations.
Section~\ref{sec:sovi-first-selection} uses this version space to
select the first action and a candidate Q-function.
Proofs of the lemmas in these three sections are deferred to
Appendix~\ref{app:section-4-proofs}.

\subsection{Candidate Q-Functions and the Decision Objective}

Value iteration uses a continuation value to evaluate each action. Since these continuation values are learned from data, we begin with a fixed class of candidate Q-functions. Using the stagewise bounds from Section~\ref{sec:setup}, define
\begin{equation}
\mathcal F^h:=\left\{g\in\cH_k:\|g\|_{\cH_k}\le B_r+B_p(H-h),\ 0\le g\le H-h+1\right\},\ \mathcal F:=\prod_{h=1}^H\mathcal F^h.
\label{eq:candidate-class}
\end{equation}
Each candidate Q-function $f=(f^1,\ldots,f^H)\in\mathcal F$ has one component $f^h:\cS\times\cA\to\mathbb R$ for each stage $h$. Its induced state-value functions and greedy policy are defined by
\[
V_f^h(s):=\max_{a\in\cA}f^h(s,a),\ \pi_f^h(s)\in\operatorname*{arg\,max}_{a\in\cA}f^h(s,a),\ V_f^{H+1}:=0.
\]
We fix measurable greedy selectors throughout. For the zero candidate, choose a fixed action $a_0\in\cA$ at every state and stage. By \eqref{eq:q-rkhs-bound}, $Q_*:=(Q_*^1,\ldots,Q_*^H)\in\mathcal F$. The terminal component is also learned from reward observations.

We define the class of induced state-value functions used as continuations by $\mathcal V_h:=\{V_f^{h+1}:f\in\mathcal F\}\cup\{0\},\ h\in[H]$.

Let $\mathcal F_t\subseteq\mathcal F$ denote the version space retained in episode $t$, constructed in Section~\ref{sec:sovi-version-space}. The continuation values induced by its candidates belong to $\mathcal V_h$. For a nonempty $\mathcal F_t$ and initial state $s_t^1$, the ideal decision objective is
\begin{equation}
\sup_{(a,f)\in\cA\times\mathcal F_t}(\mathcal B^1V_f^2)(s_t^1,a).
\label{eq:initial-objective}
\end{equation}
Here $\mathcal B^h$ is the true Bellman operator: $(\mathcal B^hV)(z)=r^h(z)+(P^hV)(z)$, with continuation term zero at $h=H$. The objective evaluates the first-step reward and transition using the true model, then uses $f$ to evaluate the continuation. Whenever $Q_*\in\mathcal F_t$, its value is at least $V_*^1(s_t^1)$. Since the model is unknown, we estimate these backups from observed rewards and next states.

\subsection{Approximating Bellman Backups with Local Legendre Polynomials}

For a continuation value $V:\cS\to[0,H-h]$, the observed Bellman response is
\begin{equation}
Y_i^h(V):=
\begin{cases}
R_i^h+V(s_i^{h+1}),&h<H,\\
R_i^H,&h=H.
\end{cases}
\label{eq:bellman-response}
\end{equation}
At $h=H$, the prescribed range forces $V\equiv0$. For each fixed $V$, the observation model gives
\[
\mathbb E[Y_i^h(V)\mid\mathcal I_i^h]=(\mathcal B^hV)(z_i^h),\ 0\le Y_i^h(V)\le H-h+1.
\]
Thus, regressing the full response estimates both the reward and continuation contributions. The learner can evaluate this response for every candidate continuation after observing $(R_i^h,s_i^{h+1})$.

\paragraph{Local features for Bellman backup approximation.}
We use low-degree polynomials to represent the local variation of a Bellman backup and retain its remaining component through a kernel feature map. This keeps the number of regression coefficients polynomial in $d:=d_S+d_A$ for fixed $\nu$.

We construct the local features using a normalized Mat\'ern reference kernel $k_0$ of smoothness $\nu$. Its lengthscale $\ell_0>0$ controls decay with distance, and we write $\kappa:=\sqrt{2\nu}/\ell_0$ for its inverse lengthscale parameter. These parameters describe the reference kernel $k_0$; a general Sobolev kernel $k$ need not itself have a lengthscale parameter.

For a general Sobolev kernel $k$ satisfying Definition~\ref{def:sobolev-kernel}, we choose $k_0$ to be the normalized Mat\'ern kernel of smoothness $\nu$ and lengthscale $\ell_0=\sqrt{2\nu}$, so that $\kappa=1$. By the spectral representation of the Mat\'ern kernel \citep{rasmussen2006} and the Sobolev norm comparison, this choice satisfies
\begin{equation}
\|g\|_{\cH_{k_0}}
\le \Lambda\|g\|_{\cH_k},
\ \forall g\in\cH_k,
\ 
\Lambda:=
\frac1{c_1}
\left\{
\frac{\Gamma(\nu)}
{(4\pi)^{d/2}\Gamma(\nu+d/2)}
\right\}^{1/2}.
\label{eq:reference-kernel-norm}
\end{equation}
Here $d=d_S+d_A$, $\Gamma$ denotes the Gamma function, and $c_1>0$ is a known valid lower norm-comparison constant satisfying $c_1\|g\|_{H^{\nu+d/2}(\cZ)}
\le \|g\|_{\cH_k},
\ \forall g\in\cH_k$.

If the original kernel $k$ is itself a normalized Mat\'ern kernel of smoothness $\nu$, we instead take $k_0=k$ and $\Lambda=1$. In this case, $\ell_0$ is the known lengthscale specified by $k$, and $\kappa=\sqrt{2\nu}/\ell_0$. Thus, in either case, $k_0$, $\kappa$, and $\Lambda$ are fixed by the known kernel specification before data collection.

We write $\Phi(z):=k_0(z,\cdot)$ for the kernel feature map and $\mathcal H_0:=\cH_{k_0}$ for the associated RKHS. The feature $\Phi(z)$ is a function indexed by $z$: for every $u$, $k_0(z,u)=\Phi(z)(u)$.

We partition $\cZ$ for both Bellman backup approximation and first-action selection, using the following definition.
\begin{definition}
\label{def:partition}
A finite measurable partition of $\cZ$ is a finite collection $\mathcal C$ of nonempty, pairwise disjoint measurable sets satisfying $\bigcup_{c\in\mathcal C}c=\cZ$.
\end{definition}

For each cell $c\in\mathcal C$, choose an axis-aligned rectangle $\overline c:=\prod_{q=1}^d[\ell_{c,q},u_{c,q}]$ with $c\subseteq\overline c\subseteq\cZ$ and $\ell_{c,q}<u_{c,q}$ for all $q\in[d]$.
We also put $\rho:=\lfloor\nu\rfloor,\ D:=\binom{d+\rho}{\rho},\ B_{\phi}^2:=3^\rho D,\ \lambda:=B_{\phi}^2$.

On $\overline c$, we use the tensor-product Legendre polynomials of total degree at most $\rho$. Their orthonormality and projection properties are stated in Lemma~\ref{lem:local-legendre-projection}. We define these polynomials below.
\begin{definition}
\label{def:local-legendre}
Let $\mathcal C$ be a finite measurable partition of $\cZ=[0,1]^d$. For each cell $c\in\mathcal C$, choose a bounding rectangle $\overline c:=\prod_{q=1}^d[\ell_{c,q},u_{c,q}]\subseteq\cZ$ containing $c$, with $\ell_{c,q}<u_{c,q}$. Let $P_j(x):=(2^j j!)^{-1}\frac{d^j}{dx^j}(x^2-1)^j$ be the degree-$j$ Legendre polynomial on $[-1,1]$. For $\mathbf b\in\mathbb N_0^d$, the local tensor Legendre polynomial is defined as
\[
p_{c,\mathbf b}(z):=\prod_{q=1}^d\sqrt{2b_q+1}\,P_{b_q}\left(\frac{2(z_q-\ell_{c,q})}{u_{c,q}-\ell_{c,q}}-1\right),\ z\in\overline c.
\]
These polynomials are orthonormal under normalized Lebesgue measure on the bounding rectangle:
\[
\frac1{|\overline c|}\int_{\overline c}p_{c,\mathbf b}(z)p_{c,\mathbf b'}(z)\,dz=\mathbf1\{\mathbf b=\mathbf b'\},\ |\overline c|:=\prod_{q=1}^d(u_{c,q}-\ell_{c,q}).
\]
\end{definition}
For $\rho\in\mathbb N_0$, the local feature vector on $c$ is $\phi_c(z):=(p_{c,\mathbf b}(z))_{|\mathbf b|\le\rho}$, where $|\mathbf b|:=\sum_{q=1}^d b_q$. It contains $D=\binom{d+\rho}{\rho}$ polynomials of total degree at most $\rho$, evaluated at $z\in c$. Its norm satisfies
\begin{equation}
\begin{aligned}
\sup_{z\in\overline c}\|\phi_c(z)\|_2^2=\sup_{z\in\overline c}\sum_{|\mathbf b|\le\rho}|p_{c,\mathbf b}(z)|^2
\le\sum_{|\mathbf b|\le\rho}\prod_{q=1}^d(2b_q+1)\le\sum_{|\mathbf b|\le\rho}3^{|\mathbf b|}
\le D3^\rho=B_{\phi}^2.
\end{aligned}
\label{eq:local-features}
\end{equation}
To separate the local polynomial component from its remainder,
define the projection coefficients by
\[
\zeta_{c,\mathbf b}:=
\frac1{|\overline c|}
\int_{\overline c}\Phi(u)p_{c,\mathbf b}(u)\,du,
\  |\mathbf b|\le\rho.
\]
We define $\Psi_{\mathcal C}$ pointwise through
\begin{equation}
\Phi(z)=
\sum_{|\mathbf b|\le\rho}p_{c,\mathbf b}(z)\zeta_{c,\mathbf b}
+\Psi_{\mathcal C}(z),\ z\in c.
\label{eq:feature-decomposition}
\end{equation}
All these features depend only on the known kernel and partition.

For $g\in\mathcal H_0$, set
$\theta_c(g):=
(\langle g,\zeta_{c,\mathbf b}\rangle_{\mathcal H_0})_{|\mathbf b|\le\rho}$.
By the reproducing property and \eqref{eq:feature-decomposition},
we have
\[
g(z)=\langle g,\Phi(z)\rangle_{\mathcal H_0}
=\phi_c(z)^\top\theta_c(g)
+\langle g,\Psi_{\mathcal C}(z)\rangle_{\mathcal H_0},\ z\in c.
\]
By the orthonormality in Lemma~\ref{lem:local-legendre-projection}
and Bessel's inequality, we have
$\|\theta_c(g)\|_2^2
\le|\overline c|^{-1}\int_{\overline c}|g(u)|^2\,du
\le\|g\|_{L^\infty(\overline c)}^2$.
The following lemma bounds the remainder in
\eqref{eq:feature-decomposition} when all bounding rectangles
have maximum side length at most $a$.
\begin{lemma}
\label{lem:feature-remainder}
Let $k_0$ be a normalized Mat\'ern kernel with smoothness $\nu>0$ and inverse lengthscale $\kappa>0$. Set $\rho:=\lfloor\nu\rfloor$, $D:=\binom{d+\rho}{\rho}$, and $B_{\phi}:=\sqrt{3^\rho D}$. If every bounding rectangle $\overline c$ has maximum side length at most $a>0$, then the remainder after local polynomial projection of degree $\rho$ satisfies
\[
\sup_{z\in\cZ}\|\Psi_{\mathcal C}(z)\|_{\mathcal H_0}\le C_\nu(1+B_{\phi})\min\{1,(\kappa\sqrt d\,a)^\nu\},
\]
where $C_\nu$ depends only on $\nu$.
\end{lemma}

We now use these local features to estimate Bellman backups by ridge regression.

\paragraph{Local polynomial ridge regression.}
At each stage $h\in[H]$, fix a deterministic measurable continuation $V:\cS\to[0,H-h]$ and use the responses in \eqref{eq:bellman-response}.
For each $c\in\mathcal C$, let $\mathcal T_t^h(c):=\{i\in[t-1]:z_i^h\in c\}$ index the past observations at step $h$ in cell $c$, and let $n_t^h(c):=|\mathcal T_t^h(c)|$ denote its size. The ridge regression objective is
\begin{equation}
\mathcal L_t^h(\theta;V,c):=\sum_{i\in\mathcal T_t^h(c)}\bigl(Y_i^h(V)-\phi_c(z_i^h)^\top\theta\bigr)^2+\lambda\|\theta\|_2^2.
\label{eq:local-ridge-loss}
\end{equation}
With $G_t^h(c):=\lambda I_D+\sum_{i\in\mathcal T_t^h(c)}\phi_c(z_i^h)\phi_c(z_i^h)^\top$, its minimizer and prediction are
\begin{equation}
\widehat\theta_t^h(V,c)=G_t^h(c)^{-1}\sum_{i\in\mathcal T_t^h(c)}\phi_c(z_i^h)Y_i^h(V),\ \widehat{\mathcal B}_t^h[V](z):=\phi_c(z)^\top\widehat\theta_t^h(V,c).
\label{eq:local-ridge-estimate}
\end{equation}
For each fixed $V$, by \eqref{eq:observation-law} and \eqref{eq:bellman-response}, we have
\[
\mathbb E[Y_i^h(V)\mid\mathcal I_i^h]=r^h(z_i^h)+(P^hV)(z_i^h)=(\mathcal B^hV)(z_i^h).
\]
Thus, $\widehat{\mathcal B}_t^h[V]$ estimates $\mathcal B^hV$ by locally regressing the observed responses on the polynomial features.
 
For $z\in c$, we set
\begin{equation}
\begin{aligned}
w_t^h(z):=\sqrt{\phi_c(z)^\top G_t^h(c)^{-1}\phi_c(z)},\ 
g_t^h(z):=\Psi_{\mathcal C}(z)-\sum_{i\in\mathcal T_t^h(c)}\phi_c(z)^\top G_t^h(c)^{-1}\phi_c(z_i^h)\Psi_{\mathcal C}(z_i^h).
\end{aligned}
\label{eq:profiled-features}
\end{equation}
The width $w_t^h$ measures uncertainty in the polynomial fit. The feature $g_t^h(z)$ subtracts from $\Psi_{\mathcal C}(z)$ the contribution of the historical kernel remainders to the fitted prediction at $z$.

We need uniform control over all candidate continuation functions $V\in \cV_h$ and all $H$ steps. The following entropy bound makes its dimension dependence explicit.

\begin{lemma}
\label{lem:continuation-entropy}
Let $d:=d_S+d_A$. Suppose that $k_0$ is a normalized Mat\'ern kernel with smoothness $\nu>0$ and inverse length scale $\kappa>0$, and that $\|g\|_{\mathcal H_0}\le\Lambda\|g\|_{\cH_k}$ for every $g\in\cH_k$. There is a constant $C_{\nu,\kappa}$ depending only on $\nu$ and $\kappa$ such that
\begin{equation}
\max_{h\in[H]}\int_0^H\sqrt{\ln\bigl(2\cN(\varepsilon,\mathcal V_h,\|\cdot\|_\infty)\bigr)}\,d\varepsilon\le J:=C_{\nu,\kappa}\left[H+d\left(\frac{2\nu+d}{2\nu}\right)^2\Lambda(B_r+B_p H)\right].
\label{eq:continuation-entropy-bound}
\end{equation}
Here $\cN(\varepsilon,\mathcal V_h,\|\cdot\|_\infty)$ denotes the covering number in the supremum norm, with centers in $\mathcal V_h$. We choose $C_{\nu,\kappa}$ large enough that the displayed integral is at most $J/2$.
\end{lemma}

Using Lemma~\ref{lem:continuation-entropy}, we define the confidence parameters for uniform concentration over all functions in $\cV_h$. For a partition with $M$ cells and a failure probability $\delta\in(0,1)$, set
\begin{equation}
\beta(M,\delta):=C\left[H\sqrt{D\ln(1+T/D)+\ln(8MHT/\delta)}+J\right],\ \overline\beta(M,\delta):=\sqrt\lambda H+\beta(M,\delta).
\label{eq:regression-confidence-parameter}
\end{equation}
Here $C$ is a sufficiently large universal constant, $D=\binom{d+\rho}{\rho}$ is the number of local polynomial features, and $\lambda=B_{\phi}^2$ is the regularization parameter. 
Any fixed enlargement of the constants defining $J$ and $\beta$ above the thresholds required by Lemmas~\ref{lem:continuation-entropy} and~\ref{lem:local-ridge} is valid after recomputing all dependent parameters. Such enlargements change only the regret constant: $C$ remains universal and $C_{\nu,\kappa}$ depends only on $\nu,\kappa$.
The bound also permits a predictable sample assignment: whether observation $i$ is included in a regression group must be measurable with respect to $\mathcal I_i^h$, before $R_i^h$ and, for $h<H$, $s_i^{h+1}$ are observed. The same assignment is used for every $V\in\mathcal V_h$. The sets $\mathcal T_t^h(c)$ then contain only the selected observations.
\begin{lemma}
\label{lem:local-ridge}
For a fixed partition with at most $M$ cells and a predictable sample assignment, for any $\delta\in(0,1)$, with probability at least $1-\delta$, uniformly over $ t\in[T],h\in[H],\ V\in\mathcal V_h,\ z\in[0,1]^d$, we have
\begin{equation}
\left|(\mathcal B^hV)(z)-\widehat{\mathcal B}_t^h[V](z)-\langle\mathcal B^hV,g_t^h(z)\rangle_{\mathcal H_0}\right|
\le\overline\beta(M,\delta)w_t^h(z).
\label{eq:local-ridge-confidence}
\end{equation}
On the same event, we also have $
\|\sum_{i\in\mathcal T_t^h(c)}
\phi_c(z_i^h)\bigl[Y_i^h(V)-(\mathcal B^hV)(z_i^h)\bigr]
\|_{G_t^h(c)^{-1}}
\le\beta(M,\delta)$.
\end{lemma}

\section{\texttt{SOVI}: Version Space Construction}
\label{sec:sovi-version-space}

In this section, we construct the version space $\mathcal F_t$ from observations collected before episode $t$ and the pseudo-code is contained in Algorithm \ref{alg:version-space}. We use the local regression with normalized Legendre polynomial features in Section~\ref{sec:skovi}. We first specify this partition and then give two tests for retaining candidates.

For steps $2,\ldots,H$, we construct a fixed rectangular partition of $\cZ=[0,1]^d$, where $d=d_S+d_A$. Set
\begin{equation}
R:=
\begin{cases}
T^{1/(2\nu+2d_S+d_A)},&d_A<4\nu\text{ or }d_S\ge d_A,\\
T^{d_A/[d(2\nu+d_A)]},&d_A\ge4\nu\text{ and }d_S<d_A.
\end{cases}
\label{eq:version-space-resolution}
\end{equation}
Let $m_R:=\lfloor\log_2R\rfloor$ and $b_R:=\lfloor d\log_2R\rfloor-dm_R$. Let $\mathcal U_q$ divide coordinate $q$ into $2^{m_R+1}$ equal intervals for $q\le b_R$ and $2^{m_R}$ otherwise, left-closed and right-open except that the final interval includes $1$. Define
\[
\mathcal C:=\left\{I_1\times\cdots\times I_d:I_q\in\mathcal U_q,\ q\in[d]\right\}.
\]
Write $c(z)$ for the unique cell containing $z$. Each cell's closure
is its bounding rectangle $\overline c$, and
\[
|\mathcal C|=2^{dm_R+b_R}=2^{\lfloor d\log_2R\rfloor}\le R^d,\ 
\max_{c\in\mathcal C}\max_{q\in[d]}(u_{c,q}-\ell_{c,q})
=2^{-m_R}\le\frac{2}{R}.
\]

For $h\in[H]$ and $t\in[T]$, use the observation sets $\mathcal T_t^h(c)$ defined above with this partition. For a candidate $f$ and $1\le p\le t$, define the observed Bellman residual and its cumulative feature-weighted sum:
\begin{equation}
d_i^h(f):=f^h(z_i^h)-Y_i^h(V_f^{h+1}),\ v_{p,c}^h(f):=\sum_{i\in\mathcal T_p^h(c)}\phi_c(z_i^h)d_i^h(f).
\label{eq:empirical-bellman-residual}
\end{equation}
\textbf{First test.} We require the cumulative residual vector in each cell to stay within its regression confidence bound. By Lemma~\ref{lem:local-ridge} and the Bellman identity
$Q_*^h=\mathcal B^hV_*^{h+1}$, with probability at least
$1-\delta/8$, the optimal candidate $Q_*$ satisfies
\[
\|v_{p,c}^h(Q_*)\|_{G_p^h(c)^{-1}}
\le\beta(|\mathcal C|,\delta/8),\
2\le h\le H,\ c\in\mathcal C,\ 1\le p\le t,
\]
at every episode $t\in[T]$. We therefore retain all
candidates $f\in\mathcal F$ such that
\begin{equation}
\|v_{p,c}^h(f)\|_{G_p^h(c)^{-1}}
\le\beta(|\mathcal C|,\delta/8),\
2\le h\le H,\ c\in\mathcal C,\ 1\le p\le t.
\label{eq:coarse-feasibility}
\end{equation}
Candidates violating any constraint are rejected, and we denote the remaining class by $\cF_t^1$. We retain every historical constraint, so an excluded candidate cannot re-enter the version space. These constraints also supply the historical residual bounds needed for the second test.

\textbf{Second test.}
To determine what remains after the first test, we decompose the Bellman residual. For $f\in\mathcal F$ and $h\in[H]$, define
\begin{equation}
\Delta_f^h:=f^h-\mathcal B^hV_f^{h+1}.
\label{eq:bellman-residual-definition}
\end{equation}
The argument of Lemma~\ref{lem:local-ridge} gives the following decomposition.

\begin{lemma}
\label{lem:bellman-residual-decomposition}
Fix a partition $\mathcal C$ with at most $M$ cells and a predictable sample assignment as in Lemma~\ref{lem:local-ridge}. For any $\delta\in(0,1)$, with probability at least $1-\delta$, uniformly over $t\in[T]$, $h\in[H]$, $f\in\mathcal F$, $c\in\mathcal C$, and $z\in c$,
\begin{equation}
\left|
\Delta_f^h(z)
-\phi_c(z)^\top G_t^h(c)^{-1}v_{t,c}^h(f)
-\langle\Delta_f^h,g_t^h(z)\rangle_{\mathcal H_0}
\right|
\le\overline\beta(M,\delta)w_t^h(z).
\label{eq:bellman-residual-confidence}
\end{equation}
The quantities $G_t^h(c)$, $v_{t,c}^h(f)$, and $g_t^h(z)$ use the same assigned observations. This inequality holds on the same event as in Lemma~\ref{lem:local-ridge}.
\end{lemma}

Apply Lemma~\ref{lem:bellman-residual-decomposition} with failure probability $\delta/8$. On this event, for every $f\in\mathcal F_t^1$, $2\le h\le H$, $1\le i\le t$, $c\in\mathcal C$, and $z\in c$, we have
\begin{equation}
\begin{aligned}
&\left|\Delta_f^h(z)-\langle\Delta_f^h,g_i^h(z)\rangle_{\mathcal H_0}\right|\\
&\le
\left|\Delta_f^h(z)-\phi_c(z)^\top G_i^h(c)^{-1}v_{i,c}^h(f)
-\langle\Delta_f^h,g_i^h(z)\rangle_{\mathcal H_0}\right|
+\left|\phi_c(z)^\top G_i^h(c)^{-1}v_{i,c}^h(f)\right|\\
&\le \overline\beta(|\mathcal C|,\delta/8)w_i^h(z)
+\|\phi_c(z)\|_{G_i^h(c)^{-1}}
\|v_{i,c}^h(f)\|_{G_i^h(c)^{-1}}\\
&\le\bigl(\overline\beta(|\mathcal C|,\delta/8)
+\beta(|\mathcal C|,\delta/8)\bigr)w_i^h(z).
\end{aligned}
\label{eq:coarse-residual-difference}
\end{equation}
Thus, the first test controls the difference between the Bellman residual and the term $\langle\Delta_f^h,g_i^h(z)\rangle_{\mathcal H_0}$ by the displayed confidence width. The second test examines this remaining term through observable residuals.

The Bellman residual depends on the unknown reward mean and transition law. We therefore use responses computed from stored rewards and next states to construct the test. For a fixed $V\in\mathcal V_h$, define $\xi_i^h(V):=Y_i^h(V)-(\mathcal B^hV)(z_i^h)$.

Let $\mathcal I_i^h$ be the information after choosing the current action and before observing its reward and next state; then $\mathbb E[\xi_i^h(V)\mid\mathcal I_i^h]=0$. For every candidate $f$,
\begin{align*}
d_i^h(f)=f^h(z_i^h)-Y_i^h(V_f^{h+1})=f^h(z_i^h)-(\mathcal B^hV_f^{h+1})(z_i^h)
-\xi_i^h(V_f^{h+1})=\Delta_f^h(z_i^h)-\xi_i^h(V_f^{h+1}).
\end{align*}
Every response, including historical ones, is evaluated at the current candidate $f$.

For $c_i:=c(z_i^h)$, define the innovation as the observed Bellman residual minus its regression prediction from earlier observations in the same cell:
\begin{equation}
U_i^h(f):=d_i^h(f)
-\phi_{c_i}(z_i^h)^\top G_i^h(c_i)^{-1}v_{i,c_i}^h(f).
\label{eq:bellman-innovation}
\end{equation}
This subtraction applies the same regression adjustment to the observed residual as in the construction of $g_i^h$. In particular,
\begin{equation}
U_i^h(f)
=\langle\Delta_f^h,g_i^h(z_i^h)\rangle_{\mathcal H_0}
+e_i^h(f)-\xi_i^h(V_f^{h+1}),
\label{eq:innovation-decomposition}
\end{equation}
where
\[
e_i^h(f):=
\Delta_f^h(z_i^h)
-\phi_{c_i}(z_i^h)^\top G_i^h(c_i)^{-1}v_{i,c_i}^h(f)
-\langle\Delta_f^h,g_i^h(z_i^h)\rangle_{\mathcal H_0}.
\]
By \eqref{eq:bellman-residual-confidence},
$|e_i^h(f)|\le
\overline\beta(|\mathcal C|,\delta/8)w_i^h(z_i^h)$.
For a fixed candidate, this decomposition expresses the observable innovation as the remaining term plus bounded regression error and conditionally mean-zero noise.

The second test looks for systematic weighted deviations that are too large to be explained by these errors.
Intuitively, systematic deviations with opposite signs can cancel in an
unweighted sum. We therefore generate test weights $\chi_i^h(Z_i)$ from Gaussian sequences $Z_i,i=1,2,\cdots$
whose covariance is determined by the shared kernel features. When the
weights align with a systematic deviation, the products
$\chi_i^h(Z_i)U_i^h(f)$ can accumulate positive evidence.
We subtract allowances for regression error and observation noise,
combine the corrected sums through an exponential average, and reject
candidates whose aggregate score exceeds the acceptance threshold.

\paragraph{Parameters for the second test.}
We first specify the quantities that determine the Gaussian sequences and their acceptance threshold. The parameters $\nu$, $\kappa$, and $\Lambda$ come from the reference-kernel construction above.

Use $\rho,D,\lambda$ from Section~\ref{sec:skovi}, together with
$R$ in \eqref{eq:version-space-resolution} and the partition
$\mathcal C$ constructed above. In this section only, abbreviate
the confidence parameters by $\beta:=\beta(|\mathcal C|,\delta/8),\
\overline\beta:=\sqrt\lambda H+\beta$.

Let $J$ be the bound on the continuation class's entropy integral in Lemma~\ref{lem:continuation-entropy}. Define
\begin{equation}
B_{\rm res}:=1+2\Lambda(B_r+HB_p),\
K_{\rm c}:=2D\ln(1+T/D).
\label{eq:vs-residual-width-parameters}
\end{equation}
Here $B_{\rm res}$ bounds $\|\Delta_f^h\|_{\mathcal H_0}$ for every candidate and stage, and $K_{\rm c}$ bounds the accumulated squared regression widths within one cell.

To include boundary points in the uniform remainder bound, we set
\[
\Psi_c(z):=\Phi(z)
-\sum_{|\mathbf b|\le\rho}p_{c,\mathbf b}(z)\zeta_{c,\mathbf b},\ z\in\overline c,
\]
with $\Psi_{\mathcal C}(z)=\Psi_{c(z)}(z)$, where $c(z)$ is the cell containing $z$. The known remainder envelope is
\begin{equation}
A_{\rm vs}:=R^{-\nu}
+\max_{c\in\mathcal C}\sup_{z\in\overline c}
\|\Psi_c(z)\|_{\mathcal H_0}.
\label{eq:vs-remainder-envelope}
\end{equation}
This positive quantity is determined by the known kernel and partition.

To define the parameter $F_{\rm vs}$ used in the analysis, put $\nu_S:=\nu+\frac{d_S}{2}$, $k_A:=\lceil2\nu\rceil+1$, and
\begin{equation}
a_{\rm mix}:=\max\left\{1,
2\sqrt{d_A}
\left[
2^{1/4}\left(\frac{\kappa}{\sqrt2}\right)^\nu
(2\pi)^{d_S/4}
\left(\frac{\Gamma(\nu_S)}{\Gamma(\nu)}\right)^{1/2}
\right]^{1/\nu_S}
\right\}.
\label{eq:mixed-action-scale}
\end{equation}
Choose
\begin{equation}
\begin{aligned}
F_{\rm vs}&:=
\begin{cases}
a_{\rm mix}R,
&d_A<4\nu\text{ or }d_S\ge d_A,\\
T^{\nu/[k_A(2\nu+d_A)]}R^{(k_A-\nu)/k_A},
&d_A\ge4\nu\text{ and }d_S<d_A,
\end{cases}\ L_{\rm vs}:=2^{\lfloor d_A\log_2F_{\rm vs}\rfloor}.
\end{aligned}
\label{eq:vs-action-resolution}
\end{equation}
We use $L_{\rm vs}$ to set the Gaussian test scale
in \eqref{eq:version-space-tuning}.

We generate all Gaussian test coordinates from one auxiliary centered
jointly Gaussian family $(W(v):v\in\mathcal H_0)$, independent of the
interaction. Write $\mathsf P$ for its fixed law on the sigma-field
generated by these evaluations, with
\begin{equation}
\mathbb E_{\mathsf P}W(v)=0,\
\mathbb E_{\mathsf P}[W(v)W(w)]=\langle v,w\rangle_{\mathcal H_0},\
W_c(z):=W(\Psi_c(z)),\ z\in\overline c.
\label{eq:vs-gaussian-law}
\end{equation}
In particular, $\operatorname{Cov}_{\mathsf P}(W_c(z),W_{c'}(z'))
=\langle\Psi_c(z),\Psi_{c'}(z')\rangle_{\mathcal H_0}$.

We use a separable version that is continuous on each cell closure.
To choose the Gaussian envelope without evaluating a Gaussian supremum,
put $\alpha:=\min\{\nu/2,1/2\}$ and define
\begin{equation}
\begin{aligned}
C_{\Psi}&:=2\kappa^\alpha
\sqrt{\frac{\Gamma(\alpha+1/2)\Gamma(\nu-\alpha)}{\sqrt\pi\,\Gamma(\nu)}}
+4\rho^2RB_{\phi}\sqrt d,\ n_{\Psi}:=3+\sqrt d\max\{1,(C_{\Psi}R^\nu)^{1/\alpha}\}.
\end{aligned}
\label{eq:vs-gaussian-cover-parameters}
\end{equation}
These parameters use only the known kernel and partition. Define
\begin{equation}
q_0:=1+48\sqrt{\ln(1+2|\mathcal C|n_{\Psi}^{d})}
+12\sqrt{\pi d/\alpha}.
\label{eq:vs-gaussian-envelope}
\end{equation}
This explicit choice ensures $q_0\ge1$ and yields
\begin{equation}
\mathsf P\left(
\max_{c\in\mathcal C}\sup_{z\in\overline c}|W_c(z)|
\le q_0A_{\rm vs}\right)\ge\frac12.
\label{eq:vs-gaussian-envelope-probability}
\end{equation}
The grid and Dudley calculation in \eqref{eq:vs-gaussian-supremum-bound},
followed by Markov's inequality in the proof of
Lemma~\ref{lem:upper-historical-residual}, verify this bound.
This constant-probability envelope controls the auxiliary Gaussian
fluctuations in the residual-detection argument; the test's confidence
level is determined separately by $\delta$.
Define the scale
\begin{equation}
\alpha_{\rm vs}:=
\left[
\frac{HB_{\rm res}\sqrt{B_p L_{\rm vs}}}
{q_0A_{\rm vs}(1+\sqrt{K_{\rm c}})\sqrt T}
\right]^{1/2}.
\label{eq:version-space-tuning}
\end{equation}
All quantities defining this scale are fixed before data collection.
This choice makes the contributions
$HB_{\rm res}\sqrt{B_pL_{\rm vs}T}/\alpha_{\rm vs}$
and
$\alpha_{\rm vs}q_0A_{\rm vs}(1+\sqrt{K_{\rm c}})T$
in \eqref{eq:upper-continuation-interface} equal.
The full substitution is given in the proof of
Proposition~\ref{prop:upper-continuation-error}.

\paragraph{A Gaussian sequence for the first $p-1$ episodes.}
For stage $2\le h\le H$ and $1\le p\le t$, using observations from episodes $1,\ldots,p-1$, define
the known covariance matrix
\begin{equation}
\Sigma_p^h[i,j]:=
\langle g_i^h(z_i^h),g_j^h(z_j^h)\rangle_{\mathcal H_0},
\ 1\le i,j<p,
\label{eq:vs-gaussian-covariance}
\end{equation}
and the auxiliary Gaussian law
\begin{equation}
\boldsymbol Z=(Z_1,\ldots,Z_{p-1})
\sim\mathsf P_p^h:=\mathcal N(0,\alpha_{\rm vs}^2\Sigma_p^h).
\label{eq:vs-gaussian-sequence}
\end{equation}
This positive semidefinite covariance is computed from the known kernel and the observed state--action pairs used to construct $g_i^h$. Each $g_i^h$ uses only observations preceding episode $i$.

Using the fixed Gaussian family in \eqref{eq:vs-gaussian-law},
define
\begin{equation}
Z_i(W):=\alpha_{\rm vs}W(g_i^h(z_i^h)).
\label{eq:vs-gaussian-coordinate-realization}
\end{equation}
With the observed history fixed, $(Z_i(W))_{i<p}$ has law
$\mathsf P_p^h$ in \eqref{eq:vs-gaussian-sequence}.
We verify this representation in the proof of
Lemma~\ref{lem:uniform-confidence}.
Thus $\mathsf P$ is the fixed auxiliary Gaussian law, whereas
$\mathsf P_p^h$ is the induced law for the history.

The threshold combines the Gaussian supremum bound with the regression confidence width. For $z\in c$, define
\begin{equation}
\tau_i^h(z):=
\alpha_{\rm vs}q_0A_{\rm vs}
\bigl(1+\sqrt{n_i^h(c)}\,w_i^h(z)\bigr)
+(\overline\beta+\beta)w_i^h(z),
\label{eq:innovation-threshold}
\end{equation}
where $n_i^h(c)$ counts stage-$h$ observations in $c$ before episode $i$. We soft-threshold each Gaussian coordinate and clip its magnitude at $H$ to obtain the test weights
\begin{equation}
\chi_i^h(Z_i):=
\operatorname{sgn}(Z_i)
\min\left\{H,\bigl(|Z_i|-\tau_i^h(z_i^h)\bigr)_+\right\},
\ x_+:=\max\{x,0\}.
\label{eq:innovation-test-weights}
\end{equation}
Take $\operatorname{sgn}(0)=0$. The Gaussian law and thresholds depend only on the observed history, so for each fixed $p$, the test based on episodes $1,\ldots,p-1$ uses the same weights for all $f\in\mathcal F$.

To average over the test weights generated by $\boldsymbol Z\sim\mathsf P_p^h$, with $\Xi_1^h(f):=0$, we define
\begin{equation}
\Xi_p^h(f):=
\ln\mathbb E_{\boldsymbol Z\sim\mathsf P_p^h}
\exp\left\{
\frac{1}{320H^2}
\sum_{i<p}\left[
\chi_i^h(Z_i)U_i^h(f)
-\beta w_i^h(z_i^h)|\chi_i^h(Z_i)|
-\frac{\chi_i^h(Z_i)^2}{16}
\right]
\right\},
\label{eq:innovation-test}
\end{equation}
For each realization of $\boldsymbol Z$, $\chi_i^h(Z_i)U_i^h(f)$ measures
signed agreement between the innovation $U_i^h(f)$ and the test weight $\chi_i^h(Z_i)$. The subtraction
$\beta w_i^h(z_i^h)|\chi_i^h(Z_i)|$ allows for historical residual
prediction controlled by the first test. The quadratic subtraction
$\chi_i^h(Z_i)^2/16$ compensates for weighted observation noise in the
exponential bound, with $320H^2$ setting the exponential scale for bounded
weights and responses.

The logarithmic exponential average combines the corrected weighted
sums in \eqref{eq:innovation-test}, emphasizing large positive values while accounting for their
probability under the Gaussian law. Thus, $\Xi_p^h(f)$ aggregates evidence
of systematic deviations beyond the regression and noise allowances.
The following lemma calibrates the score uniformly over candidates
that pass the first test.

\begin{lemma}
\label{lem:uniform-confidence}
With probability at least $1-\delta/8$, uniformly over
$t\in[T]$, $2\le h\le H$, $1\le p\le t$, and
$f\in\mathcal F_t^1$, we have
\begin{equation}
\Xi_p^h(f)
\le
\frac{J^2}{H^2}+1+\ln\frac{8HT}{\delta}
+\frac{p-1}{320H}
\|\Delta_f^h\|_\infty,
\label{eq:innovation-statistic-bound}
\end{equation}
where $J$ is specified in Lemma~\ref{lem:continuation-entropy}.
\end{lemma}

Since $\Delta_{Q_*}^h=0$, on the event of Lemma~\ref{lem:uniform-confidence} whenever
$Q_*\in\mathcal F_t^1$, the optimal candidate satisfies
\[
\Xi_p^h(Q_*)\le
\frac{J^2}{H^2}+1+\ln\frac{8HT}{\delta},
\ 2\le h\le H,\ 1\le p\le t.
\]
We use this bound as the acceptance threshold for the second test. At episode $t$, candidates must pass it for every $1\le p\le t$, using episodes $1,\ldots,p-1$, and every stage $2\le h\le H$, in addition to satisfying the first test. The resulting version space is
\begin{equation}
\mathcal F_t:=
\left\{
f\in\mathcal F_t^1:
\Xi_p^h(f)\le
\frac{J^2}{H^2}+1+\ln\frac{8HT}{\delta},
\ 2\le h\le H,\ 1\le p\le t
\right\}.
\label{eq:version-space}
\end{equation}
By the first test and Lemma~\ref{lem:uniform-confidence}, with probability at least $1-\delta/4$,
$Q_*\in\mathcal F_t$ for every $t\in[T]$.
The sets $\mathcal F_t$ depend only on past observations and satisfy
$\mathcal F_{t+1}\subseteq\mathcal F_t$ for $t<T$.

Algorithm~\ref{alg:version-space} applies both tests and returns the candidates available in episode $t$.

\begin{algorithm}[H]
\caption{Version space construction (\textsc{ConstructVersionSpace})}
\label{alg:version-space}
\small
\begin{algorithmic}[1]
\Require Episode $t$, trajectories from episodes $i<t$, class $\mathcal F$, partition $\mathcal C$, confidence level $\delta$, parameter $\beta=\beta(|\mathcal C|,\delta/8)$ from \eqref{eq:regression-confidence-parameter}, and scale and thresholds from \eqref{eq:version-space-tuning} and \eqref{eq:innovation-threshold}.
\Ensure The version space $\mathcal F_t$.
\State Form the regularized Gram matrices $G_p^h(c)$ for $2\le h\le H$, $1\le p\le t$, and $c\in\mathcal C$.
\State For each $f\in\mathcal F$, form $d_i^h(f)$ and $v_{p,c}^h(f)$ by \eqref{eq:empirical-bellman-residual}.
\State Retain in $\mathcal F_t^1$ the candidates satisfying every constraint in \eqref{eq:coarse-feasibility}.
\State Construct the shared Gaussian laws and test weights by \eqref{eq:vs-gaussian-sequence}-\eqref{eq:innovation-test-weights}.
\State For each $f\in\mathcal F_t^1$, form $U_i^h(f)$ by \eqref{eq:bellman-innovation} and $\Xi_p^h(f)$ by \eqref{eq:innovation-test}.
\State Retain in $\mathcal F_t$ the candidates satisfying the second test in \eqref{eq:version-space} for every $2\le h\le H$ and $1\le p\le t$, using observations from episodes $1,\ldots,p-1$.
\State \Return $\mathcal F_t$.
\end{algorithmic}
\end{algorithm}

\section{\texttt{SOVI}: First Action and Candidate Q-Function Selection}
\label{sec:sovi-first-selection}
Given a nonempty version space $\mathcal F_t$ from
Section~\ref{sec:sovi-version-space} and an initial state $s_t^1$,
we seek an action $a_t^1$ and a candidate Q-function
$f_t\in\mathcal F_t$ that maximize the first-step Bellman backup
\[
\sup_{a\in\cA,\ f\in\mathcal F_t}
(\mathcal B^1V_f^2)(s_t^1,a).
\]
The Bellman backups are unknown, so selecting $(a,f)$ requires
estimates from observed rewards and transitions. Because initial
states need not repeat in a continuous space, we pool nearby
observations using local polynomial regression on state-action cells.

To adapt this pooling to uneven initial-state visits, we prepare
partitions of the state space at several spatial scales, indexed by a \emph{level} $l$
and ordered from coarse to fine. At each level, a counting partition
$\mathcal D_l$ counts initial-state visits for scale selection.
A regression partition $\mathcal P_l$ groups nearby states; its product
with an action partition defines the regression cells.

After observing $s_t^1$, we select the first level whose counting
cell containing that state has not reached its prescribed visit
threshold. The selected level determines the regression partition
used in this episode. At each level, successive assigned visits
are grouped into epochs, which specify when regression records
and model weights restart.

Local polynomial regression leaves an unknown kernel remainder.
Within each epoch, a fixed finite set of remainder models completes
the estimates; each model scores all pairs $(a,f)$ and recommends
a maximizer. We sample a model according to its current weight and retain its
recommended candidate Q-function $f_t$. We mix the recommended
first action with G-optimal exploration to obtain feedback for
comparing the models. Both branches follow the greedy policy of
$f_t$ from step two, and first-step feedback updates the
model weights.

We present this construction in five parts:
\begin{enumerate}[label=(\roman*),leftmargin=*]
\item Subsection~\ref{subsec:state-action-partition} discretizes the action space and defines two state partitions
for counting visits and pooling regression data, together with an
action partition for local estimation.

\item Subsection~\ref{sec:first-step-g-optimal} constructs the G-optimal
exploration policy to collect data for comparing recommendations
at different actions.

\item Subsection~\ref{sec:first-step-selection-principle} specifies the
level selection and epoch rules, then fixes each epoch's remainder
models, learning rate, and exploration probability.

\item Subsection~\ref{sec:first-step-scores} constructs the scores,
samples a remainder model from the current distribution, and selects
the first action and candidate Q-function using its recommendation
and the exploration policy.

\item Subsection~\ref{sec:first-step-weight-updates} estimates feedback
at the models' recommended pairs and uses it to update the model
distribution for subsequent selections.
\end{enumerate}
Finally, we provide the pseudo-code of our first action and $Q$-function selection procedure in Algorithm \ref{alg:first-selection}. Combining Algorithm \ref{alg:version-space} and Algorithm \ref{alg:first-selection} together yields our complete \texttt{SOVI} in Algorithm \ref{alg:skovi}.
\subsection{State and Action Partition}
\label{subsec:state-action-partition}
First, we discretize the action space $\cA=[0,1]^{d_A}$ and restrict
the first action to a finite grid. For each integer $j\ge0$, define $\mathcal A^{(j)}
:=
\{k2^{-j}:k=0,\ldots,2^j\}^{d_A}$.
Set $\bar B:=\max\{1,\Lambda(B_r+HB_p)\}$.
Since $\|\mathcal B^1V_f^2\|_{\mathcal H_0}\le \bar B$
for every $f\in\mathcal F$, the reproducing property and
the Cauchy--Schwarz inequality give, for any $s\in\cS$ and $a,\bar a\in\cA$,
\begin{align*}
&\left|
(\mathcal B^1V_f^2)(s,a)
-(\mathcal B^1V_f^2)(s,\bar a)
\right|=
\left|
\left\langle
\mathcal B^1V_f^2,\Phi(s,a)-\Phi(s,\bar a)
\right\rangle_{\mathcal H_0}
\right|\le
\bar B
\|\Phi(s,a)-\Phi(s,\bar a)\|_{\mathcal H_0}.
\end{align*}
We choose the smallest grid index that makes this approximation
error at most $T^{-2}$:
\[
j_T
:=
\min\left\{
j\in\mathbb N_0:
\bar B
\sup_{s\in\cS,a\in\cA}
\min_{\bar a\in\mathcal A^{(j)}}
\|\Phi(s,a)-\Phi(s,\bar a)\|_{\mathcal H_0}
\le T^{-2}
\right\}.
\]
Define $\mathcal A_T:=\mathcal A^{(j_T)}$.
Thus, for every nonempty $\mathcal F_t$, we have
\begin{align}
0\le{}&
\sup_{a\in\cA,\ f\in\mathcal F_t}
(\mathcal B^1V_f^2)(s_t^1,a)
-
\sup_{a\in\mathcal A_T,\ f\in\mathcal F_t}
(\mathcal B^1V_f^2)(s_t^1,a)
\le T^{-2}.
\label{eq:first-action-grid-error}
\end{align}
Restricting the first action to $\mathcal A_T$ therefore
incurs an approximation loss of at most $T^{-1}$ over $T$ episodes.
The greedy policies from step two onward still maximize over
the action space $\cA$.

We now construct the counting and regression partitions at each level.
Recall $d=d_S+d_A$ and set the partition parameters
\[
S_{\max}:=
\min\left\{
T^{1/(2\nu+d_A)},
T^{(2\nu+d_S)/[2\nu(2\nu+2d_S+d_A)]}
\right\},\
L_{\max}:=\lceil d\log_2S_{\max}\rceil.
\]
For each level $l=0,\ldots,L_{\max}$, denote $S_l:=\min\{2^{l/d},S_{\max}\}$. To define the state partitions, we let $A_{d_S}\in\mathbb R^{d_S\times d_S}$
have entries $(A_{d_S})_{ij}:=\mathbf1\{i=j\}
+\frac{d_S-j+1}{d_S}\mathbf1\{i<j\},\ i,j\in[d_S]$.

For $m\in\mathbb Z^{d_S}$, define $Q_m:=A_{d_S}m+[0,1)^{d_S}$
and $\mathscr Q^{(d_S)}:=\{Q_m:m\in\mathbb Z^{d_S}\}$.
By Lemma~\ref{lem:colored-cube-partition}, these half-open cubes
partition $\mathbb R^{d_S}$.
For each integer $M\ge2$, we scale the cubes by $1/M$ and
intersect them with $\cS=[0,1]^{d_S}$:
\[
\mathscr Q_M:=\left\{
\left(M^{-1}A_{d_S}m+[0,M^{-1})^{d_S}\right)\cap\cS:
m\in\mathbb Z^{d_S}\right\}\setminus\{\varnothing\}.
\]
Set $\mathscr Q_1:=\{\cS\}$. The half-open cubes assign every state
to exactly one cell; all nonempty intersections with $\cS$ are retained,
including those contained in its boundary.
Applying Lemma~\ref{lem:colored-cube-partition} at scale $1/M$,
we find that every closed $\ell_\infty$-ball of radius $1/(2d_SM)$
intersects at most $d_S+1$ cells of $\mathscr Q_M$.

\textbf{The counting partition $\mathcal D_l$.}
We count all earlier initial-state visits in each cell, including
episodes assigned to other levels. Set
\[
M_l:=
\max\left\{
1,\left\lfloor
\frac{S_l}{64d_S^2\sqrt{d_A}}
\right\rfloor
\right\}.
\]
The counting partition is $\mathcal D_l:=\mathscr Q_{M_l}$.

\textbf{The regression partition $\mathcal P_l$.}
We use coarser state cells to pool observations from nearby initial
states for the same local fit, using data assigned to that level. Set
\[
\widetilde M_l:=
\max\left\{
1,\left\lfloor
\frac{S_l}{4096(d_S+1)d_S^{5/2}\sqrt{d_A}}
\right\rfloor
\right\}.
\]
The regression partition is $\mathcal P_l:=\mathscr Q_{\widetilde M_l}$.
If $\widetilde M_l=1$, its bounding rectangle is $\cS$.
Otherwise, write a cell as $P=(b+[0,h)^{d_S})\cap\cS$, where
$h=1/\widetilde M_l$ and $b$ is the lower corner of its scaled cube.
We use the rectangle
\[
\overline P:=\prod_{q=1}^{d_S}[\ell_q,\ell_q+h],\
\ell_q:=\min\{\max\{b_q,0\},1-h\}.
\]
This shifts any part of the rectangle outside the state space inward:
$P\subseteq\overline P\subseteq\cS$, and every side remains exactly $h$.
Regression cells determine observation membership; their bounding
rectangles define the local Legendre features, including for boundary cells.

We next partition the action space for estimation.
For $r\ge1$ and $q\in[d_A]$, set
\[
m_r:=\lfloor\log_2r\rfloor,\
b_r:=\lfloor d_A\log_2r\rfloor-d_Am_r,\
n_{r,q}:=2^{m_r+\mathbf1\{q\le b_r\}}.
\]
Define
\[
I_{r,q,j}:=
\begin{cases}
[j/n_{r,q},(j+1)/n_{r,q}),&0\le j<n_{r,q}-1,\\
[(n_{r,q}-1)/n_{r,q},1],&j=n_{r,q}-1,
\end{cases}
\]
and
\[
\mathscr U(r):=
\left\{\prod_{q=1}^{d_A}I_{r,q,j_q}:
j_q\in\{0,\ldots,n_{r,q}-1\},\ q\in[d_A]\right\}.
\]
The partition $\mathscr U(r)$ has
$2^{\lfloor d_A\log_2r\rfloor}\le r^{d_A}$ pairwise disjoint cells, each with side lengths at most $2/r$.

Therefore, at level $l$, the joint regression partition is
\begin{equation}
\mathcal C_l:=
\{P\times U:P\in\mathcal P_l,\ U\in\mathscr U(S_l)\}.
\label{eq:initial-partitions}
\end{equation}
For $c=P\times U$, its bounding rectangle is
$\overline c:=\overline P\times\operatorname{cl}(U)$,
where $\operatorname{cl}(U)$ denotes the closure of $U$.
Thus, $\mathcal D_l$ counts initial-state visits for level selection,
whereas $\mathcal P_l$ groups nearby states for local regression.

\subsection{G-optimal Design Policy}
\label{sec:first-step-g-optimal}

Exploration supplies data for comparing recommendations made at
different actions. At each level, we partition the action
grid into blocks and choose a G-optimal design for the polynomial
features on each block.

Recall $k_A=\lceil2\nu\rceil+1$ and define $k_T:=\min\left\{k\in\mathbb N:\ k\ge1,\
\bar B\frac{(2\sqrt{d_A})^k}{k!}\le T^{-2}
\right\}$. The following choices specify the action blocks and their polynomial degree:
\[
(F_l,j_l):=
\begin{cases}
(S_l,k_T-1),&d_A<4\nu,\\
\bigl(T^{\nu/[k_A(2\nu+d_A)]}S_l^{(k_A-\nu)/k_A},k_A-1\bigr),
&d_A\ge4\nu.
\end{cases}
\]
Define the nonempty action blocks by $\mathscr U_l:=
\{U\cap\mathcal A_T:U\in\mathscr U(F_l),\
U\cap\mathcal A_T\ne\varnothing\}$. Their number and a bound on their total polynomial dimension are
\[
L_l:=|\mathscr U_l|,\
r_l:=L_l\binom{d_A+j_l}{j_l}.
\]
On each block $U$, let $\varphi_U:U\to\mathbb R^{d_U}$ be a
basis for the restrictions of action polynomials of total degree at
most $j_l$, so that $d_U\le\binom{d_A+j_l}{j_l}$.
To control the second moment of the feedback estimator, choose an
exploration distribution $\pi_U^{\rm G}$ satisfying
\begin{equation}
\Sigma_U:=\mathbb E_{a\sim\pi_U^{\rm G}}
\varphi_U(a)\varphi_U(a)^\top,\
\max_{a\in U}\varphi_U(a)^\top\Sigma_U^{-1}\varphi_U(a)\le d_U.
\label{eq:feedback-design}
\end{equation}
Such a distribution exists on each finite block and is called a
G-optimal design.

At level $l$, define $\pi_l^{\rm G}:=L_l^{-1}\sum_{U\in\mathscr U_l}\pi_U^{\rm G}$,
extending each $\pi_U^{\rm G}$ by zero outside $U$.
To sample this policy, choose a block uniformly and draw an action
from its G-optimal design \citep{fedorov2013optimal}. The bound in
\eqref{eq:feedback-design} controls the polynomial prediction
variance within each block and enters the feedback bound in
Lemma~\ref{lem:first-step-feedback}.

\subsection{Level Selection and Epoch Initialization}
\label{sec:first-step-selection-principle}

With the partitions and exploration policies fixed at each level,
we now assign each episode to a level and determine whether to start
a new epoch or resume the current one. We then specify the models
and parameters initialized at the start of each epoch.

\textbf{Selecting a level.}
The visit threshold for selecting level $l$ is
\begin{equation}
N_l:=\left\lceil T(S_l/S_{\max})^{2\nu}\right\rceil,\
0\le l\le L_{\max}.
\label{eq:initial-scale-schedule}
\end{equation}

At the beginning of episode $t$, after observing $s_t^1$,
let $B_l(s_t^1)$ be its cell in the counting partition $\mathcal D_l$.
Let $m_t(B):=\sum_{i<t}\mathbf1\{s_i^1\in B\}$ count earlier
initial-state visits to a cell $B$. Using the thresholds $N_l$ in
\eqref{eq:initial-scale-schedule}, select the first level whose
current counting cell has not reached its threshold:
\begin{equation}
l_t:=\min\left\{l\in\{0,\ldots,L_{\max}\}:
m_t(B_l(s_t^1))<N_l\right\}.
\label{eq:adaptive-cube}
\end{equation}
The rule is evaluated at the current initial state, so the selected
level can vary with the state. The counts include all previous
initial-state visits, regardless of the levels selected in those episodes.
Since $N_{L_{\max}}=T$, a level is always available.

\textbf{Determining the epoch.}
For each level $l$, let
\begin{equation}
n_t(l):=\sum_{i=1}^{t}\mathbf1\{l_i=l\}
\label{eq:level-visit-count}
\end{equation}
be the number of episodes assigned to level $l$ up to and including
episode $t$. We divide each level's selected visits into epochs
of capacities $1,2,4,\ldots$; since $n_t(l_t)\ge1$,
the current epoch is
\begin{equation}
u_t:=\left\lfloor\log_2 n_t(l_t)\right\rfloor,\
2^{u_t}\le n_t(l_t)<2^{u_t+1}.
\label{eq:epoch-index}
\end{equation}
Epoch $u$ at a given level contains its visits numbered
$2^u,\ldots,2^{u+1}-1$, truncated when the interaction ends.
These visits need not occur in consecutive episodes, and a new
epoch begins when $n_t(l_t)=2^{u_t}$.

A new epoch resets first-step regression and the model weights
at that level. The initial-state counts and history used for the
version-space tests are retained.

\textbf{Initializing an epoch.}
At the start of each epoch, we fix a finite collection of kernel
remainder models. To construct this collection, recall that the
decomposition in
\eqref{eq:feature-decomposition} leaves the remainder
$\langle\mathcal B^1V,\Psi_{\mathcal C_l}(z)\rangle_{\mathcal H_0}$
at level $l$.
Because all continuations share the same reward and transition model,
this function belongs to the class
\[
\mathcal R_l:=\left\{\begin{aligned}
&(z,V)\mapsto
\left\langle\Theta,\Psi_{\mathcal C_l}(z)\otimes(1,V/(H-1))\right\rangle:\\
&\Theta\in\mathcal H_0\otimes(\mathbb R\oplus L^2(\mu)),\
\|\Theta\|\le\Lambda\sqrt{B_r^2+(H-1)^2B_p^2}
\end{aligned}\right\}.
\]
Here $\otimes$ denotes the Hilbert-space tensor product, combining
$\Psi_{\mathcal C_l}(z)$ with $1$ and $V/(H-1)$ to represent the
reward and continuation contributions, respectively, through
the shared coefficient $\Theta$.
The variables satisfy $z\in\cZ$ and $V\in\mathcal V_1$, and the norm and inner product
are those of the displayed tensor-product Hilbert space.
The same coefficient $\Theta$ determines the remainder at every
$(z,V)$.

We measure distance between remainder models by
$\|\varrho-\varrho'\|_\infty:=
\sup_{z\in\cZ,V\in\mathcal V_1}|\varrho(z,V)-\varrho'(z,V)|$.

Fix a level $l$ and epoch $u$. Set its capacity, confidence
allocation, and model approximation accuracy to
\[
N:=2^u,\
\delta_{l,u}:=\frac{\delta}{64(l+1)^2(u+1)^2},\
\varepsilon:=S_l^{-\nu/2}(r_l/N)^{1/4}.
\]
Choose a fixed $\varepsilon$-net $\{\varrho_1,\ldots,\varrho_M\}
\subseteq\mathcal R_l$ of minimum cardinality $M$ in the distance above.
Every function in $\mathcal R_l$ is uniformly within $\varepsilon$
of one of these models. The net is determined by the known kernel
and continuation class, independently of the observations.
Using the resulting model count $M$, set the learning rate $\eta$
for weight updates and the exploration probability $\gamma$ to
\begin{equation}
\eta:=\frac1{256r_l}
\min\left\{1,
\sqrt{\frac{r_l\ln(2M/\delta_{l,u})}
{N\ln(8L_lN/\delta_{l,u})}}\right\},\
\gamma:=128r_l\eta\le\frac12.
\label{eq:first-step-tuning}
\end{equation}
The net and parameters are fixed separately for each $(l,u)$.
At the epoch's first visit, set $\omega_{t,j}=1/M$, $j\in[M]$.
On later visits, we use the same net and parameters with the
accumulated first-step records and saved weights.

\subsection{Scoring and Selecting Action-Candidate Pairs}
\label{sec:first-step-scores}

Fix the current level $l=l_t$ and epoch $u=u_t$, with the models
and parameters specified in
Subsection~\ref{sec:first-step-selection-principle}. We construct
each model's scores and recommendation, then use the model weights
and exploration policy to select the first action and continuation.

\textbf{First-step value estimation.}
To compare $(a,f)\in\mathcal A_T\times\mathcal F_t$, we
estimate $(\mathcal B^1V_f^2)(s_t^1,a)$ from past rewards and transitions.
Let $c=P\times U\in\mathcal C_{l_t}$ be the cell containing
$(s_t^1,a)$.

On each cell's bounding rectangle, use the Legendre features $\phi_c$
of total degree at most $\rho$, with dimension $D$ and regularization
$\lambda$ as defined in Section~\ref{sec:skovi}.

The estimate uses observations in this cell from the current
level and epoch:
\[
\widetilde{\mathcal T}_t(c)
:=
\{i<t:l_i=l_t,\ u_i=u_t,\ z_i^1\in c\}.
\]

For each $V\in\mathcal V_1$, the stored
observation $(z_i^1,R_i^1,s_i^2)$ supplies the response
$Y_i^1(V)=R_i^1+V(s_i^2)$.
The shared regularized Gram matrix is $G_t(c):=
\lambda I_D+
\sum_{i\in\widetilde{\mathcal T}_t(c)}
\phi_c(z_i^1)\phi_c(z_i^1)^\top$.
The Legendre polynomial estimate is
\begin{equation}
\widehat{\mathcal P}_t[V](z)
:=
\phi_c(z)^\top G_t(c)^{-1}
\sum_{i\in\widetilde{\mathcal T}_t(c)}
\phi_c(z_i^1)\bigl(R_i^1+V(s_i^2)\bigr),\
z\in c.
\label{eq:layered-backup-estimate}
\end{equation}
For a candidate pair $(a,f)$, evaluate this estimate at $z=(s_t^1,a)$
with $V=V_f^2$ in the cell containing $z$.
The matrix $G_t(c)$ is shared by all $V\in\mathcal V_1$;
only their responses $Y_i^1(V)$ differ.

To express the estimation error, apply the regression adjustment
to the kernel remainder $\Psi_{\mathcal C_l}$ used in the model class.
For $c\in\mathcal C_{l_t}$ and $z\in c$, define the regression
width and adjusted kernel remainder by
\begin{equation}
\begin{aligned}
w_t(z)
&:=
\sqrt{\phi_c(z)^\top G_t(c)^{-1}\phi_c(z)},\ g_t(z)
:=
\Psi_{\mathcal C_{l_t}}(z)
-
\sum_{i\in\widetilde{\mathcal T}_t(c)}
\phi_c(z)^\top G_t(c)^{-1}\phi_c(z_i^1)
\Psi_{\mathcal C_{l_t}}(z_i^1).
\end{aligned}
\label{eq:layered-width}
\end{equation}
Using the confidence allocation fixed at epoch initialization,
define the regression confidence radius
\[
b_t(z):=
\overline\beta(|\mathcal C_{l_t}|,\delta_{l_t,u_t})w_t(z).
\]
Applying Lemma~\ref{lem:local-ridge} to the observations assigned
to each level and epoch gives the following result.

\begin{lemma}
\label{lem:initial-confidence}
With probability at least $1-\delta/8$, for all
$t\in[T]$, $f\in\mathcal F$, and $z\in\cZ$,
\begin{equation}
\left|
(\mathcal B^1V_f^2)(z)
-\widehat{\mathcal P}_t[V_f^2](z)
-\langle\mathcal B^1V_f^2,g_t(z)\rangle_{\mathcal H_0}
\right|
\le b_t(z).
\label{eq:initial-confidence}
\end{equation}
\end{lemma}

\textbf{Scoring each candidate pair.}
The following operations apply when $\mathcal F_t\ne\varnothing$.
Otherwise, use $a_0$ at every stage and skip the model-weight update.
For $z=(s_t^1,a)$ and $V=V_f^2$, we account for the kernel term
in Lemma~\ref{lem:initial-confidence} using each fixed remainder
model $\varrho_j$: subtract $\varrho_j(z_i^1,V)$ from each stored response,
fit a local polynomial to these adjusted responses, and add
$\varrho_j(z,V)$ at $z$.
The model approximation affects both $\varrho_j(z,V)$ and the
fitted responses. For $z\in c\in\mathcal C_l$, include both
contributions in the total error allowance
\[
\bar b_t(z):=b_t(z)
+\varepsilon\bigl(1+\sqrt{|\widetilde{\mathcal T}_t(c)|}\,w_t(z)\bigr).
\]
For $z=(s_t^1,a)\in c$ and $V=V_f^2$, the resulting score is
\begin{equation}
\widehat Q_{t,j}(a,f):=
\varrho_j(z,V)+
\phi_c(z)^\top G_t(c)^{-1}
\sum_{i\in\widetilde{\mathcal T}_t(c)}
\phi_c(z_i^1)\bigl[Y_i^1(V)-\varrho_j(z_i^1,V)\bigr]
+\bar b_t(z).
\label{eq:optimistic-score}
\end{equation}
Every term uses a fixed remainder model, stored observations,
or known features. Model $j$ recommends
\begin{equation}
(a_t(j),f_t(j))\in
\operatorname*{arg\,max}_{a\in\mathcal A_T,\ f\in\mathcal F_t}
\widehat Q_{t,j}(a,f),\ j\in[M],
\label{eq:continuation-choice}
\end{equation}
with fixed measurable tie-breaking.

The next lemma shows that one fixed remainder model provides an
optimistic approximation.
\begin{lemma}
\label{lem:finite-remainder-comparison}
On the event of Lemma~\ref{lem:initial-confidence}, for every
level and epoch $(l,u)$, one of its $M$ remainder models has an
index $j_*\in[M]$, fixed throughout that epoch, such that
\[
0\le\widehat Q_{t,j_*}(a,f)
-(\mathcal B^1V_f^2)(s_t^1,a)
\le2\bar b_t(s_t^1,a),\
\forall t:(l_t,u_t)=(l,u),\
a\in\mathcal A_T,\ f\in\mathcal F_t.
\]
\end{lemma}
The index of this model is unknown, so feedback from the selected
actions guides the model probabilities.

\textbf{Selecting the action and continuation.}
With the exploration probability $\gamma$ fixed for the current
epoch by \eqref{eq:first-step-tuning}, we mix the models'
recommended actions with the G-optimal exploration policy.
The resulting marginal distribution of the first action is
\begin{equation}
p_t(a):=(1-\gamma)\sum_{j=1}^M
\omega_{t,j}\mathbf1\{a_t(j)=a\}+\gamma\pi_l^{\rm G}(a).
\label{eq:first-step-sampling}
\end{equation}
Before drawing the model or action, we use this distribution
to compute the coefficients in
Subsection~\ref{sec:first-step-weight-updates}.

Draw $j_t\in[M]$ with probabilities $(\omega_{t,j})_{j=1}^M$
and set $f_t:=f_t(j_t)$. With probability $1-\gamma$, execute
$a_t(j_t)$; otherwise, we draw an action from the G-optimal exploration $\pi_l^{\rm G}$.
Both branches follow $\pi_{f_t}$ from step two.

\subsection{Feedback and Model-Weight Updates}
\label{sec:first-step-weight-updates}

For the selection rule above, we prepare feedback coefficients
before sampling and update the model weights after observing the
first-step transition. The update evaluates each model through its
recommended $(a_t(j),f_t(j))$.

The observed reward and next state give $Y_t^1(V)=R_t^1+V(s_t^2)$
for every continuation, but only at the executed action.
To compare models that recommend different actions, we use local
action polynomials to estimate the Bellman backups at their
recommended pairs.

\textbf{Preparing the feedback coefficients.}
Before sampling the model index or action, for each action $a\in U$ with
$U\in\mathscr U_l$, choose coefficients $h_t(a,x)$,
$x\in U$, by the following convex optimization:
\begin{equation}
\begin{aligned}
\underset{h:U\to\mathbb R}{\operatorname{minimize}}\ &
\sum_{x\in U}p_t(x)h(x)^2\\
\text{subject to }\ &
\sum_{x\in U}p_t(x)h(x)\varphi_U(x)=\varphi_U(a),\\
&\sum_{x\in U}p_t(x)|h(x)|\le2\sqrt{d_U},\
|h(x)|\le r_l/\gamma,\ x\in U.
\end{aligned}
\label{eq:minimum-variance-feedback}
\end{equation}
Set $h_t(a,x)=0$ outside $U$ and at coordinates with $p_t(x)=0$.
The first constraint reproduces action polynomials on the block.
The remaining constraints control approximation error
and the size of a single observation's contribution.
The objective minimizes the resulting second-moment bound.

After observing $(R_t^1,s_t^2)$, estimate
$(\mathcal B^1V)(s_t^1,a)$ by
\begin{equation}
\widehat\mu_t(a,V):=h_t(a,a_t^1)\bigl(R_t^1+V(s_t^2)\bigr).
\label{eq:first-step-feedback-estimate}
\end{equation}
For each action $a$, all continuation values share the same feedback coefficients.
Let $\mathbb E_t$ denote expectation conditional on the history
and $s_t^1$, before the model index and action are sampled.

\begin{lemma}
\label{lem:first-step-feedback}
The optimization in \eqref{eq:minimum-variance-feedback} is feasible.
For any $a\in U$ and $V\in\mathcal V_1$, if
$$\inf_{b\in\mathbb R^{d_U}}\max_{x\in U}
|(\mathcal B^1V)(s_t^1,x)-\varphi_U(x)^\top b|\le\epsilon,$$
for some $\epsilon\ge0$, then
\[
\begin{aligned}
\left|\mathbb E_t\widehat\mu_t(a,V)
-(\mathcal B^1V)(s_t^1,a)\right|
&\le(1+2\sqrt{d_U})\epsilon,\\
\mathbb E_t\widehat\mu_t(a,V)^2
&\le H^2\sum_{x\in U}p_t(x)h_t(a,x)^2.
\end{aligned}
\]

The coefficients satisfy
$\sum_{x\in U}p_t(x)h_t(a,x)^2\le r_l/\gamma$.
Moreover, the recommendation-weighted sum satisfies $\sum_{j=1}^M\omega_{t,j}
\sum_{x\in\mathcal A_T}p_t(x)h_t(a_t(j),x)^2
\le C r_l\ln(\mathrm e L_l/\gamma)$, where $C$ is a universal constant.
\end{lemma}

\textbf{Updating the model weights.}
Convert the estimated backup at each recommended pair into a
normalized loss $\widehat\ell_{t,j}$ and use $v_{t,j}$ to bound its conditional
second moment. These determine the weight update:
\begin{equation}
\begin{aligned}
\widehat\ell_{t,j}&:=1-
\frac{\widehat\mu_t(a_t(j),V_{f_t(j)}^2)
+\min\{H,2\bar b_t(s_t^1,a_t(j))\}}{2H},\\  v_{t,j}&:=2+\frac12
\sum_{x\in\mathcal A_T}p_t(x)h_t(a_t(j),x)^2,\\
\omega_{t+,j}&:=
\frac{\omega_{t,j}(1-\eta \widehat\ell_{t,j}+4\eta^2v_{t,j})}
{\sum_{k=1}^M\omega_{t,k}(1-\eta \widehat\ell_{t,k}+4\eta^2v_{t,k})}.
\end{aligned}
\label{eq:expert-weight-update}
\end{equation}
Here $\widehat\ell_{t,j}$ is the estimated normalized loss at model $j$'s
recommended pair $(a_t(j),f_t(j))$. It converts the estimated Bellman
backup, augmented by the correction term
$\min\{H,2\bar b_t(s_t^1,a_t(j))\}$, into a loss:
a larger estimated backup gives a smaller loss.

$v_{t,j}$ bounds the conditional second moment of this
loss estimate: $\mathbb E_t[\widehat\ell_{t,j}^2]\le v_{t,j}$ by
Lemma~\ref{lem:first-step-feedback}. It is determined before the
model index and action are sampled and measures the possible size
of the random loss estimate.

In the factor $1-\eta \widehat\ell_{t,j}+4\eta^2v_{t,j}$,
the loss term favors models with smaller estimated losses,
while the positive second-moment correction cushions their
down-weighting when the estimates have larger second-moment bounds.
We multiply each current weight by this factor and normalize
the resulting weights to obtain the next sampling distribution.
The choice of $\eta$ in \eqref{eq:first-step-tuning} ensures that
every factor is positive.
The symbol $t+$ denotes the weights saved for the next visit
to this same level and epoch.

Algorithm~\ref{alg:first-selection} selects the first action and
a candidate Q-function, executes the first step, and updates the current
epoch's records and weights. It returns the selected pair and the
observed transition, after which the main algorithm follows
$\pi_{f_t}$ from step two onward.

\begin{algorithm}[htbp]
\caption{First action and candidate Q-function selection (\textsc{SelectFirstAction})}
\label{alg:first-selection}
\small
\begin{algorithmic}[1]
\Require Episode $t$, initial state $s_t^1$, version space $\mathcal F_t$, action grid $\mathcal A_T$, partitions $\mathcal D_l,\mathcal P_l,\mathcal C_l$, and exploration policies $\pi_l^{\rm G}$ from Subsections~\ref{subsec:state-action-partition}--\ref{sec:first-step-selection-principle}.
\Require Stored visit counts and, at each level, the current epoch's fixed net, parameters, first-step records, and model weights; access to first-step observations.
\Ensure $(a_t^1,f_t,R_t^1,s_t^2)$ and saved level counts, epoch records, and model weights.
\State Select $l_t$ by \eqref{eq:adaptive-cube}, increment its visit count, and determine $u_t$ by \eqref{eq:epoch-index}.
\If{$n_t(l_t)=2^{u_t}$}
    \State Start empty first-step records for this epoch; construct its fixed $\varepsilon$-net and set $(\eta,\gamma)$ by \eqref{eq:first-step-tuning}.
    \State Initialize $\omega_{t,j}=1/M$ for $j\in[M]$.
\Else
    \State Resume this level's current epoch with the same fixed net, parameters, records, and saved weights.
\EndIf
\If{$\mathcal F_t\ne\varnothing$}
    \State Use this epoch's records to compute the scores and recommendations by \eqref{eq:optimistic-score}-\eqref{eq:continuation-choice}.
    \State Form $p_t$ by \eqref{eq:first-step-sampling} and compute $h_t$ by \eqref{eq:minimum-variance-feedback}, before sampling.
    \State Draw $j_t$ with probabilities $(\omega_{t,j})_{j=1}^M$ and set $f_t=f_t(j_t)$.
    \State With probability $1-\gamma$, set $a_t^1=a_t(j_t)$; otherwise draw $a_t^1\sim\pi_{l_t}^{\rm G}$.
\Else
    \State Set $a_t^1=a_0$ and $f_t=0$.
\EndIf
\State Play $a_t^1$ and observe $(R_t^1,s_t^2)$.
\If{$\mathcal F_t\ne\varnothing$}
    \State Compute the feedback estimates by \eqref{eq:first-step-feedback-estimate} and save the updated weights from \eqref{eq:expert-weight-update}.
\EndIf
\State Append $(s_t^1,a_t^1,R_t^1,s_t^2)$ to this epoch's first-step records.
\State \Return $(a_t^1,f_t,R_t^1,s_t^2)$.
\end{algorithmic}
\end{algorithm}

Algorithm~\ref{alg:skovi} calls Algorithms~\ref{alg:version-space}
and~\ref{alg:first-selection} in each episode, then executes the
selected continuation policy and stores its trajectory.
The stored records are shared with both subroutines.

\begin{algorithm}[ht]
\caption{Sobolev Optimal Value Iteration (\texttt{SOVI})}
\label{alg:skovi}
\small
\begin{algorithmic}[1]
\Require Spaces $\cS=[0,1]^{d_S}$ and $\cA=[0,1]^{d_A}$, kernel $k$, known norm bounds $B_p,B_r$, smoothness $\nu$, horizon $H\ge2$, episode budget $T\ge2$, and confidence level $\delta\in(0,1)$.
\Require Reference kernel $k_0$ and norm-comparison factor $\Lambda$ specified in Section~\ref{sec:skovi}, and entropy-integral bound $J$ in \eqref{eq:continuation-entropy-bound}.
\State Construct the candidate Q-function class and local features from Section~\ref{sec:skovi}, the partition $\mathcal C$ and test parameters from Section~\ref{sec:sovi-version-space}, and the state and action partitions, action grid, and G-optimal exploration policies from Subsections~\ref{subsec:state-action-partition}-\ref{sec:first-step-selection-principle}.
\State Initialize empty trajectory records, $m_1(B)=0$ for every counting cell, and $n_0(l)=0$ for every level.
\For{$t=1,\ldots,T$}
    \State $\mathcal F_t\gets$ \Call{ConstructVersionSpace}{$t$, stored trajectories}.
    \State Observe the initial state $s_t^1$.
    \State $(a_t^1,f_t,R_t^1,s_t^2)\gets$ \Call{SelectFirstAction}{$t,s_t^1,\mathcal F_t$}.
    \For{$h=2,\ldots,H$}
        \State Select $a_t^h=\pi_{f_t}^h(s_t^h)\in\operatorname*{arg\,max}_{a\in\cA}f_t^h(s_t^h,a)$ using the fixed greedy selector.
        \State Execute $a_t^h$ and observe the reward $R_t^h$.
        \If{$h<H$}
            \State Observe the next state $s_t^{h+1}$.
        \EndIf
    \EndFor
    \State Store $(s_t^1,a_t^1,R_t^1,\ldots,s_t^H,a_t^H,R_t^H)$ for future version-space tests.
    \State Increment the count of $B_l(s_t^1)$ in every counting partition $\mathcal D_l$.
\EndFor
\end{algorithmic}
\end{algorithm}

\FloatBarrier
\section{Theoretical Guarantee}
\label{sec:guarantees}

We establish a regret bound and a matching minimax exponent in $T$.
For fixed smoothness $\nu$ and reference-kernel inverse lengthscale
$\kappa$, the upper-bound prefactor is polynomial in $d,H,B_p,B_r,\Lambda$,
up to logarithmic factors. Throughout this section, we
use the continuous-action model of Section~\ref{sec:setup}, with
$H\ge2$, $T\ge2$, and $d=d_S+d_A$. Define the regret exponents
\[
e_A:=\frac{\nu+d_A}{2\nu+d_A},\
e_B:=\frac{\nu+d_A+3d_S/2}{2\nu+d_A+2d_S},\
e_*:=\max\{e_A,e_B\}.
\]
For a confidence level $\delta\in(0,1)$, put
\begin{equation}
\mathfrak B:=1+H+B_p+\Lambda(B_r+HB_p),\
\iota_T:=\ln\frac{\mathrm e(2+d)T\mathfrak B}{\delta},
\label{eq:regret-parameter-summary}
\end{equation}
where, for a general Sobolev kernel,
\[
\Lambda=\frac1{c_1}
\left\{\frac{\Gamma(\nu)}{(4\pi)^{d/2}\Gamma(\nu+d/2)}\right\}^{1/2},
\ \kappa=1.
\]
If $k$ is a normalized Mat\'ern kernel with lengthscale $\ell_0$,
we instead take $\Lambda=1$ and $\kappa=\sqrt{2\nu}/\ell_0$.

\begin{theorem}
\label{thm:regret-upper}
For any $\delta\in(0,1)$, $M\in\mathcal M(B_p,B_r)$, and admissible
initial-state adversary, run \texttt{SOVI} with confidence level
$\delta$ and the parameters of
Sections~\ref{sec:skovi}-\ref{sec:sovi-first-selection}, choosing
the constants in \eqref{eq:continuation-entropy-bound} and
\eqref{eq:regression-confidence-parameter} sufficiently large.
With probability at least $1-\delta$,
\begin{equation}
\reg(T)\le
C_{\nu,\kappa}(2+d)^{\lfloor\nu\rfloor+7\nu+6}
\mathfrak B^3\iota_T^{\,2+6\nu}T^{e_*}.
\label{eq:regret-upper}
\end{equation}
The constant $C_{\nu,\kappa}$ depends only on $\nu$ and $\kappa$.
\end{theorem}

Run \texttt{SOVI} with $\delta=T^{-1}$. By
\eqref{eq:regret-upper}, \eqref{eq:regret-parameter-summary}, and
$0\le\reg(T)\le HT$, we have
\begin{equation}
\mathbb E[\reg(T)]
\le C_{\nu,\kappa}(2+d)^{\lfloor\nu\rfloor+7\nu+6}
\mathfrak B^3
\bigl[\ln\!\bigl(\mathrm e(2+d)T^2\mathfrak B\bigr)\bigr]^{2+6\nu}
T^{e_*}+H.
\label{eq:regret-upper-expectation}
\end{equation}
The last term bounds the contribution of the failure event by
$HT\delta=H$. This estimate is uniform over the models and
admissible adversaries in Theorem~\ref{thm:regret-upper}.

For a normalized Mat\'ern kernel, $\Lambda=1$. For a general Sobolev
kernel, the bound retains its dependence on $\Lambda$ explicitly.
Since $\mathfrak B\le(1+\Lambda)(1+H)(1+B_p+B_r)$, the dependence
on these parameters is polynomial. The matching lower bound below
identifies the optimal exponent in $T$.

\paragraph{Upper-bound intuition.}
We ignore constants and logarithmic factors. Both exponents are best
read as kernel-bandit exponents. The exponent $\frac{\nu+d}{2\nu+d}$
from Section~\ref{sec:introduction} corresponds to smoothness $\nu$
and dimension $d$. With $\nu_S=\nu+d_S/2$ as in
Section~\ref{sec:sovi-version-space}, $e_A=\frac{\nu+d_A}{2\nu+d_A},\ 
e_B=\frac{\nu_S+d}{2\nu_S+d}$.

So $e_A$ is the kernel-bandit exponent in dimension $d_A$ instead of
$d$, and $e_B$ is the kernel-bandit exponent for smoothness $\nu_S$
instead of $\nu$. This exponent increases with the dimension and
decreases with the smoothness, so both $e_A$ and $e_B$ lie below
$\frac{\nu+d}{2\nu+d}$. Roughly, $e_A$ is the cost of searching the
actions at a single state, and $e_B$ is the cost of finding good
actions at all states, where averaging over the states acts like
$d_S/2$ extra units of smoothness.

\emph{Two parts of the regret.}
Suppose $Q_*\in\mathcal F_t$, which holds for all $t$ with high
probability (Section~\ref{sec:sovi-version-space}). Adding and
subtracting $(\mathcal B^1V_{f_t}^2)(z_t^1)$ and telescoping the
values along the trajectory give
\[
\begin{aligned}
V_*^1(s_t^1)-V^1(s_t^1;\pi_t)
\le{}&\underbrace{\sup_{a\in\cA,\,f\in\mathcal F_t}
(\mathcal B^1V_f^2)(s_t^1,a)
-(\mathcal B^1V_{f_t}^2)(z_t^1)}_{\text{first step}}+\underbrace{\sum_{h=2}^H
\mathbb E\bigl[\Delta_{f_t}^h(z_t^h)\,\big|\,\mathcal I_t^1\bigr]}_{\text{later steps}}.
\end{aligned}
\]
The first part is decided at a state chosen by the adversary. The
second part collects the Bellman residuals
\eqref{eq:bellman-residual-definition} of the selected candidate along
its own trajectory, at states generated by the transitions. We discuss
the second part first, because the first part reuses its ingredients.

\emph{Later steps.}
From step two onward, the state has density at most $B_p$
(Lemma~\ref{lem:state-occupancy-bound}). A residual therefore only
needs to be small on average over the states, but at every action,
since the greedy policy of $f_t$ may choose any action. On the
partition of Section~\ref{sec:sovi-version-space}, with about $R^d$
cells of side about $1/R$, three costs appear.
\begin{enumerate}[label=(\roman*),leftmargin=*]
\item \emph{Estimation.} In a cell visited $n$ times, the regression
widths add up to about $\sqrt n$. Summing over the $R^d$ cells gives
about $\sqrt{R^dT}$.
\item \emph{Approximation.} On a cell of side $1/R$, local polynomials
leave an error of order at most $R^{-\nu}$
(Lemma~\ref{lem:feature-remainder}). A function of bounded norm can be
this bad only in a few cells, because the cells share its Sobolev
norm: the squared errors of all cells add up to at most about
$R^{-2\nu}$. Even if the policy picks the worst action cell in each of
the $R^{d_S}$ state cells, the mean squared error over the state cells
is at most $R^{-2\nu}/R^{d_S}$, so the mean error is at most
$R^{-\nu-d_S/2}=R^{-\nu_S}$. Over $T$ episodes this costs
$TR^{-\nu_S}$.
\item \emph{Remainder.} The first test checks each cell separately, so
it controls a residual only up to its remainder. Different candidates
can put their remainders in different cells, so a guarantee that
holds for all candidates at once must allow the full remainder
$R^{-\nu}$ at every data point. This would cost
$TR^{-\nu}=R^{d_S/2}\cdot TR^{-\nu_S}$, more than (ii). The second
test avoids this. For each candidate, the remainder is the inner
product of a single RKHS function, $\Delta_f^h$, with known features,
so random Gaussian directions built from these features can detect
it. The Gaussian scale $\alpha_{\rm vs}$ sets a trade-off. Deviations
below about $\alpha_{\rm vs}R^{-\nu}$ go unnoticed, which costs
$\alpha_{\rm vs}TR^{-\nu}$. Larger deviations are rejected once their
squared sum on past data reaches about $\alpha_{\rm vs}^{-2}$, which
costs about $\sqrt{R^{d_A}T}/\alpha_{\rm vs}$ under the current
policies, since these may choose new actions in any of the $R^{d_A}$
action cells. The best scale gives the geometric mean
$\bigl(\sqrt{R^{d_A}T}\cdot TR^{-\nu}\bigr)^{1/2}
=R^{(d_A-2\nu)/4}T^{3/4}$.
\end{enumerate}
The resolution $R=T^{1/(2\nu_S+d)}$ in
\eqref{eq:version-space-resolution} balances (i) and (ii):
$\sqrt{R^dT}=TR^{-\nu_S}=T^{e_B}$. Moreover, for every $R$,
\[
\sqrt{R^{d_A}T}\cdot TR^{-\nu}=\sqrt{R^dT}\cdot TR^{-\nu_S},
\]
so the cost in (iii) is also the geometric mean of (i) and (ii), and
it equals $T^{e_B}$ as well. The other branch of
\eqref{eq:version-space-resolution}, used when $d_A\ge4\nu$ and
$d_S<d_A$ (where $e_*=e_A$), balances the same costs at the level
$T^{e_A}$. Its only purpose is to keep the prefactor polynomial in $d$.

\emph{First step.}
Here the state is chosen by the adversary, so
Lemma~\ref{lem:state-occupancy-bound} does not apply. If the adversary
repeats one state, the first step is a kernel bandit on the
$d_A$-dimensional action space, with exponent $e_A$. If it spreads the
initial states evenly, the three costs above give $T^{e_B}$. Other
patterns lie in between (see the lower-bound intuition). \texttt{SOVI}
adapts to the pattern through its levels. A counting cell of level $l$
serves at most $N_l$ initial states, and later visits to it go to a
finer level. This rule keeps each cost at most $T^{e_*}$ on every
level. For the approximation error, each counting cell receives at
most $N_l$ visits, and the squared errors of the cells add up to at
most about $S_l^{-2\nu}$. The Cauchy--Schwarz inequality then bounds
the total error by about $S_l^{-\nu}\sqrt{N_lT}=TS_{\max}^{-\nu}
=T^{e_*}$, by the choice of $N_l$ in \eqref{eq:initial-scale-schedule}
and $S_{\max}^{\nu}=T^{1-e_*}$. For the estimation error, a state
region reaches level $l$ only after its counting cell of level $l-1$
has received $N_{l-1}$ initial states. So at most $T/N_{l-1}$ regions
use the finer cells, and their estimation cost also stays at most
$T^{e_*}$. The remainder is again treated as one unknown function, now
through the remainder models and their weights, at the same
geometric-mean cost. Summing over the $O(\log T)$ levels gives a
first-step loss of order $T^{e_*}$. Together with the later steps,
this gives Theorem~\ref{thm:regret-upper}. The proof is in
Appendix~\ref{app:upper-bound}.

Let $\mathsf{Alg}$ range over nonanticipating algorithms that know
the model-class parameters, and $\mathsf{Adv}$ over the admissible
initial-state adversaries of Section~\ref{sec:setup}.
The minimax expected regret is
$$\inf_{\mathsf{Alg}}
\sup_{M\in\mathcal M(B_p,B_r)}
\sup_{\mathsf{Adv}}
\mathbb E^{M,\mathsf{Alg},\mathsf{Adv}}[\reg(T)].$$
Then, we have the following theorem.

\begin{theorem}
\label{thm:regret-lower}
Suppose $B_p>\|\mathbf1_{\cZ}\|_{\cH_k}$ and $B_r>0$.
There exist $c_A,c_B>0$ and $T_0\ge2$, depending only on
$\nu,d_S,d_A,c_2,B_r$ and $B_p-\|\mathbf1_{\cZ}\|_{\cH_k}$,
such that, for every $T\ge T_0$,
\begin{equation}
\begin{aligned}
\inf_{\mathsf{Alg}}
\sup_{M\in\mathcal M(B_p,B_r)}
\sup_{\mathsf{Adv}}
\mathbb E^{M,\mathsf{Alg},\mathsf{Adv}}[\reg(T)]
&\ge\max\{c_AT^{e_A},\ \lfloor H/2\rfloor c_BT^{e_B}\}\\
&\ge\min\cbr{c_A,\floor{H/2}c_B}T^{\max\cbr{e_A,e_B}}.
\end{aligned}
\label{eq:regret-lower}
\end{equation}
The constants are independent of $H,T$, and the bound holds
with deterministic rewards.
\end{theorem}
\paragraph{Lower-bound intuition.}
Each exponent comes from one family of hard models. In both families,
the unknown part of the model is a sum of bumps of height
$\varepsilon$ and width $w$ on the state-action space, and each bump
makes one small action cell better by $\varepsilon$ at the states it
covers. The bumps enter only through the transition: a better action
raises by $\varepsilon$ the probability of moving to the half of the
state space where the next reward is higher. All rewards are
deterministic, so each visit shows only which half the next state
falls in. The assumptions $B_p>\|\mathbf1_{\cZ}\|_{\cH_k}$ and $B_r>0$
leave room for the bumps and for this reward. Two constraints limit
such models.
\begin{enumerate}[label=(\roman*),leftmargin=*]
\item \emph{Norm constraint.} A change of variables in the Sobolev
norm of order $\nu+d/2$ shows that a bump of width $w$ has RKHS norm
of order $\varepsilon w^{-\nu}$, and the squared norms of bumps with
disjoint supports add (up to a constant factor). Hence $n$ such bumps
satisfy the norm bound only if $\varepsilon\sqrt n\,w^{-\nu}\lesssim1$:
the more bumps, the lower each must be. This is the same sharing of
the Sobolev norm that helps the upper bound.
\item \emph{Sample constraint.} If a state cell receives $N$ visits
and its bump hides in one of the $w^{-d_A}$ action cells, the learner
cannot find the bump unless $N\varepsilon^2w^{d_A}\gtrsim1$. Until
then, it loses about $\varepsilon$ per visit.
\end{enumerate}
\emph{One state.} The adversary always starts at the same state, and
one bump is hidden among its actions. With $n=1$ and $N=T$, the two
constraints give $\varepsilon\asymp w^{\nu}$ and
$w\asymp T^{-1/(2\nu+d_A)}$, so the regret is
$T\varepsilon\asymp T^{e_A}$. This is the kernel-bandit lower bound on
the $d_A$-dimensional action space.

\emph{All states.} The initial states are uniform, and after each even
step the transition draws a fresh uniform state. Hence, at every odd
step, each of the $w^{-d_S}$ state cells is visited about
$N=Tw^{d_S}$ times, whatever the learner does, and each state cell
hides its own bump. The norm bound is now shared by $n=w^{-d_S}$
bumps, so $\varepsilon\asymp w^{\nu_S}$ with $\nu_S=\nu+d_S/2$: each
bump is as low as a single bump of smoothness $\nu_S$. The sample
constraint then gives $w\asymp T^{-1/(2\nu_S+d)}$ and a regret of
$T\varepsilon\asymp T^{e_B}$ in every pair of steps, hence the factor
$\lfloor H/2\rfloor$.

\emph{Nothing in between is harder.} If the initial states are spread
over only $w^{-m}$ state cells, $0<m<d_S$, the same calculation with
$n=w^{-m}$ and $N=Tw^{m}$ gives the exponent
\[
1-\frac{\nu+m/2}{2\nu+d_A+2m}
=\frac34+\frac{d_A-2\nu}{4(2\nu+d_A+2m)}.
\]
It equals $e_A$ at $m=0$ and $e_B$ at $m=d_S$, and it moves
monotonically toward $3/4$ as $m$ grows: upward if $d_A<2\nu$, and
downward if $d_A>2\nu$. So the hardest pattern in this family is one
of the two extremes, which explains the switch at $d_A=2\nu$.
Theorem~\ref{thm:regret-upper} shows that no other model is harder.
The proof is in Appendix~\ref{app:lower-bound}.

Combining \eqref{eq:regret-upper-expectation} with Theorem~\ref{thm:regret-lower}, we conclude that $$\inf_{\mathsf{Alg}}
\sup_{M\in\mathcal M(B_p,B_r)}
\sup_{\mathsf{Adv}}\mathbb E^{M,\mathsf{Alg},\mathsf{Adv}}[\reg(T)]=\widetilde\Theta(T^{e_*}).$$

Thus, \texttt{SOVI} is minimax optimal in $T$ up to logarithmic
factors.


\section{Discussion}
\label{sec:discussion}

We establish matching upper and lower bounds for episodic Sobolev kernel-based reinforcement learning, determining the minimax regret rate in the number of episodes and thereby partially resolving the open problem in \citet{vakili2024}. The minimax exponent reflects the distinct roles of the state and action dimensions. Our algorithm, Sobolev Optimal Value Iteration (\texttt{SOVI}), attains the optimal exponent in $T$. For fixed smoothness and reference-kernel length scale, its upper-bound prefactor is polynomial in $d,H,B_p,B_r,\Lambda$, up to logarithmic factors. Its design separates initial action selection from continuation-value estimation and combines local polynomial approximation with the shared RKHS structure of the approximation errors.

Several directions remain open. First, sharper approximation and estimation bounds could reduce the polynomial dependence on dimension, horizon, and RKHS norm bounds in the regret upper bound. Second, extending the analysis beyond Sobolev RKHSs would clarify which kernel properties determine the minimax regret rate. Sharper lower-bound constructions are also needed for characterizing the joint minimax dependence on $T$, $d_S$, $d_A$, $H$, and the reward and transition norm bounds. Finally, \texttt{SOVI} attains statistical minimax optimality, while constructing and using its partitions and $\varepsilon$-nets can incur substantial computational costs. Developing a computationally efficient Sobolev kernel-based RL algorithm that retains the optimal regret rate in $T$ and performs well empirically is an important direction for future work.

\FloatBarrier
\bibliography{refs}

@inproceedings{abbasiYadkori2011,
  author    = {Yasin Abbasi-Yadkori and D{\'a}vid P{\'a}l and Csaba Szepesv{\'a}ri},
  title     = {Improved Algorithms for Linear Stochastic Bandits},
  booktitle = {Advances in Neural Information Processing Systems (NeurIPS)},
  volume    = {24},
  pages     = {2312--2320},
  year      = {2011}
}

@book{adamsFournier2003,
  author    = {Robert A. Adams and John J. F. Fournier},
  title     = {Sobolev Spaces},
  publisher = {Academic Press},
  address   = {Amsterdam},
  edition   = {2},
  year      = {2003}
}

@misc{anantharam2017variational,
  author = {Anantharam, Venkat},
  title = {A Variational Characterization of {R{\'e}nyi} Divergences},
  howpublished = {arXiv preprint arXiv:1701.07796},
  year = {2017}, eprint = {1701.07796}, archivePrefix = {arXiv},
}

@article{aronszajn1950,
  author  = {Nachman Aronszajn},
  title   = {Theory of Reproducing Kernels},
  journal = {Transactions of the American Mathematical Society},
  volume  = {68},
  number  = {3},
  pages   = {337--404},
  year    = {1950}
}

@inproceedings{azar2017,
  author    = {Mohammad Gheshlaghi Azar and Ian Osband and R{\'e}mi Munos},
  title     = {Minimax Regret Bounds for Reinforcement Learning},
  booktitle = {Proceedings of the 34th International Conference on Machine Learning (ICML)},
  year      = {2017}
}

@inproceedings{bayrooti2025,
  author    = {Bayrooti, Jasmine and Vakili, Sattar and Prorok, Amanda and Ek, Carl Henrik},
  title     = {No-Regret {Thompson} Sampling for Finite-Horizon {Markov} Decision Processes with {Gaussian} Processes},
  booktitle = {Advances in Neural Information Processing Systems},
  volume    = {38},
  pages     = {81043--81071},
  year      = {2025},
}

@inproceedings{cao2020,
  author    = {Tongyi Cao and Akshay Krishnamurthy},
  title     = {Provably Adaptive Reinforcement Learning in Metric Spaces},
  booktitle = {Advances in Neural Information Processing Systems (NeurIPS)},
  year      = {2020}
}

@inproceedings{chowdhury2019,
  author    = {Sayak Ray Chowdhury and Aditya Gopalan},
  title     = {Online Learning in Kernelized {M}arkov Decision Processes},
  booktitle = {Proceedings of the 22nd International Conference on Artificial Intelligence and Statistics (AISTATS)},
  year      = {2019}
}

@inproceedings{chowdhury2023,
  author    = {Sayak Ray Chowdhury and Rafael Oliveira},
  title     = {Value Function Approximations via Kernel Embeddings for No-Regret Reinforcement Learning},
  booktitle = {Proceedings of the 14th Asian Conference on Machine Learning (ACML 2022)},
  series    = {Proceedings of Machine Learning Research},
  volume    = {189},
  pages     = {249--264},
  year      = {2023}
}

@inproceedings{curi2020,
  author    = {Sebastian Curi and Felix Berkenkamp and Andreas Krause},
  title     = {Efficient Model-Based Reinforcement Learning through Optimistic Policy Search and Planning},
  booktitle = {Advances in Neural Information Processing Systems (NeurIPS)},
  year      = {2020}
}

@inproceedings{domingues2021,
  author    = {Omar Darwiche Domingues and Pierre M{\'e}nard and Matteo Pirotta and Emilie Kaufmann and Michal Valko},
  title     = {Kernel-Based Reinforcement Learning: A Finite-Time Analysis},
  booktitle = {Proceedings of the 38th International Conference on Machine Learning (ICML)},
  year      = {2021}
}

@inproceedings{flynn2026,
  author    = {Flynn, Hamish and Watson, Joe and Posner, Ingmar and Peters, Jan},
  title     = {Posterior Sampling Reinforcement Learning with {Gaussian} Processes for Continuous Control: Sublinear Regret Bounds for Unbounded State Spaces},
  booktitle = {Proceedings of the 43rd International Conference on Machine Learning},
  series    = {Proceedings of Machine Learning Research},
  volume    = {306},
  pages     = {31284--31328},
  publisher = {PMLR},
  year      = {2026},
}

@misc{hairer2025spdes,
  author = {Hairer, Martin},
  title = {An Introduction to Stochastic {PDEs}},
  year = {2025}, note = {Lecture notes, November 10, 2025},
}

@article{howard2020time,
  author = {Howard, Steven R. and Ramdas, Aaditya and McAuliffe, Jon and Sekhon, Jasjeet},
  title = {Time-uniform {Chernoff} Bounds via Nonnegative Supermartingales},
  journal = {Probability Surveys},
  volume = {17}, pages = {257--317}, year = {2020},
}

@inproceedings{jacot2018,
  author    = {Arthur Jacot and Franck Gabriel and Cl{\'e}ment Hongler},
  title     = {Neural Tangent Kernel: Convergence and Generalization in Neural Networks},
  booktitle = {Advances in Neural Information Processing Systems (NeurIPS)},
  year      = {2018}
}

@inproceedings{jin2018,
  author    = {Chi Jin and Zeyuan Allen-Zhu and Sebastien Bubeck and Michael I. Jordan},
  title     = {Is {Q}-Learning Provably Efficient?},
  booktitle = {Advances in Neural Information Processing Systems (NeurIPS)},
  year      = {2018}
}

@inproceedings{jin2020,
  author    = {Chi Jin and Zhuoran Yang and Zhaoran Wang and Michael I. Jordan},
  title     = {Provably Efficient Reinforcement Learning with Linear Function Approximation},
  booktitle = {Proceedings of the Thirty Third Conference on Learning Theory (COLT)},
  pages     = {2137--2143},
  year      = {2020}
}

@inproceedings{kakade2020,
  author    = {Sham Kakade and Akshay Krishnamurthy and Kendall Lowrey and Motoya Ohnishi and Wen Sun},
  title     = {Information Theoretic Regret Bounds for Online Nonlinear Control},
  booktitle = {Advances in Neural Information Processing Systems (NeurIPS)},
  year      = {2020}
}

@inproceedings{kar2025,
  author    = {Kar, Avik and Singh, Rahul},
  title     = {Provably Adaptive Average Reward Reinforcement Learning for Metric Spaces},
  booktitle = {Proceedings of the Forty-first Conference on Uncertainty in Artificial Intelligence},
  series    = {Proceedings of Machine Learning Research},
  volume    = {286},
  pages     = {1924--1964},
  publisher = {PMLR},
  year      = {2025},
}

@inproceedings{kayal2025,
  author    = {Kayal, Aya and Vakili, Sattar and Toni, Laura and Bernacchia, Alberto},
  title     = {Near-Optimal Sample Complexity in Reward-Free Kernel-based Reinforcement Learning},
  booktitle = {Proceedings of the 28th International Conference on Artificial Intelligence and Statistics},
  series    = {Proceedings of Machine Learning Research},
  volume    = {258},
  pages     = {874--882},
  publisher = {PMLR},
  year      = {2025},
}

@book{kechris1995,
  author    = {Alexander S. Kechris},
  title     = {Classical Descriptive Set Theory},
  series    = {Graduate Texts in Mathematics},
  volume    = {156},
  publisher = {Springer},
  address   = {New York},
  year      = {1995}
}

@article{kieferWolfowitz1960,
  author = {Kiefer, J. and Wolfowitz, J.},
  title = {The Equivalence of Two Extremum Problems},
  journal = {Canadian Journal of Mathematics},
  volume = {12}, pages = {363--366}, year = {1960},
}

@inproceedings{liScarlett2022,
  author    = {Zihan Li and Jonathan Scarlett},
  title     = {Gaussian Process Bandit Optimization with Few Batches},
  booktitle = {Proceedings of the 25th International Conference on Artificial Intelligence and Statistics (AISTATS)},
  year      = {2022}
}

@inproceedings{maran2024a,
  author    = {Davide Maran and Alberto Maria Metelli and Matteo Papini and Marcello Restelli},
  title     = {Local Linearity: The Key for No-Regret Reinforcement Learning in Continuous {MDP}s},
  booktitle = {Advances in Neural Information Processing Systems (NeurIPS)},
  year      = {2024}
}

@inproceedings{maran2024b,
  author    = {Davide Maran and Alberto Maria Metelli and Matteo Papini and Marcello Restelli},
  title     = {Projection by Convolution: Optimal Sample Complexity for Reinforcement Learning in Continuous-Space {MDP}s},
  booktitle = {Proceedings of the Thirty Seventh Conference on Learning Theory (COLT)},
  series    = {Proceedings of Machine Learning Research},
  volume    = {247},
  pages     = {3743--3774},
  year      = {2024}
}

@inproceedings{maran2024c,
  author    = {Davide Maran and Alberto Maria Metelli and Matteo Papini and Marcello Restelli},
  title     = {No-Regret Reinforcement Learning in Smooth {MDP}s},
  booktitle = {Proceedings of the 41st International Conference on Machine Learning (ICML)},
  series    = {Proceedings of Machine Learning Research},
  volume    = {235},
  pages     = {34760--34789},
  year      = {2024}
}

@misc{pavlovic2026,
  author        = {Pavlovic, Nikola and Vakili, Sattar and Zhao, Qing},
  title         = {Learning Kernel-Based {MDP}s from Episodic Preferential Feedback},
  howpublished  = {arXiv preprint arXiv:2605.23650},
  year          = {2026},
  eprint        = {2605.23650},
  archivePrefix = {arXiv},
  primaryClass  = {stat.ML},
}

@book{puterman2014,
  author    = {Martin L. Puterman},
  title     = {Markov Decision Processes: Discrete Stochastic Dynamic Programming},
  publisher = {John Wiley \& Sons},
  year      = {2014}
}

@book{rasmussen2006,
  author    = {Carl Edward Rasmussen and Christopher K. I. Williams},
  title     = {Gaussian Processes for Machine Learning},
  publisher = {MIT Press},
  year      = {2006}
}

@inproceedings{salgia2021,
  author    = {Sudeep Salgia and Sattar Vakili and Qing Zhao},
  title     = {A Domain-Shrinking Based {B}ayesian Optimization Algorithm with Order-Optimal Regret Performance},
  booktitle = {Advances in Neural Information Processing Systems (NeurIPS)},
  volume    = {34},
  year      = {2021}
}

@inproceedings{scarlett2017,
  author    = {Jonathan Scarlett and Ilija Bogunovic and Volkan Cevher},
  title     = {Lower Bounds on Regret for Noisy {G}aussian Process Bandit Optimization},
  booktitle = {Proceedings of the 2017 Conference on Learning Theory (COLT)},
  series    = {Proceedings of Machine Learning Research},
  volume    = {65},
  pages     = {1723--1742},
  year      = {2017}
}

@article{sinclair2019,
  author  = {Sean R. Sinclair and Siddhartha Banerjee and Christina Lee Yu},
  title   = {Adaptive Discretization for Episodic Reinforcement Learning in Metric Spaces},
  journal = {Proceedings of the ACM on Measurement and Analysis of Computing Systems (POMACS / SIGMETRICS 2020)},
  volume  = {3},
  number  = {3},
  year    = {2019}
}

@inproceedings{sinclair2020,
  author    = {Sean R. Sinclair and Tianyu Wang and Gauri Jain and Siddhartha Banerjee and Christina Lee Yu},
  title     = {Adaptive Discretization for Model-Based Reinforcement Learning},
  booktitle = {Advances in Neural Information Processing Systems (NeurIPS)},
  year      = {2020}
}

@article{sinclair2023,
  author  = {Sean R. Sinclair and Siddhartha Banerjee and Christina Lee Yu},
  title   = {Adaptive Discretization in Online Reinforcement Learning},
  journal = {Operations Research},
  year    = {2023}
}

@misc{song2019,
  author = {Zhao Song and Wen Sun},
  title  = {Efficient Model-Free Reinforcement Learning in Metric Spaces},
  note   = {arXiv preprint arXiv:1905.00475},
  year   = {2019}
}

@article{sunWu2009,
  author  = {Sun, Hongwei and Wu, Qiang},
  title   = {Application of Integral Operator for Regularized
             Least-Square Regression},
  journal = {Mathematical and Computer Modelling},
  volume  = {49},
  number  = {1--2},
  pages   = {276--285},
  year    = {2009},
}

@misc{touati2020,
  author = {Ahmed Touati and Adrien Ali Taiga and Marc G. Bellemare},
  title  = {Zooming for Efficient Model-Free Reinforcement Learning in Metric Spaces},
  note   = {arXiv preprint arXiv:2003.04069},
  year   = {2020}
}

@book{tsybakov2009,
  author    = {Alexandre B. Tsybakov},
  title     = {Introduction to Nonparametric Estimation},
  publisher = {Springer},
  address   = {New York},
  year      = {2009}
}

@inproceedings{vakili2021,
  author    = {Sattar Vakili and Kia Khezeli and Victor Picheny},
  title     = {On Information Gain and Regret Bounds in {G}aussian Process Bandits},
  booktitle = {Proceedings of the 24th International Conference on Artificial Intelligence and Statistics (AISTATS)},
  year      = {2021}
}

@inproceedings{vakili2024,
  author    = {Sattar Vakili},
  title     = {Open Problem: Order Optimal Regret Bounds for Kernel-Based Reinforcement Learning},
  booktitle = {Proceedings of the Thirty Seventh Conference on Learning Theory (COLT)},
  series    = {Proceedings of Machine Learning Research},
  volume    = {247},
  pages     = {5340--5344},
  year      = {2024}
}

@article{vakili2024kernelized,
  title={Kernelized reinforcement learning with order optimal regret bounds},
  author={Vakili, Sattar and Olkhovskaya, Julia},
  journal={Advances in Neural Information Processing Systems},
  volume={36},
  year={2023}
}

@inproceedings{vakiliOlkhovskaya2024,
  author    = {Vakili, Sattar and Olkhovskaya, Julia},
  title     = {Kernel-Based Function Approximation for Average Reward Reinforcement Learning: An Optimist No-Regret Algorithm},
  booktitle = {Advances in Neural Information Processing Systems},
  volume    = {37},
  pages     = {25401--25425},
  year      = {2024},
}

@inproceedings{valko2013,
  author    = {Michal Valko and Nathan Korda and R{\'e}mi Munos and Ilias Flaounas and Nello Cristianini},
  title     = {Finite-Time Analysis of Kernelised Contextual Bandits},
  booktitle = {Proceedings of the Twenty-Ninth Conference on Uncertainty in Artificial Intelligence (UAI)},
  pages     = {654--663},
  year      = {2013}
}

@book{vershynin2018high,
  title={High-dimensional probability: An introduction with applications in data science},
  author={Vershynin, Roman},
  volume={47},
  year={2018},
  publisher={Cambridge university press}
}

@book{wendland2004,
  author    = {Holger Wendland},
  title     = {Scattered Data Approximation},
  publisher = {Cambridge University Press},
  year      = {2004}
}

@inproceedings{whitehouse2023,
  author    = {Whitehouse, Justin and Ramdas, Aaditya and Wu, Steven Z.},
  title     = {On the Sublinear Regret of {GP-UCB}},
  booktitle = {Advances in Neural Information Processing Systems},
  volume    = {36},
  year      = {2023},
  
}

@inproceedings{yang2020provably,
  title={Provably Efficient Reinforcement Learning with Kernel and Neural Function Approximations},
  author={Yang, Zhuoran and Jin, Chi and Wang, Zhaoran and Wang, Mengdi and Jordan, Michael I.},
  booktitle={Advances in Neural Information Processing Systems},
  volume={33},
  year={2020}
}

@inproceedings{yangWang2020,
  author    = {Yang, Lin and Wang, Mengdi},
  title     = {Reinforcement Learning in Feature Space: Matrix Bandit, Kernels, and Regret Bound},
  booktitle = {Proceedings of the 37th International Conference on Machine Learning},
  series    = {Proceedings of Machine Learning Research},
  volume    = {119},
  pages     = {10746--10756},
  publisher = {PMLR},
  year      = {2020},
}

@inproceedings{yeh2023,
  author    = {Sing-Yuan Yeh and Fu-Chieh Chang and Chang-Wei Yueh and Pei-Yuan Wu and Alberto Bernacchia and Sattar Vakili},
  title     = {Sample Complexity of Kernel-Based {Q}-Learning},
  booktitle = {Proceedings of the 26th International Conference on Artificial Intelligence and Statistics (AISTATS)},
  year      = {2023}
}

@inproceedings{qian2024offline,
  author    = {Qian, Jian and Hu, Haichen and Simchi-Levi, David},
  title     = {Offline Oracle-Efficient Learning for Contextual {MDPs} via Layerwise Exploration-Exploitation Tradeoff},
  booktitle = {Advances in Neural Information Processing Systems},
  volume    = {37},
  year      = {2024},
}

@inproceedings{hu2025contextual,
  author    = {Hu, Haichen and Ai, Rui and Bates, Stephen and Simchi-Levi, David},
  title     = {Contextual Online Decision Making with Infinite-Dimensional Functional Regression},
  booktitle = {Proceedings of the 42nd International Conference on Machine Learning},
  series    = {Proceedings of Machine Learning Research},
  volume    = {267},
  pages     = {24176--24202},
  publisher = {PMLR},
  year      = {2025},
}

@article{hu2026model,
  author    = {Hu, Haichen and Qian, Jian and Simchi-Levi, David},
  title     = {Model-Based Reinforcement Learning with Double Oracle Efficiency in Policy Optimization and Offline Estimation},
  journal   = {arXiv preprint arXiv:2605.00393},
  year      = {2026},
}

@inproceedings{zhao2018page,
  author    = {Zhao, Xiangyu and Xia, Long and Zhang, Liang and Ding, Zhuoye and Yin, Dawei and Tang, Jiliang},
  title     = {Deep Reinforcement Learning for Page-wise Recommendations},
  booktitle = {Proceedings of the 12th ACM Conference on Recommender Systems},
  pages     = {95--103},
  publisher = {Association for Computing Machinery},
  year      = {2018},
}

@inproceedings{ie2019slateq,
  author    = {Ie, Eugene and Jain, Vihan and Wang, Jing and Narvekar, Sanmit and Agarwal, Ritesh and Wu, Rui and Cheng, Heng-Tze and Chandra, Tushar and Boutilier, Craig},
  title     = {{SlateQ}: A Tractable Decomposition for Reinforcement Learning with Recommendation Sets},
  booktitle = {Proceedings of the Twenty-Eighth International Joint Conference on Artificial Intelligence},
  pages     = {2592--2599},
  year      = {2019},
}

@article{oroojlooyjadid2022beer,
  author    = {Oroojlooyjadid, Afshin and Nazari, MohammadReza and Snyder, Lawrence V. and Tak{\'a}{\v c}, Martin},
  title     = {A Deep {Q}-Network for the Beer Game: Deep Reinforcement Learning for Inventory Optimization},
  journal   = {Manufacturing \& Service Operations Management},
  volume    = {24},
  number    = {1},
  pages     = {285--304},
  year      = {2022},
}

@article{komorowski2018clinician,
  author    = {Komorowski, Matthieu and Celi, Leo A. and Badawi, Omar and Gordon, Anthony C. and Faisal, A. Aldo},
  title     = {The Artificial Intelligence Clinician learns optimal treatment strategies for sepsis in intensive care},
  journal   = {Nature Medicine},
  volume    = {24},
  number    = {11},
  pages     = {1716--1720},
  year      = {2018},
}

@inproceedings{jia2020value,
  author    = {Jia, Zeyu and Yang, Lin F. and Szepesv{\'a}ri, Csaba and Wang, Mengdi},
  title     = {Model-Based Reinforcement Learning with Value-Targeted Regression},
  booktitle = {Proceedings of the 2nd Conference on Learning for Dynamics and Control},
  series    = {Proceedings of Machine Learning Research},
  volume    = {120},
  pages     = {666--686},
  publisher = {PMLR},
  year      = {2020},
}

@inproceedings{foster2020beyond,
  author    = {Foster, Dylan J. and Rakhlin, Alexander},
  title     = {Beyond {UCB}: Optimal and Efficient Contextual Bandits with Regression Oracles},
  booktitle = {Proceedings of the 37th International Conference on Machine Learning},
  series    = {Proceedings of Machine Learning Research},
  volume    = {119},
  pages     = {3199--3210},
  publisher = {PMLR},
  year      = {2020},
}

@inproceedings{jin2021bellman,
  author    = {Jin, Chi and Liu, Qinghua and Miryoosefi, Sobhan},
  title     = {{Bellman} Eluder Dimension: New Rich Classes of {RL} Problems, and Sample-Efficient Algorithms},
  booktitle = {Advances in Neural Information Processing Systems},
  volume    = {34},
  pages     = {13406--13418},
  year      = {2021},
}

@inproceedings{mou2022markovian,
  author    = {Mou, Wenlong and Pananjady, Ashwin and Wainwright, Martin J. and Bartlett, Peter L.},
  title     = {Optimal and Instance-Dependent Guarantees for {Markovian} Linear Stochastic Approximation},
  booktitle = {Proceedings of Thirty Fifth Conference on Learning Theory},
  series    = {Proceedings of Machine Learning Research},
  volume    = {178},
  pages     = {2060--2061},
  publisher = {PMLR},
  year      = {2022},
}

@inproceedings{xia2023krylov,
  author    = {Xia, Eric and Wainwright, Martin J.},
  title     = {{Krylov--Bellman} Boosting: Super-Linear Policy Evaluation in General State Spaces},
  booktitle = {Proceedings of the 26th International Conference on Artificial Intelligence and Statistics},
  series    = {Proceedings of Machine Learning Research},
  volume    = {206},
  pages     = {9137--9166},
  publisher = {PMLR},
  year      = {2023},
}

@inproceedings{jia2024offline,
  author    = {Jia, Zeyu and Rakhlin, Alexander and Sekhari, Ayush and Wei, Chen-Yu},
  title     = {Offline Reinforcement Learning: Role of State Aggregation and Trajectory Data},
  booktitle = {Proceedings of Thirty Seventh Conference on Learning Theory},
  series    = {Proceedings of Machine Learning Research},
  volume    = {247},
  pages     = {2644--2719},
  publisher = {PMLR},
  year      = {2024},
}

@book{szego1975orthogonal,
  author    = {Szeg{\H{o}}, G{\'a}bor},
  title     = {Orthogonal Polynomials},
  edition   = {4},
  series    = {American Mathematical Society Colloquium Publications},
  volume    = {23},
  publisher = {American Mathematical Society},
  address   = {Providence, RI},
  year      = {1975}
}

@misc{langrene2024spectral,
  author        = {Langren{\'e}, Nicolas and Warin, Xavier and Gruet, Pierre},
  title         = {New Random Projections for Isotropic Kernels Using Stable Spectral Distributions},
  howpublished  = {arXiv preprint arXiv:2411.02770v1},
  year          = {2024},
  eprint        = {2411.02770v1},
  archivePrefix = {arXiv},
  primaryClass  = {cs.LG}
}

@inproceedings{mendelson2002geometric,
  author    = {Mendelson, Shahar},
  title     = {Geometric Parameters of Kernel Machines},
  booktitle = {Computational Learning Theory},
  editor    = {Kivinen, Jyrki and Sloan, Robert H.},
  series    = {Lecture Notes in Computer Science},
  volume    = {2375},
  pages     = {29--43},
  publisher = {Springer},
  year      = {2002}
}

@article{rudelsonVershynin2006,
  author  = {Rudelson, Mark and Vershynin, Roman},
  title   = {Combinatorics of Random Processes and Sections of Convex Bodies},
  journal = {Annals of Mathematics},
  volume  = {164},
  number  = {2},
  pages   = {603--648},
  year    = {2006}
}

@article{dirksen2015,
  author  = {Sjoerd Dirksen},
  title   = {Tail Bounds via Generic Chaining},
  journal = {Electronic Journal of Probability},
  volume  = {20},
  number  = {53},
  pages   = {1--29},
  year    = {2015}
}

@book{fedorov2013optimal,
  title={Optimal design for nonlinear response models},
  author={Fedorov, Valerii V and Leonov, Sergei L},
  year={2013},
  publisher={CRC Press}
}

@misc{vanderWoude2022rounding,
  author        = {Vander Woude, Jason and Dixon, Peter and Pavan, A. and Radcliffe, Jamie and Vinodchandran, N. V.},
  title         = {The Geometry of Rounding},
  howpublished  = {arXiv preprint arXiv:2211.02694v1},
  year          = {2022},
  eprint        = {2211.02694v1},
  archivePrefix = {arXiv},
  primaryClass  = {cs.DM}
}

@book{chafai2012interactions,
  author    = {Chafa{\"i}, Djalil and Gu{\'e}don, Olivier and Lecu{\'e}, Guillaume and Pajor, Alain},
  title     = {Interactions Between Compressed Sensing Random Matrices and High Dimensional Geometry},
  series    = {Panoramas et Synth\`eses},
  volume    = {37},
  publisher = {Soci\'et\'e Math\'ematique de France},
  year      = {2012}
}

@article{carmeli2010vector,
  author  = {Carmeli, Claudio and De Vito, Ernesto and Toigo, Alessandro and Umanit{\`a}, Veronica},
  title   = {Vector Valued Reproducing Kernel Hilbert Spaces and Universality},
  journal = {Analysis and Applications},
  volume  = {8},
  number  = {1},
  pages   = {19--61},
  year    = {2010}
}

@misc{golinskii2014schoenberg,
  author        = {Golinskii, Leonid and Malamud, Mark and Oridoroga, Leonid},
  title         = {{Schoenberg} Matrices of Radial Positive Definite Functions and {Riesz} Sequences in {$L^2(\mathbb R^n)$}},
  year          = {2014},
  eprint        = {1403.2234v1},
  archivePrefix = {arXiv},
  primaryClass  = {math.CA}
}

@misc{tao2006fouriernotes,
  author       = {Tao, Terence},
  title        = {Lecture Notes 2 for 247A},
  howpublished = {Fourier analysis course notes, University of California, Los Angeles},
  year         = {2006}
}

@book{foucartRauhut2013,
  author    = {Foucart, Simon and Rauhut, Holger},
  title     = {A Mathematical Introduction to Compressive Sensing},
  series    = {Applied and Numerical Harmonic Analysis},
  publisher = {Birkh{\"a}user},
  address   = {New York},
  year      = {2013}
}

@article{harris2002markov,
  author = {Harris, Lawrence A.},
  title = {{Markov}'s Inequality for Polynomials on Normed Linear Spaces},
  journal = {Mathematica Balkanica. New Series},
  volume = {16},
  number = {1--4},
  pages = {315--326},
  year = {2002}
}

@techreport{tropp2021probability,
  author = {Tropp, Joel A.},
  title = {{ACM} 217: Probability in High Dimensions},
  institution = {California Institute of Technology},
  type = {Caltech CMS Lecture Notes},
  number = {2021-01},
  year = {2021},
  note = {Corrected March 2023}
}

@misc{bartlett2013lecture14,
  author       = {Bartlett, Peter},
  title        = {Theoretical Statistics: Lecture 14. Metric Entropy},
  howpublished = {STAT 210B lecture notes, University of California, Berkeley},
  year         = {2013},
  note         = {Dudley's entropy integral, p. 4}
}
\bibliographystyle{plainnat}

\clearpage
\appendix

\section{Mathematical Tools}
\label{app:mathematical-tools}

We collect the RKHS, polynomial approximation, covering-number, Gaussian,
martingale, and geometric partition results used in the proofs. Model-specific approximation and regret bounds
are established in the corresponding proof sections.

For a real positive definite kernel $k$ on $\mathcal X$,
write $k_x:=k(x,\cdot)$.

\begin{lemma}[\cite{aronszajn1950}]
\label{lem:point-evaluation}
For every $f\in\cH_k$ and $x,y\in\mathcal X$,
\[
|f(x)|\le\|f\|_{\cH_k}\sqrt{k(x,x)},\ |f(x)-f(y)|
\le\|f\|_{\cH_k}
\sqrt{k(x,x)-2k(x,y)+k(y,y)}.
\]
Moreover, the linear span of $\{k_x:x\in\mathcal X\}$
is dense in $\cH_k$.
\end{lemma}

\begin{lemma}[\cite{aronszajn1950}]
\label{lem:rkhs-uniqueness-restriction}
For $\mathcal X_0\subseteq\mathcal X$, let
$k_0:=k|_{\mathcal X_0\times\mathcal X_0}$. Then $\cH_{k_0}=\{f|_{\mathcal X_0}:f\in\cH_k\}$, and every $g\in\cH_{k_0}$ satisfies
\[
\|g\|_{\cH_{k_0}}
=\min\{\|f\|_{\cH_k}:f\in\cH_k,\ f|_{\mathcal X_0}=g\}.
\]
In particular, restriction cannot increase the RKHS norm.
\end{lemma}

We use the Fourier convention of Section~\ref{sec:setup}.
For $\tau\ge0$ and $F\in L^2(\mathbb R^n)$, write
\[
\|F\|_{H^\tau(\mathbb R^n)}^2
:=\int_{\mathbb R^n}
(1+4\pi^2|\varpi|^2)^\tau|\widehat F(\varpi)|^2\ d\varpi,
\]
allowing the value $+\infty$. The spectral characterization below
also gives the norm comparison for Sobolev kernels.

\begin{lemma}[\cite{wendland2004}]
\label{lem:fourier-rkhs}
\label{lem:fourier-criterion}
Let $\kappa\in C(\mathbb R^n)\cap L^1(\mathbb R^n)$
be a real positive definite function, and suppose that
$\widehat\kappa>0$ almost everywhere. For
$k(x,y):=\kappa(x-y)$, we have
\[
\cH_k=
\left\{F\in C(\mathbb R^n)\cap L^2(\mathbb R^n):
\int_{\mathbb R^n}
\frac{|\widehat F(\varpi)|^2}{\widehat\kappa(\varpi)}
\ d\varpi<\infty\right\},
\]
with norm $\|F\|_{\cH_k}^2
=\int_{\mathbb R^n}
\frac{|\widehat F(\varpi)|^2}{\widehat\kappa(\varpi)}
\ d\varpi$.

If, for some $\tau>n/2$ and $0<c_{\mathrm{lo}}\le c_{\mathrm{hi}}<\infty$, we have $c_{\mathrm{lo}}(1+4\pi^2|\varpi|^2)^{-\tau}
\le\widehat\kappa(\varpi)
\le c_{\mathrm{hi}}(1+4\pi^2|\varpi|^2)^{-\tau}$,
then $\cH_k=H^\tau(\mathbb R^n)$ and
\[
c_{\mathrm{hi}}^{-1/2}\|F\|_{H^\tau(\mathbb R^n)}
\le\|F\|_{\cH_k}
\le c_{\mathrm{lo}}^{-1/2}\|F\|_{H^\tau(\mathbb R^n)}.
\]
The same comparison holds on every restriction
$\mathcal X\subseteq\mathbb R^n$, with
$c_1=c_{\mathrm{hi}}^{-1/2}$ and $c_2=c_{\mathrm{lo}}^{-1/2}$.
\end{lemma}

The next two results give the polynomial projection and the Mat\'ern
representation used in the local approximation bounds.

\begin{lemma}[\cite{szego1975orthogonal}]
\label{lem:local-legendre-projection}
Let $\overline c$ be a rectangle and let $\rho\in\mathbb N_0$.
The polynomials $p_{c,\mathbf b}$ in Definition~\ref{def:local-legendre},
with $|\mathbf b|\le\rho$, form an orthonormal basis of the space
of polynomials of total degree at most $\rho$ under normalized
Lebesgue measure:
\begin{equation}
\frac1{|\overline c|}\int_{\overline c}
p_{c,\mathbf b}(u)p_{c,\mathbf b'}(u)\,du
=\mathbf1\{\mathbf b=\mathbf b'\}.
\label{eq:legendre-orthogonality}
\end{equation}
Put $D:=\binom{d+\rho}{\rho}$ and $B_{\phi}:=\sqrt{3^\rho D}$.
For a continuous function $F$ from $\overline c$ into a real or
complex Hilbert space $\mathcal H$, its polynomial projection is
\begin{equation}
(\Pi_cF)(z):=\frac1{|\overline c|}\int_{\overline c}K_c(z,u)F(u)\,du,\
K_c(z,u):=\sum_{|\mathbf b|\le\rho}p_{c,\mathbf b}(z)p_{c,\mathbf b}(u).
\label{eq:legendre-projection-kernel}
\end{equation}
The projection kernel satisfies
\begin{equation}
\sup_{z\in\overline c}\frac1{|\overline c|}
\int_{\overline c}|K_c(z,u)|\,du\le B_{\phi}.
\label{eq:legendre-kernel-bound}
\end{equation}
Writing $\|F\|_\infty:=\sup_{z\in\overline c}\|F(z)\|_{\mathcal H}$,
we have $\|\Pi_cF\|_\infty\le B_{\phi}\|F\|_\infty$.
For every $\mathcal H$-valued polynomial $Q$ of total degree at
most $\rho$, $\Pi_cQ=Q$ and
\begin{equation}
\|F-\Pi_cF\|_\infty\le(1+B_{\phi})\|F-Q\|_\infty.
\label{eq:legendre-projection-error}
\end{equation}
\end{lemma}

\begin{lemma}[\cite{langrene2024spectral}]
\label{lem:matern-frequency-representation}
Let $k_0$ be the normalized Mat\'ern kernel with smoothness
$\nu>0$ and inverse lengthscale $\kappa>0$.
Take independent $Z\sim\mathcal N(0,I_d)$ and
$W\sim\chi^2_{2\nu}$, where $W$ has density
$w^{\nu-1}e^{-w/2}/(2^\nu\Gamma(\nu))$ for $w>0$.
Then $\omega:=\kappa Z/\sqrt W$ has density
\begin{equation}
p_\omega(v)=\frac{\Gamma(\nu+d/2)}{\pi^{d/2}\Gamma(\nu)\kappa^d}
\left(1+\frac{\|v\|_2^2}{\kappa^2}\right)^{-\nu-d/2},\ v\in\mathbb R^d.
\label{eq:matern-frequency-density}
\end{equation}
The kernel is the following average over frequencies:
\begin{equation}
k_0(z,z')=\mathbb E e^{\mathrm i\omega^\top(z-z')}
=\mathbb E\cos\bigl(\omega^\top(z-z')\bigr),\ z,z'\in\mathbb R^d.
\label{eq:matern-frequency-representation}
\end{equation}
\end{lemma}

The following regularity properties give compact candidate classes
and admissible smooth functions for the lower bound.

\begin{lemma}[\cite{adamsFournier2003}]
\label{lem:sobolev-spaces}
If $n\ge1$ and $\tau>n/2$, the following statements hold.
\begin{enumerate}
\item[(i)] Every $F\in H^\tau(\mathbb R^n)$ has a
bounded continuous representative, and there is a
constant $C$, depending only on $n,\tau$, such that $\|F\|_\infty\le C\|F\|_{H^\tau(\mathbb R^n)}$.

\item[(ii)] $\forall 0<\rho\le1$,
$\rho<\tau-n/2$, there is a constant $C_\rho$,
depending only on $n,\tau,\rho$, such that
\[
|F(x)-F(y)|
\le C_\rho\|F\|_{H^\tau(\mathbb R^n)}|x-y|^\rho
\ \text{for all }x,y\in\mathbb R^n.
\]

\item[(iii)] For every $\tau'\ge0$, $C_c^\infty(\mathbb R^n)\subseteq H^{\tau'}(\mathbb R^n)$.
\end{enumerate}
\end{lemma}

By Lemma~\ref{lem:rkhs-uniqueness-restriction}, restriction to
$\mathcal X\subseteq\mathbb R^n$ admits a norm-preserving extension:
\[
\|g\|_{H^\tau(\mathcal X)}
=\min_{\substack{F\in H^\tau(\mathbb R^n)\\F|_{\mathcal X}=g}}
\|F\|_{H^\tau(\mathbb R^n)}.
\]
Thus Lemma~\ref{lem:sobolev-spaces} also applies on compact restrictions,
with the same constants.

\begin{lemma}[\cite{sunWu2009}]
\label{lem:continuity-compactness}
Let $\mathcal X$ be a compact metric space and let $k$ be a
continuous positive definite kernel on $\mathcal X\times\mathcal X$.
Then $\cH_k\subseteq C(\mathcal X)$, and each closed ball
$\{g\in\cH_k:\|g\|_{\cH_k}\le R\}$, $R<\infty$, is compact
in the uniform norm.
\end{lemma}

For a Sobolev kernel of order $\tau>n/2$ on compact
$\mathcal X\subseteq\mathbb R^n$, minimum-norm extensions and
Lemma~\ref{lem:sobolev-spaces} give uniformly bounded and
H\"older-continuous representatives on each RKHS ball.
Fix $0<\rho\le1$ with $\rho<\tau-n/2$. By the reproducing
property and norm duality, we obtain
\[
\|k_x-k_y\|_{\cH_k}
=\sup_{\|g\|_{\cH_k}\le1}|g(x)-g(y)|
\le C|x-y|^\rho.
\]
Thus $x\mapsto k_x$ is continuous. Since
$k(x,y)=\langle k_x,k_y\rangle_{\cH_k}$, the kernel is jointly
continuous, so Lemma~\ref{lem:continuity-compactness} applies.

Measurable maximization applies to both greedy actions and
history-dependent candidate sets.

\begin{lemma}[\cite{kechris1995}]
\label{lem:measurable-maximum}
Let $\mathcal X$ be a standard Borel space, $\mathcal Y$ a Polish
space, and $\Gamma\subseteq\mathcal X\times\mathcal Y$ a Borel set
with compact sections $\Gamma_x$.
Its projection onto $\mathcal X$ is Borel.
If $Q:\Gamma\to\mathbb R$ is Borel and $Q(x,\cdot)$ is continuous
on each nonempty $\Gamma_x$, then
$x\mapsto\max_{y\in\Gamma_x}Q(x,y)$ is Borel on that projection
and admits a Borel maximizing selector there.
The same conclusion holds for minimization.
\end{lemma}

For quantitative covers, write $\mathcal N(\varepsilon,\mathcal G,d)$
for the covering number with centers in $\mathcal G$.
The fat-shattering dimension $\operatorname{fat}_\varepsilon(\mathcal G)$
is the largest $q$ for which there are points $x_1,\ldots,x_q$
and thresholds $b_1,\ldots,b_q$ such that every
$\sigma\in\{-1,1\}^q$ has a witness $g\in\mathcal G$ satisfying
$\sigma_i(g(x_i)-b_i)\ge\varepsilon$ for all $i\in[q]$; it is zero if no point is shattered.

\begin{lemma}[{\cite{mendelson2002geometric}}]
\label{lem:rkhs-fat-spectrum}
Let $k$ be a real positive semidefinite kernel on $\mathcal X$,
with RKHS $\mathcal H_k$, and let $B,\varepsilon>0$.
If $x_1,\ldots,x_n\in\mathcal X$, $n\ge1$, are
$\varepsilon$-fat-shattered by
$\{g\in\mathcal H_k:\|g\|_{\mathcal H_k}\le B\}$,
then the Gram matrix $K_n=(k(x_i,x_j))_{i,j=1}^n$ is positive
definite. Its eigenvalues $\theta_1,\ldots,\theta_n>0$ satisfy
\begin{equation}
\varepsilon^2\sum_{i=1}^n\theta_i^{-1}\le B^2.
\label{eq:rkhs-fat-spectrum}
\end{equation}
\end{lemma}

The next two inequalities control the variation of a polynomial
and the Gaussian supremum over tensor-product features.

\begin{lemma}[{\cite{harris2002markov}}]
\label{lem:polynomial-markov}
Let $\mathcal H$ be a real Hilbert space and let
$Q:\mathbb R^n\to\mathcal H$ be a polynomial of total degree at most
$m\in\mathbb N_0$. For a rectangle
$\overline c=\prod_{q=1}^n[\ell_{c,q},u_{c,q}]$ with positive side lengths,
we have, for all $z,z'\in\overline c$,
\[
\|Q(z)-Q(z')\|_{\mathcal H}
\le2m^2\sup_{u\in\overline c}\|Q(u)\|_{\mathcal H}
\max_{q\in[n]}\frac{|z_q-z'_q|}{u_{c,q}-\ell_{c,q}}.
\]
This follows from the cited derivative bound on the unit ball of
$(\mathbb R^n,\|\cdot\|_\infty)$ by affine rescaling and integration
along the segment from $z'$ to $z$.
\end{lemma}

\begin{lemma}[{\cite{tropp2021probability}}]
\label{lem:chevet}
Let $E,F$ be finite-dimensional real Hilbert spaces and let
$S\subset E$, $T\subset F$ be nonempty compact sets.
Let $G$, $G_E$, and $G_F$ be centered Gaussian families on
$E\otimes F$, $E$, and $F$, respectively, each with covariance
equal to the inner product of its indexing space. Then
\[
\mathbb E\sup_{s\in S,t\in T}|G(s\otimes t)|
\le\sup_{t\in T}\|t\|\,\mathbb E\sup_{s\in S}|G_E(s)|
+\sup_{s\in S}\|s\|\,\mathbb E\sup_{t\in T}|G_F(t)|.
\]
This is the symmetric-set form of Chevet's inequality.
For subsets of separable Hilbert spaces, the same bound holds for
separable versions when the right-hand side is finite: apply the
finite-dimensional result to finite subsets and then take increasing
suprema over countable dense subsets.
\end{lemma}

\begin{lemma}[{\cite[Theorem~4.4]{rudelsonVershynin2006}}]
\label{lem:fat-covering}
Let $|\mathcal X|=m\ge1$ and
$\varnothing\ne\mathcal G\subseteq[-M,M]^{\mathcal X}$, where $M>0$.
There are universal constants $C\ge e$ and $c\in(0,1/2)$ such that,
for $0<\varepsilon\le M$,
\[
\ln\mathcal N(\varepsilon,\mathcal G,\|\cdot\|_\infty)
\le C\operatorname{fat}_{c\varepsilon}(\mathcal G)
\ln^{3/2}\!\left[Cm\left(1+\frac M\varepsilon\right)\right].
\]
The covering centers are in $\mathcal G$. This is the specialization
of the cited theorem with its logarithmic exponent parameter set to $1/2$.
\end{lemma}

Gaussian entropy and comparison control the envelope in
\eqref{eq:vs-gaussian-envelope} and the Gaussian width in
Lemma~\ref{lem:upper-remainder-cover}.

\begin{lemma}[\cite{vershynin2018high}]
\label{lem:gaussian-entropy-comparison}
Let $(X_t)_{t\in\mathcal T}$ be a centered separable Gaussian process
with canonical pseudometric
$d_X(s,t):=(\mathbb E|X_s-X_t|^2)^{1/2}$.
If $(\mathcal T,d_X)$ is totally bounded, then, for $t_0\in\mathcal T$,
\[
\mathbb E\sup_{t\in\mathcal T}|X_t-X_{t_0}|
\le C\int_0^{\operatorname{diam}(\mathcal T,d_X)}
\sqrt{\ln\!\bigl(2\mathcal N(\varepsilon,\mathcal T,d_X)\bigr)}\,d\varepsilon,
\]
where $C$ is universal.
For a version with $d_X$-continuous paths, the one-sided bound can be taken as
\begin{equation}
\mathbb E\sup_{t\in\mathcal T}X_t
\le12\int_0^{\operatorname{diam}(\mathcal T,d_X)}
\sqrt{\ln\mathcal N(\varepsilon,\mathcal T,d_X)}\,d\varepsilon.
\label{eq:dudley-explicit-constant}
\end{equation}
Indeed, \citet[p.~4]{bartlett2013lecture14} gives the smaller constant $8\sqrt2$.
If $(Y_t)_{t\in\mathcal T}$ is another centered separable Gaussian
process, both absolute suprema are integrable, and
$d_X(s,t)\le d_Y(s,t)$ for all $s,t$, then
$\mathbb E\sup_tX_t\le\mathbb E\sup_tY_t$.
\end{lemma}

The regression estimates use the following two bounds with
a fixed deterministic regularization parameter $\lambda>0$. For $A\succ0$, write
$\|x\|_A:=\sqrt{x^\top Ax}$.

\begin{lemma}[{\cite{abbasiYadkori2011}}]
\label{lem:self-normalized}
Let $(\mathcal J_i)_{i\ge0}$ be a filtration.
Suppose that $x_i\in\mathbb R^D$ is $\mathcal J_{i-1}$-measurable
and $\xi_i$ is $\mathcal J_i$-measurable, with
\[
\mathbb E[\exp(\theta\xi_i)\mid\mathcal J_{i-1}]
\le\exp(\theta^2\sigma_{\mathrm{sg}}^2/2)
\ \text{for every }\theta\in\mathbb R,
\]
where $\sigma_{\mathrm{sg}}>0$ is deterministic. Define $G_n:=\lambda I_D+\sum_{i=1}^nx_ix_i^\top,\ 
M_n:=\sum_{i=1}^nx_i\xi_i$.
For each fixed $n\ge0$, we have
\begin{equation}
\mathbb E\left[
\sqrt{\frac{\lambda^D}{\det G_n}}
\exp\left\{\frac{\|M_n\|_{G_n^{-1}}^2}{2\sigma_{\mathrm{sg}}^2}\right\}
\right]\le1.
\label{eq:self-normalized-exponential-moment}
\end{equation}
For every $\delta\in(0,1)$, with probability at least $1-\delta$,
simultaneously for all $n\ge0$, we have
\[
\|M_n\|_{G_n^{-1}}
\le\sigma_{\mathrm{sg}}\sqrt{\ln\frac{\det G_n}{\lambda^D}+2\ln(1/\delta)}.
\]
\end{lemma}

\begin{lemma}[\cite{abbasiYadkori2011}]
\label{lem:elliptical-potential}
Let $\|x_i\|_2\le L$ and
$G_i:=\lambda I_D+\sum_{j=1}^ix_jx_j^\top$. Then
\[
\sum_{i=1}^n\min\{1,\|x_i\|_{G_{i-1}^{-1}}^2\}
\le2\ln\frac{\det G_n}{\lambda^D}
\le2D\ln\!\left(1+\frac{nL^2}{D\lambda}\right).
\]
If $\lambda\ge L^2$, every $\|x_i\|_{G_{i-1}^{-1}}^2\le1$,
so the minimum can be omitted. In particular, $\lambda=L^2>0$
gives the bound $2D\ln(1+n/D)$.
\end{lemma}

We record the deterministic consequences used in both continuation
and epoch regression. With the definitions in
\eqref{eq:local-ridge-estimate}--\eqref{eq:profiled-features}, for $z\in c$, we have
\begin{equation}
\begin{aligned}
\sum_{i\in\mathcal T_t^h(c)}
[\phi_c(z)^\top G_t^h(c)^{-1}\phi_c(z_i^h)]^2
&=\phi_c(z)^\top G_t^h(c)^{-1}
(G_t^h(c)-\lambda I_D)G_t^h(c)^{-1}\phi_c(z)\\
&\le\phi_c(z)^\top G_t^h(c)^{-1}\phi_c(z)=w_t^h(z)^2.
\end{aligned}
\label{eq:regression-coefficient-square}
\end{equation}
The equality substitutes the sum defining $G_t^h(c)$; the inequality
uses $G_t^h(c)-\lambda I_D\preceq G_t^h(c)$.
By the Cauchy--Schwarz inequality, we then have
\begin{equation}
\sum_{i\in\mathcal T_t^h(c)}
|\phi_c(z)^\top G_t^h(c)^{-1}\phi_c(z_i^h)|
\le\sqrt{n_t^h(c)}
\left(\sum_{i\in\mathcal T_t^h(c)}
[\phi_c(z)^\top G_t^h(c)^{-1}\phi_c(z_i^h)]^2\right)^{1/2}
\le\sqrt{n_t^h(c)}\,w_t^h(z).
\label{eq:regression-coefficient-sum}
\end{equation}

For a fixed stage, consider a sequence of at most $N$ observations
occupying at most $M$ cells, and form each $G_t^h(c)$ using only
earlier observations from this sequence. All sums below run over
this sequence. Since $\lambda=B_{\phi}^2$ and $\|\phi_c\|_2\le B_{\phi}$,
we have $w_t^h(z)^2\le\|\phi_c(z)\|_2^2/\lambda\le1$.
Applying Lemma~\ref{lem:elliptical-potential} with feature radius $B_{\phi}$
and $\lambda=B_{\phi}^2$, we obtain
\begin{equation}
\sum_{t:z_t^h\in c}w_t^h(z_t^h)^2
\le2D\ln\left(1+\frac{|\{t:z_t^h\in c\}|}{D}\right)
\le2D\ln(1+N/D).
\label{eq:cell-width-square-sum}
\end{equation}
The last inequality uses at most $N$ visits to the cell.
Applying the Cauchy--Schwarz inequality and
\eqref{eq:cell-width-square-sum} to the full sequence gives
\begin{equation}
\sum_tw_t^h(z_t^h)
\le\sqrt{N\sum_c\sum_{t:z_t^h\in c}w_t^h(z_t^h)^2}
\le\sqrt{2DMN\ln(1+N/D)}.
\label{eq:regression-width-sum}
\end{equation}
Each historical count $n_t^h(c)$ is at most the final number of
visits to $c$. Applying the Cauchy--Schwarz inequality in each cell
and then \eqref{eq:cell-width-square-sum}, we obtain
\begin{equation}
\begin{aligned}
\sum_t\sqrt{n_t^h(c(z_t^h))}\,w_t^h(z_t^h)
&\le\sum_c
\left(\sum_{t:z_t^h\in c}n_t^h(c)\right)^{1/2}
\left(\sum_{t:z_t^h\in c}w_t^h(z_t^h)^2\right)^{1/2}\\
&\le\sum_c|\{t:z_t^h\in c\}|\sqrt{2D\ln(1+N/D)}
\le N\sqrt{2D\ln(1+N/D)}.
\end{aligned}
\label{eq:regression-history-width-sum}
\end{equation}
These estimates hold pathwise, including for predictable sample
assignments when all histories and sums are restricted to the assigned
observations.

The following entropy bound applies to normed-space-valued
processes and does not require independent increments.

\begin{lemma}[{\cite[Theorem~3.2 and (2.3)]{dirksen2015}}]
\label{lem:subgaussian-entropy-tail}
Let $(\mathcal T,d)$ be totally bounded and let $(X_t)_{t\in\mathcal T}$
be a separable process in a normed space, with almost surely
$d$-continuous paths. Suppose that, for all $s,t\in\mathcal T$ and $r>0$,
\[
\mathbb P\bigl(\|X_s-X_t\|>r d(s,t)\bigr)\le2e^{-r^2/2}.
\]
Put $\Delta:=\operatorname{diam}(\mathcal T,d)$. There is a universal
constant $C$ such that, for every $t_0\in\mathcal T$ and $x\ge1$,
with probability at least $1-e^{-x}$,
\[
\sup_{t\in\mathcal T}\|X_t-X_{t_0}\|
\le C\left[\int_0^\Delta
\sqrt{\ln\!\bigl(2\mathcal N(\varepsilon,\mathcal T,d)\bigr)}\,d\varepsilon
+\Delta\sqrt x\right].
\]
\end{lemma}

We also use maximal inequalities for the feedback and occupancy sums.

\begin{lemma}[\cite{howard2020time}]
\label{lem:martingale-tools}
Let $(\mathcal J_i)_{i\ge0}$ be a filtration, and fix $x>0$ and $N\in\mathbb N$. Then, we have
\begin{enumerate}
\item[(i)] If $(L_i)_{i\ge0}$ is a nonnegative supermartingale
with $\mathbb E L_0\le1$, then
$\mathbb P(\sup_{i\ge0}L_i\ge e^x)\le e^{-x}$.

\item[(ii)] If $Z_i$ is $\mathcal J_i$-measurable and $0\le Z_i\le B$
for deterministic $B>0$, put $\mu_i:=\mathbb E[Z_i\mid\mathcal J_{i-1}]$.
Then $\mathbb E[e^{\theta(Z_i-\mu_i)}\mid\mathcal J_{i-1}]
\le e^{\theta^2B^2/8}$ for every $\theta\in\mathbb R$, and
\[
\begin{aligned}
\mathbb P\!\left(\max_{k\le N}\sum_{i=1}^k(Z_i-\mu_i)
\ge B\sqrt{Nx/2}\right)&\le e^{-x},\ 
\mathbb P\!\left(\exists k\ge1:
\sum_{i=1}^k\mu_i>2\sum_{i=1}^kZ_i+2Bx\right)&\le e^{-x}.
\end{aligned}
\]
\end{enumerate}
\end{lemma}


The exploration distribution uses a G-optimal design on the span
of the action features.

\begin{lemma}[\cite{kieferWolfowitz1960}]
\label{lem:g-optimal-design}
Let $U$ be a nonempty compact metric space and
$\varphi:U\to\mathbb R^p$ continuous.
Put $E:=\operatorname{span}\{\varphi(a):a\in U\}$ and $q:=\dim E$.
There is a finitely supported probability distribution $\pi$ on $U$
such that, for $\Sigma_\pi:=\int_U\varphi(a)\varphi(a)^\top\,d\pi(a)$,
\[
\operatorname{range}(\Sigma_\pi)=E,\ 
\max_{a\in U}\varphi(a)^\top\Sigma_\pi^\dagger\varphi(a)=q,
\]
where $\Sigma_\pi^\dagger$ is the Moore-Penrose inverse.
\end{lemma}

Finally, change of measure is used in the second version-space test
defined by \eqref{eq:innovation-test} and \eqref{eq:version-space},
and Pinsker's inequality is used in the lower bound.
For probability measures $P_0,P_1$, define
$\operatorname{KL}(P_0\|P_1):=\int\ln(dP_0/dP_1)\ dP_0$
when $P_0\ll P_1$, and $+\infty$ otherwise.

\begin{lemma}[\cite{hairer2025spdes}]
\label{lem:isonormal-shift}
Let $W$ be a centered Gaussian family indexed by a separable real
Hilbert space $\mathcal H$, with
$\mathbb E[W(v)W(w)]=\langle v,w\rangle_{\mathcal H}$.
Write $\mathsf P$ for its law on the sigma-field generated by its
evaluations, and let $\mathsf P_h$ be the law of
$(W(v)+\langle h,v\rangle_{\mathcal H})_{v\in\mathcal H}$ for $h\in\mathcal H$.
Then $\mathsf P_h$ and $\mathsf P$ are equivalent, and
\[
\frac{d\mathsf P_h}{d\mathsf P}(W)
=\exp\!\left\{W(h)-\frac12\|h\|_{\mathcal H}^2\right\},\ 
\operatorname{KL}(\mathsf P_h\|\mathsf P)=\frac12\|h\|_{\mathcal H}^2.
\]
\end{lemma}

\begin{lemma}[\cite{anantharam2017variational}]
\label{lem:gibbs-variational}
For a probability measure $\mathsf P$, a bounded measurable
real-valued function $F$, and any probability measure
$\mathsf Q\ll\mathsf P$, we have
\[
\mathbb E_{\mathsf Q}F
\le\ln\mathbb E_{\mathsf P}e^F
+\operatorname{KL}(\mathsf Q\|\mathsf P).
\]
\end{lemma}

\begin{lemma}[\cite{tsybakov2009}]
\label{lem:pinsker}
For probability measures $P_0,P_1$ on the same measurable space,
\[
\operatorname{TV}(P_0,P_1)
\le\sqrt{\frac12\operatorname{KL}(P_0\|P_1)},
\]
where $\operatorname{TV}(P_0,P_1):=\sup_E|P_0(E)-P_1(E)|$.
\end{lemma}

The following cube partition supplies both the small-ball intersection
bound for level selection and the separation used in action approximation.

\begin{lemma}[{\cite{vanderWoude2022rounding}}]
\label{lem:colored-cube-partition}
For $n\in\mathbb N$, define the matrix $A_n\in\mathbb R^{n\times n}$ by
\[
(A_n)_{ij}:=\mathbf1\{i=j\}
+\frac{n-j+1}{n}\mathbf1\{i<j\},\ i,j\in[n].
\]
The half-open cubes
$Q_m:=A_nm+[0,1)^n$, $m\in\mathbb Z^n$, form a partition
$\mathscr Q^{(n)}$ of $\mathbb R^n$.
Color $Q_m$ by $\sum_{i=1}^ni m_i$ modulo $n+1$.
Distinct cubes of the same color satisfy
\[
\inf_{x\in Q_m,\,y\in Q_{m'}}\|x-y\|_\infty\ge\frac1n.
\]
For every $x\in\mathbb R^n$, the closed ball
$\{y:\|y-x\|_\infty\le1/(2n)\}$ intersects at most $n+1$ cubes.
\end{lemma}

We use the dual Sudakov inequality to cover a Euclidean ball in
the norm induced by the remainder features.

\begin{lemma}[{\cite{chafai2012interactions}}]
\label{lem:dual-sudakov}
Let $E$ be a finite-dimensional real Hilbert space,
$B_E:=\{x\in E:\|x\|_2\le1\}$, and $\|\cdot\|_*$ any norm on $E$.
If $G$ is a standard Gaussian vector in $E$, then, for $R,\varepsilon>0$,
\[
\ln\mathcal N(\varepsilon,RB_E,\|\cdot\|_*)
\le\frac{CR^2}{\varepsilon^2}\bigl(\mathbb E\|G\|_*\bigr)^2,
\]
where $C$ is universal and the covering centers lie in $RB_E$.
The internal-center convention changes only $C$: replace each
covering ball that meets $RB_E$ by a point in that intersection
and double its radius.
\end{lemma}

\begin{lemma}[{\cite{golinskii2014schoenberg}}]
\label{lem:separated-exponential-quadratic}
Let $x_1,\ldots,x_N\in\mathbb R^n$ satisfy
$\|x_i-x_j\|_2\ge32n$ for $i\ne j$.
Then, for all $c_1,\ldots,c_N\in\mathbb C$,
\[
\sum_{i,j=1}^N|c_ic_j|e^{-\|x_i-x_j\|_2}
\le2\sum_{i=1}^N|c_i|^2.
\]
\end{lemma}

\begin{lemma}[{\cite{carmeli2010vector}}]
\label{lem:rkhs-frequency-representation}
Let $k_0$ be the normalized Mat\'ern kernel restricted to
$\mathcal Z\subseteq\mathbb R^d$, with real RKHS $\mathcal H_0$.
Let $Z,W,\omega=\kappa Z/\sqrt W$ be as in
Lemma~\ref{lem:matern-frequency-representation}.
For every $g\in\mathcal H_0$, there exists a square-integrable
complex coefficient $\theta_g=\theta_g(Z,W)$ of minimum $L^2$ norm
among all coefficients representing $g$, and
\[
g(z)=\mathbb E_{Z,W}\!\left[
\theta_g(Z,W)e^{\mathrm i\omega^\top z}\right],\
z\in\mathcal Z,\ 
\|\theta_g\|_{L^2}=\|g\|_{\mathcal H_0}.
\]
This minimum-norm coefficient is unique up to almost-sure equality,
and the map $g\mapsto\theta_g$ is real-linear.
The expectation and $L^2$ norm are taken under the joint law of $(Z,W)$.
\end{lemma}
\section{Proofs in Section~\ref{sec:setup}}
\label{app:section-3-proofs}

Throughout this section, $\mu$ is Lebesgue measure on
$\cS=[0,1]^{d_S}$, so $\mu(\cS)=1$.
For a single episode, write $\mathcal I^h$ for the information
available after choosing $a^h$ and before observing the current
reward and, when $h<H$, the next state.

\begin{proof}[Proof of Lemma~\ref{lem:optimal-values}]
Starting from $V_*^{H+1}=0$, suppose that $V_*^{h+1}$ is measurable
and $0\le V_*^{h+1}\le H-h$. The Bellman backup is
\[
Q_*^h(s,a)=r^h(s,a)+\int_{\cS}V_*^{h+1}(s')P^h(s'\mid s,a)\,\mu(ds'),\
0\le Q_*^h\le1+(H-h)=H-h+1,
\]
with the integral omitted at $h=H$.
Finite kernel expansions are measurable and converge pointwise under
RKHS convergence by Lemma~\ref{lem:point-evaluation}, so $r^h$ is
measurable. Joint measurability and normalization of $P^h$ give
measurability of the backup and the displayed bound.

The reward is continuous in the action by
Assumption~\ref{ass:action-regularity}. For $h<H$, we apply
Lemma~\ref{lem:point-evaluation} and
Assumption~\ref{ass:kernel-transition} to obtain
\begin{align*}
\left|(P^hV_*^{h+1})(s,a)-(P^hV_*^{h+1})(s,b)\right|
&\le\int_{\cS}|V_*^{h+1}(s')|
\left|P^h(s'\mid s,a)-P^h(s'\mid s,b)\right|\,\mu(ds')\\
&\le B_p\|V_*^{h+1}\|_{L^1(\mu)}
\|k_{(s,a)}-k_{(s,b)}\|_{\cH_k}.
\end{align*}
The second inequality uses
$|P^h(s'\mid s,a)-P^h(s'\mid s,b)|
\le\|P^h(s'\mid\cdot)\|_{\cH_k}
\|k_{(s,a)}-k_{(s,b)}\|_{\cH_k}$ and
$\|P^h(s'\mid\cdot)\|_{\cH_k}\le B_p$.
By Assumption~\ref{ass:action-regularity}, we have
\[
\|k_{(s,a)}-k_{(s,b)}\|_{\cH_k}^2
=k((s,a),(s,a))-2k((s,a),(s,b))+k((s,b),(s,b))
\longrightarrow0\ \text{as }b\to a.
\]
Thus $Q_*^h(s,\cdot)$ is continuous, including at $h=H$.
Since $\cS$ is standard Borel and $\cA$ is compact,
Lemma~\ref{lem:measurable-maximum} supplies a measurable selector
$\pi_*^h$ and a measurable maximum satisfying
\[
V_*^h(s)=Q_*^h(s,\pi_*^h(s))=\max_{a\in\cA}Q_*^h(s,a),\
0\le V_*^h\le H-h+1.
\]
This closes the backward induction. Substituting these selectors
into the policy recursion, starting at $V^{H+1}(\cdot;\pi_*)=0$, gives
\[
V^h(s;\pi_*)=r^h(s,\pi_*^h(s))+(P^hV_*^{h+1})(s,\pi_*^h(s))
=Q_*^h(s,\pi_*^h(s))=V_*^h(s).
\]
For any history-dependent policy $\pi\in\Pi$, the observation law gives
$r^h(s^h,a^h)+\mathbb E_s^{M,\pi}[V_*^{h+1}(s^{h+1})\mid\mathcal I^h]
=Q_*^h(s^h,a^h)\le V_*^h(s^h)$ for $h<H$,
and $r^H(s^H,a^H)\le V_*^H(s^H)$. Taking expectations and summing,
we obtain
\[
V^1(s;\pi)=\sum_{h=1}^H\mathbb E_s^{M,\pi}[r^h(s^h,a^h)]
\le\sum_{h=1}^H\mathbb E_s^{M,\pi}[V_*^h(s^h)]
-\sum_{h=1}^{H-1}\mathbb E_s^{M,\pi}[V_*^{h+1}(s^{h+1})]
=V_*^1(s).
\]
The inequality uses the tower property, and the last equality cancels
stages $2,\ldots,H$. Equality holds for the deterministic Markov
policy $\pi_*$, proving optimality and $0\le\reg(T)\le HT$.
\end{proof}

\begin{proof}[Proof of Lemma~\ref{lem:bellman-completeness}]
For $h<H$ and bounded measurable $V$, the reproducing property gives
\[
(P^hV)(z)=\int_{\cS}V(s')P^h(s'\mid z)\,\mu(ds')
=\int_{\cS}V(s')\langle P^h(s'\mid\cdot),k_z\rangle_{\cH_k}\,\mu(ds').
\]
To prove $P^hV\in\cH_k$, we extend the integral from $k_z$ to
arbitrary $g\in\cH_k$ and show that it defines a bounded linear
functional. Its Riesz representative will then equal $P^hV$.
First we verify measurability: choose finite kernel expansions
$g_n=\sum_{i=1}^{m_n}b_{n,i}k_{z_{n,i}}\to g$ in $\cH_k$
by Lemma~\ref{lem:point-evaluation}. Each function
$\langle P^h(s'\mid\cdot),g_n\rangle_{\cH_k}
=\sum_{i=1}^{m_n}b_{n,i}P^h(s'\mid z_{n,i})$ is measurable in $s'$.
The Cauchy--Schwarz inequality and
Assumption~\ref{ass:kernel-transition} give
\[
|\langle P^h(s'\mid\cdot),g_n-g\rangle_{\cH_k}|
\le\|P^h(s'\mid\cdot)\|_{\cH_k}\|g_n-g\|_{\cH_k}
\le B_p\|g_n-g\|_{\cH_k}\longrightarrow0,
\]
so the limiting pairing is measurable as well.
Define $L_V(g):=\int_{\cS}V(s')
\langle P^h(s'\mid\cdot),g\rangle_{\cH_k}\,\mu(ds')$.
Using the Cauchy--Schwarz inequality and
Assumption~\ref{ass:kernel-transition}, we obtain
\begin{equation}
|L_V(g)|\le\int_{\cS}|V(s')|
\|P^h(s'\mid\cdot)\|_{\cH_k}\|g\|_{\cH_k}\,\mu(ds')
\le B_p\|V\|_{L^1(\mu)}\|g\|_{\cH_k}.
\label{eq:transition-functional-bound}
\end{equation}
Thus $L_V$ is a bounded linear functional. Its Riesz representative
$F\in\cH_k$ satisfies
\[
F(z)=\langle F,k_z\rangle_{\cH_k}=L_V(k_z)=(P^hV)(z).
\]
By \eqref{eq:transition-functional-bound} and $\mu(\cS)=1$, we obtain
\[
\|F\|_{\cH_k}=\|L_V\|\le B_p\|V\|_{L^1(\mu)}\le B_p\|V\|_\infty,
\]
which proves \eqref{eq:bellman-completeness}.
Adding the reward and applying the triangle inequality, we obtain
\[
\|\mathcal B^hV\|_{\cH_k}
\le\|r^h\|_{\cH_k}+\|P^hV\|_{\cH_k}
\le B_r+B_p\|V\|_{L^1(\mu)}\le B_r+B_p\|V\|_\infty.
\]
At $h=H$, we have $\mathcal B^HV=r^H$, so
\eqref{eq:bellman-backup-norm} holds at every stage.

For $h<H$, transition normalization gives
$P^h\mathbf1_{\cS}=\mathbf1_{\cZ}$. Substituting this continuation
into \eqref{eq:bellman-completeness} yields $\mathbf1_{\cZ}\in\cH_k$ and
\[
\|\mathbf1_{\cZ}\|_{\cH_k}=\|P^h\mathbf1_{\cS}\|_{\cH_k}
\le B_p\int_{\cS}|\mathbf1_{\cS}(s')|\,\mu(ds')=B_p\mu(\cS)=B_p.
\]
By Lemma~\ref{lem:point-evaluation} and $k(z,z)\le1$, we have
\[
1=\mathbf1_{\cZ}(z)\le\|\mathbf1_{\cZ}\|_{\cH_k}\sqrt{k(z,z)}
\le\|\mathbf1_{\cZ}\|_{\cH_k},
\]
which proves $1\le\|\mathbf1_{\cZ}\|_{\cH_k}\le B_p$.
For a measurable Markov policy $\pi$, backward integration of its
recursion gives measurable $V^{h+1}(\cdot;\pi)\in[0,H-h]$;
Lemma~\ref{lem:optimal-values} gives the same range for $V_*^{h+1}$.
Applying \eqref{eq:bellman-backup-norm} to these two continuations, we obtain
\[
\max\{\|Q_*^h\|_{\cH_k},\|Q^h(\cdot;\pi)\|_{\cH_k}\}
\le B_r+B_p\max\{\|V_*^{h+1}\|_\infty,\|V^{h+1}(\cdot;\pi)\|_\infty\}
\le B_r+B_p(H-h).
\]
This also establishes the RKHS membership in \eqref{eq:q-rkhs-bound}.
\end{proof}

\begin{proof}[Proof of Lemma~\ref{lem:state-occupancy-bound}]
By Lemma~\ref{lem:point-evaluation},
Assumption~\ref{ass:kernel-transition}, and $k(z,z)\le1$, we obtain
\begin{equation}
0\le P^h(s'\mid z)
\le\|P^h(s'\mid\cdot)\|_{\cH_k}\sqrt{k(z,z)}\le B_p,\
h<H,\ z\in\cZ,\ s'\in\cS.
\label{eq:transition-density-upper-bound}
\end{equation}
For an initial state $s$, a history-dependent policy $\pi$, and $h\ge2$,
put $p_h(s'):=\mathbb E_s^{M,\pi}[P^{h-1}(s'\mid s^{h-1},a^{h-1})]$.
Joint measurability and \eqref{eq:transition-density-upper-bound}
make $p_h$ measurable with $0\le p_h\le B_p$.
The observation law and Tonelli's theorem give, for measurable $E\subseteq\cS$,
\[
\mathbb P_s^{M,\pi}(s^h\in E)
=\mathbb E_s^{M,\pi}\left[\int_E P^{h-1}(s'\mid s^{h-1},a^{h-1})\,\mu(ds')\right]
=\int_Ep_h(s')\,\mu(ds')\le B_p\mu(E).
\]
Taking $E=\cS$ shows that $\int p_h\,d\mu=1$, so $p_h$ is the required density.

For the conditional assertion, $\mathcal I_t^1\subseteq\mathcal I_t^{h-1}$
and the same observation law yield
\[
\mathbb P(s_t^h\in E\mid\mathcal I_t^1)
=\mathbb E\left[\mathbb P(s_t^h\in E\mid\mathcal I_t^{h-1})\mid\mathcal I_t^1\right]
=\mathbb E\left[\int_E P^{h-1}(s'\mid z_t^{h-1})\,\mu(ds')\mid\mathcal I_t^1\right]
\le B_p\mu(E).
\]
The first equality is the tower property; the inequality uses
\eqref{eq:transition-density-upper-bound} inside the conditional expectation.
Both conclusions allow history-dependent policies.
\end{proof}

\section{Proofs in Sections~\ref{sec:skovi}-\ref{sec:sovi-first-selection}}
\label{app:section-4-proofs}

Write $\mathcal I_i^h$ for the information available after choosing
$a_i^h$ and before observing its reward and next state.
Sample assignments are measurable with respect to this information.
For every fixed $V\in\mathcal V_h$, we have
$\mathbb E[Y_i^h(V)\mid\mathcal I_i^h]
=(\mathcal B^hV)(z_i^h)$ and $0\le Y_i^h(V)\le H-h+1$.
At $h=H$, we have $V=0$ and $Y_i^H(0)=R_i^H$.
Throughout, $C$ denotes a universal constant, while $C_\nu$ and
$C_{\nu,\kappa}$ depend only on their indicated parameters.

\subsection{Proofs for Section~\ref{sec:skovi}}
\label{app:local-approximation-regression}

The simultaneous regression bounds in Lemma~\ref{lem:local-ridge}
are used in the version-space construction and first-action selection.

\begin{proof}[Proof of Lemma~\ref{lem:feature-remainder}]
Fix a cell $c$. By \eqref{eq:feature-decomposition}, the error to
be bounded is $\Psi_{\mathcal C}=\Phi-\Pi_c\Phi$ on $c$, where
Lemma~\ref{lem:local-legendre-projection} gives
\[
(\Pi_c\Phi)(z)=\frac1{|\overline c|}\int_{\overline c}K_c(z,u)\Phi(u)\,du,\
K_c(z,u)=\sum_{|\mathbf b|\le\rho}p_{c,\mathbf b}(z)p_{c,\mathbf b}(u).
\]
We first express its squared RKHS norm as an integral of scalar
approximation errors. By Lemma~\ref{lem:matern-frequency-representation},
take independent $Z\sim\mathcal N(0,I_d)$ and $W\sim\chi^2_{2\nu}$,
and put $\omega=\kappa Z/\sqrt W$ and
$\varphi_\omega(z):=e^{\mathrm i\omega^\top z}$.
Here expectation means integration over this fixed frequency law.
Expanding the squared norm and substituting
\eqref{eq:matern-frequency-representation}, we obtain
\begin{equation}
\begin{aligned}
&\|\Phi(z)-(\Pi_c\Phi)(z)\|_{\mathcal H_0}^2\\
&=k_0(z,z)-\frac2{|\overline c|}
\int_{\overline c}K_c(z,u)k_0(z,u)\,du
+\frac1{|\overline c|^2}
\int_{\overline c}\int_{\overline c}
K_c(z,u)K_c(z,v)k_0(u,v)\,du\,dv\\
&=\mathbb E\left|\varphi_\omega(z)-(\Pi_c\varphi_\omega)(z)\right|^2.
\end{aligned}
\label{eq:appendix-spectral-projection-identity}
\end{equation}
The last equality follows by expanding the squared modulus;
$|\varphi_\omega|=1$ and \eqref{eq:legendre-kernel-bound}
justify interchanging the integrals and expectation.

Thus we can bound the original error by approximating the scalar
function $z\mapsto e^{\mathrm i\omega^\top z}$.
Use the norm $\|F(z)\|_{L^2}:=(\mathbb E|F(z)|^2)^{1/2}$.
The map $z\mapsto\varphi_\omega(z)$ is continuous in this norm
by dominated convergence. Using
\eqref{eq:appendix-spectral-projection-identity} and applying
\eqref{eq:legendre-projection-error} in this Hilbert space, we obtain the following bound
for every polynomial $P_c$ of degree at most $\rho$ with $L^2$ coefficients:
\begin{equation}
\sup_{z\in\overline c}\|\Phi(z)-(\Pi_c\Phi)(z)\|_{\mathcal H_0}
\le(1+B_{\phi})\sup_{z\in\overline c}
\left(\mathbb E|\varphi_\omega(z)-P_c(z)|^2\right)^{1/2}.
\label{eq:appendix-projection-comparison}
\end{equation}
It remains to construct $P_c$ with mean-square error of order
$t^{2\nu}$, where $t:=\kappa\sqrt d\,a$.
For $t\ge1$, $P_c=0$ has error one. Suppose $0<t<1$ and let
$z_c$ be the center of $\overline c$ and $k:=\rho+1>\nu$.

An unrestricted Taylor bound would involve $\mathbb EW^{-k}=\infty$,
since $\int_0^1w^{\nu-k-1}\,dw=\infty$.
We therefore use a cutoff ${w_0}>0$,
chosen below, and define
\[
P_{c,w_0}(z):=\mathbf1\{W\ge{w_0}\}e^{\mathrm i\omega^\top z_c}
\sum_{j=0}^{\rho}\frac{[\mathrm i\omega^\top(z-z_c)]^j}{j!}.
\]
This polynomial has degree at most $\rho$ in $z$. Its coefficients belong
to $L^2$, since for $0\le j\le\rho$,
\[
\mathbb E[\mathbf1\{W\ge{w_0}\}\|\omega\|_2^{2j}]
\le\kappa^{2j}{w_0}^{-j}\mathbb E\|Z\|_2^{2j}<\infty.
\]
On the event $\{W<{w_0}\}$, the polynomial is zero and $|\varphi_\omega|=1$.
The contribution of this event is therefore bounded by
\begin{equation}
\mathbb P(W<{w_0})
=\frac1{2^\nu\Gamma(\nu)}\int_0^{w_0} w^{\nu-1}e^{-w/2}\,dw
\le\frac{{w_0}^\nu}{2^\nu\Gamma(\nu+1)}.
\label{eq:appendix-cutoff-probability}
\end{equation}
The inequality uses $e^{-w/2}\le1$ and integrates $w^{\nu-1}$.
For $\cbr{W\ge {w_0}}$, the same density gives
\begin{equation}
\mathbb E[W^{-k}\mathbf1\{W\ge{w_0}\}]
\le\frac1{2^\nu\Gamma(\nu)}\int_{w_0}^\infty w^{\nu-k-1}\,dw
=\frac{{w_0}^{\nu-k}}{2^\nu\Gamma(\nu)(k-\nu)}.
\label{eq:appendix-cutoff-inverse-moment}
\end{equation}
Here we again use $e^{-w/2}\le1$; the integral is finite because
$k>\nu$.

Taylor's formula gives
$|e^{\mathrm ix}-\sum_{j=0}^{\rho}(\mathrm ix)^j/j!|\le|x|^k/k!$
for real $x$.
Since $Z^\top v\sim\mathcal N(0,\|v\|_2^2)$,
we have $\mathbb E|Z^\top v|^{2k}=(2k-1)!!\|v\|_2^{2k}$.
Using these identities and independence, we obtain
\begin{equation}
\begin{aligned}
\mathbb E\left[|\varphi_\omega(z)-P_{c,w_0}(z)|^2
\mathbf1\{W\ge{w_0}\}\right]
&\le\frac{\kappa^{2k}}{(k!)^2}
\mathbb E|Z^\top(z-z_c)|^{2k}
\mathbb E[W^{-k}\mathbf1\{W\ge{w_0}\}]\\
&\le\frac{(2k-1)!!\,t^{2k}{w_0}^{\nu-k}}
{(k!)^2 2^\nu\Gamma(\nu)(k-\nu)}.
\end{aligned}
\label{eq:appendix-retained-taylor-error}
\end{equation}
The last inequality uses \eqref{eq:appendix-cutoff-inverse-moment} and
$\kappa\|z-z_c\|_2\le\kappa\sqrt d\,a=t$.
Combining \eqref{eq:appendix-cutoff-probability} and
\eqref{eq:appendix-retained-taylor-error} yields
\[
\sup_{z\in\overline c}\mathbb E
|\varphi_\omega(z)-P_{c,w_0}(z)|^2
\le C_\nu\bigl({w_0}^\nu+t^{2k}{w_0}^{\nu-k}\bigr).
\]
With ${w_0}=t^2$, the two contributions coincide: ${w_0}^\nu=t^{2k}{w_0}^{\nu-k}=t^{2\nu}$.
Substituting $P_c=P_{c,t^2}$ into
\eqref{eq:appendix-projection-comparison}, and using $P_c=0$
when $t\ge1$, we conclude that
\begin{equation}
\sup_{z\in\overline c}\|\Phi(z)-(\Pi_c\Phi)(z)\|_{\mathcal H_0}
\le C_\nu(1+B_{\phi})\min\{1,(\kappa\sqrt d\,a)^\nu\}.
\label{eq:appendix-rectangle-remainder}
\end{equation}
On $c$, this is the bound for $\Psi_{\mathcal C}$.
The constant is independent of the cell, so taking the supremum
over $c\in\mathcal C$ proves the lemma.
\end{proof}

\begin{proof}[Proof of Lemma~\ref{lem:continuation-entropy}]
We put $B:=\Lambda(B_r+B_p H)$. If $h=H$ or $B=0$, then $\mathcal V_h=\{0\}$ and the integral is $H\sqrt{\ln2}$. Fix $h<H$ and $B>0$. Each continuation $V\in\mathcal V_h$ is either zero or the induced state-value function $V(s)=\max_{a\in\cA}f^{h+1}(s,a)$, $s\in\cS$, for some candidate Q-function $f\in\mathcal F$.
By \eqref{eq:candidate-class} and \eqref{eq:reference-kernel-norm},
$g=f^{h+1}:\cS\times\cA\to\mathbb R$ belongs to
$\mathcal G_B:=\{g\in\mathcal H_0:\|g\|_{\mathcal H_0}\le B\}$;
zero is represented by $g=0$. For $g,g'\in\mathcal G_B$, we have
\[
\sup_{s\in\cS}\left|\max_{a\in\cA}g(s,a)
-\max_{a\in\cA}g'(s,a)\right|
\le\sup_{(s,a)\in\cZ}|g(s,a)-g'(s,a)|.
\]
The displayed inequality maps an $\varepsilon/2$-cover of $\mathcal G_B$
to balls covering $\mathcal V_h$. If $V,V'\in\mathcal V_h$ belong to
the ball centered at $s\mapsto\max_a g_j(s,a)$, then we have
\[
\|V-V'\|_\infty
\le\sup_s|V(s)-\max_a g_j(s,a)|
+\sup_s|\max_a g_j(s,a)-V'(s)|\le\varepsilon.
\]
Choosing one continuation from each nonempty ball therefore gives
an internal $\varepsilon$-cover, so
\begin{equation}
\mathcal N(\varepsilon,\mathcal V_h,\|\cdot\|_\infty)
\le\mathcal N(\varepsilon/2,\mathcal G_B,\|\cdot\|_\infty).
\label{eq:appendix-continuation-cover-transfer}
\end{equation}

To bound the right side of \eqref{eq:appendix-continuation-cover-transfer},
we will apply Lemma~\ref{lem:fat-covering} on a finite grid.
We need both a fat-shattering bound and control between grid points;
we use the following feature-increment estimate for both.
Set $\alpha:=\min\{\nu/2,1/2\}$.
Lemma~\ref{lem:matern-frequency-representation} gives
$\omega=\kappa Z/\sqrt W$ for independent
$Z\sim\mathcal N(0,I_d)$ and $W\sim\chi^2_{2\nu}$.
For $v\in\mathbb R^d$, the scalar $Z^\top v$ has the law of
$\|v\|_2Z_1$, so we obtain
\[
\mathbb E_\omega|\omega^\top v|^{2\alpha}
=\kappa^{2\alpha}\|v\|_2^{2\alpha}
\mathbb E|Z_1|^{2\alpha}\mathbb EW^{-\alpha}
=\frac{\kappa^{2\alpha}\Gamma(\alpha+1/2)\Gamma(\nu-\alpha)}
{\sqrt\pi\,\Gamma(\nu)}\|v\|_2^{2\alpha}.
\]
The first equality uses independence. In the second, integration of the two densities gives $\mathbb E|Z_1|^{2\alpha}=2^\alpha\Gamma(\alpha+1/2)/\sqrt\pi$ and $\mathbb EW^{-\alpha}=2^{-\alpha}\Gamma(\nu-\alpha)/\Gamma(\nu)$; the inverse moment is finite because $\alpha<\nu$.
Using $|e^{\mathrm ix}-1|\le\min\{2,|x|\}\le2|x|^\alpha$, we obtain
\begin{equation}
\|\Phi(z)-\Phi(z')\|_{\mathcal H_0}^2=2-2k_0(z,z')=\mathbb E|e^{\mathrm i\omega^\top(z-z')}-1|^2\le4\mathbb E|\omega^\top(z-z')|^{2\alpha}\le C_{\nu,\kappa}\|z-z'\|_2^{2\alpha}.
\label{eq:appendix-feature-holder}
\end{equation}
Here we used $k_0(z,z)=1$ and
\eqref{eq:matern-frequency-representation}; the last inequality
substitutes the moment just computed. Its constant is independent of $d$.

We use this continuity bound to control the number of shattered
points. Suppose $z_1,\ldots,z_n$ are $\varepsilon$-fat-shattered
by $\mathcal G_B$, with $n\ge2$ and $0<\varepsilon\le B$.
Lemma~\ref{lem:rkhs-fat-spectrum} gives
$K:=(k_0(z_i,z_j))_{i,j=1}^n\succ0$ and
$\operatorname{tr}K^{-1}\le B^2/\varepsilon^2$.
To bound $n$, we compare this with a lower bound on
$\operatorname{tr}K^{-1}$. We approximate the features in at most
$n/2$ dimensions, so the Gram quadratic form depends only on the
approximation error in at least $n/2$ orthogonal directions.
Choose $m:=\lfloor(n/2)^{1/d}\rfloor$ and let $\Pi$ be the
orthogonal projection onto the linear span of the features at the
centers of the $m^d$ equal cubes partitioning $\cZ$.
Thus we have $\dim\operatorname{range}(\Pi)\le m^d\le n/2$. For the center $\bar z$
of the cube containing $z$, by \eqref{eq:appendix-feature-holder}, we have
\begin{equation}
\|(I-\Pi)\Phi(z)\|_{\mathcal H_0}^2
\le\|\Phi(z)-\Phi(\bar z)\|_{\mathcal H_0}^2
\le C_{\nu,\kappa}\left(\frac{\sqrt d}{2m}\right)^{2\alpha}
\le C_{\nu,\kappa}d^\alpha n^{-2\alpha/d}.
\label{eq:appendix-grid-feature-error}
\end{equation}
The first inequality uses $\Pi\Phi(\bar z)=\Phi(\bar z)$;
the last uses $m\ge(n/2)^{1/d}/2$ and $2^{2\alpha/d}\le2$.

Let $A$ and $E$ be the Gram matrices of $\Pi\Phi(z_i)$ and
$(I-\Pi)\Phi(z_i)$, respectively. By orthogonality, we have $K=A+E$,
with $\operatorname{rank}A\le n/2$. By
\eqref{eq:appendix-grid-feature-error}, we obtain
\begin{equation}
\operatorname{tr}E
=\sum_{i=1}^n\|(I-\Pi)\Phi(z_i)\|_{\mathcal H_0}^2
\le C_{\nu,\kappa}n d^\alpha n^{-2\alpha/d}.
\label{eq:appendix-grid-error-trace}
\end{equation}
To combine this error with $\operatorname{tr}K^{-1}$, let $\Pi_A$
be the orthogonal projection onto the range of $A$. Applying
the Cauchy--Schwarz inequality to the matrix entries of
$K^{-1/2}$ and $(I_n-\Pi_A)K^{1/2}$, we obtain
\[
\begin{aligned}
\frac{n^2}{4}
&\le[\operatorname{tr}(I_n-\Pi_A)]^2
=\bigl[\operatorname{tr}\{K^{-1/2}(I_n-\Pi_A)K^{1/2}\}\bigr]^2\\
&\le\operatorname{tr}(K^{-1})\operatorname{tr}\{(I_n-\Pi_A)K\}
=\operatorname{tr}(K^{-1})\operatorname{tr}\{(I_n-\Pi_A)E\}\\
&\le\frac{B^2}{\varepsilon^2}\operatorname{tr}E
\le C_{\nu,\kappa}\frac{B^2}{\varepsilon^2}
nd^\alpha n^{-2\alpha/d}.
\end{aligned}
\]
We use the fact that $(I_n-\Pi_A)A=0$.
The next inequality uses $E\succeq0$, $0\preceq I_n-\Pi_A\preceq I_n$,
and Lemma~\ref{lem:rkhs-fat-spectrum}; the last uses
\eqref{eq:appendix-grid-error-trace}. Solving for $n$ gives
\begin{equation}
\operatorname{fat}_\varepsilon(\mathcal G_B)
\le\left(\frac{C_{\nu,\kappa}d^{\alpha/2}B}{\varepsilon}\right)^p,
\ p:=\frac{2d}{d+2\alpha}<2,\ 0<\varepsilon\le B.
\label{eq:appendix-rkhs-fat}
\end{equation}
This also covers $n=1$, since $B/\varepsilon\ge1$.

We now apply Lemma~\ref{lem:fat-covering} on a finite grid.
The reproducing property and \eqref{eq:appendix-feature-holder} give
\[
|g(z)-g(z')|
=|\langle g,\Phi(z)-\Phi(z')\rangle_{\mathcal H_0}|
\le B\|\Phi(z)-\Phi(z')\|_{\mathcal H_0}
\le C_{\nu,\kappa}B\|z-z'\|_2^\alpha.
\]
For $0<\varepsilon\le B$, divide each coordinate into $\lceil\sqrt d(4C_{\nu,\kappa}B/\varepsilon)^{1/\alpha}\rceil$ equal intervals and denote the grid of endpoints by $\mathcal D_\varepsilon$.
Every $g\in\mathcal G_B$ changes by at most $\varepsilon/4$
between a point and its nearest grid point. Moreover, we have
\begin{equation}
|\mathcal D_\varepsilon|
\le\left[2+\sqrt d
\left(\frac{4C_{\nu,\kappa}B}{\varepsilon}\right)^{1/\alpha}\right]^d,
\ \ln|\mathcal D_\varepsilon|
\le C_{\nu,\kappa}d\ln(C_{\nu,\kappa}dB/\varepsilon).
\label{eq:appendix-entropy-grid-size}
\end{equation}
For an internal $\varepsilon/2$-cover on the grid, its center $\bar g$
and a nearest grid point $\bar z$ satisfy
\[
|g(z)-\bar g(z)|
\le|g(z)-g(\bar z)|+|g(\bar z)-\bar g(\bar z)|
+|\bar g(\bar z)-\bar g(z)|
\le\varepsilon/4+\varepsilon/2+\varepsilon/4=\varepsilon.
\]
Thus the same centers cover $\cZ$. Since $\|g\|_\infty\le B$,
Lemma~\ref{lem:fat-covering},
\eqref{eq:appendix-rkhs-fat}, and \eqref{eq:appendix-entropy-grid-size} give
\begin{equation}
\ln\mathcal N(\varepsilon,\mathcal G_B,\|\cdot\|_\infty)\le C\operatorname{fat}_{c\varepsilon}(\mathcal G_B)\ln^{3/2}\!\left[C|\mathcal D_\varepsilon|(1+B/\varepsilon)\right]\le C_{\nu,\kappa}d^{3/2+\alpha p/2}(B/\varepsilon)^p\ln^{3/2}(C_{\nu,\kappa}dB/\varepsilon).
\label{eq:appendix-rkhs-entropy}
\end{equation}
Here $c>0$ is universal. The restriction to the grid cannot
increase the fat-shattering dimension, and $p<2$ absorbs all
fixed changes of scale into $C_{\nu,\kappa}$.

Finally, put $s:=1-p/2=2\alpha/(d+2\alpha)$.
We have $3/4+\alpha p/4\le1$ and
$\ln^{3/4}x\le\ln x$ for $x\ge e$.
Using \eqref{eq:appendix-continuation-cover-transfer} and
\eqref{eq:appendix-rkhs-entropy}, and substituting $x=\varepsilon/B$,
we obtain
\begin{align*}
\int_0^H\sqrt{\ln\bigl(2\mathcal N(\varepsilon,\mathcal V_h,
\|\cdot\|_\infty)\bigr)}\,d\varepsilon
&\le H\sqrt{\ln2}+C_{\nu,\kappa}dB
\int_0^1x^{s-1}\ln(C_{\nu,\kappa}d/x)\,dx\\
&=H\sqrt{\ln2}+C_{\nu,\kappa}dB
\left[\frac{\ln(C_{\nu,\kappa}d)}s+\frac1{s^2}\right].
\end{align*}
The zero function covers $\mathcal V_h$ when $\varepsilon\ge B$;
extending the remaining integral to $B$ also covers $H<B$.
The equality integrates $x^{s-1}$ and $x^{s-1}\ln(1/x)$.
Since $\ln(C_{\nu,\kappa}d)\le C_{\nu,\kappa}/s$ and
\[
\frac1s=\frac{d+2\alpha}{2\alpha}
\le\frac{\nu}{\alpha}\frac{d+2\nu}{2\nu},
\]
substitution of $B=\Lambda(B_r+B_p H)$ gives
\[
\int_0^H\sqrt{\ln\bigl(2\mathcal N(\varepsilon,\mathcal V_h,
\|\cdot\|_\infty)\bigr)}\,d\varepsilon
\le C_{\nu,\kappa}
\left[H+d\left(\frac{2\nu+d}{2\nu}\right)^2\Lambda(B_r+B_p H)\right].
\]
This holds for every $h$. Choosing the constant in $J$ twice as
large makes the integral at most $J/2$.
\end{proof}

\begin{proof}[Proof of Lemma~\ref{lem:local-ridge}]
Fix a cell $c$ and a stage $h$. To bound the adjusted error in
\eqref{eq:local-ridge-confidence}, we expand the target and fitted responses.
Put $m:=\mathcal B^hV$ and $\theta:=\theta_c(m)$.
By Lemma~\ref{lem:bellman-completeness} and
\eqref{eq:reference-kernel-norm}, we know that $m\in\mathcal H_0$.
Then, by the reproducing property and \eqref{eq:feature-decomposition},
we have
\[
m(z)=(\mathcal B^hV)(z)=\langle m,\Phi(z)\rangle_{\mathcal H_0}
=\phi_c(z)^\top\theta+\langle m,\Psi_{\mathcal C}(z)\rangle_{\mathcal H_0},\ z\in c.
\]
Write $x_i=\phi_c(z_i^h)$ for an observation assigned to $c$ and
$x_i=0$ otherwise, so $G_t:=G_t^h(c)=\lambda I_D+\sum_{i<t}x_ix_i^\top$.
We define the response-noise sum
$S_t(V):=\sum_{i<t}x_i[Y_i^h(V)-m(z_i^h)]$.
Using \eqref{eq:local-ridge-estimate}, we separate the observed
responses into this expansion and their noise:
\begin{equation}
\begin{aligned}
G_t\widehat\theta_t^h(V,c)
&=\sum_{i<t}x_iY_i^h(V)
=\sum_{i<t}x_im(z_i^h)+\sum_{i<t}x_i[Y_i^h(V)-m(z_i^h)]\\
&=\left(\sum_{i<t}x_ix_i^\top\right)\theta
+\sum_{i<t}x_i\langle m,\Psi_{\mathcal C}(z_i^h)\rangle_{\mathcal H_0}
+S_t(V)\\
&=(G_t-\lambda I_D)\theta
+\sum_{i<t}x_i\langle m,\Psi_{\mathcal C}(z_i^h)\rangle_{\mathcal H_0}
+S_t(V).
\end{aligned}
\label{eq:appendix-ridge-response-sum}
\end{equation}
The second line applies the local expansion at assigned observations;
all other summands vanish because $x_i=0$. The last line uses the
definition of $G_t$.

For $z\in c$, the adjustment in \eqref{eq:profiled-features}
satisfies
\[
\langle m,g_t^h(z)\rangle_{\mathcal H_0}
=\langle m,\Psi_{\mathcal C}(z)\rangle_{\mathcal H_0}
-\phi_c(z)^\top G_t^{-1}
\sum_{i<t}x_i\langle m,\Psi_{\mathcal C}(z_i^h)\rangle_{\mathcal H_0}.
\]
We now expand $m(z)$ and the fitted prediction, then subtract this
adjustment:
\begin{equation}
\begin{aligned}
m(z)-\widehat{\mathcal B}_t^h[V](z)
-\langle m,g_t^h(z)\rangle_{\mathcal H_0}
&=\phi_c(z)^\top\theta-\phi_c(z)^\top G_t^{-1}
\left[\sum_{i<t}x_iY_i^h(V)
-\sum_{i<t}x_i\langle m,\Psi_{\mathcal C}(z_i^h)\rangle_{\mathcal H_0}\right]\\
&=\phi_c(z)^\top\theta-\phi_c(z)^\top G_t^{-1}
\left[(G_t-\lambda I_D)\theta+S_t(V)\right]\\
&=\lambda\phi_c(z)^\top G_t^{-1}\theta
-\phi_c(z)^\top G_t^{-1}S_t(V).
\end{aligned}
\label{eq:appendix-ridge-error-identity}
\end{equation}
The first equality cancels the two copies of
$\langle m,\Psi_{\mathcal C}(z)\rangle_{\mathcal H_0}$; the second uses
\eqref{eq:appendix-ridge-response-sum}. The last uses
$G_t^{-1}(G_t-\lambda I_D)=I_D-\lambda G_t^{-1}$.
It remains to bound the coefficient norm and the noise sum.
Since $0\le m\le H$, the orthonormality in
Lemma~\ref{lem:local-legendre-projection} gives
\begin{equation}
\|\theta\|_2^2
=\sum_{|\mathbf b|\le\rho}
\left|\frac1{|\overline c|}\int_{\overline c}m(u)p_{c,\mathbf b}(u)\,du\right|^2
\le\frac1{|\overline c|}\int_{\overline c}|m(u)|^2\,du\le H^2.
\label{eq:appendix-ridge-coefficient-bound}
\end{equation}
The first inequality is Bessel's inequality; the last uses normalized
Lebesgue measure. Applying the Cauchy--Schwarz inequality to the
two terms in \eqref{eq:appendix-ridge-error-identity}, with
$G_t^{-1}\preceq\lambda^{-1}I_D$, we obtain
\begin{equation}
\begin{aligned}
\left|m(z)-\widehat{\mathcal B}_t^h[V](z)
-\langle m,g_t^h(z)\rangle_{\mathcal H_0}\right|
&\le\left(\lambda\|\theta\|_{G_t^{-1}}
+\|S_t(V)\|_{G_t^{-1}}\right)w_t^h(z)\\
&\le\left(\sqrt\lambda H+\|S_t(V)\|_{G_t^{-1}}\right)w_t^h(z).
\end{aligned}
\label{eq:appendix-ridge-error}
\end{equation}

We bound the noise uniformly over $V$ using the entropy integral
from Lemma~\ref{lem:continuation-entropy}.
The vectors $x_i$ are $\mathcal I_i^h$-measurable and shared by all
$V\in\mathcal V_h$; the $i$th response is observed before
$\mathcal I_{i+1}^h$. For $h<H$, write
\[
S_t(V)=S_t(0)+S_{P,t}(V),\
S_{P,t}(V):=\sum_{i<t}x_i[V(s_i^{h+1})-(P^hV)(z_i^h)].
\]
For fixed bounded $U$, the conditional range of $U(s_i^{h+1})$
has length at most $2\|U\|_\infty$. By
Lemma~\ref{lem:martingale-tools}(ii), its centered increment is
conditionally $\|U\|_\infty$-sub-Gaussian; the reward noise is $1/2$-sub-Gaussian because $R_i^h\in[0,1]$.
Since $\lambda=B_{\phi}^2$ and $\|x_i\|_2\le B_{\phi}$, the determinant bound in
Lemma~\ref{lem:elliptical-potential} gives
\begin{equation}
\ln\frac{\det G_t}{\lambda^D}\le D\ln(1+(t-1)/D)\le D\ln(1+T/D).
\label{eq:appendix-ridge-logdet}
\end{equation}

Fix $t$ and put $Z_V:=G_t^{-1/2}S_{P,t}(V)$, so $Z_0=0$.
Applying \eqref{eq:self-normalized-exponential-moment}
with $U=V-V'$, we obtain
\[
\mathbb E\left[
\sqrt{\frac{\lambda^D}{\det G_t}}
\exp\left\{\frac{\|Z_V-Z_{V'}\|_2^2}{2\|V-V'\|_\infty^2}\right\}
\right]\le1,\ V\ne V'.
\]
Here $\mathbb P$ denotes the probability  law of the interaction trajectory
under the fixed MDP, algorithm, and initial-state rule; $G_t$ is random
under this law. The determinant factor is common to all continuations.
To obtain the increment bound required by
Lemma~\ref{lem:subgaussian-entropy-tail}, we define
$d\mathbb Q:=\sqrt{\frac{\lambda^D}{\det G_t}}\,d\mathbb P$ on the original
sample space. Its total mass is at most one because
$G_t\succeq\lambda I_D$. We assign probability
$1-\mathbb E[\sqrt{\lambda^D/\det G_t}]$ to one extra point where
$Z_V=0$ for every $V$, making the total mass exactly one. 

By Markov's inequality, we have
\begin{equation}
\mathbb Q\bigl(\|Z_V-Z_{V'}\|_2>r\|V-V'\|_\infty\bigr)
\le e^{-r^2/2},\ r>0.
\label{eq:appendix-ridge-increment-tail}
\end{equation}
The added point contributes nothing to this event.
Lemma~\ref{lem:continuation-entropy} implies that $\mathcal V_h$
is totally bounded. Since $\|x_i\|_{G_t^{-1}}\le B_{\phi}/\sqrt\lambda=1$,
the process has continuous paths:
\[
\|Z_V-Z_{V'}\|_2\le2\sum_{i<t}\|x_i\|_{G_t^{-1}}\|V-V'\|_\infty
\le2T\|V-V'\|_\infty.
\]
By \eqref{eq:appendix-ridge-increment-tail},
Lemma~\ref{lem:subgaussian-entropy-tail} applies with the supremum-norm
metric, whose diameter on $\mathcal V_h$ is at most $H$.
Substituting the entropy integral bound
\eqref{eq:continuation-entropy-bound}, we obtain
\begin{equation}
\mathbb Q\left(\sup_{V\in\mathcal V_h}\|Z_V\|_2
>C\left[
C_{\nu,\kappa}\left(H+d\left(\frac{2\nu+d}{2\nu}\right)^2
\Lambda(B_r+B_p H)\right)+H\sqrt x
\right]\right)\le e^{-x},\ x\ge1.
\label{eq:appendix-ridge-tilted-tail}
\end{equation}
By \eqref{eq:appendix-ridge-logdet}, for every event $E$ in the original sample space, we have
\[
\mathbb Q(E)
=\mathbb E\left[\mathbf1_E\sqrt{\lambda^D/\det G_t}\right]
\ge\exp\{-\tfrac D2\ln(1+T/D)\}\mathbb P(E).
\]
Taking $x=1+\tfrac D2\ln(1+T/D)+\ln(2MHT/\delta)$, for the event $E$ in \eqref{eq:appendix-ridge-tilted-tail}, we have
$$\mathbb P(E)\le\exp\{\tfrac D2\ln(1+T/D)-x\}
=\delta/(2\mathrm e MHT).$$
A union bound over $t,c,h$ therefore gives,
with failure probability at most $\delta/2$,
\begin{equation}
\sup_{V\in\mathcal V_h}\|S_{P,t}(V)\|_{G_t^{-1}}
\le C\left[J+H\sqrt{D\ln(1+T/D)+\ln(2MHT/\delta)}\right]
\ \text{for all }t\in[T],\ c\in\mathcal C,\ h<H.
\label{eq:appendix-ridge-transition-noise}
\end{equation}

For $S_t(0)$, Lemma~\ref{lem:self-normalized} with scale $1/2$,
\eqref{eq:appendix-ridge-logdet}, and a union bound over cells and stages give
\begin{equation}
\|S_t(0)\|_{G_t^{-1}}
\le\tfrac12\sqrt{D\ln(1+T/D)+2\ln(2MH/\delta)}
\ \text{for all }t\in[T],\ c\in\mathcal C,\ h\in[H]
\label{eq:appendix-ridge-reward-noise}
\end{equation}
with failure probability at most $\delta/2$.
Intersecting the events in \eqref{eq:appendix-ridge-transition-noise}
and \eqref{eq:appendix-ridge-reward-noise}, and using $H\ge1$, we obtain
\begin{equation}
\begin{aligned}
\sup_{V\in\mathcal V_h}\|S_t(V)\|_{G_t^{-1}}
&\le\|S_t(0)\|_{G_t^{-1}}
+\sup_{V\in\mathcal V_h}\|S_{P,t}(V)\|_{G_t^{-1}}\\
&\le C\left[J+H\sqrt{D\ln(1+T/D)+\ln(8MHT/\delta)}\right]
\le\beta(M,\delta).
\end{aligned}
\label{eq:appendix-ridge-total-noise}
\end{equation}
At $h=H$, we have $\mathcal V_H=\{0\}$ and $S_t(V)=S_t(0)$, so
\eqref{eq:appendix-ridge-reward-noise} already bounds the total noise.
Substitution into \eqref{eq:appendix-ridge-error} proves
\eqref{eq:local-ridge-confidence}, since
$\overline\beta(M,\delta)=\sqrt\lambda H+\beta(M,\delta)$.
Both conclusions hold on this single event for every $t,h,c,V,z$. Empty cells have
$G_t=\lambda I_D$ and $S_t(V)=0$ and satisfy the same bounds.
\end{proof}

\subsection{Proofs for Section~\ref{sec:sovi-version-space}}

\begin{proof}[Proof of Lemma~\ref{lem:bellman-residual-decomposition}]
Fix $t,h,f,c$ and put $G:=G_t^h(c)$. By
\eqref{eq:empirical-bellman-residual} and
\eqref{eq:bellman-residual-definition}, the observed residual is a response
for $\Delta_f^h$ with noise of the opposite sign:
\[
d_i^h(f)=f^h(z_i^h)-Y_i^h(V_f^{h+1})
=\Delta_f^h(z_i^h)-\xi_i^h(V_f^{h+1}),\
\xi_i^h(V)=Y_i^h(V)-(\mathcal B^hV)(z_i^h).
\]
By \eqref{eq:candidate-class}, Lemma~\ref{lem:bellman-completeness},
and \eqref{eq:reference-kernel-norm}, we have
$\Delta_f^h=f^h-\mathcal B^hV_f^{h+1}\in\mathcal H_0$.
Applying \eqref{eq:appendix-ridge-error-identity} with this target
and responses $d_i^h(f)$, we obtain
\[
\begin{aligned}
&\Delta_f^h(z)-\phi_c(z)^\top G^{-1}v_{t,c}^h(f)
-\langle\Delta_f^h,g_t^h(z)\rangle_{\mathcal H_0}\\
&=\lambda\phi_c(z)^\top G^{-1}\theta_c(\Delta_f^h)
+\phi_c(z)^\top G^{-1}
\sum_{i\in\mathcal T_t^h(c)}\phi_c(z_i^h)\xi_i^h(V_f^{h+1}).
\end{aligned}
\]
The plus sign comes from the response error $-\xi_i^h(V_f^{h+1})$.
Both $f^h$ and $\mathcal B^hV_f^{h+1}$ take values in $[0,H-h+1]$;
hence \eqref{eq:appendix-ridge-coefficient-bound} gives
$\|\theta_c(\Delta_f^h)\|_2\le\|\Delta_f^h\|_\infty\le H$.
On the event of Lemma~\ref{lem:local-ridge},
\eqref{eq:appendix-ridge-total-noise} applies because
$V_f^{h+1}\in\mathcal V_h$:
\[
\left\|\sum_{i\in\mathcal T_t^h(c)}\phi_c(z_i^h)
\xi_i^h(V_f^{h+1})\right\|_{G^{-1}}\le\beta(M,\delta).
\]
Using \eqref{eq:appendix-ridge-coefficient-bound} and
\eqref{eq:appendix-ridge-total-noise} in \eqref{eq:appendix-ridge-error},
with the same target and response error, we conclude that
\[
\left|\Delta_f^h(z)-\phi_c(z)^\top G^{-1}v_{t,c}^h(f)
-\langle\Delta_f^h,g_t^h(z)\rangle_{\mathcal H_0}\right|
\le\bigl(\sqrt\lambda H+\beta(M,\delta)\bigr)w_t^h(z)
=\overline\beta(M,\delta)w_t^h(z).
\]
The event of Lemma~\ref{lem:local-ridge} is uniform over $t,h,c,V$,
so this conclusion holds for every $f,z$, including data-dependent $f$,
on that same event.
\end{proof}
\paragraph{Verification of the Gaussian envelope.}
Recall $\alpha=\min\{\nu/2,1/2\}$ and $C_{\Psi},n_{\Psi}$ from
\eqref{eq:vs-gaussian-cover-parameters}. For the family in
\eqref{eq:vs-gaussian-law}, its canonical distance on each rectangle is
$d_c(z,z')=\|\Psi_c(z)-\Psi_c(z')\|_{\mathcal H_0}$.
The frequency-moment calculation preceding
\eqref{eq:appendix-feature-holder} gives
\[
\|\Phi(z)-\Phi(z')\|_{\mathcal H_0}
\le2\kappa^\alpha
\sqrt{\frac{\Gamma(\alpha+1/2)\Gamma(\nu-\alpha)}{\sqrt\pi\,\Gamma(\nu)}}
\|z-z'\|_2^\alpha.
\]
The partition after \eqref{eq:version-space-resolution} has sides
at least $1/(2R)$. Since $\|\Phi(z)\|^2=k_0(z,z)=1$,
Lemma~\ref{lem:local-legendre-projection} gives
$\|\Pi_c\Phi\|_\infty\le B_{\phi}$.
Applying Lemma~\ref{lem:polynomial-markov} on this rectangle, we obtain
\[
\|(\Pi_c\Phi)(z)-(\Pi_c\Phi)(z')\|
\le4\rho^2RB_{\phi}\sqrt d\,\|z-z'\|_2.
\]
The triangle inequality, $\alpha\le1$, and \eqref{eq:vs-gaussian-cover-parameters}
therefore imply, for $\|z-z'\|_2\le1$,
\begin{equation}
\begin{aligned}
d_c(z,z')
&\le2\kappa^\alpha
\sqrt{\frac{\Gamma(\alpha+1/2)\Gamma(\nu-\alpha)}{\sqrt\pi\,\Gamma(\nu)}}
\|z-z'\|_2^\alpha
+4\rho^2RB_{\phi}\sqrt d\,\|z-z'\|_2\\
&\le C_{\Psi}\|z-z'\|_2^\alpha.
\end{aligned}
\label{eq:gaussian-remainder-increment}
\end{equation}
Gaussian moment bounds and \eqref{eq:gaussian-remainder-increment}
give a continuous version on each rectangle by
\citet{hairer2025spdes}. The modifications preserve the joint
finite-dimensional laws, and each supremum is determined by a
countable dense subset.

For $0<x\le2$, use a coordinate mesh with spacing at most
$d^{-1/2}\min\{1,(A_{\rm vs}x/C_{\Psi})^{1/\alpha}\}$.
Every point is within Euclidean distance
$\min\{1,(A_{\rm vs}x/C_{\Psi})^{1/\alpha}\}$ of a mesh point,
so its remainder feature is approximated within $A_{\rm vs}x$.
Since every side is at most one and $A_{\rm vs}\ge R^{-\nu}$,
the grid size satisfies
\begin{equation}
2+\left\lceil\sqrt d\max\{1,(C_{\Psi}/(A_{\rm vs}x))^{1/\alpha}\}\right\rceil
\le n_{\Psi}\max\{1,x^{-1/\alpha}\}.
\label{eq:gaussian-remainder-cover}
\end{equation}
Including zero and both signs, we obtain
\begin{equation}
\ln\mathcal N\!\left(A_{\rm vs}x,
\{0\}\cup\bigcup_{c\in\mathcal C}\{\pm\Psi_c(z):z\in\overline c\},\|\cdot\|\right)
\le\ln(1+2|\mathcal C|n_{\Psi}^d)
+\frac d\alpha(\ln(1/x))_+.
\label{eq:vs-gaussian-explicit-cover}
\end{equation}
Here the factor two counts the signs; the last term bounds the grid refinement.
Under \eqref{eq:vs-gaussian-law}, the canonical metric on this set is the RKHS
distance, its diameter is at most $2A_{\rm vs}$, and its Gaussian supremum is
$\max_c\sup_z|W_c(z)|$.
Applying \eqref{eq:dudley-explicit-constant}, then
$\sqrt{a+b}\le\sqrt a+\sqrt b$ and
$\int_0^1\sqrt{\ln(1/x)}\,dx=\sqrt\pi/2$, we obtain
\begin{equation}
\begin{aligned}
\mathbb E_{\mathsf P}\max_c\sup_{z\in\overline c}|W_c(z)|
&\le12A_{\rm vs}\int_0^2
\sqrt{\ln(1+2|\mathcal C|n_{\Psi}^d)+\frac d\alpha(\ln(1/x))_+}\,dx\\
&\le A_{\rm vs}\left[24\sqrt{\ln(1+2|\mathcal C|n_{\Psi}^d)}
+6\sqrt{\pi d/\alpha}\right]
=\frac{q_0-1}{2}A_{\rm vs}.
\end{aligned}
\label{eq:vs-gaussian-supremum-bound}
\end{equation}
The equality substitutes \eqref{eq:vs-gaussian-envelope}.

The statistic $\Xi_p^h(f)$ in \eqref{eq:innovation-test}
subtracts a quadratic term to control weighted observation noise
uniformly over continuation values. The following bound justifies
the factor $320H^2$ in its definition.

\begin{lemma}
\label{lem:innovation-offset-moment}
Fix a stage $h\in[H]$ and $p\in[T]$, and consider episodes $1,\ldots,p-1$.
Let $(a_i)_{i<p}$ be finite predictable scalar weights shared by all
$V\in\mathcal V_h$. Choose the constant defining $J$ in
Lemma~\ref{lem:continuation-entropy} so that the entropy integral
there is at most $J/2$.
Then quadratic compensation bounds the weighted-noise supremum
for either sign:
\[
\mathbb E\exp\!\left\{
\frac{1}{320H^2}\left[
\sup_{V\in\mathcal V_h}\left(\pm\sum_{i<p}a_i\xi_i^h(V)\right)
-\frac1{16}\sum_{i<p}a_i^2
\right]\right\}
\le\exp\{J^2/H^2\}.
\]
\end{lemma}

\begin{proof}[Proof of Lemma~\ref{lem:innovation-offset-moment}]
We prove the positive-sign bound. We write
$M(V):=\sum_{i<p}a_i\xi_i^h(V)$, $Q:=\sum_{i<p}a_i^2$, and
$\lambda_*:=(320H^2)^{-1}$.
For deterministic $U,U'\in\mathcal V_h$ and $h<H$, subtracting
two noise terms cancels the reward:
\begin{equation}
\xi_i^h(U)-\xi_i^h(U')
=(U-U')(s_i^{h+1})-(P^h(U-U'))(z_i^h).
\label{eq:offset-noise-difference}
\end{equation}
For $h=H$, we have $\mathcal V_H=\{0\}$ and all continuation increments vanish.
We reduce the supremum to finite sums of increments.

By Lemma~\ref{lem:continuation-entropy}, we choose fixed internal $r_j$-nets
$\mathcal N_j$, where $r_j:=H2^{-j}$, containing zero and satisfying
$|\mathcal N_j|\le2\mathcal N(r_j,\mathcal V_h,\|\cdot\|_\infty)$
for $j\ge1$. Set $\mathcal N_0:=\{0\}$.
Choose $V^{(j)}\in\mathcal N_j$ with
$\|V-V^{(j)}\|_\infty\le r_j$ and $V^{(0)}=0$.
Since $\|V\|_\infty\le H$, we have, for every $j\ge1$,
\[
\|V^{(j)}-V^{(j-1)}\|_\infty
\le\|V^{(j)}-V\|_\infty+\|V-V^{(j-1)}\|_\infty
\le3r_j=:b_j.
\]
Let $L_j:=\max\cbr{M(U)-M(U'):U\in\mathcal N_j,U'\in\mathcal N_{j-1},\|U-U'\|_\infty\le 3r_j}$ and put
$K_j:=|\mathcal N_j||\mathcal N_{j-1}|$.
Set $L_0:=M(0)$, $b_0:=H$, and $K_0:=1$. The finite approximations satisfy
\begin{equation}
M(V^{(n)})=M(0)+\sum_{j=1}^n
[M(V^{(j)})-M(V^{(j-1)})]\le\sum_{j=0}^nL_j.
\label{eq:offset-chain-domination}
\end{equation}
The equality telescopes, and each increment is at most its defining
maximum $L_j$. By \eqref{eq:offset-noise-difference}, we have
\begin{equation}
\sup_{V\in\mathcal V_h}|M(V)-M(V^{(n)})|
\le2r_n\sum_{i<p}|a_i|\longrightarrow0,
\label{eq:offset-net-remainder}
\end{equation}
because $\sum_{i=1}^{p-1}|a_i|<\infty$ and $r_n\to0$.
Thus $\sup_VM(V^{(n)})$ converges to the measurable supremum
$\sup_VM(V)$. By \eqref{eq:offset-chain-domination}, it remains to bound
the compensated exponential moment of $\sum_{j=0}^nL_j$, uniformly in $n$.

For fixed $U,U'\in\mathcal V_h$ and $h<H$,
\eqref{eq:offset-noise-difference} has conditional range of length at most
$2\|U-U'\|_\infty$. Given $\mathcal I_i^h$, the weight $a_i$ is known,
so Lemma~\ref{lem:martingale-tools}(ii) gives, for $\theta\in\mathbb R$,
\[
\mathbb E\left[\exp\left\{
\theta a_i[\xi_i^h(U)-\xi_i^h(U')]
-\frac{\theta^2a_i^2}{2}\|U-U'\|_\infty^2
\right\}\mid\mathcal I_i^h\right]\le1.
\]
Iterating conditional expectations, we obtain
\begin{equation}
\mathbb E\exp\left\{\theta[M(U)-M(U')]
-\frac{\theta^2Q}{2}\|U-U'\|_\infty^2\right\}\le1.
\label{eq:offset-increment-moment}
\end{equation}
Here $Q$ remains inside the exponential.
Since $M(0)=\sum_{i<p}a_i(R_i^h-r^h(z_i^h))$ and $R_i^h\in[0,1]$,
we apply Lemma~\ref{lem:martingale-tools}(ii) with the relaxed scale
$H\ge1$ and iterate as in \eqref{eq:offset-increment-moment} to obtain
\begin{equation}
\mathbb E [e^{\theta M(0)-\theta^2H^2Q/2}]\le1.
\label{eq:offset-reward-moment}
\end{equation}

For $\theta>0$ and $j\ge1$, bounding a maximum of nonnegative
exponentials by their sum gives
\begin{equation}
\mathbb E e^{\theta L_j-\theta^2b_j^2Q/2}
\le\sum_{\substack{U\in\mathcal N_j,\ U'\in\mathcal N_{j-1}\\
\|U-U'\|_\infty\le 3r_j}}
\mathbb E e^{\theta[M(U)-M(U')]-\theta^2b_j^2Q/2}
\le K_j.
\label{eq:offset-link-moment}
\end{equation}
The last inequality uses $\|U-U'\|_\infty^2\le b_j^2=9r_j^2$
in \eqref{eq:offset-increment-moment}.
For $j=0$, \eqref{eq:offset-reward-moment} yields
$\mathbb E e^{\theta L_0-\theta^2b_0^2Q/2}\le K_0=1$.

Using $\ln K_j\le2\ln(2\mathcal N(r_j,\mathcal V_h,\|\cdot\|_\infty))$ and writing $A:=\sum_{j\ge1}3r_j\sqrt{\ln K_j}$, we obtain
\begin{equation}
\begin{aligned}
\sum_{j\ge1}3r_j\sqrt{\ln K_j}
&\le3\sqrt2\sum_{j\ge1}r_j
\sqrt{\ln\bigl(2\mathcal N(r_j,\mathcal V_h,\|\cdot\|_\infty)\bigr)}\\
&\le6\sqrt2\sum_{j\ge1}\int_{r_{j+1}}^{r_j}
\sqrt{\ln\bigl(2\mathcal N(\varepsilon,\mathcal V_h,\|\cdot\|_\infty)\bigr)}\,d\varepsilon\le 3\sqrt2 J.
\end{aligned}
\label{eq:offset-entropy-sum}
\end{equation}
The second inequality uses monotonicity of the covering number
and $r_j-r_{j+1}=r_j/2$; the last uses the disjoint intervals
and Lemma~\ref{lem:continuation-entropy}.
If $A=0$, every net equals $\{0\}$, hence $\mathcal V_h=\{0\}$.
Equation~\eqref{eq:offset-reward-moment} proves the claim. We may therefore assume $A>0$.

We combine \eqref{eq:offset-link-moment} over all levels using
H\"older's inequality, choosing
parameters that keep the total quadratic compensation below $Q/16$:
\[
\lambda_j:=\frac{4\lambda_*H+\sqrt{\ln K_j}/(16A)}{b_j},\
p_j:=\frac{\lambda_*}{\lambda_j},\ j\ge0.
\]
Since $b_0=H$ and $b_j=3H2^{-j}$ for $j\ge1$,
we have $\sum_{j\ge0}b_j=H+3H\sum_{j\ge1}2^{-j}=4H$. Hence, we obtain
\begin{equation}
\sum_{j\ge0}p_j\le\sum_{j\ge0}\frac{b_j}{4H}=1.
\label{eq:offset-holder-weights}
\end{equation}
The quadratic coefficient is controlled by the same parameter choice:
\begin{equation}
\frac12\sum_{j\ge0}\lambda_jb_j^2
=2\lambda_*H\sum_{j\ge0}b_j
+\frac1{32A}\sum_{j\ge1}b_j\sqrt{\ln K_j}
=8\lambda_*H^2+\frac1{32}=\frac9{160}<\frac1{16}.
\label{eq:offset-quadratic-cost}
\end{equation}
For the logarithm of the resulting moment bound, we use
$\lambda_j\ge\sqrt{\ln K_j}/(16Ab_j)$ to obtain
\begin{equation}
\sum_{j\ge0}p_j\ln K_j
\le16\lambda_*A\sum_{j\ge1}b_j\sqrt{\ln K_j}
=16\lambda_*A^2\le\frac{9J^2}{10H^2}.
\label{eq:offset-holder-cost}
\end{equation}
Terms with $K_j=1$ are zero; the last inequality uses
\eqref{eq:offset-entropy-sum} and $\lambda_*=(320H^2)^{-1}$.

Since $p_j\lambda_j=\lambda_*$ and $Q\ge0$,
\eqref{eq:offset-quadratic-cost} implies
\[
\lambda_*\left[\sum_{j=0}^nL_j-\frac Q{16}\right]
\le\sum_{j=0}^np_j\left[\lambda_jL_j-\frac{\lambda_j^2b_j^2Q}{2}\right].
\]
Exponentiating this pointwise bound and taking expectations, we obtain
\begin{equation}
\begin{aligned}
\mathbb E e^{\lambda_*[\sum_{j=0}^nL_j-Q/16]}
&\le\prod_{j=0}^n
\left(\mathbb E e^{\lambda_jL_j-\lambda_j^2b_j^2Q/2}\right)^{p_j}\le\exp\left\{\sum_{j=0}^np_j\ln K_j\right\}
\le e^{J^2/H^2}.
\end{aligned}
\label{eq:offset-finite-chain-moment}
\end{equation}
The first inequality applies the generalized H\"older inequality to the
exponential factors and $1^{1-\sum_{j=0}^np_j}=1$;
all weights are nonnegative and sum to one by \eqref{eq:offset-holder-weights}.
The second inequality applies \eqref{eq:offset-link-moment} to each factor.
In the last, we use \eqref{eq:offset-holder-cost}, since $p_j\ln K_j\ge0$
and its finite sum is bounded by the full sum.

Finally, by Fatou's lemma, \eqref{eq:offset-net-remainder},
\eqref{eq:offset-chain-domination}, and \eqref{eq:offset-finite-chain-moment},
we obtain
\[
\mathbb E e^{\lambda_*[\sup_VM(V)-Q/16]}
\le\liminf_{n\to\infty}
\mathbb E e^{\lambda_*[\sup_VM(V^{(n)})-Q/16]}
\le e^{J^2/H^2}.
\]
Replacing $a_i$ by $-a_i$ proves the negative-sign bound.
\end{proof}
\begin{proof}[Proof of Lemma~\ref{lem:uniform-confidence}]
Fix $2\le h\le H$ and $p\le t$. Recall from \eqref{eq:innovation-test} that
\[
\Xi_p^h(f)=\ln\mathbb E_{\boldsymbol Z\sim\mathsf P_p^h}
\exp\left\{\frac1{320H^2}\sum_{i<p}
\left[\chi_i^h(Z_i)U_i^h(f)
-\beta w_i^h(z_i^h)|\chi_i^h(Z_i)|-\frac{\chi_i^h(Z_i)^2}{16}\right]\right\}.
\]
We first remove the fitted residual from this score.
For $f\in\mathcal F_t^1$, \eqref{eq:coarse-feasibility} and the
Cauchy--Schwarz inequality imply
\begin{equation}
\left|\phi_{c_i}(z_i^h)^\top G_i^h(c_i)^{-1}v_{i,c_i}^h(f)\right|
\le w_i^h(z_i^h)\|v_{i,c_i}^h(f)\|_{G_i^h(c_i)^{-1}}
\le\beta w_i^h(z_i^h),\ c_i:=c(z_i^h),\ i<p.
\label{eq:uniform-coarse-fit}
\end{equation}
By \eqref{eq:bellman-innovation}, we have
$$U_i^h(f)=\Delta_f^h(z_i^h)-\xi_i^h(V_f^{h+1})
-\phi_{c_i}(z_i^h)^\top G_i^h(c_i)^{-1}v_{i,c_i}^h(f),$$
where $\xi_i^h(V)=Y_i^h(V)-(\mathcal B^hV)(z_i^h)$.
For every real $a_i$, plugging this back therefore gives
\begin{equation}
\begin{aligned}
a_iU_i^h(f)-\beta w_i^h(z_i^h)|a_i|
&=a_i\Delta_f^h(z_i^h)-a_i\xi_i^h(V_f^{h+1})
-a_i\phi_{c_i}(z_i^h)^\top G_i^h(c_i)^{-1}v_{i,c_i}^h(f)
-\beta w_i^h(z_i^h)|a_i|\\
&\le a_i\Delta_f^h(z_i^h)-a_i\xi_i^h(V_f^{h+1}).
\end{aligned}
\label{eq:uniform-innovation-drift}
\end{equation}
To control the remaining noise, use the fixed Gaussian law
$W\sim\mathsf P$ in \eqref{eq:vs-gaussian-law},
independently of the interaction.
With the observed history fixed, all features below are deterministic.
Recall $c_i=c(z_i^h)$.
By \eqref{eq:vs-gaussian-coordinate-realization},
\eqref{eq:profiled-features}, and linearity of $W$, we have
\[
\begin{aligned}
Z_i(W)
&=\alpha_{\rm vs}W(g_i^h(z_i^h))\\
&=\alpha_{\rm vs}\left[
W_{c_i}(z_i^h)
-\sum_{j\in\mathcal T_i^h(c_i)}
\phi_{c_i}(z_i^h)^\top G_i^h(c_i)^{-1}
\phi_{c_i}(z_j^h)W_{c_i}(z_j^h)
\right].
\end{aligned}
\]
These coordinates are jointly Gaussian and centered.
By \eqref{eq:vs-gaussian-law}, their covariance satisfies
\[
\mathbb E_{\mathsf P}[Z_i(W)Z_j(W)]
=\alpha_{\rm vs}^2
\langle g_i^h(z_i^h),g_j^h(z_j^h)\rangle_{\mathcal H_0}
=\alpha_{\rm vs}^2\Sigma_p^h[i,j].
\]
The last equality uses \eqref{eq:vs-gaussian-covariance}.
Consequently, $(Z_i(W))_{i<p}$ has the Gaussian law
$\mathsf P_p^h$ defined in \eqref{eq:vs-gaussian-sequence}.

The weights in \eqref{eq:innovation-test-weights} are
\[
\chi_i^h(x)=\operatorname{sgn}(x)
\min\{H,(|x|-\tau_i^h(z_i^h))_+\},\ (y)_+:=\max\{y,0\}.
\]
For the residual contribution, $|\chi_i^h|\le H$ gives
\begin{equation}
\sum_{i<p}\chi_i^h(Z_i(W))\Delta_f^h(z_i^h)
\le\sum_{i<p}|\chi_i^h(Z_i(W))|\,|\Delta_f^h(z_i^h)|
\le H(p-1)\|\Delta_f^h\|_\infty.
\label{eq:uniform-residual-sum}
\end{equation}

Define the Gaussian average of the compensated noise by
\begin{equation}
L_p^h:=\mathbb E_{\mathsf P}\exp\left\{\frac1{320H^2}\left[
\sup_{V\in\mathcal V_h}\left(-\sum_{i<p}\chi_i^h(Z_i(W))\xi_i^h(V)\right)
-\frac1{16}\sum_{i<p}\chi_i^h(Z_i(W))^2\right]\right\}.
\label{eq:uniform-negative-offset-average}
\end{equation}
We substitute \eqref{eq:uniform-innovation-drift} into
\eqref{eq:innovation-test}, apply \eqref{eq:uniform-residual-sum}, and
include $V_f^{h+1}$ in the supremum in
\eqref{eq:uniform-negative-offset-average}. The residual bound is
independent of $W$, so it factors out of the Gaussian average and we obtain
\begin{equation}
e^{\Xi_p^h(f)}
\le\exp\left\{\frac{p-1}{320H}\|\Delta_f^h\|_\infty\right\}L_p^h.
\label{eq:uniform-score-comparison}
\end{equation}

It remains to bound $L_p^h$ uniformly.
For fixed $W$, the feature $g_i^h(z_i^h)$ in
\eqref{eq:profiled-features} and the threshold in
\eqref{eq:innovation-threshold} use the current $z_i^h$ and observations
before episode $i$. Thus $\chi_i^h(Z_i(W))$ is
$\mathcal I_i^h$-measurable, bounded by $H$, and shared by all
$V\in\mathcal V_h$.

Let $\mathbb E$ and $\Pr$ be the expectation and probability distribution with respect to the interaction trajectory.
For each fixed $W$, the negative-sign bound in
Lemma~\ref{lem:innovation-offset-moment}, with
$a_i=\chi_i^h(Z_i(W))$, bounds the interaction expectation of the
exponential in \eqref{eq:uniform-negative-offset-average} by $e^{J^2/H^2}$.
Since the exponential is nonnegative, Tonelli's theorem lets us
average this inequality over $W$:
\begin{equation}
\mathbb E L_p^h\le\int e^{J^2/H^2}\,\mathsf P(dW)=e^{J^2/H^2}.
\label{eq:uniform-negative-offset-moment}
\end{equation}
By Markov's inequality and \eqref{eq:uniform-negative-offset-moment}, we have
\[
\Pr\left(L_p^h>\frac{8HT}{\delta}e^{J^2/H^2}\right)
\le\frac{\mathbb E L_p^h}{(8HT/\delta)e^{J^2/H^2}}
\le\frac{\delta}{8HT}.
\]
A union bound over at most $HT$ pairs $(h,p)$ gives, with probability
at least $1-\delta/8$,
\begin{equation}
\ln L_p^h\le J^2/H^2+\ln(8HT/\delta),\
2\le h\le H,\ p\in[T].
\label{eq:uniform-negative-offset-event}
\end{equation}

Taking logarithms in \eqref{eq:uniform-score-comparison} and applying
\eqref{eq:uniform-negative-offset-event}, we obtain
\[
\Xi_p^h(f)\le\frac{p-1}{320H}\|\Delta_f^h\|_\infty+\ln L_p^h
\le\frac{p-1}{320H}\|\Delta_f^h\|_\infty
+\frac{J^2}{H^2}+\ln\frac{8HT}{\delta},
\]
which implies \eqref{eq:innovation-statistic-bound}.
The event \eqref{eq:uniform-negative-offset-event} is shared by all candidates,
so the bound holds for all $t\ge p$ and data-dependent
$f\in\mathcal F_t^1$. At $p=1$, we have $\Xi_1^h(f)=0$.
\end{proof}

\subsection{Proofs for Section~\ref{sec:sovi-first-selection}}

\begin{proof}[Proof of Lemma~\ref{lem:initial-confidence}]
At a visit to level $l$ and epoch $u$, the regression uses exactly
\[
\widetilde{\mathcal T}_t(c)
=\{i<t:l_i=l,\ u_i=u,\ z_i^1\in c\}.
\]
The assignment $\mathbf1\{l_i=l,u_i=u\}$ is determined by the
current initial state and earlier visit counts, before observing
the response. Applying Lemma~\ref{lem:local-ridge} on $\mathcal C_l$
with this predictable assignment and failure probability
$\delta_{l,u}$, we obtain
\[
\left|(\mathcal B^1V)(z)-\widehat{\mathcal P}_t[V](z)
-\langle\mathcal B^1V,g_t(z)\rangle_{\mathcal H_0}\right|
\le\overline\beta(|\mathcal C_l|,\delta_{l,u})w_t(z)=b_t(z)
\]
at all visits to this epoch, uniformly over $V\in\mathcal V_1$
and $z$. The first visit is included because the lemma allows
empty observation sets.

Taking a union bound over the fixed pairs $(l,u)$, we bound the
probability of any failure by
\begin{equation}
\sum_{l,u\ge0}\delta_{l,u}
=\frac{\delta}{64}\left(\sum_{m\ge1}\frac1{m^2}\right)^2
=\frac{\pi^4\delta}{2304}<\frac\delta8.
\label{eq:initial-confidence-failure-sum}
\end{equation}
On the complementary event, we may use the selected pair
$(l_t,u_t)$ and any $V_f^2\in\mathcal V_1$, proving
\eqref{eq:initial-confidence} simultaneously for all $t,f,z$.
\end{proof}

\begin{proof}[Proof of Lemma~\ref{lem:finite-remainder-comparison}]
Fix a level and epoch $(l,u)$. The algorithm uses a fixed
$\varepsilon$-net $\{\varrho_1,\ldots,\varrho_M\}\subseteq\mathcal R_l$
from Subsection~\ref{sec:first-step-selection-principle}.
We first place the true remainder in this class, so one fixed net
model approximates it for every continuation. Recall that
\[
(\mathcal B^1V)(z)=r^1(z)+\int_{\cS}P^1(s'\mid z)V(s')\,\mu(ds'),\
V\in\mathcal V_1.
\]
We represent the reward and transition by
$\Theta_*:=(r^1,(H-1)P^1(\cdot\mid\cdot))$ in
$\mathcal H_0\oplus L^2(\mu;\mathcal H_0)$.
The proof of Lemma~\ref{lem:bellman-completeness} establishes weak
measurability of $s'\mapsto P^1(s'\mid\cdot)$ in $\mathcal H_k$.
The bounded inclusion \eqref{eq:reference-kernel-norm} transfers it
to $\mathcal H_0$; separability of $\mathcal H_0$ then gives strong
measurability. By the same norm comparison and $\mu(\cS)=1$, we have
\begin{equation}
\|\Theta_*\|^2=\|r^1\|_{\mathcal H_0}^2
+(H-1)^2\int_{\cS}\|P^1(s'\mid\cdot)\|_{\mathcal H_0}^2\,\mu(ds')
\le\Lambda^2[B_r^2+(H-1)^2B_p^2].
\label{eq:appendix-true-remainder-norm}
\end{equation}
By algebra and the integral representation in
Lemma~\ref{lem:bellman-completeness}, we obtain
\[
\begin{aligned}
\left\langle\Theta_*,\Psi_{\mathcal C_l}(z)\otimes(1,V/(H-1))\right\rangle
&=\langle r^1,\Psi_{\mathcal C_l}(z)\rangle_{\mathcal H_0}
+\int_{\cS}V(s')\langle P^1(s'\mid\cdot),\Psi_{\mathcal C_l}(z)\rangle_{\mathcal H_0}\,\mu(ds')\\
&=\langle\mathcal B^1V,\Psi_{\mathcal C_l}(z)\rangle_{\mathcal H_0}.
\end{aligned}
\]
Together with \eqref{eq:appendix-true-remainder-norm}, this proves
that the true remainder belongs to $\mathcal R_l$.
Lemma~\ref{lem:upper-remainder-cover} ensures that the finite internal
net contains an index $j_*$ satisfying
\begin{equation}
\sup_{z\in\cZ,\,V\in\mathcal V_1}
\left|\varrho_{j_*}(z,V)
-\langle\mathcal B^1V,\Psi_{\mathcal C_l}(z)\rangle_{\mathcal H_0}\right|
\le\varepsilon.
\label{eq:appendix-true-remainder-net-error}
\end{equation}
The net and the true model are fixed independently of the observations,
so the same index works throughout the epoch.

At a visit $t$ to $(l,u)$, fix $z\in c$ and write
$\mathcal T:=\widetilde{\mathcal T}_t(c)$ and $G:=G_t(c)$.
Recall from \eqref{eq:layered-width} that
\[
g_t(z)=\Psi_{\mathcal C_l}(z)
-\sum_{i\in\mathcal T}\phi_c(z)^\top G^{-1}\phi_c(z_i^1)
\Psi_{\mathcal C_l}(z_i^1).
\]
Taking the inner product with $\mathcal B^1V$ gives
\begin{equation}
\langle\mathcal B^1V,g_t(z)\rangle_{\mathcal H_0}
=\langle\mathcal B^1V,\Psi_{\mathcal C_l}(z)\rangle_{\mathcal H_0}
-\phi_c(z)^\top G^{-1}\sum_{i\in\mathcal T}\phi_c(z_i^1)
\langle\mathcal B^1V,\Psi_{\mathcal C_l}(z_i^1)\rangle_{\mathcal H_0}.
\label{eq:appendix-remainder-correction}
\end{equation}
For $z=(s_t^1,a)$ and $V=V_f^2$, we expand
\eqref{eq:optimistic-score} and \eqref{eq:layered-backup-estimate},
then use \eqref{eq:appendix-remainder-correction} to get:
\begin{equation}
\begin{aligned}
&\widehat Q_{t,j_*}(a,f)-\bar b_t(z)-\widehat{\mathcal P}_t[V](z)
-\langle\mathcal B^1V,g_t(z)\rangle_{\mathcal H_0}\\
=&\varrho_{j_*}(z,V)+\phi_c(z)^\top G^{-1}
\sum_{i\in\mathcal T}\phi_c(z_i^1)[Y_i^1(V)-\varrho_{j_*}(z_i^1,V)]\\
&\ -\phi_c(z)^\top G^{-1}\sum_{i\in\mathcal T}\phi_c(z_i^1)Y_i^1(V)
-\langle\mathcal B^1V,g_t(z)\rangle_{\mathcal H_0}\\
=&\varrho_{j_*}(z,V)-\langle\mathcal B^1V,\Psi_{\mathcal C_l}(z)\rangle_{\mathcal H_0}\\
&\ -\phi_c(z)^\top G^{-1}\sum_{i\in\mathcal T}\phi_c(z_i^1)
[\varrho_{j_*}(z_i^1,V)-\langle\mathcal B^1V,\Psi_{\mathcal C_l}(z_i^1)\rangle_{\mathcal H_0}].
\end{aligned}
\label{eq:appendix-remainder-score-error}
\end{equation}
The second equality cancels the two response sums and substitutes
\eqref{eq:appendix-remainder-correction}, grouping terms with the same
regression coefficient.
Applying the triangle inequality to
\eqref{eq:appendix-remainder-score-error}, then using
\eqref{eq:appendix-true-remainder-net-error} and
\eqref{eq:regression-coefficient-sum} for the epoch history, we obtain
\begin{equation}
\begin{aligned}
&\left|\widehat Q_{t,j_*}(a,f)-\bar b_t(z)-\widehat{\mathcal P}_t[V](z)
-\langle\mathcal B^1V,g_t(z)\rangle_{\mathcal H_0}\right|\\
&\le\varepsilon+\varepsilon\sum_{i\in\mathcal T}
|\phi_c(z)^\top G^{-1}\phi_c(z_i^1)|\le\varepsilon(1+\sqrt{|\mathcal T|}\,w_t(z)).
\end{aligned}
\label{eq:appendix-adjusted-remainder-error}
\end{equation}
On the event of Lemma~\ref{lem:initial-confidence}, the triangle inequality,
\eqref{eq:initial-confidence}, and
\eqref{eq:appendix-adjusted-remainder-error} give
\[
|\widehat Q_{t,j_*}(a,f)-\bar b_t(z)-(\mathcal B^1V)(z)|
\le\varepsilon(1+\sqrt{|\mathcal T|}\,w_t(z))+b_t(z)=\bar b_t(z).
\]
Adding $\bar b_t(z)$ gives the required interval $[0,2\bar b_t(z)]$.
Both \eqref{eq:appendix-true-remainder-net-error} and
\eqref{eq:initial-confidence} are uniform, so this same $j_*$ works
for every visit to $(l,u)$, every $a\in\mathcal A_T$, and every
$f\in\mathcal F_t$.
\end{proof}

\begin{proof}[Proof of Lemma~\ref{lem:first-step-feedback}]
Condition on the history and $s_t^1$, before the model index and
first action are drawn, and write $p:=p_t$.
Fix an action $a\in U$, with $U\in\mathscr U_l$, and a state-value
function $V\in\mathcal V_1$ given this information.
The feedback $\widehat\mu_t(a,V)=h_t(a,a_t^1)Y_t^1(V)$,
with $Y_t^1(V)=R_t^1+V(s_t^2)$, estimates
\[
(\mathcal B^1V)(s_t^1,a)
=r^1(s_t^1,a)+\int_{\mathcal S}V(s')P^1(s'\mid s_t^1,a)\,\mu(ds').
\]
We first prove the first two inequalities for any feasible coefficients;
their existence will be established below.
By conditioning on the executed action, we have
\begin{equation}
\mathbb E_t\widehat\mu_t(a,V)
=\sum_{\substack{x\in U\\p(x)>0}}p(x)h_t(a,x)
\mathbb E_t[Y_t^1(V)\mid a_t^1=x]
=\sum_{x\in U}p(x)h_t(a,x)(\mathcal B^1V)(s_t^1,x).
\label{eq:feedback-conditional-mean}
\end{equation}
The last equality is the observation law; coefficients outside $U$
and at $p(x)=0$ are zero.
Recall that $\varphi_U(x)=(\varphi_{U,1}(x),\ldots,\varphi_{U,d_U}(x))^\top$,
whose coordinates form a basis for the restrictions to $U$ of
polynomials of total degree at most $j_l$.
Since $U$ is finite, these restrictions form a closed subspace of
$\mathbb R^U$, so the infimum in the hypothesis is attained.
We can therefore choose $b\in\mathbb R^{d_U}$ with
$\max_{x\in U}|(\mathcal B^1V)(s_t^1,x)-\varphi_U(x)^\top b|\le\epsilon$.

The optimization in \eqref{eq:minimum-variance-feedback} requires
$\sum_{x\in U}p(x)h_t(a,x)\varphi_U(x)=\varphi_U(a)$.
Using this constraint, we recover the polynomial part exactly:
\begin{equation}
\sum_{x\in U}p(x)h_t(a,x)\varphi_U(x)^\top b
=\left[\sum_{x\in U}p(x)h_t(a,x)\varphi_U(x)\right]^\top b
=\varphi_U(a)^\top b.
\label{eq:feedback-polynomial-identity}
\end{equation}
Subtracting \eqref{eq:feedback-polynomial-identity} from
\eqref{eq:feedback-conditional-mean} leaves only approximation errors:
\[
\begin{aligned}
\mathbb E_t\widehat\mu_t(a,V)-(\mathcal B^1V)(s_t^1,a)
&=\sum_{x\in U}p(x)h_t(a,x)
[(\mathcal B^1V)(s_t^1,x)-\varphi_U(x)^\top b]\\
&\ -[(\mathcal B^1V)(s_t^1,a)-\varphi_U(a)^\top b].
\end{aligned}
\]
Each bracketed error has absolute value at most $\epsilon$, so we obtain
\[
\left|\mathbb E_t\widehat\mu_t(a,V)-(\mathcal B^1V)(s_t^1,a)\right|
\le\epsilon\left(1+\sum_{x\in U}p(x)|h_t(a,x)|\right)
\le(1+2\sqrt{d_U})\epsilon.
\]
The first inequality is the triangle inequality; the second uses
the absolute-sum constraint in \eqref{eq:minimum-variance-feedback}.

For the second inequality in the lemma, $R_t^1\in[0,1]$ and
$0\le V(s)\le H-1$ for every $s\in\mathcal S$ give
$0\le Y_t^1(V)\le H$. Thus, we have
\[
\mathbb E_t\widehat\mu_t(a,V)^2
=\sum_{\substack{x\in U\\p(x)>0}}p(x)h_t(a,x)^2
\mathbb E_t[Y_t^1(V)^2\mid a_t^1=x]
\le H^2\sum_{x\in U}p(x)h_t(a,x)^2.
\]
Thus the objective controls the feedback's second moment.

Now, we construct feasible coefficients to prove the remaining two inequalities.
We verify their constraints and bound their objective values.
Since the algorithm minimizes $\sum_{x\in U}p(x)h_t(a,x)^2$
over the same constraints, each feasible choice provides an upper bound
on this minimum.
Recall that $L_l=|\mathscr U_l|$ counts the action blocks and
$r_l=L_l\binom{d_A+j_l}{j_l}$. By \eqref{eq:first-step-tuning},
we have $0<\gamma\le1/2$; only this range is needed here.
The recommendation distribution is
$q(x):=\sum_j\omega_{t,j}\mathbf1\{a_t(j)=x\}$, whereas
\eqref{eq:first-step-sampling} executes actions with probabilities
$p(x)=(1-\gamma)q(x)+(\gamma/L_l)\pi_U^{\rm G}(x)$ for $x\in U$.
By \eqref{eq:feedback-design}, the exploration matrix
$\Sigma_U=\sum_{x\in U}\pi_U^{\rm G}(x)\varphi_U(x)\varphi_U(x)^\top$
is positive definite and satisfies
$\varphi_U(x)^\top\Sigma_U^{-1}\varphi_U(x)\le d_U$ for every $x\in U$.

\paragraph{Proof of the third inequality.}
We construct coefficients that are feasible for every action $a\in U$.
Set $\widetilde p(x):=(\gamma/L_l)\pi_U^{\rm G}(x)$ and
$A:=\sum_{x\in U}\widetilde p(x)\varphi_U(x)\varphi_U(x)^\top
=(\gamma/L_l)\Sigma_U$.
Since $\Sigma_U\succ0$, we have $A\succ0$ and can define
\begin{equation}
\widetilde h(a,x):=\frac{\widetilde p(x)}{p(x)}
\varphi_U(a)^\top A^{-1}\varphi_U(x),\ x\in U,\ p(x)>0.
\label{eq:feedback-comparison-coefficients}
\end{equation}
Set $\widetilde h(a,x)=0$ outside $U$ and when $p(x)=0$.
We now verify the three constraints in
\eqref{eq:minimum-variance-feedback}.

First, $0\le\widetilde p(x)\le p(x)$ implies
$\widetilde p(x)=0$ whenever $p(x)=0$. Substituting the coefficients,
we obtain
\begin{equation}
\sum_{x\in U}p(x)\widetilde h(a,x)\varphi_U(x)
=\left[\sum_{x\in U}\widetilde p(x)\varphi_U(x)\varphi_U(x)^\top\right]
A^{-1}\varphi_U(a)
=AA^{-1}\varphi_U(a)=\varphi_U(a).
\label{eq:feedback-feature-identity}
\end{equation}
The first equality uses the symmetry of $A^{-1}$; the second uses
its defining covariance, so the equality constraint holds.
The same covariance also gives the quadratic identity
\begin{equation}
\begin{aligned}
\sum_{x\in U}\widetilde p(x)
[\varphi_U(a)^\top A^{-1}\varphi_U(x)]^2
&=\varphi_U(a)^\top A^{-1}
\left[\sum_{x\in U}\widetilde p(x)\varphi_U(x)\varphi_U(x)^\top\right]
A^{-1}\varphi_U(a)
\\
&=\varphi_U(a)^\top A^{-1}\varphi_U(a).
\end{aligned}
\label{eq:feedback-covariance-identity}
\end{equation}
For the absolute-sum constraint, the Cauchy--Schwarz inequality and
\eqref{eq:feedback-covariance-identity} give
\begin{equation}
\begin{aligned}
&\sum_{x\in U}p(x)|\widetilde h(a,x)|
=\sum_{x\in U}\widetilde p(x)|\varphi_U(a)^\top A^{-1}\varphi_U(x)|\\
&\le\left(\sum_{x\in U}\widetilde p(x)\right)^{1/2}
\left(\sum_{x\in U}\widetilde p(x)
[\varphi_U(a)^\top A^{-1}\varphi_U(x)]^2\right)^{1/2}
=\sqrt{\left(\sum_{x\in U}\widetilde p(x)\right)
\varphi_U(a)^\top A^{-1}\varphi_U(a)}.
\end{aligned}
\label{eq:feedback-absolute-sum}
\end{equation}
Here $\sum_x\widetilde p(x)=\gamma/L_l$ and
$A=(\gamma/L_l)\Sigma_U$, so we have
\[
\left(\sum_xp(x)|\widetilde h(a,x)|\right)^2
\le\frac{\gamma}{L_l}\varphi_U(a)^\top
\left(\frac{\gamma}{L_l}\Sigma_U\right)^{-1}\varphi_U(a)
=\varphi_U(a)^\top\Sigma_U^{-1}\varphi_U(a)\le d_U.
\]
The last inequality is \eqref{eq:feedback-design}; hence
$\sum_xp(x)|\widetilde h(a,x)|\le\sqrt{d_U}\le2\sqrt{d_U}$.

To verify the coordinate constraint, we use
$A\succeq(\gamma/L_l)\Sigma_U$ and the fact that inversion reverses
the positive definite matrix order to get
\begin{equation}
\varphi_U(x)^\top A^{-1}\varphi_U(x)
\le\frac{L_l}{\gamma}\varphi_U(x)^\top\Sigma_U^{-1}\varphi_U(x)
\le\frac{L_ld_U}{\gamma},\ x\in U.
\label{eq:feedback-quadratic-bound}
\end{equation}
The second inequality is \eqref{eq:feedback-design}.
Using $\widetilde p/p\le1$ and the Cauchy--Schwarz inequality
in the $A^{-1}$ inner product, we obtain
\begin{equation}
|\widetilde h(a,x)|
\le\sqrt{[\varphi_U(a)^\top A^{-1}\varphi_U(a)]
[\varphi_U(x)^\top A^{-1}\varphi_U(x)]}
\le\frac{L_ld_U}{\gamma}\le\frac{r_l}{\gamma}.
\label{eq:feedback-coordinate-bound}
\end{equation}
The second inequality uses \eqref{eq:feedback-quadratic-bound};
the last uses $d_U\le\binom{d_A+j_l}{j_l}$.
Thus all three constraints hold.

We now bound the objective value of this feasible choice:
\begin{equation}
\begin{aligned}
\sum_{x\in U}p(x)\widetilde h(a,x)^2
&=\sum_{\substack{x\in U\\p(x)>0}}
\frac{\widetilde p(x)^2}{p(x)}
[\varphi_U(a)^\top A^{-1}\varphi_U(x)]^2\\
&\le\sum_{x\in U}\widetilde p(x)
[\varphi_U(a)^\top A^{-1}\varphi_U(x)]^2
=\varphi_U(a)^\top A^{-1}\varphi_U(a).
\end{aligned}
\label{eq:feedback-comparison-identities}
\end{equation}
The inequality uses $\widetilde p^2/p\le\widetilde p$ on $\{p>0\}$;
the last equality is \eqref{eq:feedback-covariance-identity}.
The constraints in \eqref{eq:minimum-variance-feedback} define a closed
subset of $[-r_l/\gamma,r_l/\gamma]^{|U|}$, and the constructed
coefficients make it nonempty. The continuous objective therefore attains
its minimum. 

The feasible-set graph and objective are Borel in $p$;
Lemma~\ref{lem:measurable-maximum} supplies a measurable minimizer.
Let $v(a):=\sum_xp(x)h_t(a,x)^2$ denote this minimum. We then have
\begin{equation}
v(a)\le\sum_xp(x)\widetilde h(a,x)^2
\le\varphi_U(a)^\top A^{-1}\varphi_U(a)
\le\frac{L_ld_U}{\gamma}\le\frac{r_l}{\gamma}.
\label{eq:feedback-pointwise-objective}
\end{equation}
The first inequality uses minimality, the second is
\eqref{eq:feedback-comparison-identities}, and the third is
\eqref{eq:feedback-quadratic-bound}. This proves the third inequality.

\paragraph{Proof of the fourth inequality.}
We must control $\sum_{a\in U}q(a)v(a)$ on each action block $U$.
By \eqref{eq:feedback-pointwise-objective}, we can write
$v(a)=\int_0^{L_ld_U/\gamma}\mathbf1\{v(a)>y\}\,dy$.
Multiplying by $q(a)$ and interchanging the finite sum and integral, we obtain
\begin{equation}
\sum_{a\in U}q(a)v(a)
=\int_0^{L_ld_U/\gamma}\sum_{a\in U:v(a)>y}q(a)\,dy.
\label{eq:feedback-objective-integral}
\end{equation}
To obtain logarithmic dependence on $1/\gamma$, we bound the probability
inside this integral.

Fix $y>0$ and put $E:=\{a\in U:v(a)>y\}$ and
$q(E):=\sum_{a\in E}q(a)$.
When $q(E)>0$, we use the explicit weights
$\widetilde p(x):=(1-\gamma)q(x)\mathbf1\{x\in E\}
+(\gamma/L_l)\pi_U^{\rm G}(x)$, define
$A:=\sum_{x\in U}\widetilde p(x)\varphi_U(x)\varphi_U(x)^\top$,
and define $\widetilde h$ by \eqref{eq:feedback-comparison-coefficients}
with these $\widetilde p$ and $A$.

Here $0\le\widetilde p\le p$ and
$A\succeq(\gamma/L_l)\Sigma_U\succ0$, so
\eqref{eq:feedback-feature-identity} and \eqref{eq:feedback-coordinate-bound}
verify the equality and coordinate constraints.
The calculations \eqref{eq:feedback-absolute-sum} and
\eqref{eq:feedback-comparison-identities} apply to these weights as well.
We need an action in $E$ satisfying the absolute-sum constraint; to find it,
we use the weighted average
\begin{equation}
\begin{aligned}
(1-\gamma)\sum_{a\in E}q(a)\varphi_U(a)^\top A^{-1}\varphi_U(a)
&=\operatorname{tr}\left[A^{-1}(1-\gamma)
\sum_{a\in E}q(a)\varphi_U(a)\varphi_U(a)^\top\right]\\
&=\operatorname{tr}\left[A^{-1}\left(A-\frac{\gamma}{L_l}\Sigma_U\right)\right]
=d_U-\frac{\gamma}{L_l}\operatorname{tr}(A^{-1/2}\Sigma_U A^{-1/2})\le d_U.
\end{aligned}
\label{eq:feedback-leverage-sum}
\end{equation}
The first equality writes each quadratic form as a trace; the next two
use the definition of $A$ and cyclicity of trace.
The inequality uses $A^{-1/2}\Sigma_U A^{-1/2}\succeq0$.
Dividing \eqref{eq:feedback-leverage-sum} by $(1-\gamma)q(E)$
bounds the average with weights $q(a)/q(E)$, which sum to one.
Thus some $a\in E$ with $q(a)>0$ satisfies
\begin{equation}
\varphi_U(a)^\top A^{-1}\varphi_U(a)
\le\frac{d_U}{(1-\gamma)q(E)}.
\label{eq:feedback-selected-action}
\end{equation}
For this action, \eqref{eq:feedback-absolute-sum} gives
\[
\left(\sum_xp(x)|\widetilde h(a,x)|\right)^2
\le\left((1-\gamma)q(E)+\frac{\gamma}{L_l}\right)
\varphi_U(a)^\top A^{-1}\varphi_U(a)
\le d_U+\frac{\gamma}{L_l}\frac{L_ld_U}{\gamma}=2d_U.
\]
The second inequality uses \eqref{eq:feedback-selected-action} for the
$(1-\gamma)q(E)$ term and \eqref{eq:feedback-quadratic-bound} for the
$\gamma/L_l$ term.
Hence, the absolute-sum constraint holds since
$\sum_xp(x)|\widetilde h(a,x)|\le\sqrt{2d_U}\le2\sqrt{d_U}$.
Together with \eqref{eq:feedback-feature-identity} and
\eqref{eq:feedback-coordinate-bound}, this verifies the feasibility.

Combining $a\in E$ with minimality and
\eqref{eq:feedback-comparison-identities}, we obtain
\[
y<v(a)\le\sum_xp(x)\widetilde h(a,x)^2
\le\varphi_U(a)^\top A^{-1}\varphi_U(a)
\le\frac{d_U}{(1-\gamma)q(E)}.
\]
The first inequality uses $a\in E$; the last is
\eqref{eq:feedback-selected-action}. Rearranging and including $q(E)=0$ gives that
\begin{equation}
q(E)\le\min\{1,d_U/[(1-\gamma)y]\}.
\label{eq:feedback-objective-tail}
\end{equation}

Substituting \eqref{eq:feedback-objective-tail} into
\eqref{eq:feedback-objective-integral}, we obtain
\[
\begin{aligned}
\sum_{a\in U}q(a)v(a)
&\le\int_0^{d_U/(1-\gamma)}1\,dy
+\int_{d_U/(1-\gamma)}^{L_ld_U/\gamma}
\frac{d_U}{(1-\gamma)y}\,dy\\
&=\frac{d_U}{1-\gamma}\left[1+\ln\frac{L_l(1-\gamma)}{\gamma}\right]
\le2d_U\ln(\mathrm eL_l/\gamma).
\end{aligned}
\]
The first inequality uses the two parts of
\eqref{eq:feedback-objective-tail}, split at $d_U/(1-\gamma)$;
this point lies in the integration interval
because $L_l(1-\gamma)\ge\gamma$.
The next equality evaluates the integrals, and the last inequality uses
$1/(1-\gamma)\le2$ and $1-\gamma\le1$.
Finally, by the definition of $q$ and the zero coefficients outside
each action's block, we obtain
\[
\sum_{j=1}^M\omega_{t,j}\sum_{x\in\mathcal A_T}p(x)h_t(a_t(j),x)^2
=\sum_{U\in\mathscr U_l}\sum_{a\in U}q(a)v(a)
\le2\ln(\mathrm eL_l/\gamma)\sum_{U\in\mathscr U_l}d_U
\le2r_l\ln(\mathrm eL_l/\gamma).
\]
The last inequality uses $\sum_Ud_U\le L_l\binom{d_A+j_l}{j_l}=r_l$.
This proves the fourth inequality. The minimizing coefficients are feasible,
so the first two inequalities established at the start apply to them as well.
\end{proof}

\section{Proof of Theorem~\ref{thm:regret-upper}}
\label{app:section-5-proofs}
\label{app:upper-bound}

The decomposition \eqref{eq:upper-final-decomposition} reduces regret
to first-step selection and continuation evaluation errors.
Propositions~\ref{prop:upper-first-step} and~\ref{prop:upper-continuation-error}
control these contributions, respectively. We state the inputs to each
proposition below; their proofs are in Appendix~\ref{app:upper-auxiliary-proofs}.

Fix a model, an admissible initial-state adversary, and the budget
$T\ge2$. We retain $\rho$, $\mathfrak B$, and $\iota_T$ from the
algorithm and theorem. A constant $C$ is numerical, whereas
$C_{\nu,\kappa}$ depends only on $\nu,\kappa$; neither hides a
dimension or model-norm factor.

\paragraph{First-step selection.}
We first count visits and regression cells. The visit bound is used
in Lemma~\ref{lem:upper-action-approximation} to sum approximation errors;
the total cell count enters Lemma~\ref{lem:upper-remainder-cover}.
The occupied-cell and level counts also control the sums of regression
and epoch errors in Proposition~\ref{prop:upper-first-step}.

\begin{lemma}
\label{lem:upper-partition-counts}
With $S_{\max}$ as in Section~\ref{subsec:state-action-partition}, recall
$d=d_S+d_A$, $S_l=\min\{2^{l/d},S_{\max}\}$, and
$N_l=\lceil T(S_l/S_{\max})^{2\nu}\rceil$.
For each level $l$, let $T_l:=|\{t\le T:l_t=l\}|$ count its
selected episodes, and let $P_l$ count the state regression cells
in $\mathcal P_l$ visited by these episodes.
Every counting cell in $\mathcal D_l$ receives at most $N_l$ visits
selected at level $l$, and
\[
T_l\le\min\{T,S_l^{d_S}N_l\},\ P_0\le1,\
P_l\le\min\{S_l^{d_S},(d_S+1)T/N_{l-1}\},\ l\ge1.
\]
Moreover, $|\mathcal C_l|\le S_l^d$ and
$L_{\max}+1\le C_\nu d\iota_T$.
\end{lemma}

By Lemma~\ref{lem:first-step-feedback}, a local polynomial error
$\epsilon$ gives feedback bias at most $(1+2\sqrt{d_U})\epsilon$.
Lemma~\ref{lem:upper-action-approximation} uses the visit bound in
Lemma~\ref{lem:upper-partition-counts} to control the cumulative error
$\sum_{t:l_t=l}\epsilon_l(s_t^1)$. It also bounds the polynomial dimension
$r_l/L_l$, which enters the feedback bounds in
Lemma~\ref{lem:upper-epoch-comparison}.

\begin{lemma}
\label{lem:upper-action-approximation}
Fix a level $l$. For each initial state $s$, let $\epsilon_l(s)$
be the largest best action-polynomial approximation error over
the continuation values $V\in\mathcal V_1$ and action blocks used by the algorithm:
\[
\epsilon_l(s):=\sup_{V\in\mathcal V_1}\max_{U\in\mathscr U_l}
\inf_{b\in\mathbb R^{d_U}}\max_{a\in U}
\left|(\mathcal B^1V)(s,a)-\varphi_U(a)^\top b\right|.
\]
The cumulative error on episodes assigned to this level satisfies
\begin{equation}
\begin{aligned}
\sum_{t:l_t=l}\epsilon_l(s_t^1)
&\le C_{\nu,\kappa}\bar B\sqrt{d_S+1}\,d_A^{\nu/2}
S_l^{-\nu}\sqrt{N_lT_l}+T_lT^{-2},\ &d_A<4\nu,\\
\sum_{t:l_t=l}\epsilon_l(s_t^1)
&\le C_{\nu,\kappa}\bar B d_A^{\nu/2}
\left(\sqrt{d_S+1}\,S_l^{-\nu}\sqrt{N_lT_l}
+T_lT^{-\nu/(2\nu+d_A)}\right),\ &d_A\ge4\nu.
\end{aligned}
\label{eq:upper-first-approximation-sum}
\end{equation}
The corresponding number of polynomial features per block obeys
\[
\begin{aligned}
\frac{r_l}{L_l}=\binom{d_A+j_l}{j_l}
&\le C_\nu\iota_T^{d_A},\ &d_A<4\nu,\\
\frac{r_l}{L_l}=\binom{d_A+j_l}{j_l}
&\le C_\nu(2+d)^{\lceil2\nu\rceil},\ &d_A\ge4\nu.
\end{aligned}
\]
\end{lemma}

The epoch comparison in Lemma~\ref{lem:upper-epoch-comparison} also
depends on $\ln M$, where $M$ counts the models in an $\varepsilon$-net
of $\mathcal R_l$. We bound it using the cell count from
Lemma~\ref{lem:upper-partition-counts} and the continuation entropy bound
\eqref{eq:continuation-entropy-bound}. Proposition~\ref{prop:upper-first-step}
will substitute this bound into the epoch analysis.

\begin{lemma}
\label{lem:upper-remainder-cover}
For every level $l$ and $\varepsilon>0$, the class $\mathcal R_l$
has an internal $\varepsilon$-net in the uniform norm over
$z\in\mathcal Z$ and $V\in\mathcal V_1$. Its minimum cardinality $M$
satisfies
\begin{equation}
\ln M\le C_{\nu,\kappa}\left[
1+\frac{\bar B^2S_l^{-2\nu}}{\varepsilon^2}
(2+d)^{\rho+9\nu}
\left(d\iota_T+\frac{J^2}{(H-1)^2}\right)\right].
\label{eq:upper-remainder-cover}
\end{equation}
The net is fixed independently of the observations.
\end{lemma}

On the event of Lemma~\ref{lem:initial-confidence},
Lemma~\ref{lem:finite-remainder-comparison} supplies a fixed optimistic
model $j_*$ within each epoch. Using the feedback bias and second-moment bounds in
Lemma~\ref{lem:first-step-feedback}, we analyze
\eqref{eq:expert-weight-update} to compare the selected recommendations
with $j_*$. The resulting epoch bound keeps the cell count, model count
$M$, and action-approximation errors explicit.

\begin{lemma}
\label{lem:upper-epoch-comparison}
Fix a level $l$ and epoch $u$ with capacity $N=2^u$.
Use the algorithm's $M$ remainder models, net accuracy $\varepsilon$,
and update parameters, and let $P_{l,u}$ count the joint regression
cells occupied in this epoch. The action-approximation error is
$\epsilon_l$ from Lemma~\ref{lem:upper-action-approximation}. There is an event of probability at least
$1-\delta_{l,u}$ on whose intersection with the event of
Lemma~\ref{lem:initial-confidence}, the following bound holds simultaneously after every visit to this epoch,
before the first empty version space:
\begin{equation}
\begin{aligned}
&\sum_{t:(l_t,u_t)=(l,u)}
\left[\sup_{a\in\mathcal A_T,\ f\in\mathcal F_t}
(\mathcal B^1V_f^2)(s_t^1,a)
-(\mathcal B^1V_{f_t}^2)(s_t^1,a_t^1)\right]\\
&\le C\Bigg[
\overline\beta(|\mathcal C_l|,\delta_{l,u})
\sqrt{DP_{l,u}N\ln(1+N/D)}
+\varepsilon N\left(1+\sqrt{2D\ln(1+N/D)}\right)\\
&
+H\sqrt{r_lN\ln\frac{8L_lN}{\delta_{l,u}}
\ln\frac{2M}{\delta_{l,u}}}
+\left(1+\sqrt{r_l/L_l}\right)
\sum_{t:(l_t,u_t)=(l,u)}\epsilon_l(s_t^1)\Bigg].
\end{aligned}
\label{eq:upper-epoch-comparison}
\end{equation}
Both sums include exactly the visits to this epoch up to and including the visit under consideration.
\end{lemma}

Proposition~\ref{prop:upper-first-step} combines the epoch analysis of
Lemma~\ref{lem:upper-epoch-comparison} with three estimates:
Lemma~\ref{lem:upper-partition-counts} controls visits and occupied cells,
Lemma~\ref{lem:upper-action-approximation} bounds the cumulative approximation errors and
polynomial dimensions, and Lemma~\ref{lem:upper-remainder-cover} bounds
$\ln M$. The discretization estimate \eqref{eq:first-action-grid-error}
then extends the comparison from $\mathcal A_T$ to $\mathcal A$.

\begin{proposition}
\label{prop:upper-first-step}
With probability at least $1-\delta/4$, for every $q\le T$ such that
$\mathcal F_t\ne\varnothing$ for all $t\le q$,
\begin{equation}
\begin{aligned}
&\sum_{t=1}^q\left[
\sup_{a\in\cA,\ f\in\mathcal F_t}(\mathcal B^1V_f^2)(s_t^1,a)
-(\mathcal B^1V_{f_t}^2)(s_t^1,a_t^1)\right]\le C_{\nu,\kappa}\mathfrak B^2
(2+d)^{\rho+5\nu+6}\iota_T^{2+2\nu}T^{e_*}.
\end{aligned}
\label{eq:upper-first-step-total}
\end{equation}
The bound uses the fixed episode budget $T$ also when $q<T$.
\end{proposition}

\paragraph{Continuation policies.}
We need to bound residuals under the current policy, whereas the
version-space tests constrain earlier observations.
Lemma~\ref{lem:upper-historical-residual} extracts two inputs for
Lemma~\ref{lem:upper-residual-conversion}: a historical squared-excess
bound and a bound on the cumulative clipped thresholds.
For these continuation bounds, define integration against the stage-$h$
occupancy after the first action and candidate Q-function are selected by
\begin{equation}
\mathbb E_i^h\varphi:=\int\varphi(z)\,\mathbb P(z_i^h\in dz\mid\mathcal I_i^1).
\label{eq:historical-occupancy-integral}
\end{equation}
A candidate $f$
chosen later is held fixed inside $\mathbb E_i^h\varphi_f$; this is
an integral against a historical occupancy, not the conditional
expectation of a retrospectively selected random function.

\begin{lemma}
\label{lem:upper-historical-residual}
Define the threshold
$\sigma_i^h(z):=\min\{H,2\tau_i^h(z)\}$.
Put $\ell_\delta:=\ln(8HT/\delta)$.
With probability at least $1-\delta/2$, the squared excess of
$|\Delta_f^h|$ over $\sigma_i^h$ satisfies, simultaneously for
$2\le h\le H$, $t\in[T]$, and $f\in\mathcal F_t$,
\begin{equation}
\sum_{i<t}\mathbb E_i^h
\bigl[(|\Delta_f^h(z_i^h)|-\sigma_i^h(z_i^h))_+^2\bigr]
\le C\left[J^2+H^2\left(1+
\frac{B_{\rm res}^2}{\alpha_{\rm vs}^2}+\ell_\delta\right)\right].
\label{eq:upper-historical-residual}
\end{equation}
On the same event, the sum $\sum_{i=1}^T\mathbb E_i^h\sigma_i^h(z_i^h)$
satisfies the following bound, where $M:=|\mathcal C|$ and
$\beta:=\beta(M,\delta/8)$:
\begin{equation}
\sum_{i=1}^T\mathbb E_i^h \sigma_i^h(z_i^h)
\le C\left[(B_{\phi}H+\beta)\sqrt{MTK_{\rm c}}
+\alpha_{\rm vs}q_0A_{\rm vs}(1+\sqrt{K_{\rm c}})T
+H\ell_\delta\right].
\label{eq:upper-slack-sum}
\end{equation}
Here and below, $C$ denotes a numerical constant.
\end{lemma}

A second input to Lemma~\ref{lem:upper-residual-conversion} is a
finite-dimensional action approximation with small state-integrated
error. By Lemma~\ref{lem:state-occupancy-bound}, later-stage state densities
satisfy $p\le B_p$, so we have
\[
\int p(s)\varepsilon_g(s)\,d\mu(s)
\le\|p\|_{L^2(\mu)}\|\varepsilon_g\|_{L^2(\mu)}
\le\sqrt{B_p}\|\varepsilon_g\|_{L^2(\mu)},
\]
by the Cauchy--Schwarz inequality and $\int p\,d\mu=1$.
Thus Lemma~\ref{lem:upper-continuation-approximation} constructs one
error envelope valid for every action and bounded in $L^2(\mu)$.
Its proof reuses \eqref{eq:upper-action-polynomial} and
\eqref{eq:upper-action-taylor-bound} from the proof of
Lemma~\ref{lem:upper-action-approximation}, replacing the cellwise error
bound by an $L^2(\mu)$ estimate.

\begin{lemma}
\label{lem:upper-continuation-approximation}
Use $R$ and $F_{\rm vs}$ from \eqref{eq:version-space-resolution}
and \eqref{eq:vs-action-resolution}. For every $g\in\mathcal H_0$
with $\|g\|_{\mathcal H_0}\le B$, there is a function $p_g(s,a)$,
polynomial in $a$ on each block of $\mathscr U(F_{\rm vs})$,
and a nonnegative measurable envelope $\varepsilon_g$ such that
\[
\sup_{a\in\cA}|g(s,a)-p_g(s,a)|\le\varepsilon_g(s),\ s\in\cS.
\]
Let $r$ denote the largest dimension of the polynomial spaces used on the action blocks.

If $d_A<4\nu$ or $d_S\ge d_A$, then
\[
r\le[C\ln(\mathrm eT)]^{d_A},\
\|\varepsilon_g\|_{L^2(\mu)}\le B(R^{-\nu_S}+T^{-2}),
\]
where $\nu_S=\nu+d_S/2$ and
$a_{\rm mix}\le C_{\nu,\kappa}(2+d)$.
If $d_A\ge4\nu$ and $d_S<d_A$, they satisfy
\[
r\le C_\nu(2+d)^{\lceil2\nu\rceil},\
\|\varepsilon_g\|_{L^2(\mu)}
\le C_{\nu,\kappa}B(2+d)^{2\nu}T^{-\nu/(2\nu+d_A)}.
\]
The same envelope controls every action block; its norm is taken
after maximizing over actions, without summing over blocks.
\end{lemma}

In the application below, we take $g_t=\Delta_{f_t}^h$ and
$\sigma_i(z)=\sigma_i^h(z)=\min\{H,2\tau_i^h(z)\}$.
Recall from \eqref{eq:innovation-threshold} that, for $z$ in regression cell $c$,
\[
\tau_i^h(z)=\alpha_{\rm vs}q_0A_{\rm vs}
\bigl(1+\sqrt{n_i^h(c)}\,w_i^h(z)\bigr)
+(\overline\beta+\beta)w_i^h(z).
\]
The function $\tau_i^h$ is the algorithm's test threshold.
The next lemma converts historical squared-excess bounds into
$\sum_t\mathbb E_t^h|g_t(z_t^h)|$ under the current occupancies.
It is a deterministic implication for any functions satisfying its
hypotheses. In Proposition~\ref{prop:upper-continuation-error},
Lemma~\ref{lem:upper-historical-residual} supplies the historical budget
and threshold sum, while Lemma~\ref{lem:upper-continuation-approximation}
supplies the polynomial dimensions and approximation errors.

\begin{lemma}
\label{lem:upper-residual-conversion}
Fix $h\ge2$ and a realized policy sequence with state densities
$p_i\le B_p$. For the occupancies in
\eqref{eq:historical-occupancy-integral}, suppose
$g_t:\mathcal S\times\mathcal A\to[-H,H]$ and
$\sigma_i:\mathcal S\times\mathcal A\to[0,\infty)$ satisfy
\[
\sum_{i<t}\mathbb E_i^h
[(|g_t(z_i^h)|-\sigma_i(z_i^h))_+^2]\le b,\ t\le T.
\]
Fix a partition into $L$ action rectangles, each with a polynomial
space of dimension at most $r\ge1$. Suppose each $g_t$ has measurable
approximants in these spaces whose error, uniformly over actions,
is bounded by a nonnegative measurable $\varepsilon_t(s)$ with
$\|\varepsilon_t\|_{L^2(\mu)}\le E$.
Then, with $G_T:=1+\lceil\log_2(2T)\rceil$ and
$\Omega_T:=\ln(\mathrm e+LB_pT)$,
\begin{equation}
\begin{aligned}
\sum_{t=1}^T\mathbb E_t^h|g_t(z_t^h)|
\le C\bigg[&B_pHLr\ln(1+T)
+G_T\sqrt{B_pLrT(b+H^2)\ln(1+T)}\\
&+\sqrt{r\Omega_T}\sum_{i=1}^T\mathbb E_i^h\sigma_i(z_i^h)
+\sqrt{B_pr\Omega_T}\,ET\bigg]+H.
\end{aligned}
\label{eq:upper-residual-conversion}
\end{equation}
The functions may depend on the realized data.
\end{lemma}

We apply Lemma~\ref{lem:upper-residual-conversion} with these inputs
and the density bound from Lemma~\ref{lem:state-occupancy-bound}.
To evaluate the Gaussian test scale, we also reuse
\eqref{eq:upper-remainder-feature-cover}--\eqref{eq:upper-remainder-scalar-entropy}
from the proof of Lemma~\ref{lem:upper-remainder-cover}.
Nonemptiness is established separately: the event of
Lemma~\ref{lem:local-ridge} included in the proof of
Lemma~\ref{lem:upper-historical-residual}, together with
Lemma~\ref{lem:uniform-confidence}, ensures that both version-space tests
\eqref{eq:coarse-feasibility} and~\eqref{eq:version-space} retain $Q_*$.

\begin{proposition}
\label{prop:upper-continuation-error}
With probability at least $1-5\delta/8$, $\mathcal F_t\ne\varnothing$
for all $t\in[T]$, and \texttt{SOVI} satisfies
\begin{equation}
\sum_{h=2}^H\sum_{t=1}^T
\mathbb E_t^h|\Delta_{f_t}^h(z_t^h)|
\le C_{\nu,\kappa}\mathfrak B^3
(2+d)^{\rho+7\nu+6}\iota_T^{\,2+6\nu}T^{e_*}.
\label{eq:upper-continuation-error}
\end{equation}
\end{proposition}

Finally, with all the lemmas and propositions above, we are ready to prove our main regret guarantee.

\begin{proof}[Proof of Theorem~\ref{thm:regret-upper}]
Starting from \eqref{eq:regret}, we add and subtract the Bellman backup
at the selected first action, where $z_t^1=(s_t^1,a_t^1)$:
\begin{equation}
\begin{aligned}
\reg(T)
&=\sum_{t=1}^T\left[V_*^1(s_t^1)-(\mathcal B^1V_{f_t}^2)(z_t^1)\right]+\sum_{t=1}^T\left[(\mathcal B^1V_{f_t}^2)(z_t^1)-V^1(s_t^1;\pi_t)\right].
\end{aligned}
\label{eq:upper-final-decomposition}
\end{equation}
Propositions~\ref{prop:upper-first-step} and~\ref{prop:upper-continuation-error}
control these two sums. We work on the intersection of the events
constructed in their proofs. The latter proof also establishes
$Q_*\in\mathcal F_t$ for every $t$, so all version spaces are nonempty
and Proposition~\ref{prop:upper-first-step} applies with $q=T$.

For the first sum, \eqref{eq:bellman-optimality},
\eqref{eq:bellman-backup}, and $V_{Q_*}^2=V_*^2$ imply
\[
V_*^1(s_t^1)=\max_{a\in\cA}(\mathcal B^1V_*^2)(s_t^1,a)
\le\sup_{a\in\cA,\ f\in\mathcal F_t}
(\mathcal B^1V_f^2)(s_t^1,a).
\]
The inequality uses $Q_*\in\mathcal F_t$.
Subtracting the selected backup and applying
\eqref{eq:upper-first-step-total} in Proposition \ref{prop:upper-first-step} with $q=T$, we obtain
\[
\begin{aligned}
\sum_{t=1}^T\left[V_*^1(s_t^1)-(\mathcal B^1V_{f_t}^2)(z_t^1)\right]
&\le\sum_{t=1}^T\left[
\sup_{a\in\cA,\ f\in\mathcal F_t}(\mathcal B^1V_f^2)(s_t^1,a)
-(\mathcal B^1V_{f_t}^2)(z_t^1)\right]\\
&\le C_{\nu,\kappa}\mathfrak B^2
(2+d)^{\rho+5\nu+6}\iota_T^{2+2\nu}T^{e_*}.
\end{aligned}
\]

For the second sum, recall from \eqref{eq:historical-occupancy-integral}
that $\mathbb E_t^h\varphi=\int\varphi(z)\,
\mathbb P(z_t^h\in dz\mid\mathcal I_t^1)$.
Given $\mathcal I_t^1$, the selected $f_t,a_t^1$ are fixed, and
$\pi_t$ plays $a_t^1$ before following the greedy policy of $f_t$.
Thus $f_t^h(z_t^h)=V_{f_t}^h(s_t^h)$ for $h\ge2$.
The value terms telescope as
\[
V_{f_t}^2(s_t^2)
=\sum_{h=2}^{H-1}\left[
V_{f_t}^h(s_t^h)-V_{f_t}^{h+1}(s_t^{h+1})\right]
+V_{f_t}^H(s_t^H).
\]
Using this identity, \eqref{eq:bellman-backup}, and
\eqref{eq:observation-law}, we have
\[
\begin{aligned}
(\mathcal B^1V_{f_t}^2)(z_t^1)-V^1(s_t^1;\pi_t)
&=\mathbb E\left[V_{f_t}^2(s_t^2)-\sum_{h=2}^Hr^h(z_t^h)
\,\middle|\,\mathcal I_t^1\right]=\sum_{h=2}^H\mathbb E_t^h
\left[f_t^h(z_t^h)-(\mathcal B^hV_{f_t}^{h+1})(z_t^h)\right]\\
&=\sum_{h=2}^H\mathbb E_t^h\Delta_{f_t}^h(z_t^h)
\le\sum_{h=2}^H\mathbb E_t^h|\Delta_{f_t}^h(z_t^h)|.
\end{aligned}
\]
The first equality cancels the first-stage reward. The second uses
the telescoping identity, the greedy action identity, and the tower
property; the terminal continuation is zero. The third recalls
$\Delta_{f_t}^h=f_t^h-\mathcal B^hV_{f_t}^{h+1}$ from
\eqref{eq:bellman-residual-definition}.

Summing over episodes and applying \eqref{eq:upper-continuation-error} in Proposition \ref{prop:upper-continuation-error}, we have
\[
\begin{aligned}
\sum_{t=1}^T\left[(\mathcal B^1V_{f_t}^2)(z_t^1)-V^1(s_t^1;\pi_t)\right]
&\le\sum_{h=2}^H\sum_{t=1}^T
\mathbb E_t^h|\Delta_{f_t}^h(z_t^h)|\le C_{\nu,\kappa}\mathfrak B^3
(2+d)^{\rho+7\nu+6}\iota_T^{2+6\nu}T^{e_*}.
\end{aligned}
\]

Substituting these two bounds into \eqref{eq:upper-final-decomposition},
we obtain
\[
\reg(T)\le C_{\nu,\kappa}T^{e_*}
\left[\mathfrak B^2(2+d)^{\rho+5\nu+6}\iota_T^{2+2\nu}
+\mathfrak B^3(2+d)^{\rho+7\nu+6}\iota_T^{2+6\nu}\right].
\]
Recall $\rho=\lfloor\nu\rfloor$ from Section~\ref{sec:skovi},
$d=d_S+d_A$ and $e_*=\max\{e_A,e_B\}$ from
Section~\ref{sec:guarantees}, and the definitions
\eqref{eq:regret-parameter-summary}:
\[
\mathfrak B=1+H+B_p+\Lambda(B_r+HB_p),\
\iota_T=\ln\frac{\mathrm e(2+d)T\mathfrak B}{\delta}.
\]
These imply $\mathfrak B,2+d,\iota_T\ge1$, so the coefficients satisfy
\[
\frac{\mathfrak B^2(2+d)^{\rho+5\nu+6}\iota_T^{2+2\nu}}
{\mathfrak B^3(2+d)^{\rho+7\nu+6}\iota_T^{2+6\nu}}
=\frac1{\mathfrak B(2+d)^{2\nu}\iota_T^{4\nu}}\le1.
\]
Their sum is therefore at most twice the second coefficient.
Substituting $\rho=\lfloor\nu\rfloor$ and absorbing the numerical
factor $2$ into $C_{\nu,\kappa}$ yields \eqref{eq:regret-upper}.
The two events fail with total probability at most
$\delta/4+5\delta/8=7\delta/8<\delta$, completing the proof.
\end{proof}

\section{Proof of Theorem~\ref{thm:regret-lower}}
\label{app:lower-bound}

\begin{proof}[Proof of Theorem~\ref{thm:regret-lower}]
Fix an algorithm. To prove \eqref{eq:regret-lower}, we construct
models in which choosing an action with a smaller perturbation value
causes a proportional loss in expected reward. We first realize this
loss in one pair of stages, then repeat the construction across
$\lfloor H/2\rfloor$ pairs.

\paragraph{Encoding the perturbation in the transition.}
Recall $\cS=[0,1]^{d_S}$ and that $\mu$ is Lebesgue measure.
Split the state space into
$\cS^+:=[0,1/2)\times[0,1]^{d_S-1}$ and
$\cS^-:=\cS\setminus\cS^+$, so $\mu(\cS^+)=\mu(\cS^-)=1/2$.
For a scalar perturbation $g:\cZ\to\mathbb R$ with
$\|g\|_\infty\le1/4$, define
\begin{equation}
P_g(s'\mid z):=(1+2g(z))\mathbf1_{\cS^+}(s')
 +(1-2g(z))\mathbf1_{\cS^-}(s').
\label{eq:two-region-transition}
\end{equation}
The factor $2$ makes $g(z)$ the change in the probability of entering
$\cS^+$: directly integrating the density, we have
\[
\int_{\cS^+}P_g(s'\mid z)\,\mu(ds')
=\frac{1+2g(z)}2=\frac12+g(z),\
\int_{\cS^-}P_g(s'\mid z)\,\mu(ds')=\frac12-g(z).
\]
Adding these probabilities gives one, and $\|g\|_\infty\le1/4$
ensures $1/2\le P_g\le3/2$. Thus \eqref{eq:two-region-transition}
defines a transition density.
Put $\overline B_p:=(B_p-\|\mathbf1_{\cZ}\|_{\cH_k})/2>0$.
For fixed $s'$, the section $P_g(s'\mid\cdot)$ is
$\mathbf1_{\cZ}+2g$ or $\mathbf1_{\cZ}-2g$, so we have
\[
\|P_g(s'\mid\cdot)\|_{\cH_k}
\le\|\mathbf1_{\cZ}\|_{\cH_k}+2\|g\|_{\cH_k}
\le\|\mathbf1_{\cZ}\|_{\cH_k}+2\overline B_p=B_p
\]
whenever $\|g\|_{\cH_k}\le\overline B_p$.
This verifies Assumption~\ref{ass:kernel-transition} once the
perturbation norm is bounded.

\paragraph{Turning the transition probability into a reward difference.}
We choose a common reward that is larger on $\cS^+$ than on $\cS^-$.
Fix a smooth function $\chi_0:\mathbb R\to[0,1]$ supported in
$(-1/2,1/2)$ and equal to one on $[-1/4,1/4]$, and put $\tau:=\nu+d/2$.
The smooth compactly supported function
$z\mapsto\prod_{j=1}^d\chi_0((z_j-1/2)/2)$ equals one on $\cZ$.
By Lemma~\ref{lem:sobolev-spaces}(iii) and
Definition~\ref{def:sobolev-kernel}, we have
\[
\|\mathbf1_{\cZ}\|_{\cH_k}
\le c_2\left\|\prod_{j=1}^d\chi_0((z_j-1/2)/2)
\right\|_{H^\tau(\mathbb R^d)}<\infty.
\]
Similarly, $\widetilde q(x):=(1-x_1)\prod_{j=1}^d\chi_0((x_j-1/2)/2)$
extends $1-s_1$ from $\cZ$. Define
\[
c:=\min\left\{1,\frac{B_r}{c_2\|\widetilde q\|_{H^\tau(\mathbb R^d)}}\right\}>0,\
q(s,a):=c(1-s_1).
\]
Then we have $0\le q\le c\le1$, and the same norm comparison gives
\[
\|q\|_{\cH_k}\le cc_2\|\widetilde q\|_{H^\tau(\mathbb R^d)}\le B_r.
\]
To compute its expectation under $P_g$, integrate over the first
state coordinate; every other coordinate contributes a factor one:
\[
\int_{\cS^+}q\,d\mu
=c\int_0^{1/2}(1-x)\,dx=\frac{3c}{8},\
\int_{\cS^-}q\,d\mu
=c\int_{1/2}^1(1-x)\,dx=\frac c8.
\]
Consequently, by \eqref{eq:two-region-transition}, we obtain
\begin{equation}
(P_gq)(z)=(1+2g(z))\frac{3c}{8}+(1-2g(z))\frac c8
=\frac c2+\frac c2g(z).
\label{eq:lower-reward-expectation}
\end{equation}
Thus increasing $g(z)$ by $\varepsilon$ increases the expected next
reward by $c\varepsilon/2$. All rewards in the construction are
observed deterministically; the unknown perturbation affects only
the transitions.

\paragraph{Realizing the loss in one pair of stages.}
Let $M_g$ use $P^1=P_g$, $r^2=q$, and uniform transitions and zero
rewards at all other stages. Since all rewards after stage two vanish,
\eqref{eq:bellman-optimality} and \eqref{eq:lower-reward-expectation} give
\[
Q_*^1(s,a)=(P_gq)(s,a)=\frac c2+\frac c2g(s,a),\
V_*^1(s)=\frac c2+\frac c2\max_{a'\in\cA}g(s,a').
\]
Subtracting, we obtain
\[
V_*^1(s)-Q_*^1(s,a)
=\frac c2\left[\max_{a'\in\cA}g(s,a')-g(s,a)\right].
\]
In particular, if the maximum is $\varepsilon$ and the chosen action
has $g(s,a)=0$, its value loss is $c\varepsilon/2$.

\paragraph{Repeating the construction across stage pairs.}
Let $r\in[\lfloor H/2\rfloor]$ index stages
$2r-1,2r$. We assign a potentially different maximizing action region
to each state cell and stage pair.
For a dyadic width $w=2^{-l}$, let $\mathscr B_w$ and $\mathscr A_w$
be the regular state and action partitions into cubes of side $w$,
with fixed boundary assignments. Their cardinalities are $w^{-d_S}$
and $w^{-d_A}$; write $x_B,x_U$ for their cell centers.
Set $\psi:=\chi_0^{\otimes d_S}$,
$\vartheta:=\chi_0^{\otimes d_A}$, $\chi:=\psi\otimes\vartheta$, and
$\vartheta_U(a):=\vartheta((a-x_U)/w)$.
For a height $\varepsilon>0$ and an array of action cells
$\boldsymbol U=(U_{r,B})\in\mathscr A_w^{[{\lfloor H/2\rfloor}]\times\mathscr B_w}$, define
\begin{equation}
g_r(s,a):=\varepsilon\sum_{B\in\mathscr B_w}
\psi((s-x_B)/w)\vartheta_{U_{r,B}}(a),\ r\in[{\lfloor H/2\rfloor}].
\label{eq:lower-array-perturbations}
\end{equation}
Let $\widetilde M_{\boldsymbol U}$ use
\[
P^{2r-1}=P_{g_r},\ r^{2r}=q\ (r\in[{\lfloor H/2\rfloor}]),\
P^{2r}(s'\mid z)=1\ (2r<H),\ r^{2r-1}=0.
\]
An extra terminal odd stage has zero reward.
The uniform transition after stage $2r$ makes the next pair start
from a uniform state independently of the preceding history.
It also makes the subsequent optimal value independent of the current
action. Indeed, by \eqref{eq:bellman-optimality}, we have
\[
Q_*^{2r}(s,a)=q(s,a)+\int_{\cS}V_*^{2r+1}(s')\,\mu(ds')
=V_*^{2r}(s).
\]
The integral is zero when $2r=H$, since $V_*^{H+1}=0$.
Using \eqref{eq:lower-reward-expectation} and $P_{g_r}\mathbf1=1$, we have
\[
Q_*^{2r-1}(s,a)=(P_{g_r}q)(s,a)+\int_{\cS}V_*^{2r+1}(s')\,\mu(ds')
=\frac c2+\frac c2g_r(s,a)+\int_{\cS}V_*^{2r+1}(s')\,\mu(ds').
\]
The integral is independent of $a$ and cancels when we subtract
$Q_*^{2r-1}(s,a)$ from its maximum over actions. Hence we obtain
\begin{equation}
V_*^{2r-1}(s)-Q_*^{2r-1}(s,a)
=\frac c2\left[\max_{a'\in\cA}g_r(s,a')-g_r(s,a)\right].
\label{eq:lower-bellman-gaps}
\end{equation}
The even-stage gaps are zero. Both $M_g$ and
$\widetilde M_{\boldsymbol U}$ belong to $\mathcal M(B_p,B_r)$ when
their perturbations satisfy $\|g\|_\infty\le1/4$ and
$\|g\|_{\cH_k}\le\overline B_p$, as checked above.

\paragraph{Converting the value gaps into expected regret.}
Expectations below include the algorithm's randomness, the transitions,
and the chosen initial-state randomization. Recall $z_t^h=(s_t^h,a_t^h)$
and $\mathcal I_t^h$ from \eqref{eq:observation-law}.
By \eqref{eq:bellman-optimality} and \eqref{eq:observation-law}, we have
\[
Q_*^h(z_t^h)=r^h(z_t^h)
+\mathbb E_{\widetilde M_{\boldsymbol U}}
[V_*^{h+1}(s_t^{h+1})\mid\mathcal I_t^h],
\]
with zero continuation term at $h=H$.
Starting from \eqref{eq:regret}, we therefore obtain
\[
\begin{aligned}
\mathbb E_{\widetilde M_{\boldsymbol U}}[\reg(T)]
&=\mathbb E_{\widetilde M_{\boldsymbol U}}\sum_{t=1}^T
\left[V_*^1(s_t^1)-\sum_{h=1}^Hr^h(z_t^h)\right]=\mathbb E_{\widetilde M_{\boldsymbol U}}\sum_{t=1}^T\sum_{h=1}^H
\left[V_*^h(s_t^h)-r^h(z_t^h)-V_*^{h+1}(s_t^{h+1})\right]\\
&=\sum_{t=1}^T\sum_{h=1}^H\mathbb E_{\widetilde M_{\boldsymbol U}}
\left[V_*^h(s_t^h)-Q_*^h(z_t^h)\right].
\end{aligned}
\]
The first equality uses the expected return of the selected policy;
the second telescopes the value terms and uses $V_*^{H+1}=0$;
the third applies the tower property to the preceding Bellman identity.
Substituting the zero even-stage gaps and \eqref{eq:lower-bellman-gaps},
we conclude
\begin{equation}
\mathbb E_{\widetilde M_{\boldsymbol U}}[\reg(T)]
=\frac c2\,\mathbb E_{\widetilde M_{\boldsymbol U}}
\sum_{t=1}^T\sum_{r=1}^{\lfloor H/2\rfloor}
\left[\max_{a\in\cA}\{g_r(s_t^{2r-1},a)\}-g_r(z_t^{2r-1})\right].
\label{eq:lower-expected-gap}
\end{equation}
For $M_g$, the same calculation retains only the first pair of stages,
with gap $(c/2)[\max_{a\in\cA}g(s_t^1,a)-g(z_t^1)]$.
These identities hold for history-dependent policies. It remains to
bound the perturbation norms and then show that the algorithm incurs
enough of these nonnegative gaps.

\paragraph{Verifying the RKHS constraints.}
To check the perturbation norms in both families, we establish the
following estimate for sums with disjoint state supports.
There is a constant $C_*>0$, depending only on $\nu,d_S,d_A$,
with the following property. Let $0<w\le1$ and $\varepsilon>0$.
Choose state centers $y_j$ so that the open cubes
$y_j+w(-1/2,1/2)^{d_S}$, $j\in[N]$, are pairwise disjoint,
and choose arbitrary action centers $x_j\in\cA$.
Then we obtain a perturbation with the following norm bound:
\begin{equation}
g(s,a):=\varepsilon\sum_{j=1}^N
\psi((s-y_j)/w)\vartheta((a-x_j)/w),\ 
\|g\|_{\cH_k}\le c_2C_*\varepsilon\sqrt N\,w^{-\nu}.
\label{eq:lower-perturbation-norm}
\end{equation}

We extend the defining formula to the whole space by
\[
F(s,a):=\varepsilon\sum_{j=1}^N
\psi\!\left(\frac{s-y_j}{w}\right)
\vartheta\!\left(\frac{a-x_j}{w}\right),\
(s,a)\in\mathbb R^{d_S}\times\mathbb R^{d_A}.
\]
Since $F|_{\cZ}=g$, Definition~\ref{def:sobolev-kernel} and
Lemma~\ref{lem:rkhs-uniqueness-restriction} give
$\|g\|_{\cH_k}\le c_2\|g\|_{H^\tau(\cZ)}
\le c_2\|F\|_{H^\tau(\mathbb R^d)}$, where $\tau=\nu+d/2$.
Derivatives of distinct summands have disjoint state supports.
With $\chi=\psi\otimes\vartheta$, differentiation and a change of variables give,
for every multi-index $\mathbf b$,
\[
\begin{aligned}
\|\partial^{\mathbf b}F\|_2^2
&=\varepsilon^2w^{-2|\mathbf b|}\sum_{j=1}^N
\int_{\mathbb R^d}\left|\partial^{\mathbf b}\chi
\left(\frac{z-(y_j,x_j)}w\right)\right|^2\,dz=\varepsilon^2Nw^{d-2|\mathbf b|}\|\partial^{\mathbf b}\chi\|_2^2.
\end{aligned}
\]
Put $m:=\lfloor\tau\rfloor$ and $\sigma:=\tau-m$.
For $j=m,m+1$, expanding the integer Fourier weight and applying
Plancherel's identity yields
\[
\|F\|_{H^j(\mathbb R^d)}^2
=\varepsilon^2N\sum_{|\mathbf b|\le j}
\frac{j!}{(j-|\mathbf b|)!\,\mathbf b!}
 w^{d-2|\mathbf b|}\|\partial^{\mathbf b}\chi\|_2^2
\le\varepsilon^2Nw^{d-2j}\|\chi\|_{H^j(\mathbb R^d)}^2,
\]
where $\mathbf b!:=\prod_{j=1}^db_j!$; the inequality uses $w\le1$.
For $0<\sigma<1$, let $\widehat F$ denote the Fourier transform of $F$.
Applying the H\"older inequality to its weighted integral, we obtain
\[
\begin{aligned}
\|F\|_{H^\tau}^2
&=\int\bigl[(1+4\pi^2|\varpi|^2)^m|\widehat F(\varpi)|^2\bigr]^{1-\sigma}
\bigl[(1+4\pi^2|\varpi|^2)^{m+1}|\widehat F(\varpi)|^2\bigr]^\sigma\,d\varpi\\
&\le\|F\|_{H^m}^{2(1-\sigma)}\|F\|_{H^{m+1}}^{2\sigma}
\le\varepsilon^2Nw^{d-2\tau}
\|\chi\|_{H^m}^{2(1-\sigma)}\|\chi\|_{H^{m+1}}^{2\sigma}.
\end{aligned}
\]
These norms are on $\mathbb R^d$. The last inequality substitutes
the integer-order bounds and uses
$(1-\sigma)(d-2m)+\sigma(d-2m-2)=d-2\tau=-2\nu$.
For $\sigma=0$, the integer-order estimate applies directly.
Taking square roots and restricting to $\cZ$, we obtain
$\|g\|_{\cH_k}\le c_2C_*\varepsilon\sqrt N\,w^{-\nu}$,
where $C_*:=\|\chi\|_{H^m}^{1-\sigma}\|\chi\|_{H^{m+1}}^\sigma$
depends only on $\nu,d_S,d_A$.

\paragraph{The $T^{e_A}$ lower bound.}
\begin{lemma}
\label{lem:one-needle}
Fix $s^\circ=(1/2,\ldots,1/2)$ and a dyadic width $w\le1/4$.
Choose a height $0<\varepsilon\le1/4$ satisfying the RKHS constraint
$c_2C_*\varepsilon w^{-\nu}\le\overline{B}_p$ and the testing condition
$\varepsilon^2Tw^{d_A}\le1/6$.
For each hidden action cell $U\in\mathscr A_w$, define
$g_U(s,a):=\varepsilon\psi((s-s^\circ)/w)\vartheta_U(a)$.
Under the fixed initial states $s_t^1=s^\circ$, every algorithm
satisfies
\[
\frac1{|\mathscr A_w|}\sum_{U\in\mathscr A_w}
\mathbb E_{M_{g_U}}[\reg(T)]\ge\frac{c\varepsilon T}{16}.
\]
\end{lemma}

Put
$c_{\mathrm{lb}}:=\overline{B}_p/(c_2C_*)=(B_p-\|\mathbf1_{\cZ}\|_{\cH_k})/(2c_2C_*)>0$.
For Lemma~\ref{lem:one-needle}, choose
\[
l=\left\lceil\frac{\log_2(6c_{\mathrm{lb}}^2T)}{2\nu+d_A}\right\rceil,\
w=2^{-l},\ \varepsilon=c_{\mathrm{lb}} w^\nu.
\]
Since $x\le\lceil x\rceil<x+1$, the width satisfies
\[
\tfrac12(6c_{\mathrm{lb}}^2T)^{-1/(2\nu+d_A)}\le w
\le(6c_{\mathrm{lb}}^2T)^{-1/(2\nu+d_A)}.
\]
Substituting the height and upper width bound, we obtain
\[
c_2C_*\varepsilon w^{-\nu}=\overline{B}_p,\
\varepsilon^2Tw^{d_A}=c_{\mathrm{lb}}^2Tw^{2\nu+d_A}\le\tfrac16.
\]
For sufficiently large $T$, we also have $w,\varepsilon\le1/4$.
By Lemma~\ref{lem:one-needle}, some model therefore satisfies
\[
\mathbb E[\reg(T)]\ge\frac{cc_{\mathrm{lb}}}{16}w^\nu T
\ge\frac{cc_{\mathrm{lb}}}{2^{\nu+4}}(6c_{\mathrm{lb}}^2)^{-\nu/(2\nu+d_A)}
T^{1-\nu/(2\nu+d_A)}=:c_AT^{e_A}.
\]
The second inequality uses the lower width bound, and
$1-\nu/(2\nu+d_A)=e_A$.

\paragraph{The $\lfloor H/2\rfloor T^{e_B}$ lower bound.}
The next family assigns a hidden action cell to every state cube
and every odd stage.

\begin{lemma}
\label{lem:state-cube-needles}
Fix a dyadic width $w\le1/2$ and a height $0<\varepsilon\le1/4$.
Assume the RKHS constraint
$c_2C_*\varepsilon w^{-(\nu+d_S/2)}\le\overline{B}_p$ and the
visit and testing conditions
\[
Tw^{d_S}\ge1,\ \varepsilon^2Tw^{d_S+d_A}\le1/48.
\]
For each array $\boldsymbol U\in\mathscr A_w^{[{\lfloor H/2\rfloor}]\times\mathscr B_w}$,
use the model $\widetilde M_{\boldsymbol U}$ constructed from
\eqref{eq:lower-array-perturbations}.
If the initial states are independent and uniform, every algorithm
satisfies
\[
\frac1{|\mathscr A_w|^{{\lfloor H/2\rfloor}|\mathscr B_w|}}
\sum_{\boldsymbol U}
\mathbb E_{\widetilde M_{\boldsymbol U}}[\reg(T)]
\ge {\lfloor H/2\rfloor}\,2^{-d_S-4}c\varepsilon T.
\]
\end{lemma}

For Lemma~\ref{lem:state-cube-needles}, choose instead
\[
l=\left\lceil\frac{\log_2(48c_{\mathrm{lb}}^2T)}{2\nu+2d_S+d_A}\right\rceil,\
w=2^{-l},\ \varepsilon=c_{\mathrm{lb}} w^{\nu+d_S/2}.
\]
The same rounding rule yields
\[
\tfrac12(48c_{\mathrm{lb}}^2T)^{-1/(2\nu+2d_S+d_A)}\le w
\le(48c_{\mathrm{lb}}^2T)^{-1/(2\nu+2d_S+d_A)}.
\]
Substituting the height and the upper width bound, we obtain
\[
c_2C_*\varepsilon w^{-(\nu+d_S/2)}=\overline{B}_p,\
\varepsilon^2Tw^{d_S+d_A}
=c_{\mathrm{lb}}^2Tw^{2\nu+2d_S+d_A}\le\tfrac1{48}.
\]
The lower width bound verifies the visit condition:
\[
Tw^{d_S}\ge2^{-d_S}(48c_{\mathrm{lb}}^2)^{-d_S/(2\nu+2d_S+d_A)}
T^{1-d_S/(2\nu+2d_S+d_A)}\longrightarrow\infty.
\]
The exponent is positive, and $w,\varepsilon\to0$, so all hypotheses
hold for sufficiently large $T$. Lemma~\ref{lem:state-cube-needles}
and the lower width bound then give a model with
\[
\begin{aligned}
&\mathbb E[\reg(T)]\ge {\lfloor H/2\rfloor}2^{-d_S-4}cc_{\mathrm{lb}} w^{\nu+d_S/2}T\\
&\ge {\lfloor H/2\rfloor}2^{-\nu-3d_S/2-4}cc_{\mathrm{lb}}
(48c_{\mathrm{lb}}^2)^{-(\nu+d_S/2)/(2\nu+2d_S+d_A)}
T^{1-(\nu+d_S/2)/(2\nu+2d_S+d_A)}=:{\lfloor H/2\rfloor}c_BT^{e_B}.
\end{aligned}
\]
The exponent is $e_B$. The uniform initial states are admissible;
averaging over their draws also supplies a deterministic sequence
with at least the same expected regret.

\paragraph{Combining the two lower bounds.}
Both constants are positive and have the stated parameter dependence.
The threshold ensuring the width, height, and visit conditions is
independent of $H$, and both constructions use deterministic rewards.
Since the supremum can select either model, for the fixed algorithm we have
\[
\sup_{M\in\mathcal M(B_p,B_r)}\sup_{\mathsf{Adv}}\mathbb E[\reg(T)]
\ge\max\{c_AT^{e_A},{\lfloor H/2\rfloor}c_BT^{e_B}\}
\ge\min\{c_A,{\lfloor H/2\rfloor}c_B\}T^{\max\{e_A,e_B\}}.
\]
Taking the infimum over algorithms
proves \eqref{eq:regret-lower}.
\end{proof}

\begin{proof}[Proof of Lemma~\ref{lem:one-needle}]
By the perturbation definition in Lemma~\ref{lem:one-needle}, we have
$0\le g_U\le\varepsilon\le1/4$.
By \eqref{eq:lower-perturbation-norm} with $N=1$, we also have
$\|g_U\|_{\cH_k}\le c_2C_*\varepsilon w^{-\nu}\le\overline{B}_p$.
Thus $M_{g_U}$ and the zero-perturbation model are admissible.
For the fixed algorithm and initial states $s_t^1=s^\circ$, let
$\mathbb P_U$ and $\mathbb P_0$ be the laws of the complete interaction
under $M_{g_U}$ and $M_0$, respectively. Here $M_0$ is $M_g$ with
$g=0$. Write $\mathbb E_U,\mathbb E_0$ for the corresponding expectations.
Let $n(U):=\sum_{t=1}^T\mathbf1\{a_t^1\in U\}$.
Since $\psi(0)=1$, $\max_a\vartheta_U(a)=1$, and $\vartheta_U=0$ outside $U$, we have
\[
\max_{a'}g_U(s^\circ,a')-g_U(s^\circ,a)
=\varepsilon(1-\vartheta_U(a))\ge\varepsilon\mathbf1\{a\notin U\}.
\]
Substituting into \eqref{eq:lower-expected-gap}, we obtain
\[
\mathbb E_U[\reg(T)]\ge\frac{c\varepsilon}{2}
\bigl(T-\mathbb E_U n(U)\bigr).
\]

Only the first transition at actions in $U$ differs between the two models.
At perturbation value $x$, its next-state region is Bernoulli with
parameter $1/2+x$; conditional on either region, the state law is unchanged.
For $|x|\le1/4$, we have
\begin{equation}
\operatorname{kl}(1/2,1/2+x)
=-\tfrac12\ln(1-4x^2)
\le\frac{2x^2}{1-4x^2}\le3x^2.
\label{eq:lower-bernoulli-kl}
\end{equation}
Here we use $-\ln(1-y)\le y/(1-y)$ and $1-4x^2\ge3/4$.
The algorithm's decision rule is the same in both models, so the
chain rule and \eqref{eq:lower-bernoulli-kl} give
\begin{equation}
\operatorname{KL}(\mathbb P_0\|\mathbb P_U)
=\mathbb E_0\sum_{t=1}^T\operatorname{kl}
\bigl(1/2,1/2+g_U(s^\circ,a_t^1)\bigr)
\le3\varepsilon^2\mathbb E_0 n(U).
\label{eq:lower-single-kl}
\end{equation}
The last inequality uses $g_U(s^\circ,a)^2\le\varepsilon^2\mathbf1\{a\in U\}$.
Because $0\le n(U)\le T$, Lemma~\ref{lem:pinsker} implies
\[
\mathbb E_U n(U)-\mathbb E_0 n(U)
\le T\operatorname{TV}(\mathbb P_0,\mathbb P_U)
\le T\sqrt{\operatorname{KL}(\mathbb P_0\|\mathbb P_U)/2}
\le T\varepsilon\sqrt{\tfrac32\mathbb E_0 n(U)}.
\]
Recall $|\mathscr A_w|=w^{-d_A}\ge4$.
Under the common law $\mathbb P_0$, we have $\sum_{U\in\mathscr A_w}n(U)=T$.
Averaging and applying the Cauchy--Schwarz inequality, we obtain
\[
\begin{aligned}
\frac1{|\mathscr A_w|}\sum_{U\in\mathscr A_w}\mathbb E_U n(U)
&\le\frac T{|\mathscr A_w|}+\frac{T\varepsilon}{|\mathscr A_w|}
\sum_{U\in\mathscr A_w}\sqrt{\tfrac32\mathbb E_0 n(U)}\le Tw^{d_A}+T\varepsilon\sqrt{\frac{3Tw^{d_A}}2}
\le\frac{3T}{4}.
\end{aligned}
\]
The second inequality uses
$\sum_{U\in\mathscr A_w}\sqrt{\mathbb E_0n(U)}
\le\sqrt{|\mathscr A_w|T}$ and $|\mathscr A_w|=w^{-d_A}$;
the last uses $w^{d_A}\le1/4$ and $\varepsilon^2Tw^{d_A}\le1/6$.
Hence the average regret is at least
$(c\varepsilon/2)(T-3T/4)=c\varepsilon T/8\ge c\varepsilon T/16$.
\end{proof}

\begin{proof}[Proof of Lemma~\ref{lem:state-cube-needles}]
By \eqref{eq:lower-array-perturbations} and disjointness of the state supports, we have
$0\le g_r\le\varepsilon\le1/4$.
By \eqref{eq:lower-perturbation-norm} with $N=w^{-d_S}$, we also have
$\|g_r\|_{\cH_k}\le c_2C_*\varepsilon w^{-(\nu+d_S/2)}\le\overline{B}_p$.
Thus all models are admissible, including those with one bump deleted.
For $B\in\mathscr B_w$, let
$B_{\mathrm{in}}:=x_B+w[-1/4,1/4]^{d_S}$ be its concentric cube of
half the side length. For any measurable $C\subseteq\cS$ and
$U\in\mathscr A_w$, define
\[
n_r(C):=\sum_{t=1}^T\mathbf1\{s_t^{2r-1}\in C\},\
n_r(C,U):=\sum_{t=1}^T
\mathbf1\{s_t^{2r-1}\in C,\ a_t^{2r-1}\in U\}.
\]
We use these counts with $C=B$ and $C=B_{\mathrm{in}}$.
On $B_{\mathrm{in}}$, the state factor in
\eqref{eq:lower-array-perturbations} equals one; outside $U_{r,B}$,
its action factor is zero. Therefore we obtain
\begin{equation}
\sum_{t=1}^T\sum_{r=1}^{\lfloor H/2\rfloor}
\left[\max_{a\in\cA}g_r(s_t^{2r-1},a)-g_r(z_t^{2r-1})\right]
\ge\varepsilon\sum_{r=1}^{\lfloor H/2\rfloor}\sum_{B\in\mathscr B_w}
\bigl[n_r(B_{\mathrm{in}})-n_r(B_{\mathrm{in}},U_{r,B})\bigr].
\label{eq:lower-core-gap}
\end{equation}
The omitted gaps are nonnegative.

To apply Lemma~\ref{lem:pinsker} with the fixed upper bound
$n_r(B_{\mathrm{in}})$, generate all uniform initial states and uniform
reset states independently of the algorithm's random seed and the
randomness used to sample the transitions $P_{g_r}$.
Let $\mathcal J$ be the $\sigma$-algebra generated by these uniform
states and the algorithm's seed, including the terminal uniform reset
when $H$ is odd. These generating variables have the same joint law
in every model. Conditional on $\mathcal J$, $n_r(B)$ and
$n_r(B_{\mathrm{in}})$ are fixed, whereas the counts that also restrict
the action can still be random. Unsubscripted expectations below refer
to the common law of these generating variables.

Fix $(r,B)$ and all entries of $\boldsymbol U$ other than $U_{r,B}$.
For this comparison, let $\mathbb P_0$ denote the joint law of the
complete interaction, the uniform states, and the algorithm's seed
when the summand indexed by $B$ in \eqref{eq:lower-array-perturbations}
is replaced by zero. Let $\mathbb P_U$ denote the corresponding law
when $U_{r,B}=U$, and let $\mathbb E_0,\mathbb E_U$ denote their
expectations. All other perturbations are kept fixed in both models.
Conditional on $\mathcal J$, the transition densities $P_{g_r}$ are
unchanged and the algorithm uses the same decision rule; each state
is still observed only at its prescribed time.
Only transitions at stage $2r-1$ in $B\times U$ differ.
Applying the chain-rule calculation in \eqref{eq:lower-single-kl}
conditionally on $\mathcal J$, with the bound
\eqref{eq:lower-bernoulli-kl} and $|\varepsilon\psi\vartheta_U|\le\varepsilon$, we obtain
\[
\operatorname{KL}(\mathbb P_0(\cdot\mid\mathcal J)
\|\mathbb P_U(\cdot\mid\mathcal J))
\le3\varepsilon^2\mathbb E_0[n_r(B,U)\mid\mathcal J].
\]
Since $n_r(B_{\mathrm{in}})$ is fixed under this conditioning,
$0\le n_r(B_{\mathrm{in}},U)\le n_r(B_{\mathrm{in}})$ and Lemma~\ref{lem:pinsker} yield
\[
\begin{aligned}
\mathbb E_U[n_r(B_{\mathrm{in}},U)\mid\mathcal J]
-\mathbb E_0[n_r(B_{\mathrm{in}},U)\mid\mathcal J]
&\le n_r(B_{\mathrm{in}})\sqrt{\tfrac12
\operatorname{KL}(\mathbb P_0(\cdot\mid\mathcal J)\|\mathbb P_U(\cdot\mid\mathcal J))}\\
&\le\varepsilon n_r(B_{\mathrm{in}})
\sqrt{\tfrac32\mathbb E_0[n_r(B,U)\mid\mathcal J]}.
\end{aligned}
\]
The model defining $\mathbb P_0$ is common to all $U\in\mathscr A_w$,
and the partition identities give
$\sum_U n_r(B_{\mathrm{in}},U)=n_r(B_{\mathrm{in}})$ and
$\sum_U n_r(B,U)=n_r(B)$.
Averaging the preceding comparison and applying the Cauchy--Schwarz
inequality, we obtain
\[
\begin{aligned}
\frac1{|\mathscr A_w|}\sum_U\mathbb E_U[n_r(B_{\mathrm{in}},U)\mid\mathcal J]
&\le\frac{n_r(B_{\mathrm{in}})}{|\mathscr A_w|}
+\frac{\varepsilon n_r(B_{\mathrm{in}})}{|\mathscr A_w|}
\sum_U\sqrt{\tfrac32\mathbb E_0[n_r(B,U)\mid\mathcal J]}\\
&\le n_r(B_{\mathrm{in}})w^{d_A}
+\varepsilon n_r(B_{\mathrm{in}})\sqrt{\frac{3n_r(B)w^{d_A}}2}.
\end{aligned}
\]
For the second inequality, we use
$\sum_U\sqrt{\mathbb E_0[n_r(B,U)\mid\mathcal J]}
\le\sqrt{|\mathscr A_w|n_r(B)}$ and $|\mathscr A_w|=w^{-d_A}$.

The uniform states give $n_r(B)\sim\operatorname{Bin}(T,w^{d_S})$ and
$\mathbb E[n_r(B_{\mathrm{in}})\mid n_r(B)]=2^{-d_S}n_r(B)$,
since $|B_{\mathrm{in}}|/|B|=2^{-d_S}$.
Thus we have $\mathbb E n_r(B_{\mathrm{in}})=2^{-d_S}Tw^{d_S}$ and
\[
\begin{aligned}
\mathbb E[n_r(B_{\mathrm{in}})\sqrt{n_r(B)}]
&=2^{-d_S}\mathbb E[n_r(B)^{3/2}]
\le2^{-d_S}\sqrt{\mathbb E n_r(B)\,\mathbb E[n_r(B)^2]}\\
&\le2^{-d_S}\sqrt{Tw^{d_S}[(Tw^{d_S})^2+Tw^{d_S}]}
\le2^{-d_S}\sqrt2\,(Tw^{d_S})^{3/2}.
\end{aligned}
\]
The inequalities use the Cauchy--Schwarz inequality, the binomial
variance bound, and $Tw^{d_S}\ge1$, respectively.
Integrating the conditional comparison over the common law of the
variables generating $\mathcal J$, we obtain
\[
\begin{aligned}
\mathbb E n_r(B_{\mathrm{in}})
-\frac1{|\mathscr A_w|}\sum_U\mathbb E_U n_r(B_{\mathrm{in}},U)
&\ge(1-w^{d_A})\mathbb E n_r(B_{\mathrm{in}})
-\varepsilon\sqrt{\tfrac32w^{d_A}}\, 
\mathbb E[n_r(B_{\mathrm{in}})\sqrt{n_r(B)}]\\
&\ge2^{-d_S}Tw^{d_S}
\left(1-w^{d_A}-\varepsilon\sqrt{3Tw^{d_S+d_A}}\right)
\ge2^{-d_S-2}Tw^{d_S}.
\end{aligned}
\]
The second inequality substitutes the preceding two moment bounds;
the last uses $w^{d_A}\le1/2$ and
$\varepsilon\sqrt{3Tw^{d_S+d_A}}\le1/4$.
This bound holds for every choice of the other hidden entries.
Averaging over them and substituting into \eqref{eq:lower-core-gap}
and \eqref{eq:lower-expected-gap}, we conclude
\[
\begin{aligned}
&\frac1{|\mathscr A_w|^{{\lfloor H/2\rfloor}|\mathscr B_w|}}\sum_{\boldsymbol U}
\mathbb E_{\widetilde M_{\boldsymbol U}}[\reg(T)]
\ge\frac{c\varepsilon}{2}\sum_{r=1}^{\lfloor H/2\rfloor}\sum_{B\in\mathscr B_w}
2^{-d_S-2}Tw^{d_S}
\\
&={\lfloor H/2\rfloor}2^{-d_S-3}c\varepsilon T
\ge {\lfloor H/2\rfloor}2^{-d_S-4}c\varepsilon T.
\end{aligned}
\]
The equality uses $|\mathscr B_w|=w^{-d_S}$.
Conditioning preserved the adaptive decision rule, so the result
holds for every nonanticipating algorithm.
\end{proof}

\section{Proofs of Lemmas and Propositions in Appendix~\ref{app:upper-bound}}
\label{app:upper-auxiliary-proofs}

\subsection{Proof of Lemma~\ref{lem:upper-partition-counts}}

\begin{proof}
The counting partition $\mathcal D_l$ determines the selected level,
whereas the regression partition $\mathcal P_l$ groups observations
assigned to that level. We bound the corresponding numbers of episodes
and occupied regression cells,
\[
T_l=|\{t\le T:l_t=l\}|,\
P_l=\bigl|\{P\in\mathcal P_l:\exists t\le T,\ l_t=l,\ s_t^1\in P\}\bigr|.
\]
Recall $S_l=\min\{2^{l/d},S_{\max}\}$,
$N_l=\lceil T(S_l/S_{\max})^{2\nu}\rceil$, and
$m_t(B)=\sum_{i<t}\mathbf1\{s_i^1\in B\}$.

Fix $B\in\mathcal D_l$. By \eqref{eq:adaptive-cube},
each visit with $l_t=l$ and $s_t^1\in B$ satisfies $m_t(B)<N_l$.
For two such visits $t<t'$, we have
\[
m_{t'}(B)-m_t(B)
=\sum_{i=t}^{t'-1}\mathbf1\{s_i^1\in B\}\ge1,
\]
because the sum includes the visit at $t$.
Thus these counts are distinct integers in $\{0,\ldots,N_l-1\}$,
so each cell contributes at most $N_l$ selected visits. Summing over cells,
we obtain
\begin{equation}
T_l=\sum_{B\in\mathcal D_l}|\{t\le T:l_t=l,\ s_t^1\in B\}|
\le|\mathcal D_l|N_l.
\label{eq:upper-level-visit-count}
\end{equation}
We will bound $|\mathcal D_l|$ after controlling the occupied regression cells.

For $l\ge1$, the minimum in \eqref{eq:adaptive-cube} implies
\[
l_t=l\ \Longrightarrow\
m_t(B_{l-1}(s_t^1))\ge N_{l-1}
\ \Longrightarrow\
m_{T+1}(B_{l-1}(s_t^1))\ge N_{l-1}.
\]
The first implication holds because level $l-1$ was not selected;
the second uses that visit counts are nondecreasing.
Since $\mathcal D_{l-1}$ partitions $\cS$, we have
\begin{equation}
N_{l-1}\bigl|\{B\in\mathcal D_{l-1}:m_{T+1}(B)\ge N_{l-1}\}\bigr|
\le\sum_{B\in\mathcal D_{l-1}}m_{T+1}(B)
=\sum_{i=1}^T\sum_{B\in\mathcal D_{l-1}}\mathbf1\{s_i^1\in B\}=T.
\label{eq:upper-saturated-cell-count}
\end{equation}
Consequently, the selected initial states lie in at most $T/N_{l-1}$
preceding-level counting cells. To bound $P_l$, we now bound the number
of regression cells intersecting each of these cells.

Recall $\mathcal D_{l-1}=\mathscr Q_{M_{l-1}}$ and
$\mathcal P_l=\mathscr Q_{\widetilde M_l}$.
Write $B_\infty(s,r):=\{x\in\mathbb R^{d_S}:\|x-s\|_\infty\le r\}$.
By Lemma~\ref{lem:colored-cube-partition}, with $n=d_S$ and scale $1/M$,
we have
\begin{equation}
\left|\left\{C\in\mathscr Q_M:
C\cap B_\infty\!\left(s,\frac1{2d_SM}\right)\ne\varnothing
\right\}\right|\le d_S+1,\ s\in\cS,\ M\ge2.
\label{eq:scaled-cube-intersections}
\end{equation}
Under the map $x\mapsto Mx$, we have
$M B_\infty(s,1/(2d_SM))=B_\infty(Ms,1/(2d_S))$ and
$M(M^{-1}Q_m)=Q_m$. Intersecting the cubes with $\cS$ can only remove
intersections, so Lemma~\ref{lem:colored-cube-partition} applies.

Recall $M_{l-1}=\max\{1,\lfloor S_{l-1}/(64d_S^2\sqrt{d_A})\rfloor\}$.
If $S_{l-1}/(64d_S^2\sqrt{d_A})\ge2$, the definition of the floor gives
\[
M_{l-1}\ge\frac{S_{l-1}}{64d_S^2\sqrt{d_A}}-1
\ge\frac{S_{l-1}}{128d_S^2\sqrt{d_A}}.
\]
Otherwise, we have $M_{l-1}=1>S_{l-1}/(128d_S^2\sqrt{d_A})$, so the same
lower bound holds. Applying the same calculation to
$\widetilde M_l=\max\{1,\lfloor S_l/[4096(d_S+1)d_S^{5/2}\sqrt{d_A}]\rfloor\}$,
and using $\widetilde M_l\le S_l$, we obtain
\begin{equation}
\frac1{S_l}\le\frac1{\widetilde M_l}
\le\frac{8192(d_S+1)d_S^{5/2}\sqrt{d_A}}{S_l}.
\label{eq:state-partition-mesh}
\end{equation}

If $\widetilde M_l=1$, every counting cell meets only the single
regression cell $\cS$. Otherwise, each $B\in\mathcal D_{l-1}$ lies in a
cube of side $1/M_{l-1}$, including when $M_{l-1}=1$. Hence we have
\begin{equation}
\operatorname{diam}_\infty(B)\le\frac1{M_{l-1}}
\le\frac{128d_S^2\sqrt{d_A}}{S_{l-1}}
\le\frac{256d_S^2\sqrt{d_A}}{S_l},\ B\in\mathcal D_{l-1}.
\label{eq:scaled-state-partition-bounds}
\end{equation}
The last inequality uses $S_l/S_{l-1}\le2^{1/d}\le2$.
By the definition of $\widetilde M_l$ and $\widetilde M_l\ge2$, we have
\[
\frac1{2d_S\widetilde M_l}
\ge\frac{2048(d_S+1)d_S^{3/2}\sqrt{d_A}}{S_l}
>\frac{256d_S^2\sqrt{d_A}}{S_l}
\ge\operatorname{diam}_\infty(B).
\]
Here the first inequality substitutes
$\widetilde M_l\le S_l/[4096(d_S+1)d_S^{5/2}\sqrt{d_A}]$,
the strict inequality uses $8(d_S+1)>\sqrt{d_S}$, and the last
inequality is \eqref{eq:scaled-state-partition-bounds}.
For any $s_B\in B$, it follows that
$B\subseteq B_\infty(s_B,1/(2d_S\widetilde M_l))$.
Applying \eqref{eq:scaled-cube-intersections}, obtained from
Lemma~\ref{lem:colored-cube-partition}, with $M=\widetilde M_l$,
we obtain
\[
\bigl|\{P\in\mathcal P_l:P\cap B\ne\varnothing\}\bigr|
\le\left|\left\{P\in\mathcal P_l:
P\cap B_\infty\!\left(s_B,\frac1{2d_S\widetilde M_l}\right)
\ne\varnothing\right\}\right|\le d_S+1.
\]
Each occupied regression cell meets a preceding-level cell whose count
reached $N_{l-1}$. Summing the intersection counts and applying
\eqref{eq:upper-saturated-cell-count}, we obtain
\[
P_l
\le\sum_{\substack{B\in\mathcal D_{l-1}\\m_{T+1}(B)\ge N_{l-1}}}
\bigl|\{P\in\mathcal P_l:P\cap B\ne\varnothing\}\bigr|
\le(d_S+1)\frac{T}{N_{l-1}}.
\]
At level zero, we have $S_0=1$ and $\mathcal P_0=\{\cS\}$, which implies $P_0\le1$.

It remains to bound the total numbers of cells. For $M\ge2$, every
cube $M^{-1}Q_m$ meeting $\cS$ lies in $[-1/M,1+1/M]^{d_S}$.
By Lemma~\ref{lem:colored-cube-partition}, these cubes are disjoint,
and each has volume $M^{-d_S}$. Comparing their total volume with that
of the containing box, we obtain
\begin{equation}
|\mathscr Q_M|M^{-d_S}\le(1+2/M)^{d_S}
\ \Longrightarrow\ |\mathscr Q_M|\le(M+2)^{d_S}.
\label{eq:upper-state-partition-volume}
\end{equation}
For either $M\in\{M_l,\widetilde M_l\}$ with $M\ge2$, their
definitions imply $M\le M_l\le S_l/(64d_S^2\sqrt{d_A})$.
By \eqref{eq:upper-state-partition-volume}, we have
\[
|\mathscr Q_M|\le(M+2)^{d_S}\le(2M)^{d_S}
\le\left(\frac{S_l}{32d_S^2\sqrt{d_A}}\right)^{d_S}
\le S_l^{d_S}.
\]
The second inequality uses $M\ge2$; the last uses $d_S,d_A\ge1$.
When either scale is one, its partition is $\{\cS\}$, so we have
$|\mathcal D_l|,|\mathcal P_l|\le S_l^{d_S}$ in all cases.
Substituting these cardinalities into \eqref{eq:upper-level-visit-count} yields
\[
T_l\le\min\{T,|\mathcal D_l|N_l\}
\le\min\{T,S_l^{d_S}N_l\}.
\]
Similarly, combining $P_l\le|\mathcal P_l|$ with the intersection count gives
\[
P_l\le\min\left\{|\mathcal P_l|,\frac{(d_S+1)T}{N_{l-1}}\right\}
\le\min\left\{S_l^{d_S},\frac{(d_S+1)T}{N_{l-1}}\right\},\ l\ge1.
\]

Finally, by \eqref{eq:initial-partitions} and
$|\mathscr U(S_l)|\le S_l^{d_A}$, we have
\[
|\mathcal C_l|=|\mathcal P_l|\,|\mathscr U(S_l)|
\le S_l^{d_S}S_l^{d_A}=S_l^d.
\]
Using $L_{\max}=\lceil d\log_2S_{\max}\rceil$,
$S_{\max}\le T^{1/(2\nu+d_A)}$, and $\iota_T\ge\ln(\mathrm eT)$,
we also obtain
\[
L_{\max}+1\le d\log_2S_{\max}+2
\le\frac{d\ln T}{(2\nu+d_A)\ln2}+2\le C_\nu d\iota_T.
\]
\end{proof}

\subsection{Proof of Lemma~\ref{lem:upper-action-approximation}}

\begin{proof}
To bound $\sum_{t:l_t=l}\epsilon_l(s_t^1)$, we construct a nonnegative
function $e_l$, common to all $V\in\mathcal V_1$, such that
\[
\epsilon_l(s)\le e_l(s)+
\begin{cases}
T^{-2},&d_A<4\nu,\\
C_{\nu,\kappa}\bar B d_A^{\nu/2}T^{-\nu/(2\nu+d_A)},&d_A\ge4\nu.
\end{cases}
\]
We also need the cell bound
\[
\sum_{B\in\mathcal D_l}\sup_{s\in B}e_l(s)^2
\le C_{\nu,\kappa}\bar B^2(d_S+1)d_A^\nu S_l^{-2\nu}.
\]
By Lemma~\ref{lem:upper-partition-counts}, each cell contributes at most
$N_l$ selected visits. The Cauchy--Schwarz inequality will therefore
bound their total envelope error by
$C_{\nu,\kappa}\bar B\sqrt{d_S+1}\,d_A^{\nu/2}S_l^{-\nu}\sqrt{N_lT_l}$.
We construct $e_l$ from the high-frequency part of the Bellman backup,
prove its cell bound, and approximate the remaining low-frequency part
by action polynomials.

\paragraph{Constructing the common envelope $e_l$.}

For $V\in\mathcal V_1$,
we have $0\le V\le H-1$ and $\mu(\mathcal S)=1$.
By \eqref{eq:reference-kernel-norm} and
Lemma~\ref{lem:bellman-completeness}, $g=\mathcal B^1V$ satisfies
\[
\|g\|_{\mathcal H_0}
\le\Lambda\|\mathcal B^1V\|_{\mathcal H_k}
\le\Lambda\left[B_r+B_p\int_{\mathcal S}|V(y)|\,\mu(dy)\right]
\le\Lambda[B_r+(H-1)B_p]\le\bar B.
\]
Recall from Example~\ref{ex:Matern} and
Lemma~\ref{lem:matern-frequency-representation} that the normalized
Mat\'ern reference kernel $k_0$ is an average over frequencies. The vector
$\omega=2\pi\varpi\in\mathbb R^d$ is the angular frequency.
By Lemma~\ref{lem:matern-frequency-representation}, we may sample it as
\[
\omega=\frac{\kappa Z}{\sqrt W},\
Z\sim\mathcal N(0,I_d),\ W\sim\chi^2_{2\nu},\ Z\perp W,
\]
and the kernel satisfies
\[
k_0(z,z')=\mathbb E_{Z,W}[e^{\mathrm i\omega^\top(z-z')}]
=\mathbb E_{Z,W}[\cos\bigl(\omega^\top(z-z')\bigr)].
\]
For $z=(s,a)$, we write $\omega=(\xi,\zeta)$ with
$\xi\in\mathbb R^{d_S}$ and $\zeta\in\mathbb R^{d_A}$, so that
$\omega^\top z=\xi^\top s+\zeta^\top a$; these are the state and
action frequencies. Throughout this argument, expectations and
unqualified $L^2$ norms are taken under the joint law of the auxiliary
variables $(Z,W)$, with $g$ and the evaluation points fixed.
Using this kernel identity, we apply
\citet[Proposition~1]{carmeli2010vector} to the feature
$e^{-\mathrm i\omega^\top z}$.
For every $g\in\mathcal H_0$, this proposition supplies the unique
minimum-$L^2$-norm frequency coefficient $\theta_g=\theta_g(Z,W)$ with
\begin{equation}
g(z)=\mathbb E_{Z,W}[\theta_g(Z,W)e^{\mathrm i\omega^\top z}],\ z\in\cZ.
\label{eq:upper-spectral-coefficient}
\end{equation}
The same proposition states that $g\mapsto\theta_g$ is linear and
\begin{equation}
\|\theta_g\|_{L^2}=\|g\|_{\mathcal H_0}.
\label{eq:upper-spectral-coefficient-norm}
\end{equation}
For $S>0$, we truncate on the frequency and define the real-valued high-frequency part by
\[
(T_Sg)(s,a):=\operatorname{Re}\mathbb E\left[
\theta_g\mathbf1\{\|\zeta\|_2>S\}
 e^{\mathrm i(\xi^\top s+\zeta^\top a)}\right].
\]

We can now define the common envelope using the fixed reward
and transition functions:
\[
e_l(s):=\sup_{a\in\cA}|(T_{S_l}r^1)(s,a)|
+(H-1)\int_{\cS}\sup_{a\in\cA}
|(T_{S_l}P^1(y\mid\cdot))(s,a)|\,\mu(dy).
\]
This function is nonnegative and independent of $V$.
We first verify that it bounds $|T_{S_l}\mathcal B^1V|$, using
$\mathcal B^1V=r^1+\int_{\cS}V(y)P^1(y\mid\cdot)\,\mu(dy)$.
The coefficient map in \eqref{eq:upper-spectral-coefficient} is linear,
so $T_S$ is real-linear. By the Cauchy--Schwarz inequality and
\eqref{eq:upper-spectral-coefficient-norm}, we have
\[
\|T_Sg\|_\infty
\le\|\theta_g\|_{L^2}\Pr(\|\zeta\|_2>S)^{1/2}
\le\|g\|_{\mathcal H_0}.
\]
The RKHS integral is well defined by Lemma~\ref{lem:bellman-completeness}
and the measurability argument in the proof of
Lemma~\ref{lem:finite-remainder-comparison}. We may therefore pass
this bounded real-linear map through the integral:
\[
(T_{S_l}\mathcal B^1V)(s,a)
=(T_{S_l}r^1)(s,a)+\int_{\cS}V(y)
(T_{S_l}P^1(y\mid\cdot))(s,a)\,\mu(dy).
\]
Taking absolute values and using $0\le V\le H-1$, we obtain
\begin{equation}
\begin{aligned}
|(T_{S_l}\mathcal B^1V)(s,a)|
&\le|(T_{S_l}r^1)(s,a)|+\int_{\cS}|V(y)|
|(T_{S_l}P^1(y\mid\cdot))(s,a)|\,\mu(dy)\\
&\le\sup_{a\in\cA}|(T_{S_l}r^1)(s,a)|
+(H-1)\int_{\cS}\sup_{a\in\cA}
|(T_{S_l}P^1(y\mid\cdot))(s,a)|\,\mu(dy)
=e_l(s).
\end{aligned}
\label{eq:upper-action-common-envelope}
\end{equation}
The first inequality is the triangle inequality; the second uses
$|V(y)|\le H-1$ and bounds each action value by its supremum.
It remains to control the cell suprema of these fixed reward and
transition terms. We first derive a bound for $T_{S_l}g$ valid for
every fixed $g\in\mathcal H_0$; we will then apply it to the terms
defining $e_l$ in \eqref{eq:upper-action-common-envelope}.

\paragraph{Spatial decay of the high-frequency part.}
To control many cells at once, we need a bound that decreases with
the distance between their states. We obtain it from
Lemma~\ref{lem:matern-frequency-representation}, starting with the
Gaussian tail and the Laplace transform of $W$.
The action coordinates $Z_{d_S+1},\ldots,Z_d$ are independent
$\mathcal N(0,1)$ variables by
Lemma~\ref{lem:matern-frequency-representation}. Direct integration gives
\begin{equation}
\mathbb E\exp\left\{\frac1{4d_A}\sum_{q=1}^{d_A}Z_{d_S+q}^2\right\}
=\left[\frac1{\sqrt{2\pi}}\int_{\mathbb R}
e^{-(1/2-1/(4d_A))x^2}\,dx\right]^{d_A}
=\left(1-\frac1{2d_A}\right)^{-d_A/2}\le\sqrt2.
\label{eq:upper-action-gaussian-moment}
\end{equation}
The last inequality follows from
$(1-1/(2d_A))^{d_A}\ge1-d_A/(2d_A)=1/2$.
Applying Markov's inequality to $\exp\{\sum_{q=1}^{d_A}Z_{d_S+q}^2/(4d_A)\}$, we obtain,
for $x\ge0$,
\begin{equation}
\Pr\left(\sum_{q=1}^{d_A}Z_{d_S+q}^2>x\right)
\le e^{-x/(4d_A)}\mathbb E\exp\left\{\frac1{4d_A}\sum_{q=1}^{d_A}Z_{d_S+q}^2\right\}
\le\sqrt2e^{-x/(4d_A)}.
\label{eq:upper-action-frequency-tail}
\end{equation}
The last inequality is \eqref{eq:upper-action-gaussian-moment}.
We also integrate the density of $W$ from
Lemma~\ref{lem:matern-frequency-representation}: for $c\ge0$,
\begin{equation}
\mathbb E[ e^{-cW}]
=\frac1{2^\nu\Gamma(\nu)}\int_0^\infty
w^{\nu-1}e^{-(c+1/2)w}\,dw
=\frac{\Gamma(\nu)}{2^\nu\Gamma(\nu)(c+1/2)^\nu}
=(1+2c)^{-\nu}.
\label{eq:upper-chi-square-laplace}
\end{equation}
We now estimate the high-frequency features
$T_S\Phi(s,a):=\mathbf1\{\|\zeta\|_2>S\}
e^{\mathrm i(\xi^\top s+\zeta^\top a)}$ in complex $L^2$.
We expand their inner product using the definition above:
\[
\begin{aligned}
\langle T_S\Phi(s,a),T_S\Phi(s',a')\rangle
&=\mathbb E\left[\mathbf1\{\|\zeta\|_2>S\}
e^{\mathrm i\xi^\top(s-s')+\mathrm i\zeta^\top(a-a')}\right]\\
&=\mathbb E_W\left[\mathbb E[e^{\mathrm i\xi^\top(s-s')}\mid W]
\mathbb E\left[e^{\mathrm i\zeta^\top(a-a')}
\mathbf1\{\|\zeta\|_2>S\}\mid W\right]\right].
\end{aligned}
\]
The first equality uses that the indicator's square equals itself;
the second uses conditional independence of $\xi$ and $\zeta$ given $W$.
Conditional on $W$, the coordinates of $\xi$ are independent centered
Gaussians of variance $\kappa^2/W$, so we have
\[
\mathbb E[e^{\mathrm i\xi^\top(s-s')}\mid W]
=\prod_{q=1}^{d_S}e^{-\kappa^2(s_q-s'_q)^2/(2W)}
=e^{-\kappa^2\|s-s'\|_2^2/(2W)}.
\]
Taking the absolute value inside its conditional
expectation and using \eqref{eq:upper-action-frequency-tail}, we obtain
\[
\begin{aligned}
\left|\mathbb E\left[e^{\mathrm i\zeta^\top(a-a')}
\mathbf1\{\|\zeta\|_2>S\}\mid W\right]\right|
&\le\mathbb E[\mathbf1\{\|\zeta\|_2>S\}\mid W]
=\Pr\left(\sum_{q=1}^{d_A}Z_{d_S+q}^2>WS^2/\kappa^2\,\middle|\,W\right)\\
&\le\sqrt2e^{-WS^2/(4d_A\kappa^2)}.
\end{aligned}
\]
Here we use $|e^{\mathrm i\zeta^\top(a-a')}|=1$ and
$\|\zeta\|_2^2=\kappa^2\sum_{q=1}^{d_A}Z_{d_S+q}^2/W$.
Substituting both factors into the inner product, we obtain
\begin{equation}
\left|\langle T_S\Phi(s,a),T_S\Phi(s',a')\rangle\right|
\le\sqrt2\,\mathbb E\exp\left\{
-\frac{\kappa^2\|s-s'\|_2^2}{2W}-\frac{WS^2}{4d_A\kappa^2}\right\}.
\label{eq:upper-highpass-mixture}
\end{equation}
By direct algebra, we split the second
term in the exponent and apply the AM-GM inequality:
\[
\begin{aligned}
\frac{\kappa^2\|s-s'\|_2^2}{2W}+\frac{WS^2}{4d_A\kappa^2}
&=\frac{\kappa^2\|s-s'\|_2^2}{2W}
+\frac{WS^2}{8d_A\kappa^2}+\frac{WS^2}{8d_A\kappa^2}\\
&\ge2\sqrt{\frac{\kappa^2\|s-s'\|_2^2}{2W}
\frac{WS^2}{8d_A\kappa^2}}+\frac{WS^2}{8d_A\kappa^2}
=\frac{S\|s-s'\|_2}{2\sqrt{d_A}}+\frac{WS^2}{8d_A\kappa^2}.
\end{aligned}
\]
Using \eqref{eq:upper-chi-square-laplace} with
$c=S^2/(8d_A\kappa^2)$, we conclude
\begin{equation}
\begin{aligned}
\left|\langle T_S\Phi(s,a),T_S\Phi(s',a')\rangle\right|
&\le\sqrt2e^{-S\|s-s'\|_2/(2\sqrt{d_A})}
\mathbb E e^{-WS^2/(8d_A\kappa^2)}\\
&=\sqrt2e^{-S\|s-s'\|_2/(2\sqrt{d_A})}
\left(1+\frac{S^2}{4d_A\kappa^2}\right)^{-\nu}
\le C_{\nu,\kappa}d_A^\nu S^{-2\nu}e^{-S\|s-s'\|_2/(2\sqrt{d_A})}.
\end{aligned}
\label{eq:upper-highpass-kernel}
\end{equation}
The last inequality uses
$(1+S^2/(4d_A\kappa^2))^{-\nu}
\le(S^2/(4d_A\kappa^2))^{-\nu}
=(4\kappa^2)^\nu d_A^\nu S^{-2\nu}$.

\paragraph{Bounding $e_l$ across the counting cells.}
We use the separation of cells of the same color to sum
\eqref{eq:upper-highpass-kernel} without a factor $|\mathcal D_l|$.
Suppose $M_l\ge2$, and choose $s_B\in B$ and $a_B\in\mathcal A$
for each cell of one color in $\mathcal D_l$.
Recall $M_l=\max\{1,\lfloor S_l/(64d_S^2\sqrt{d_A})\rfloor\}$.
By Lemma~\ref{lem:colored-cube-partition}, distinct cells of this color satisfy
\[
\begin{aligned}
\left\|\frac{S_ls_B}{2\sqrt{d_A}}-
\frac{S_ls_{B'}}{2\sqrt{d_A}}\right\|_2
&=\frac{S_l}{2\sqrt{d_A}}\|s_B-s_{B'}\|_2
\ge\frac{S_l}{2\sqrt{d_A}}\|s_B-s_{B'}\|_\infty\\
&\ge\frac{S_l}{2\sqrt{d_A}\,d_SM_l}
\ge\frac{S_l}{2\sqrt{d_A}\,d_S}
\frac{64d_S^2\sqrt{d_A}}{S_l}=32d_S,\ B\ne B'.
\end{aligned}
\]
The first inequality uses $\|\cdot\|_2\ge\|\cdot\|_\infty$;
the second is the same-color separation in
Lemma~\ref{lem:colored-cube-partition} after scaling by $M_l^{-1}$.
The last uses $M_l\le S_l/(64d_S^2\sqrt{d_A})$, valid when $M_l\ge2$.
For any real vector $v$ indexed by these cells, we apply
Lemma~\ref{lem:separated-exponential-quadratic}, obtained from
\citet[Theorem~3.4(i)]{golinskii2014schoenberg}, with dimension
$d_S$, points $S_ls_B/(2\sqrt{d_A})$, and coefficients $v_B$.
Expanding the squared norm and using
\eqref{eq:upper-highpass-kernel}, we obtain
\begin{equation}
\begin{aligned}
\left\|\sum_Bv_BT_{S_l}\Phi(s_B,a_B)\right\|_{L^2}^2
&=\sum_{B,B'}v_Bv_{B'}
\langle T_{S_l}\Phi(s_B,a_B),T_{S_l}\Phi(s_{B'},a_{B'})\rangle\\
&\le\sum_{B,B'}|v_Bv_{B'}|
|\langle T_{S_l}\Phi(s_B,a_B),T_{S_l}\Phi(s_{B'},a_{B'})\rangle|\\
&\le C_{\nu,\kappa}d_A^\nu S_l^{-2\nu}
\sum_{B,B'}|v_Bv_{B'}|e^{-S_l\|s_B-s_{B'}\|_2/(2\sqrt{d_A})}\le2C_{\nu,\kappa}d_A^\nu S_l^{-2\nu}\sum_Bv_B^2.
\end{aligned}
\label{eq:upper-action-feature-quadratic}
\end{equation}
The first inequality is the triangle inequality; the second bounds
each inner product by \eqref{eq:upper-highpass-kernel}.
The last is the quadratic bound
in Lemma~\ref{lem:separated-exponential-quadratic}.
For every $g\in\mathcal H_0$, we express the sum of squares using
linear combinations of the sampled values:
\begin{align*}
\sum_B|(T_{S_l}g)(s_B,a_B)|^2
&=\sup_{\|v\|_2\le1}\left|\sum_Bv_B(T_{S_l}g)(s_B,a_B)\right|^2=\sup_{\|v\|_2\le1}\left|\operatorname{Re}\mathbb E\!\left[
\theta_g\sum_Bv_BT_{S_l}\Phi(s_B,a_B)\right]\right|^2\\
&\le\|\theta_g\|_{L^2}^2\sup_{\|v\|_2\le1}
\left\|\sum_Bv_BT_{S_l}\Phi(s_B,a_B)\right\|_{L^2}^2\\
&\le2C_{\nu,\kappa}d_A^\nu S_l^{-2\nu}\|\theta_g\|_{L^2}^2
\sup_{\|v\|_2\le1}\sum_Bv_B^2
\le C_{\nu,\kappa}d_A^\nu S_l^{-2\nu}\|g\|_{\mathcal H_0}^2.
\end{align*}
The first equality is Euclidean norm duality; the second uses the
definition of $T_{S_l}g$. The first inequality is the
Cauchy--Schwarz inequality together with $|\operatorname{Re}z|\le|z|$.
The last two use \eqref{eq:upper-action-feature-quadratic} and
\eqref{eq:upper-spectral-coefficient-norm}.
For the finitely many cells of this color, the points $(s_B,a_B)$
can be chosen separately. Therefore, we obtain
\[
\sum_B\sup_{s\in B,\ a\in\cA}|(T_{S_l}g)(s,a)|^2
=\sup_{\substack{s_B\in B,\ a_B\in\cA\ \text{for each }B}}
\sum_B|(T_{S_l}g)(s_B,a_B)|^2
\le C_{\nu,\kappa}d_A^\nu S_l^{-2\nu}\|g\|_{\mathcal H_0}^2.
\]
The equality uses a supremum and does not require maximizers to exist.
Adding this bound for the $d_S+1$ colors yields
\begin{equation}
\sum_{B\in\mathcal D_l}\sup_{s\in B,\ a\in\cA}|(T_{S_l}g)(s,a)|^2
\le C_{\nu,\kappa}(d_S+1)d_A^\nu S_l^{-2\nu}\|g\|_{\mathcal H_0}^2.
\label{eq:upper-cell-highpass}
\end{equation}
If $M_l=1$, there is one cell. The Cauchy-Schwarz inequality and
\eqref{eq:upper-highpass-kernel} with $(s',a')=(s,a)$ imply
\[
\sup_{s\in\cS,\ a\in\cA}|(T_{S_l}g)(s,a)|^2
\le\|\theta_g\|_{L^2}^2
\sup_{s\in\cS,\ a\in\cA}\|T_{S_l}\Phi(s,a)\|_{L^2}^2
\le C_{\nu,\kappa}d_A^\nu S_l^{-2\nu}\|g\|_{\mathcal H_0}^2.
\]
This also implies \eqref{eq:upper-cell-highpass} because $d_S+1\ge1$.

We now apply \eqref{eq:upper-cell-highpass} to the fixed functions
$r^1$ and $P^1(y\mid\cdot)$ in \eqref{eq:upper-action-common-envelope}.
This controls one envelope for all $V$, without taking a separate
supremum over $V$ in each cell.
Applying Minkowski's inequality to the finite vector of cell suprema
in \eqref{eq:upper-action-common-envelope}, and then
\eqref{eq:upper-cell-highpass}, we obtain
\begin{equation}
\begin{aligned}
\left(\sum_B\sup_{s\in B}e_l(s)^2\right)^{1/2}
&\le\left(\sum_B\sup_{s\in B,a}|(T_{S_l}r^1)(s,a)|^2\right)^{1/2}\\
&\ +(H-1)\int\left(\sum_B\sup_{s\in B,a}
|(T_{S_l}P^1(y\mid\cdot))(s,a)|^2\right)^{1/2}\,d\mu(y)\\
&\le C_{\nu,\kappa}\sqrt{d_S+1}\,d_A^{\nu/2}S_l^{-\nu}
\left[\|r^1\|_{\mathcal H_0}+(H-1)
\int\|P^1(y\mid\cdot)\|_{\mathcal H_0}\,d\mu(y)\right]\\
&\le C_{\nu,\kappa}\bar B\sqrt{d_S+1}\,d_A^{\nu/2}S_l^{-\nu}.
\end{aligned}
\label{eq:upper-action-envelope-energy}
\end{equation}
The first inequality moves the state supremum inside the integral and
applies Minkowski's inequality; the second uses
\eqref{eq:upper-cell-highpass} for the fixed reward and transition functions.
For the last, \eqref{eq:reference-kernel-norm} and the model assumptions give
$\|r^1\|_{\mathcal H_0}+(H-1)\int\|P^1(y\mid\cdot)\|_{\mathcal H_0}\,d\mu(y)
\le\Lambda[B_r+(H-1)B_p]\le\bar B$.
Thus $e_l$ has the required cell bound. We now approximate
$(I-T_{S_l})\mathcal B^1V$; its Taylor error is the remaining term
in the desired bound for $\epsilon_l$.

\paragraph{Polynomial approximation of the low-frequency part.}
Fix $s,V,U$ and recall $g=\mathcal B^1V$.
Since $g$ is real, \eqref{eq:upper-spectral-coefficient}
and the definition of $T_{S_l}$ give
\[
(I-T_{S_l})g(s,a)
=\operatorname{Re}\mathbb E\left[
\theta_g\mathbf1\{\|\zeta\|_2\le S_l\}
e^{\mathrm i(\xi^\top s+\zeta^\top a)}\right].
\]
For any $b\in\mathbb R^{d_U}$, the triangle inequality gives
\begin{equation}
|g(s,a)-\varphi_U(a)^\top b|
\le |(T_{S_l}g)(s,a)|
+|(I-T_{S_l})g(s,a)-\varphi_U(a)^\top b|,\ a\in U.
\label{eq:upper-action-error-split}
\end{equation}
By \eqref{eq:upper-action-common-envelope}, the first term is at most
$e_l(s)$. We choose $b$ to bound the second term uniformly in $s,V,U$. Let $a_U$ be the center of the rectangle
underlying $U$. Its sides are at most $2/F_l$, so, for $a\in U$, we have
$\|a-a_U\|_2^2=\sum_{q=1}^{d_A}(a_q-(a_U)_q)^2
\le d_A/F_l^2$ and hence $\|a-a_U\|_2\le\sqrt{d_A}/F_l$.
We approximate the low-frequency part $(I-T_{S_l})g$ by expanding
$e^{\mathrm i\zeta^\top(a-a_U)}$:
\begin{equation}
\varphi_U(a)^\top b=\operatorname{Re}\mathbb E\left[
\theta_g\mathbf1\{\|\zeta\|_2\le S_l\}
 e^{\mathrm i(\xi^\top s+\zeta^\top a_U)}
\sum_{j=0}^{j_l}\frac{[\mathrm i\zeta^\top(a-a_U)]^j}{j!}
\right],\ a\in U.
\label{eq:upper-action-polynomial}
\end{equation}
The right side is a real polynomial of degree at most $j_l$.
Its coefficients are finite: for $0\le j\le j_l$, we have
\[
\mathbb E\left[|\theta_g|\,\|\zeta\|_2^j
\mathbf1\{\|\zeta\|_2\le S_l\}\right]
\le S_l^j\mathbb E|\theta_g|
\le S_l^j\|\theta_g\|_{L^2}<\infty.
\]
Thus, the basis $\varphi_U$ represents it with some
$b\in\mathbb R^{d_U}$, allowed to depend on $s,V,U$.
We bound the second term by subtracting \eqref{eq:upper-action-polynomial}
from the low-frequency part of \eqref{eq:upper-spectral-coefficient}:
\begin{equation}
\begin{aligned}
&\left|(I-T_{S_l})g(s,a)-\varphi_U(a)^\top b\right|\\
&=\left|\operatorname{Re}\mathbb E\left[
\theta_g\mathbf1\{\|\zeta\|_2\le S_l\}
e^{\mathrm i(\xi^\top s+\zeta^\top a_U)}
\left(e^{\mathrm i\zeta^\top(a-a_U)}
-\sum_{j=0}^{j_l}\frac{[\mathrm i\zeta^\top(a-a_U)]^j}{j!}\right)
\right]\right|\\
&\le\|\theta_g\|_{L^2}\left(\mathbb E\left[
\mathbf1\{\|\zeta\|_2\le S_l\}
\left|e^{\mathrm i\zeta^\top(a-a_U)}
-\sum_{j=0}^{j_l}\frac{[\mathrm i\zeta^\top(a-a_U)]^j}{j!}\right|^2
\right]\right)^{1/2}\\
&\le\frac{\|\theta_g\|_{L^2}}{(j_l+1)!}
\left(\mathbb E\left[|\zeta^\top(a-a_U)|^{2(j_l+1)}
\mathbf1\{\|\zeta\|_2\le S_l\}\right]\right)^{1/2}.
\end{aligned}
\label{eq:upper-action-taylor-bound}
\end{equation}
The equality uses $\zeta^\top a=\zeta^\top a_U+\zeta^\top(a-a_U)$.
For the first inequality, we use $|\operatorname{Re}z|\le|z|$,
the Cauchy--Schwarz inequality, and the unit modulus of the exponential
factor. The second uses the scalar Taylor inequality
$|e^{\mathrm ix}-\sum_{j=0}^{j_l}(\mathrm ix)^j/j!|
\le |x|^{j_l+1}/(j_l+1)!$ for real $x$.
We now substitute the degree and block size from
Section~\ref{sec:first-step-g-optimal} into
\eqref{eq:upper-action-taylor-bound}. The two choices balance the
Taylor error against the number of polynomial features.

\paragraph{Taylor error and dimension when $d_A<4\nu$.}
When $d_A<4\nu$, the algorithm uses $F_l=S_l$ and $j_l=k_T-1$.
For $\|\zeta\|_2\le S_l$ and $a\in U$, we have
\[
|\zeta^\top(a-a_U)|
\le\|\zeta\|_2\|a-a_U\|_2
\le S_l\frac{\sqrt{d_A}}{F_l}=\sqrt{d_A}\le2\sqrt{d_A}.
\]
The first inequality is the Cauchy–Schwarz inequality; the second
uses the frequency cutoff and the action-block radius.
Substituting this pointwise bound into
\eqref{eq:upper-action-taylor-bound} and using
\eqref{eq:upper-spectral-coefficient-norm}, we obtain
\begin{equation}
\left|(I-T_{S_l})g(s,a)-\varphi_U(a)^\top b\right|
\le\bar B\frac{(2\sqrt{d_A})^{k_T}}{k_T!}
\Pr(\|\zeta\|_2\le S_l)^{1/2}
\le\bar B\frac{(2\sqrt{d_A})^{k_T}}{k_T!}\le T^{-2}.
\label{eq:upper-small-action-taylor}
\end{equation}
The second inequality uses that a probability is at most one;
the last is the definition of $k_T$ in Section~\ref{sec:first-step-g-optimal}.
For any integer $k\ge\max\{1,4\mathrm e\sqrt{d_A},\log_2(\bar BT^2)\}$, we have
\[
\bar B\frac{(2\sqrt{d_A})^k}{k!}
\le\bar B\left(\frac{2\mathrm e\sqrt{d_A}}k\right)^k
\le\bar B2^{-k}\le T^{-2},
\]
where the first inequality uses $k!\ge(k/\mathrm e)^k$ and the
remaining inequalities use the two lower bounds on $k$.
Recall from \eqref{eq:regret-parameter-summary} that
$\bar B\le\mathfrak B$ and
$\iota_T=\ln(\mathrm e(2+d)T\mathfrak B/\delta)\ge1$.
Thus we have $\ln(\bar BT^2)\le2\ln(T\mathfrak B)\le2\iota_T$.
Since $d_A<4\nu$, the preceding choice of $k$ gives
\[
k_T\le\left\lceil\max\left\{1,4\mathrm e\sqrt{d_A},
\log_2(\bar BT^2)\right\}\right\rceil
\le1+\max\left\{1,8\mathrm e\sqrt\nu,
\frac{2\iota_T}{\ln2}\right\}\le C_\nu\iota_T.
\]
We therefore obtain
\begin{equation}
\begin{aligned}
\frac{r_l}{L_l}&=\binom{d_A+k_T-1}{d_A}
\le\frac{(d_A+k_T-1)^{d_A}}{d_A!}
\le\left[\frac{\mathrm e(d_A+k_T-1)}{d_A}\right]^{d_A}\le[\mathrm e(1+C_\nu\iota_T)]^{d_A}
\le C_\nu\iota_T^{d_A}.
\end{aligned}
\label{eq:upper-action-small-rank}
\end{equation}
The first two inequalities bound the binomial numerator and use
$d_A!\ge(d_A/\mathrm e)^{d_A}$. The last two substitute
$k_T\le C_\nu\iota_T$ and use $d_A<4\nu$, so constants raised to
$d_A$ are absorbed into $C_\nu$.

\paragraph{Taylor error and dimension when $d_A\ge4\nu$.}
When $d_A\ge4\nu$, the algorithm instead fixes
$j_l=k_A-1=\lceil2\nu\rceil$ and reduces the action-block size.
The fixed degree ensures polynomial dependence on $d$:
\begin{equation}
\begin{aligned}
\frac{r_l}{L_l}&=\binom{d_A+k_A-1}{k_A-1}
 =\frac{\prod_{j=1}^{k_A-1}(d_A+j)}{(k_A-1)!}
\\
&\le\frac{(d_A+k_A-1)^{k_A-1}}{(k_A-1)!}\le\frac{[k_A(2+d)]^{k_A-1}}{(k_A-1)!}
\le C_\nu(2+d)^{\lceil2\nu\rceil}.
\end{aligned}
\label{eq:upper-action-fixed-rank}
\end{equation}
The first inequality bounds each factor in the product; the second
uses $d_A+k_A-1\le k_A(2+d)$. The last uses
$k_A-1=\lceil2\nu\rceil$ and absorbs $k_A^{k_A-1}/(k_A-1)!$
into $C_\nu$.
For the fixed-degree Taylor error, we use
\eqref{eq:upper-action-frequency-tail} and
\eqref{eq:upper-chi-square-laplace} to bound the truncated frequency
moment in \eqref{eq:upper-action-taylor-bound}.
Since $\|\zeta\|_2^2=\kappa^2\sum_{q=1}^{d_A}Z_{d_S+q}^2/W$,
conditioning on $W$ gives, for $x>0$,
\begin{equation}
\begin{aligned}
\Pr(\|\zeta\|_2>x)
&=\mathbb E_W\Pr\left(\sum_{q=1}^{d_A}Z_{d_S+q}^2>Wx^2/\kappa^2\,\middle|\,W\right)
\le\sqrt2\,\mathbb E e^{-Wx^2/(4d_A\kappa^2)}\\
&=\sqrt2\left(1+\frac{x^2}{2d_A\kappa^2}\right)^{-\nu}
\le\sqrt2\left(\frac{x^2}{2d_A\kappa^2}\right)^{-\nu}
\le C_{\nu,\kappa}d_A^\nu x^{-2\nu}.
\end{aligned}
\label{eq:upper-action-radial-tail}
\end{equation}
The first inequality applies \eqref{eq:upper-action-frequency-tail}
conditionally on $W$; the next equality uses
\eqref{eq:upper-chi-square-laplace} with $c=x^2/(4d_A\kappa^2)$.
The remaining inequalities use $(1+x)^{-\nu}\le x^{-\nu}$ and absorb
$\sqrt2(2\kappa^2)^\nu$ into $C_{\nu,\kappa}$.
To integrate this tail, we first write
\[
\min\{\|\zeta\|_2,S_l\}^{2k_A}
=\int_0^{S_l}2k_Ax^{2k_A-1}
\mathbf1\{\|\zeta\|_2>x\}\,dx.
\]
The integrand is nonnegative, so interchanging expectation and integration
gives
\begin{equation}
\begin{aligned}
\mathbb E[\|\zeta\|_2^{2k_A}\mathbf1\{\|\zeta\|_2\le S_l\}]
&\le\mathbb E\min\{\|\zeta\|_2,S_l\}^{2k_A}
=2k_A\int_0^{S_l}x^{2k_A-1}\Pr(\|\zeta\|_2>x)\,dx\\
&\le2k_AC_{\nu,\kappa}d_A^\nu
\int_0^{S_l}x^{2k_A-2\nu-1}\,dx
=C_{\nu,\kappa}\frac{k_A}{k_A-\nu}d_A^\nu S_l^{2k_A-2\nu}.
\end{aligned}
\label{eq:upper-action-radial-moment}
\end{equation}
The second inequality uses \eqref{eq:upper-action-radial-tail};
the last evaluates the integral, using $k_A>2\nu>\nu$.
The direction of $\zeta$ equals that of the Gaussian action coordinates
and is independent of their length and of $W$. We compute its moment
by factoring the corresponding Gaussian moment:
\begin{equation}
\begin{aligned}
\mathbb E\left[
\left(\frac{|\zeta^\top(a-a_U)|}{\|\zeta\|_2}\right)^{2k_A}\right]
&=\frac{\mathbb E\left|\sum_{q=1}^{d_A}
Z_{d_S+q}(a-a_U)_q\right|^{2k_A}}
{\mathbb E\left(\sum_{q=1}^{d_A}Z_{d_S+q}^2\right)^{k_A}}\\
&=\frac{(2k_A-1)!!\|a-a_U\|_2^{2k_A}}
{2^{k_A}\Gamma(d_A/2+k_A)/\Gamma(d_A/2)}
=\frac{(2k_A-1)!!\|a-a_U\|_2^{2k_A}}
{d_A(d_A+2)\cdots(d_A+2k_A-2)}.
\end{aligned}
\label{eq:upper-action-direction-moment}
\end{equation}
The first equality uses independence of Gaussian direction and length.
In the second, the numerator is the $2k_A$-th moment of a centered
Gaussian with variance $\|a-a_U\|_2^2$, while the denominator is
the $k_A$-th moment of $\sum_{q=1}^{d_A}Z_{d_S+q}^2$, which has the chi-square distribution with $d_A$ degrees of freedom. The last equality uses
$\Gamma(d_A/2+k_A)/\Gamma(d_A/2)=\prod_{j=0}^{k_A-1}(d_A/2+j)$.
Combining \eqref{eq:upper-action-radial-moment} and
\eqref{eq:upper-action-direction-moment}, and using independence
of direction and length, we obtain
\begin{equation}
\begin{aligned}
&\mathbb E[|\zeta^\top(a-a_U)|^{2k_A}\mathbf1\{\|\zeta\|_2\le S_l\}]\\
&=\mathbb E[\|\zeta\|_2^{2k_A}\mathbf1\{\|\zeta\|_2\le S_l\}]
\frac{(2k_A-1)!!\|a-a_U\|_2^{2k_A}}
{d_A(d_A+2)\cdots(d_A+2k_A-2)}\\
&\le C_{\nu,\kappa}d_A^\nu S_l^{2k_A-2\nu}
\frac{(2k_A-1)!!(\sqrt{d_A}/F_l)^{2k_A}}{d_A^{k_A}}
\le C_{\nu,\kappa}d_A^\nu F_l^{-2k_A}S_l^{2k_A-2\nu}.
\end{aligned}
\label{eq:upper-action-projected-moment}
\end{equation}
The first inequality uses \eqref{eq:upper-action-radial-moment},
$\|a-a_U\|_2\le\sqrt{d_A}/F_l$, and
$\prod_{j=0}^{k_A-1}(d_A+2j)\ge d_A^{k_A}$.
The second cancels $d_A^{k_A}$; all remaining factors depend only on $\nu$.
For the algorithm's choice
$F_l=T^{\nu/[k_A(2\nu+d_A)]}S_l^{(k_A-\nu)/k_A}$, we have
\[
\begin{aligned}
F_l^{-k_A}S_l^{k_A-\nu}
&=\left(T^{\nu/[k_A(2\nu+d_A)]}
S_l^{(k_A-\nu)/k_A}\right)^{-k_A}S_l^{k_A-\nu}\\
&=T^{-\nu/(2\nu+d_A)}S_l^{-(k_A-\nu)}S_l^{k_A-\nu}
=T^{-\nu/(2\nu+d_A)}.
\end{aligned}
\]
Applying \eqref{eq:upper-action-taylor-bound} with $j_l+1=k_A$
and $\|\theta_g\|_{L^2}\le\bar B$, we obtain
\begin{equation}
\begin{aligned}
\left|(I-T_{S_l})g(s,a)-\varphi_U(a)^\top b\right|
&\le\frac{\bar B}{k_A!}
\left(\mathbb E[|\zeta^\top(a-a_U)|^{2k_A}
\mathbf1\{\|\zeta\|_2\le S_l\}]\right)^{1/2}\\
&\le C_{\nu,\kappa}\bar B d_A^{\nu/2}F_l^{-k_A}S_l^{k_A-\nu}
=C_{\nu,\kappa}\bar B d_A^{\nu/2}T^{-\nu/(2\nu+d_A)}.
\end{aligned}
\label{eq:upper-large-action-taylor}
\end{equation}
The second inequality uses \eqref{eq:upper-action-projected-moment};
the equality substitutes the displayed formula for $F_l$.

\paragraph{Summing over the selected episodes.}
For each $s,V,U$, the polynomial in \eqref{eq:upper-action-polynomial}
is an admissible choice in the infimum defining $\epsilon_l(s)$.
Taking the maximum over blocks and supremum over $V$ in
\eqref{eq:upper-action-error-split}, and applying
\eqref{eq:upper-action-common-envelope},
\eqref{eq:upper-small-action-taylor}, and
\eqref{eq:upper-large-action-taylor}, we obtain
\[
\epsilon_l(s)\le e_l(s)+
\begin{cases}
T^{-2},&d_A<4\nu,\\
C_{\nu,\kappa}\bar B d_A^{\nu/2}T^{-\nu/(2\nu+d_A)},&d_A\ge4\nu.
\end{cases}
\]
Both terms on the right are independent of $V$ and $U$.
By Lemma~\ref{lem:upper-partition-counts}, each counting cell contains
at most $N_l$ visits assigned to level $l$. Consequently, we obtain
\begin{equation}
\begin{aligned}
\sum_{t:l_t=l}e_l(s_t^1)
&\le\sqrt{T_l}\left(\sum_{t:l_t=l}e_l(s_t^1)^2\right)^{1/2}
=\sqrt{T_l}\left(\sum_{B\in\mathcal D_l}
\sum_{\substack{t:l_t=l\\s_t^1\in B}}e_l(s_t^1)^2\right)^{1/2}\\
&\le\sqrt{T_l}\left(\sum_{B\in\mathcal D_l}
|\{t:l_t=l,\ s_t^1\in B\}|\sup_{s\in B}e_l(s)^2\right)^{1/2}\le\sqrt{N_lT_l}\left(\sum_{B\in\mathcal D_l}
\sup_{s\in B}e_l(s)^2\right)^{1/2}.
\end{aligned}
\label{eq:upper-envelope-counting}
\end{equation}
The first inequality is the Cauchy-Schwarz inequality, and the equality
groups visits by their cells. The second inequality bounds each value
by its cell supremum; the last uses
$|\{t:l_t=l,\ s_t^1\in B\}|\le N_l$ from
Lemma~\ref{lem:upper-partition-counts}.
Summing over the selected visits and using \eqref{eq:upper-envelope-counting}
and \eqref{eq:upper-action-envelope-energy}, we bound the envelope contribution
by $C_{\nu,\kappa}\bar B\sqrt{d_S+1}\,d_A^{\nu/2}
S_l^{-\nu}\sqrt{N_lT_l}$. The uniform remainder is added once per visit:
\[
\sum_{t:l_t=l}\epsilon_l(s_t^1)
\le C_{\nu,\kappa}\bar B\sqrt{d_S+1}\,d_A^{\nu/2}
S_l^{-\nu}\sqrt{N_lT_l}
+\begin{cases}
T_lT^{-2},&d_A<4\nu,\\
C_{\nu,\kappa}\bar B d_A^{\nu/2}T_lT^{-\nu/(2\nu+d_A)},&d_A\ge4\nu.
\end{cases}
\]
This proves \eqref{eq:upper-first-approximation-sum}.
The dimension bounds are \eqref{eq:upper-action-small-rank} and
\eqref{eq:upper-action-fixed-rank}.
\end{proof}

\subsection{Proof of Lemma~\ref{lem:upper-remainder-cover}}

\begin{proof}
Recall that $\mathcal R_l$ consists of
$\varrho_\Theta(z,V)=\langle\Theta,\psi(z,V)\rangle$, where
$\psi(z,V)=\Psi_{\mathcal C_l}(z)\otimes(1,V/(H-1))$ and
$\|\Theta\|\le B_1:=\Lambda\sqrt{B_r^2+(H-1)^2B_p^2}\le\bar B$.
If $B_1=0$, then $\mathcal R_l=\{0\}$ and $M=1$; assume $B_1>0$.
The uniform distance between two remainder models is
\[
\|\varrho_\Theta-\varrho_{\Theta'}\|_\infty
=\sup_{z,V}|\langle\Theta-\Theta',\psi(z,V)\rangle|
=:\|\Theta-\Theta'\|_\psi.
\]
Let $W_l:=\mathbb E\sup_{z,V}|G(\psi(z,V))|$, where $G$ is a centered
Gaussian family on $\mathcal H_0\otimes(\mathbb R\oplus L^2(\mu))$
with covariance $\mathbb E[G(v)G(w)]=\langle v,w\rangle$.
Lemma~\ref{lem:dual-sudakov} reduces the required covering bound to
\begin{equation}
\ln M\le\frac{CB_1^2W_l^2}{\varepsilon^2};
\label{eq:upper-remainder-packing}
\end{equation}
we justify its finite-dimensional application at the end.
To bound $W_l$, we first control the size and variation of
$\Psi_{\mathcal C_l}$, then use Lemma~\ref{lem:chevet} to include $V$.

By \eqref{eq:initial-partitions}, each regression cell is
$c=P\times U$ with $P\in\mathcal P_l$ and $U\in\mathscr U(S_l)$.
On $\overline c=\overline P\times\operatorname{cl}(U)$, we use
$\Psi_c:=\Phi-\Pi_c\Phi$, where $\Pi_c$ is defined in
\eqref{eq:legendre-projection-kernel}.
By \eqref{eq:feature-decomposition}, this agrees with
$\Psi_{\mathcal C_l}$ on $c$.
The state sides satisfy, by \eqref{eq:state-partition-mesh},
\[
\frac1{S_l}\le\frac1{\widetilde M_l}
\le\frac{8192(d_S+1)d_S^{5/2}\sqrt{d_A}}{S_l}.
\]
For the action sides, recall from
Section~\ref{subsec:state-action-partition} that
$m_{S_l}=\lfloor\log_2S_l\rfloor$ and that coordinate $q$ is divided
into $n_{S_l,q}\in\{2^{m_{S_l}},2^{m_{S_l}+1}\}$ equal intervals.
Thus, we have
\[
\frac{S_l}{2}\le2^{m_{S_l}}\le n_{S_l,q}
\le2^{m_{S_l}+1}\le2S_l\Rightarrow
\frac1{2S_l}\le\frac1{n_{S_l,q}}\le\frac2{S_l}.
\]
Since $d_S,d_A\ge1$, both kinds of side therefore satisfy
\[
\frac1{2S_l}\le u_{c,q}-\ell_{c,q}
\le\frac{8192(d_S+1)d_S^{5/2}\sqrt{d_A}}{S_l},\ q\in[d].
\]
Applying \eqref{eq:appendix-rectangle-remainder} with this maximum
side length, we obtain
\begin{equation}
\max_c\sup_{z\in\overline c}\|\Psi_c(z)\|
\le C_\nu(1+B_{\phi})
\left[\frac{8192\kappa\sqrt d\,(d_S+1)d_S^{5/2}\sqrt{d_A}}{S_l}\right]^\nu
\le A:=C_{\nu,\kappa}(2+d)^{\rho/2+9\nu/2}S_l^{-\nu}.
\label{eq:upper-remainder-feature-size}
\end{equation}
For the last inequality, recall $\rho=\lfloor\nu\rfloor$ and
$B_{\phi}^2=3^\rho D$. We have
\[
D=\binom{d+\rho}{\rho}
=\frac{\prod_{j=1}^{\rho}(d+j)}{\rho!}
\le\frac{(d+\rho)^\rho}{\rho!}
\le C_\nu(2+d)^\rho,
\]
so we have $1+B_{\phi}\le C_\nu(2+d)^{\rho/2}$. Moreover, we have
\[
\sqrt d\,(d_S+1)d_S^{5/2}\sqrt{d_A}
\le(2+d)^{1/2+1+5/2+1/2}=(2+d)^{9/2}.
\]
These two bounds give the stated power of $2+d$ in $A$.

For the variation, Lemma~\ref{lem:local-legendre-projection} gives
$\|\Pi_c\Phi\|_\infty\le B_{\phi}\|\Phi\|_\infty=B_{\phi}$.
The polynomial $\Pi_c\Phi$ has total degree at most $\rho$.
Applying Lemma~\ref{lem:polynomial-markov} and using the minimum side length,
we obtain
\begin{equation}
\|(\Pi_c\Phi)(z)-(\Pi_c\Phi)(z')\|
\le2\rho^2B_{\phi}\max_{q\in[d]}
\frac{|z_q-z'_q|}{u_{c,q}-\ell_{c,q}}
\le4\rho^2B_{\phi}S_l\|z-z'\|_2.
\label{eq:upper-remainder-polynomial-gradient}
\end{equation}
When $\rho=0$, we have $(\Pi_c\Phi)(z)-(\Pi_c\Phi)(z')=0$.
Combining this with \eqref{eq:appendix-feature-holder}, for
$s:=\min\{\nu/2,1/2\}$ and $z,z'\in\overline c$ with
$\|z-z'\|_2\le1$, we have
\begin{equation}
\begin{aligned}
\|\Psi_c(z)-\Psi_c(z')\|
&\le\|\Phi(z)-\Phi(z')\|
+\|(\Pi_c\Phi)(z)-(\Pi_c\Phi)(z')\|\\
&\le C_{\nu,\kappa}\|z-z'\|_2^s
+4\rho^2B_{\phi}S_l\|z-z'\|_2
\le C_{\nu,\kappa}(1+S_lD)\|z-z'\|_2^s.
\end{aligned}
\label{eq:upper-remainder-feature-holder}
\end{equation}
The last inequality uses $B_{\phi}=3^{\rho/2}\sqrt D\le C_\nu D$ and $s\le1$.

Let $\mathcal T_l:=\bigcup_{c\in\mathcal C_l}\Psi_c(\overline c)$.
To cover these features at radius $\varepsilon A$, we use a coordinate mesh
of size at most
\[
d^{-1/2}\min\left\{1,
\left[\frac{\varepsilon A}{C_{\nu,\kappa}(1+S_lD)}\right]^{1/s}\right\}
\]
in each rectangle; \eqref{eq:upper-remainder-feature-holder} bounds
the error by $\varepsilon A$.
Each side has length at most one by the partition definitions.
Counting the grid points and using
$|\mathcal C_l|\le S_l^d$ from Lemma~\ref{lem:upper-partition-counts},
we obtain
\begin{equation}
\begin{aligned}
\ln\mathcal N(\varepsilon A,\mathcal T_l,\|\cdot\|_{\mathcal H_0})
&\le d\ln S_l+d\ln\left[1+C\sqrt d\max\left\{1,
\left(\frac{C_{\nu,\kappa}(1+S_lD)}{\varepsilon A}\right)^{1/s}\right\}\right]\\
&\le C_{\nu,\kappa}d[\ln(\mathrm e+d+S_l)+\ln(1/\varepsilon)],\ 0<\varepsilon\le1.
\end{aligned}
\label{eq:upper-remainder-feature-cover}
\end{equation}
The second inequality uses
$(1+S_lD)/A\le C_{\nu,\kappa}S_l^{\nu+1}(2+d)^{\rho/2}$;
$s$ and $\rho$ depend only on $\nu$.
Integrating this bound and recalling $S_l\le T^{1/\nu}$, we get
\begin{equation}
\begin{aligned}
\int_0^1\sqrt{\ln\left(2\mathcal N\left(\frac{\varepsilon A}{2\sqrt2},\mathcal T_l,\|\cdot\|\right)\right)}\,d\varepsilon
&\le C_{\nu,\kappa}\left[\sqrt{d\ln(\mathrm e+d+S_l)}
+\sqrt d\int_0^1\sqrt{\ln(1/\varepsilon)}\,d\varepsilon\right]\le C_{\nu,\kappa}\sqrt{d\iota_T}.
\end{aligned}
\label{eq:upper-remainder-scalar-entropy}
\end{equation}
Here we use $\sqrt{x+y}\le\sqrt x+\sqrt y$,
$\int_0^1\sqrt{\ln(1/\varepsilon)}\,d\varepsilon=\sqrt\pi/2$, and the definition
of $\iota_T$ in \eqref{eq:regret-parameter-summary}.

We now apply Lemma~\ref{lem:chevet} to the two factors of $\psi$.
The first factor lies in $\mathcal T_l$ and has norm at most $A$.
For the second, $\mu(\mathcal S)=1$ and $0\le V\le H-1$ imply
\[
\|(1,V/(H-1))\|^2
=1+\frac{\|V\|_{L^2(\mu)}^2}{(H-1)^2}\le2.
\]
Its increments have norm
$\|V-V'\|_{L^2(\mu)}/(H-1)\le\|V-V'\|_\infty/(H-1)$.
Applying Lemma~\ref{lem:gaussian-entropy-comparison} separately to
the two factors, we combine \eqref{eq:upper-remainder-scalar-entropy}
with Lemma~\ref{lem:chevet} to obtain
\begin{equation}
\begin{aligned}
W_l
&\le C_{\nu,\kappa}A\left[\sqrt{d\iota_T}+1+\frac1{H-1}\int_0^H
\sqrt{\ln(2\mathcal N(\varepsilon,\mathcal V_1,\|\cdot\|_\infty))}\,d\varepsilon\right]\le C_{\nu,\kappa}A\left(\sqrt{d\iota_T}+\frac J{H-1}\right).
\end{aligned}
\label{eq:upper-lifted-width}
\end{equation}
For the first factor, adding zero gives
$\mathcal N(x,\mathcal T_l\cup\{0\})\le\mathcal N(x,\mathcal T_l)+1$.
The substitution $x=2A\varepsilon$ and \eqref{eq:upper-remainder-scalar-entropy} yield
\[
\begin{aligned}
\int_0^{2A}\sqrt{\ln(2\mathcal N(x,\mathcal T_l\cup\{0\}))}\,dx
&=2A\int_0^1\sqrt{\ln(2\mathcal N(2A\varepsilon,\mathcal T_l\cup\{0\}))}\,d\varepsilon\\
&\le C_{\nu,\kappa}A\sqrt{d\iota_T}.
\end{aligned}
\]
Here covering numbers decrease with the radius, and
$\ln(2(N+1))\le2\ln(2N)$ for $N\ge1$.
For the second factor, the base point is $(1,0)$, whose Gaussian
coordinate has expected absolute value $\sqrt{2/\pi}\le1$.
The second inequality uses \eqref{eq:continuation-entropy-bound}
and $d\iota_T\ge1$.
The sets are separable, and the finite bounds in
\eqref{eq:upper-remainder-scalar-entropy} and
\eqref{eq:continuation-entropy-bound} allow us to apply
Lemmas~\ref{lem:gaussian-entropy-comparison} and~\ref{lem:chevet}.

To justify \eqref{eq:upper-remainder-packing}, take any finite family
$\{\varrho_{\theta_j}\}\subseteq\mathcal R_l$ whose pairwise distances
exceed $\varepsilon$.
Let $\Pi_\psi$ denote the orthogonal projection onto
$\overline{\operatorname{span}}\{\psi(z,V)\}$. For every coefficient, we have
\[
\langle\Pi_\psi\theta_j,\psi(z,V)\rangle
=\langle\theta_j,\Pi_\psi\psi(z,V)\rangle
=\langle\theta_j,\psi(z,V)\rangle,\
\|\Pi_\psi\theta_j\|\le\|\theta_j\|\le B_1.
\]
Thus these projections represent the same functions. Let
$E:=\operatorname{span}\{\Pi_\psi\theta_j\}$ and let $\Pi_E$ be its
orthogonal projection. If $v\in E$ and $\|v\|_\psi=0$, then $v$ is
orthogonal to the closed span containing $E$, so $\|v\|^2=0$.
Hence $\|\cdot\|_\psi$ is a norm on $E$.
For the standard Gaussian vector $G_E$ in $E$, we have
\[
\mathbb E|G(\Pi_Ev)-G(\Pi_Ew)|^2
=\|\Pi_E(v-w)\|^2\le\|v-w\|^2
=\mathbb E|G(v)-G(w)|^2.
\]
Applying Lemma~\ref{lem:gaussian-entropy-comparison} to the signed
features $v,w\in\{\pm\psi(z,V)\}$, we obtain
$$\mathbb E\|G_E\|_\psi
=\mathbb E\sup_{z,V}|G(\Pi_E\psi(z,V))|\le W_l.$$
An $\varepsilon/2$-ball contains at most one member of the separated
family. Lemma~\ref{lem:dual-sudakov} therefore yields
\[
\ln|\{\varrho_{\theta_j}\}|
\le\ln\mathcal N(\varepsilon/2,\{u\in E:\|u\|\le B_1\},\|\cdot\|_\psi)
\le\frac{CB_1^2}{\varepsilon^2}(\mathbb E\|G_E\|_\psi)^2
\le\frac{CB_1^2W_l^2}{\varepsilon^2}.
\]
The bound on every finite separated family makes greedy selection
terminate at an internal $\varepsilon$-net, which proves
\eqref{eq:upper-remainder-packing}.
Finally, substituting \eqref{eq:upper-lifted-width}, $B_1\le\bar B$,
and $A$ from \eqref{eq:upper-remainder-feature-size}, we obtain
\[
\ln M\le\frac{C_{\nu,\kappa}\bar B^2A^2}{\varepsilon^2}
\left(d\iota_T+\frac{J^2}{(H-1)^2}\right)
\le\frac{C_{\nu,\kappa}\bar B^2S_l^{-2\nu}}{\varepsilon^2}
(2+d)^{\rho+9\nu}\left(d\iota_T+\frac{J^2}{(H-1)^2}\right).
\]
The first inequality uses $(x+y)^2\le2x^2+2y^2$.
This proves \eqref{eq:upper-remainder-cover}; the class and the net
construction are fixed independently of the observations.
\end{proof}
\subsection{Proof of Lemma~\ref{lem:upper-epoch-comparison}}

\begin{proof}
Fix an epoch $(l,u)$, with capacity $N=2^u$, confidence allocation
$\delta_{l,u}=\delta/[64(l+1)^2(u+1)^2]$, and $M$ models in its fixed
$\varepsilon$-net. We compare the best pair in
$\mathcal A_T\times\mathcal F_t$ with the executed pair $(a_t^1,f_t)$.
All sums and products below include the visits to this epoch from its start through the visit under consideration.
For the probability bounds, once $\mathcal F_t=\varnothing$, we set
all subsequent update multipliers to one and all subsequent increments
to zero. The condition $\mathcal F_t=\varnothing$ is known before
sampling the model and action. These bounds use the original probability law, without
conditioning on the event of Lemma~\ref{lem:initial-confidence}.

\paragraph{Reduction to the fixed model.}
Recall from Section~\ref{sec:first-step-scores} that, for $z\in c$,
\[
\bar b_t(z)=\overline\beta(|\mathcal C_l|,\delta_{l,u})w_t(z)
+\varepsilon\bigl(1+\sqrt{|\widetilde{\mathcal T}_t(c)|}\,w_t(z)\bigr).
\]
Here $\widetilde{\mathcal T}_t(c)$ contains the earlier observations
in $c$ from this epoch, $w_t(z)=\|\phi_c(z)\|_{G_t(c)^{-1}}$ is defined
in \eqref{eq:layered-width}, and $\overline\beta$ is given by
\eqref{eq:regression-confidence-parameter}.
Lemma~\ref{lem:finite-remainder-comparison} supplies a fixed index
$j_*$, independent of the observations. On the event of
Lemma~\ref{lem:initial-confidence}, we have
\[
\begin{aligned}
\sup_{a\in\mathcal A_T,\ f\in\mathcal F_t}
(\mathcal B^1V_f^2)(s_t^1,a)
&\le\min\{H,\max_{a,f}\widehat Q_{t,j_*}(a,f)\}=\min\{H,\widehat Q_{t,j_*}(a_t(j_*),f_t(j_*))\}\\
&\le\min\{H,(\mathcal B^1V_{f_t(j_*)}^2)(s_t^1,a_t(j_*))
+2\bar b_t(s_t^1,a_t(j_*))\}\\
&\le(\mathcal B^1V_{f_t(j_*)}^2)(s_t^1,a_t(j_*))
+\min\{H,2\bar b_t(s_t^1,a_t(j_*))\}.
\end{aligned}
\]
The first and second inequalities use
Lemma~\ref{lem:finite-remainder-comparison} and the fact that Bellman backups lie in $[0,H]$.
The equality uses \eqref{eq:continuation-choice}; the last inequality
uses the nonnegativity of the Bellman backup.

To use the weight update, recall from \eqref{eq:first-step-feedback-estimate}
that $\widehat\mu_t(a,V)=h_t(a,a_t^1)(R_t^1+V(s_t^2))$, with $h_t$
chosen by \eqref{eq:minimum-variance-feedback}. Equation~\eqref{eq:expert-weight-update}
defines
\[
\widehat\ell_{t,j}=1-\frac{\widehat\mu_t(a_t(j),V_{f_t(j)}^2)
+\min\{H,2\bar b_t(s_t^1,a_t(j))\}}{2H}.
\]
Take $\mathbb E_t$ as expectation conditioned on the history and $s_t^1$
before sampling the model and action. The recommendations, probability
vector $\omega_t$, and $\bar b_t$ are already fixed, so we have
\[
2H(1-\mathbb E_t\widehat\ell_{t,j})
=\mathbb E_t\widehat\mu_t(a_t(j),V_{f_t(j)}^2)
+\min\{H,2\bar b_t(s_t^1,a_t(j))\}.
\]
Using the action-approximation error $\epsilon_l$ from
Lemma~\ref{lem:upper-action-approximation} and $d_U\le r_l/L_l$,
we apply Lemma~\ref{lem:first-step-feedback} to obtain
\[
\left|\mathbb E_t\widehat\mu_t(a_t(j),V_{f_t(j)}^2)
-(\mathcal B^1V_{f_t(j)}^2)(s_t^1,a_t(j))\right|
\le(1+2\sqrt{r_l/L_l})\epsilon_l(s_t^1).
\]
By the sampling rule in \eqref{eq:first-step-sampling}, the
nonexploration branch uses $(a_t(j),f_t(j))$ with probability
$(1-\gamma)\omega_{t,j}$. We therefore have
\begin{equation}
\begin{aligned}
&\sup_{a,f}(\mathcal B^1V_f^2)(s_t^1,a)
-\mathbb E_t\left[(\mathcal B^1V_{f_t}^2)(s_t^1,a_t^1)
+\min\{H,2\bar b_t(z_t^1)\}\right]\\
\le&(\mathcal B^1V_{f_t(j_*)}^2)(s_t^1,a_t(j_*))
+\min\{H,2\bar b_t(s_t^1,a_t(j_*))\}\\
-&(1-\gamma)\sum_j\omega_{t,j}
\left[(\mathcal B^1V_{f_t(j)}^2)(s_t^1,a_t(j))
+\min\{H,2\bar b_t(s_t^1,a_t(j))\}\right]\\
\le&2H(1-\mathbb E_t\widehat\ell_{t,j_*})
-2H\left(1-\sum_j\omega_{t,j}\mathbb E_t\widehat\ell_{t,j}\right)
+2(1+2\sqrt{r_l/L_l})\epsilon_l(s_t^1)+2H\gamma\\
=&2H\left[\sum_j\omega_{t,j}\mathbb E_t\widehat\ell_{t,j}
-\mathbb E_t\widehat\ell_{t,j_*}\right]
+2(1+2\sqrt{r_l/L_l})\epsilon_l(s_t^1)+2H\gamma.
\end{aligned}
\label{eq:upper-epoch-loss-reduction}
\end{equation}
The first inequality discards the nonnegative exploration contribution.
For the second, each bracket lies in $[0,2H]$, so removing $1-\gamma$ costs at most
$\gamma\sum_j\omega_{t,j}\,2H=2H\gamma$; we then substitute the
identity for $\mathbb E_t\widehat\ell_{t,j}$ and the feedback bound twice.

\paragraph{Control of the conditional loss difference.}
If $r_l\ln(2M/\delta_{l,u})\ge N\ln(8L_lN/\delta_{l,u})$, then we have
\[
H\sqrt{r_lN\ln\frac{8L_lN}{\delta_{l,u}}\ln\frac{2M}{\delta_{l,u}}}
\ge HN\ln\frac{8L_lN}{\delta_{l,u}}\ge HN,
\]
which bounds the target since each backup difference is at most $H$.
Otherwise, \eqref{eq:first-step-tuning} gives
\[
\eta=\frac1{256}\sqrt{\frac{\ln(2M/\delta_{l,u})}
{r_lN\ln(8L_lN/\delta_{l,u})}},\
\gamma=128r_l\eta
=\frac12\sqrt{\frac{r_l\ln(2M/\delta_{l,u})}
{N\ln(8L_lN/\delta_{l,u})}}.
\]
Recall the predictable second-moment bound in
\eqref{eq:expert-weight-update},
$v_{t,j}=2+\frac12\sum_{x\in\mathcal A_T}p_t(x)h_t(a_t(j),x)^2$.
By the displayed definition of $\widehat\ell_{t,j}$ and
Lemma~\ref{lem:first-step-feedback}, we have
\begin{equation}
\begin{aligned}
\mathbb E_t\widehat\ell_{t,j}^2
&\le2\left(1-\frac{\min\{H,2\bar b_t(s_t^1,a_t(j))\}}{2H}\right)^2\\
&\ +\frac{\mathbb E_t\widehat\mu_t(a_t(j),V_{f_t(j)}^2)^2}{2H^2}\\
&\le2+\frac12\sum_xp_t(x)h_t(a_t(j),x)^2=v_{t,j}.
\end{aligned}
\label{eq:upper-update-loss-moment}
\end{equation}
The first inequality is $(x-y)^2\le2x^2+2y^2$; the second uses
the clipped bonus in $[0,H]$ and the feedback second-moment bound.

Also, \eqref{eq:minimum-variance-feedback} and $0\le R_t^1+V(s_t^2)\le H$
give
\[
|\widehat\mu_t(a,V)|\le Hr_l/\gamma.
\]
Together with $\sum_xp_t(x)h_t(a_t(j),x)^2\le r_l/\gamma$
from Lemma~\ref{lem:first-step-feedback}, this yields
\begin{equation}
\max\{|\widehat\ell_{t,j}|,v_{t,j}\}\le2+\frac{r_l}{2\gamma},\
\eta\left(2+\frac{r_l}{2\gamma}\right)
=2\eta+\frac1{256}\le\frac3{256}.
\label{eq:upper-update-size}
\end{equation}
Here $\eta\le1/256$ follows from \eqref{eq:first-step-tuning}.
Thus we have $1-\eta \widehat\ell_{t,j}+4\eta^2v_{t,j}\ge1-\eta|\widehat\ell_{t,j}|
\ge1-3/256\ge3/4$.

We now control the square root of the weight ratio. For $|x|\le1/4$ and
$0\le y\le1/4$, direct algebra gives
\[
\frac1{1-x+y}=1+x+\frac{x^2}{1-x}
-\frac{y}{(1-x)(1-x+y)}
\le1+x+2x^2-y/2,
\]
since $1-x\ge3/4$ and $(1-x)(1-x+y)\le15/8<2$.
By \eqref{eq:upper-update-size}, we have
$\eta|\widehat\ell_{t,j_*}|\le3/256$ and $4\eta^2v_{t,j_*}\le3/16384$.
Substituting $x=\eta \widehat\ell_{t,j_*}$ and $y=4\eta^2v_{t,j_*}$,
then using \eqref{eq:upper-update-loss-moment}, we obtain
\begin{equation}
\mathbb E_t\frac1{1-\eta \widehat\ell_{t,j_*}+4\eta^2v_{t,j_*}}
\le1+\eta\mathbb E_t\widehat\ell_{t,j_*}
+2\eta^2(\mathbb E_t\widehat\ell_{t,j_*}^2-v_{t,j_*})
\le e^{\eta\mathbb E_t\widehat\ell_{t,j_*}}.
\label{eq:upper-update-multiplier-moment}
\end{equation}
The correction cancels the comparator's second-moment term.
Using \eqref{eq:expert-weight-update} and the conditional
Cauchy--Schwarz inequality, we have
\[
\begin{aligned}
\mathbb E_t\sqrt{\frac{\omega_{t,j_*}}{\omega_{t+,j_*}}}
&=\mathbb E_t\sqrt{\frac{\sum_j\omega_{t,j}
(1-\eta \widehat\ell_{t,j}+4\eta^2v_{t,j})}
{1-\eta \widehat\ell_{t,j_*}+4\eta^2v_{t,j_*}}}\\
&\le\left[\left(1-\eta\sum_j\omega_{t,j}\mathbb E_t\widehat\ell_{t,j}
+4\eta^2\sum_j\omega_{t,j}v_{t,j}\right)
\mathbb E_t\frac1{1-\eta \widehat\ell_{t,j_*}+4\eta^2v_{t,j_*}}\right]^{1/2}\\
&\le\exp\left\{\frac\eta2\left(\mathbb E_t\widehat\ell_{t,j_*}
-\sum_j\omega_{t,j}\mathbb E_t\widehat\ell_{t,j}\right)
+2\eta^2\sum_j\omega_{t,j}v_{t,j}\right\}.
\end{aligned}
\]
Here $t+$ denotes the saved weights after this visit; the last inequality
uses \eqref{eq:upper-update-multiplier-moment} and $1+x\le e^x$.
Multiplying the square-root ratios by the reciprocal exponentials and using
$\prod_{i\le t}\sqrt{\omega_{i,j_*}/\omega_{i+,j_*}}
=1/\sqrt{M\omega_{t+,j_*}}$, we obtain the nonnegative supermartingale
\[
\frac1{\sqrt{M\omega_{t+,j_*}}}
\exp\left\{\frac\eta2\sum_{i\le t}\left[
\sum_j\omega_{i,j}\mathbb E_i\widehat\ell_{i,j}-\mathbb E_i\widehat\ell_{i,j_*}\right]
-2\eta^2\sum_{i\le t,j}\omega_{i,j}v_{i,j}\right\},
\]
which starts at one because the initial weights are $1/M$.
Applying Lemma~\ref{lem:martingale-tools}(i) with threshold $2/\delta_{l,u}$,
and using $\omega_{t+,j_*}\le1$, we obtain the following bound simultaneously after every visit to this epoch:
\begin{equation}
\sum_t\left[\sum_j\omega_{t,j}\mathbb E_t\widehat\ell_{t,j}
-\mathbb E_t\widehat\ell_{t,j_*}\right]
\le\frac{\ln M+2\ln(2/\delta_{l,u})}{\eta}
+4\eta\sum_{t,j}\omega_{t,j}v_{t,j},
\label{eq:upper-multiplier-comparison}
\end{equation}
except on an event of probability at most $\delta_{l,u}/2$.
Substituting the definition of $v_{t,j}$ and using the weighted
coefficient bound in Lemma~\ref{lem:first-step-feedback}, we obtain
\begin{equation}
\sum_j\omega_{t,j}v_{t,j}
=2+\frac12\sum_j\omega_{t,j}\sum_xp_t(x)h_t(a_t(j),x)^2
\le Cr_l\ln(\mathrm eL_l/\gamma)
\le Cr_l\ln\frac{8L_lN}{\delta_{l,u}}.
\label{eq:upper-epoch-moments}
\end{equation}
For the last inequality, the formula for $\gamma$ gives
\[
\ln\frac{\mathrm eL_l}{\gamma}
=\ln(2\mathrm eL_l)+\frac12
\ln\frac{N\ln(8L_lN/\delta_{l,u})}{r_l\ln(2M/\delta_{l,u})}
\le2\ln\frac{8L_lN}{\delta_{l,u}}.
\]
We used $r_l\ln(2M/\delta_{l,u})\ge1$ and bounded
$\ln(2\mathrm eL_l)$, $\ln N$, and
$\ln\ln(8L_lN/\delta_{l,u})$ by $\ln(8L_lN/\delta_{l,u})$.

\paragraph{From conditional losses to the claimed bound.}
To recover the executed backups, we apply Lemma~\ref{lem:martingale-tools}(ii)
to the stopped sequence $2H-(\mathcal B^1V_{f_t}^2)(s_t^1,a_t^1)
-\min\{H,2\bar b_t(z_t^1)\}\in[0,2H]$,
with at most $N$ visits and $x=\ln(2/\delta_{l,u})$.
With failure probability at most $\delta_{l,u}/2$, the following bound holds simultaneously after every visit to this epoch:
\[
\begin{aligned}
&\sum_t\Big(\mathbb E_t[(\mathcal B^1V_{f_t}^2)(s_t^1,a_t^1)
+\min\{H,2\bar b_t(z_t^1)\}]\\
&\ -(\mathcal B^1V_{f_t}^2)(s_t^1,a_t^1)
-\min\{H,2\bar b_t(z_t^1)\}\Big)\\
&\le H\sqrt{2N\ln(2/\delta_{l,u})}.
\end{aligned}
\]
Adding this to the sum of \eqref{eq:upper-epoch-loss-reduction},
substituting \eqref{eq:upper-multiplier-comparison} and
\eqref{eq:upper-epoch-moments}, and using
$\ln M+2\ln(2/\delta_{l,u})\le2\ln(2M/\delta_{l,u})$
and $\min\{H,2\bar b_t\}\le2\bar b_t$, we obtain
\begin{equation}
\begin{aligned}
&\sum_t\left[\sup_{a,f}(\mathcal B^1V_f^2)(s_t^1,a)
-(\mathcal B^1V_{f_t}^2)(s_t^1,a_t^1)\right]\\
&\le2H\left[\frac{2\ln(2M/\delta_{l,u})}{\eta}
+C\eta r_lN\ln\frac{8L_lN}{\delta_{l,u}}+\gamma N\right]
+H\sqrt{2N\ln(2/\delta_{l,u})}\\
&+2(1+2\sqrt{r_l/L_l})\sum_t\epsilon_l(s_t^1)
+2\sum_t\bar b_t(z_t^1).
\end{aligned}
\label{eq:upper-epoch-combined-errors}
\end{equation}
The parameter identities in \eqref{eq:first-step-tuning} give
\[
\frac{2\ln(2M/\delta_{l,u})}{\eta}
+C\eta r_lN\ln\frac{8L_lN}{\delta_{l,u}}
=\left(512+\frac C{256}\right)
\sqrt{r_lN\ln\frac{8L_lN}{\delta_{l,u}}\ln\frac{2M}{\delta_{l,u}}}.
\]
For the exploration and sampling terms, we have
\[
\begin{aligned}
\gamma N+\sqrt{\frac N2\ln\frac2{\delta_{l,u}}}
&=\frac12\sqrt{\frac{r_lN\ln(2M/\delta_{l,u})}
{\ln(8L_lN/\delta_{l,u})}}
+\sqrt{\frac N2\ln\frac2{\delta_{l,u}}}\le\left(\frac12+\frac1{\sqrt2}\right)
\sqrt{r_lN\ln\frac{8L_lN}{\delta_{l,u}}\ln\frac{2M}{\delta_{l,u}}},
\end{aligned}
\]
because $r_l\ge1$, $\ln(8L_lN/\delta_{l,u})\ge1$, and
$\ln(2/\delta_{l,u})\le\ln(2M/\delta_{l,u})$.

It remains to sum $\bar b_t$. We apply \eqref{eq:regression-width-sum} and
\eqref{eq:regression-history-width-sum} to the same epoch histories,
with at most $N$ observations and $P_{l,u}$ occupied cells:
\begin{equation}
\begin{aligned}
\sum_t\bar b_t(z_t^1)
&=\overline\beta(|\mathcal C_l|,\delta_{l,u})\sum_tw_t(z_t^1)
+\varepsilon\sum_t1
+\varepsilon\sum_c\sum_{t:z_t^1\in c}
\sqrt{|\widetilde{\mathcal T}_t(c)|}\,w_t(z_t^1)\\
&\le\overline\beta(|\mathcal C_l|,\delta_{l,u})
\sqrt{2DP_{l,u}N\ln(1+N/D)}
+\varepsilon N\left(1+\sqrt{2D\ln(1+N/D)}\right).
\end{aligned}
\label{eq:upper-epoch-bonus-sum}
\end{equation}
Substituting these bounds into \eqref{eq:upper-epoch-combined-errors}
and using $2(1+2\sqrt{r_l/L_l})\le4(1+\sqrt{r_l/L_l})$
proves \eqref{eq:upper-epoch-comparison}.
The two failure probabilities sum to $\delta_{l,u}$.
Intersecting their complement with the high probability event that
Lemma~\ref{lem:initial-confidence} holds gives the claimed bound
simultaneously after every visit to this epoch before the first empty version space.
\end{proof}
\subsection{Proof of Proposition~\ref{prop:upper-first-step}}
\begin{proof}[Proof of Proposition~\ref{prop:upper-first-step}]
Fix $q\le T$ with $\mathcal F_t\ne\varnothing$ for $t\le q$; all
parameters still use $T$. By \eqref{eq:first-action-grid-error}, restricting
the comparison actions to $\mathcal A_T$ costs at most $qT^{-2}\le T^{-1}$.
We add and subtract the conditional expected backup with its bonus:
\begin{equation}
\begin{aligned}
&\sup_{a\in\mathcal A_T,\ f\in\mathcal F_t}
(\mathcal B^1V_f^2)(s_t^1,a)
-(\mathcal B^1V_{f_t}^2)(s_t^1,a_t^1)\\
&=\Bigl[\sup_{a\in\mathcal A_T,\ f\in\mathcal F_t}
(\mathcal B^1V_f^2)(s_t^1,a)
-\mathbb E_t\!\left[(\mathcal B^1V_{f_t}^2)(s_t^1,a_t^1)
+\min\{H,2\bar b_t(z_t^1)\}\right]\Bigr]\\
&\ +\Bigl[\mathbb E_t\!\left[(\mathcal B^1V_{f_t}^2)(s_t^1,a_t^1)
+\min\{H,2\bar b_t(z_t^1)\}\right]\\
&\ \ -(\mathcal B^1V_{f_t}^2)(s_t^1,a_t^1)
-\min\{H,2\bar b_t(z_t^1)\}\Bigr]
+\min\{H,2\bar b_t(z_t^1)\}.
\end{aligned}
\label{eq:upper-first-three-term-decomposition}
\end{equation}
Here $z_t^1=(s_t^1,a_t^1)$; $\mathbb E_t$ conditions on the history
and $s_t^1$ before sampling.

\paragraph{Controlling the first term.}
In epoch $(l,u)$, recall $N=2^u$,
$\delta_{l,u}=\delta/[64(l+1)^2(u+1)^2]$,
$\varepsilon=S_l^{-\nu/2}(r_l/N)^{1/4}$, and its fixed net size $M$
from Subsection~\ref{sec:first-step-selection-principle}.
Recall $L_l=|\mathscr U_l|$, $r_l=L_l\binom{d_A+j_l}{j_l}$ from
Subsection~\ref{sec:first-step-g-optimal}, and $\epsilon_l$ from
Lemma~\ref{lem:upper-action-approximation}. 
Suppose first that
$r_l\ln(2M/\delta_{l,u})<N\ln(8L_lN/\delta_{l,u})$.
On the event of Lemma~\ref{lem:initial-confidence}, we obtain the following bound by summing
\eqref{eq:upper-epoch-loss-reduction} and substituting
\eqref{eq:upper-multiplier-comparison} and \eqref{eq:upper-epoch-moments}:
\[
\begin{aligned}
&\sum_{t\le q:(l_t,u_t)=(l,u)}
\Bigl[\sup_{a\in\mathcal A_T,\ f\in\mathcal F_t}
(\mathcal B^1V_f^2)(s_t^1,a)
-\mathbb E_t[(\mathcal B^1V_{f_t}^2)(s_t^1,a_t^1)
+\min\{H,2\bar b_t(z_t^1)\}]\Bigr]\\
&\le2H\left[\frac{2\ln(2M/\delta_{l,u})}{\eta}
+C\eta r_lN\ln\frac{8L_lN}{\delta_{l,u}}+\gamma N\right]
+C(1+\sqrt{r_l/L_l})
\sum_{t\le q:(l_t,u_t)=(l,u)}\epsilon_l(s_t^1).
\end{aligned}
\]
We also used $\ln M+2\ln(2/\delta_{l,u})\le2\ln(2M/\delta_{l,u})$.
By \eqref{eq:first-step-tuning}, we have
$\eta=\frac1{256}\sqrt{\frac{\ln(2M/\delta_{l,u})}{r_lN\ln(8L_lN/\delta_{l,u})}}$
and $\gamma=128r_l\eta$. We substitute these values to obtain
\[
\begin{aligned}
\frac{2\ln(2M/\delta_{l,u})}{\eta}
+C\eta r_lN\ln\frac{8L_lN}{\delta_{l,u}}+\gamma N
&=\left(512+\frac C{256}+\frac1{2\ln(8L_lN/\delta_{l,u})}\right)
\sqrt{r_lN\ln\frac{8L_lN}{\delta_{l,u}}\ln\frac{2M}{\delta_{l,u}}}\\
&\le C\sqrt{r_lN\ln\frac{8L_lN}{\delta_{l,u}}\ln\frac{2M}{\delta_{l,u}}}.
\end{aligned}
\]
We used $\ln(8L_lN/\delta_{l,u})\ge1$ in the inequality.

Otherwise, each first bracket is at most $H$ and
$HN\le H\sqrt{r_lN\ln(8L_lN/\delta_{l,u})\ln(2M/\delta_{l,u})}$.
The bound holds at every visit cutoff, with additional failure probability
$\delta_{l,u}/2$.

We now substitute the model-count and approximation bounds.
Subsection~\ref{subsec:state-action-partition} defines
$S_l=\min\{2^{l/d},S_{\max}\}$. By the definitions of $e_A,e_B$ in
Section~\ref{sec:setup}, we have
\[
S_{\max}=\min\{T^{(1-e_A)/\nu},T^{(1-e_B)/\nu}\}
=T^{(1-e_*)/\nu}.
\]
By Subsection~\ref{sec:first-step-g-optimal}, we have
either $F_l=S_l$, or $d_A\ge4\nu$ and
$F_l/S_{\max}=(S_l/S_{\max})^{1-\nu/k_A}\le1$, where
$k_A=\lceil2\nu\rceil+1>2\nu$. Hence we have $L_l\le F_l^{d_A}\le T^{d_A(1-e_*)/\nu}\le T$.
Recall $\rho=\lfloor\nu\rfloor$,
$\mathfrak B=1+H+B_p+\Lambda(B_r+HB_p)$, and
$\iota_T=\ln[\mathrm e(2+d)T\mathfrak B/\delta]$ from
\eqref{eq:regret-parameter-summary}.

With $u\le\log_2T$ and
$l\le C_\nu d\iota_T$ from Lemma~\ref{lem:upper-partition-counts}, we have
\begin{equation}
\ln\frac2{\delta_{l,u}}
=\ln\frac{128}{\delta}+2\ln(l+1)+2\ln(u+1)\le C_\nu\iota_T,\
\ln\frac{8L_lN}{\delta_{l,u}}\le\ln(4T^2)+C_\nu\iota_T\le C_\nu\iota_T.
\label{eq:upper-first-confidence-log}
\end{equation}
Recall $\bar B=\max\{1,\Lambda(B_r+HB_p)\}$ and $J$ from
\eqref{eq:continuation-entropy-bound}. By
Lemma~\ref{lem:upper-remainder-cover}, $\sqrt{x+y}\le\sqrt x+\sqrt y$,
and $H/(H-1)\le2$, we obtain
\begin{equation}
\begin{aligned}
H\sqrt{r_lN\ln\frac{8L_lN}{\delta_{l,u}}\ln\frac{2M}{\delta_{l,u}}}&\le C_{\nu,\kappa}H\sqrt{r_lN\iota_T
\left[\iota_T+\bar B^2S_l^{-2\nu}\varepsilon^{-2}
(2+d)^{\rho+9\nu}\left(d\iota_T+\frac{J^2}{(H-1)^2}\right)\right]}\\
&\le C_{\nu,\kappa}H\iota_T\sqrt{r_lN}
+C_{\nu,\kappa}\bar B(2+d)^{\rho/2+9\nu/2}
(H\sqrt d\,\iota_T+J\sqrt{\iota_T})
S_l^{-\nu}\varepsilon^{-1}\sqrt{r_lN}.
\end{aligned}
\label{eq:upper-first-model-split}
\end{equation}
The choice of $\varepsilon$ balances this factor with the bonus approximation:
\[
S_l^{-\nu}\varepsilon^{-1}\sqrt{r_lN}
=\varepsilon N=S_l^{-\nu/2}r_l^{1/4}N^{3/4}.
\]
For $T_l=|\{t\le q:l_t=l\}|>0$, started epochs satisfy $2^u\le T_l$
by \eqref{eq:epoch-index}. Using
$\sum_{u:2^u\le T_l}2^{pu}\le T_l^p/(1-2^{-p})$ for $p>0$, we get
\begin{equation}
\begin{aligned}
\sum_{u:2^u\le T_l}\sqrt{r_l2^u}
&\le C\sqrt{r_lT_l},\ \sum_{u:2^u\le T_l}S_l^{-\nu/2}r_l^{1/4}2^{3u/4}
\le CS_l^{-\nu/2}r_l^{1/4}T_l^{3/4}.
\end{aligned}
\label{eq:upper-first-epoch-sum}
\end{equation}
By \eqref{eq:initial-scale-schedule}, we have
$N_l=\lceil T(S_l/S_{\max})^{2\nu}\rceil
=\lceil T^{2e_*-1}S_l^{2\nu}\rceil$.
Since $e_*\ge1/2$ and $S_l\ge1$, Lemma~\ref{lem:upper-partition-counts} gives
\begin{equation}
T_l\le\min\{T,S_l^{d_S}N_l\}
\le\min\{T,2T^{2e_*-1}S_l^{2\nu+d_S}\}.
\label{eq:upper-first-visit-count}
\end{equation}
By Section~\ref{sec:setup}, we have
\begin{equation}
e_*=\begin{cases}
e_B=(\nu+d_A+3d_S/2)/(2\nu+d_A+2d_S),&d_A\le2\nu,\\
e_A=(\nu+d_A)/(2\nu+d_A),&d_A\ge2\nu.
\end{cases}
\label{eq:upper-first-exponent-cases}
\end{equation}
If $d_A\le2\nu$, then we have $F_l=S_l$; substituting
\eqref{eq:upper-first-visit-count}, we have
\[
\begin{aligned}
S_l^{-\nu/2}F_l^{d_A/4}T_l^{3/4}
&\le C\min\left\{T^{3/4}S_l^{(d_A-2\nu)/4},
T^{3(2e_B-1)/4}S_l^{(d_A+4\nu+3d_S)/4}\right\}\\
&\le CT^{3/4+(d_A-2\nu)/[4(2\nu+d_A+2d_S)]}=CT^{e_B}.
\end{aligned}
\]
These expressions are nonincreasing and increasing, respectively;
their common value bounds the minimum at
\begin{equation}
S_l^{2\nu+d_S}=T^{2-2e_B}
\ \Longleftrightarrow\ S_l=T^{1/(2\nu+d_A+2d_S)}.
\label{eq:upper-first-low-scale}
\end{equation}
If $d_A\ge2\nu$, then we have $e_*=e_A$. Since $S_l/S_{\max}\le1$,
$d_A\ge2\nu$, and $d_A(1-\nu/k_A)>2\nu$ when $F_l\ne S_l$, we have
\[
F_l^{d_A}\le S_l^{2\nu}S_{\max}^{d_A-2\nu},\
S_l^{-\nu/2}F_l^{d_A/4}T_l^{3/4}
\le T^{3/4+(d_A-2\nu)/[4(2\nu+d_A)]}=T^{e_A}.
\]
Using $L_l\le F_l^{d_A}$, we conclude in both cases that
\begin{equation}
S_l^{-\nu/2}r_l^{1/4}T_l^{3/4}
\le C(r_l/L_l)^{1/4}T^{e_*},\
\sqrt{r_lT_l}\le(r_l/L_l)^{1/2}
T^{1/2+d_A(1-e_*)/(2\nu)}\le(r_l/L_l)^{1/2}T^{e_*}.
\label{eq:upper-first-balanced-scale}
\end{equation}
Here we use $e_*\ge e_A$.

For the feedback bias, \eqref{eq:upper-first-approximation-sum} applies
since each counting cell contributes at most $N_l$ visits
(Lemma~\ref{lem:upper-partition-counts}). We have
\[
S_l^{-\nu}\sqrt{N_lT_l}\le\sqrt2T^{e_*-1/2}\sqrt{T_l}
\le\sqrt2T^{e_*},\
T_lT^{-2}\le T^{-1},\
T_lT^{-\nu/(2\nu+d_A)}\le T^{e_A}\le T^{e_*}.
\]
Substituting into \eqref{eq:upper-first-approximation-sum}, we get
\[
\sum_{t\le q:l_t=l}\epsilon_l(s_t^1)
\le C_{\nu,\kappa}\bar B\sqrt{d_S+1}\,d_A^{\nu/2}T^{e_*}.
\]
By \eqref{eq:continuation-entropy-bound}, we have
$J\le C_{\nu,\kappa}\mathfrak B(2+d)^3$.
By the rank bounds in Lemma~\ref{lem:upper-action-approximation}, we have
\begin{equation}
(r_l/L_l)^{1/2}\le C_\nu(2+d)^{\nu+1/2}\iota_T^{2\nu},\
(r_l/L_l)^{1/4}\le C_\nu(2+d)^{\nu/2+1/4}\iota_T^\nu,
\label{eq:upper-first-rank-factors}
\end{equation}
since $d_A<4\nu$ in the first case and $\lceil2\nu\rceil\le2\nu+1$ in the second.
By \eqref{eq:upper-first-model-split}, \eqref{eq:upper-first-epoch-sum},
and \eqref{eq:upper-first-balanced-scale}, the first bracket summed over
level $l$, including the feedback bias, is bounded by
\[
\begin{aligned}
C_{\nu,\kappa}T^{e_*}\Bigl[&H\iota_T(r_l/L_l)^{1/2}
+\bar B(2+d)^{\rho/2+9\nu/2}
(H\sqrt d\,\iota_T+J\sqrt{\iota_T})(r_l/L_l)^{1/4}\\
&+(1+\sqrt{r_l/L_l})\bar B\sqrt{d_S+1}\,d_A^{\nu/2}\Bigr]
\le C_{\nu,\kappa}\mathfrak B^2(2+d)^{\rho+5\nu+5}
\iota_T^{1+2\nu}T^{e_*}.
\end{aligned}
\]
By \eqref{eq:upper-first-rank-factors}, $H,\bar B\le\mathfrak B$,
and the bound on $J$, the dimension exponents satisfy
$\max\{\nu+1/2,\rho/2+5\nu+13/4,3\nu/2+1\}\le\rho+5\nu+5$,
and the logarithmic exponents are at most $1+2\nu$.
Multiplying by $L_{\max}+1\le C_\nu d\iota_T$ from
Lemma~\ref{lem:upper-partition-counts} bounds the first term by the right-hand side of \eqref{eq:upper-first-step-total}.

\paragraph{Controlling the sampling deviation.}
By \eqref{eq:candidate-class} and $0\le r^1\le1$, we have
$2H-(\mathcal B^1V_{f_t}^2)(s_t^1,a_t^1)-\min\{H,2\bar b_t(z_t^1)\}\in[0,2H]$.
We stop before the first empty candidate class and apply
Lemma~\ref{lem:martingale-tools}(ii) with bound $2H$, $N=T$, and
$x=\ln(16/\delta)$. With failure probability at most $\delta/16$,
simultaneously for every admissible $q\le T$, we have
\[
\begin{aligned}
&\sum_{t=1}^q\Bigl[\mathbb E_t[(\mathcal B^1V_{f_t}^2)(s_t^1,a_t^1)
+\min\{H,2\bar b_t(z_t^1)\}]-(\mathcal B^1V_{f_t}^2)(s_t^1,a_t^1)
-\min\{H,2\bar b_t(z_t^1)\}\Bigr]\\
&\le2H\sqrt{T\ln(16/\delta)/2}
=H\sqrt{2T\ln(16/\delta)}\le C\mathfrak B\iota_T^{1/2}T^{e_*}.
\end{aligned}
\]
We used $H\le\mathfrak B$, $e_*\ge1/2$, and $\ln(16/\delta)\le C\iota_T$.

\paragraph{Controlling the clipped bonus.}
Recall $D=\binom{d+\rho}{\rho}$ and, from
Subsection~\ref{sec:first-step-scores},
\[
\bar b_t(z_t^1)=\overline\beta(|\mathcal C_l|,\delta_{l,u})w_t(z_t^1)
+\varepsilon\left(1+\sqrt{|\widetilde{\mathcal T}_t(c)|}\,w_t(z_t^1)\right),
\]
where $c$ contains $z_t^1$, $w_t$ is defined in \eqref{eq:layered-width},
and $\widetilde{\mathcal T}_t(c)$ contains earlier epoch observations in $c$.
Recall that $P_l$ and $P_{l,u}$ count visited state cells at level $l$
and joint cells in epoch $(l,u)$, respectively, up to $q$
(Lemmas~\ref{lem:upper-partition-counts} and~\ref{lem:upper-epoch-comparison}).
By \eqref{eq:initial-partitions}, we have $P_{l,u}\le P_lS_l^{d_A}$.

We apply \eqref{eq:upper-epoch-bonus-sum} with $N=2^u$ and
$\varepsilon N=S_l^{-\nu/2}r_l^{1/4}2^{3u/4}$:
\[
\begin{aligned}
&\sum_{t\le q:l_t=l}\min\{H,2\bar b_t(z_t^1)\}\\
\le&2\sum_{t\le q:l_t=l}\bar b_t(z_t^1)\\
\le&2\sum_{u:2^u\le T_l}\Bigl[
\overline\beta(|\mathcal C_l|,\delta_{l,u})
\sqrt{2DP_{l,u}2^u\ln(1+2^u/D)} +S_l^{-\nu/2}r_l^{1/4}2^{3u/4}
\bigl(1+\sqrt{2D\ln(1+2^u/D)}\bigr)\Bigr]\\
\le &C\left(\max_{u:2^u\le T_l}\overline\beta(|\mathcal C_l|,\delta_{l,u})\right)
\sqrt{D\iota_T}\sqrt{P_lS_l^{d_A}T_l}
+C(1+\sqrt{D\iota_T})S_l^{-\nu/2}r_l^{1/4}T_l^{3/4}.
\end{aligned}
\]
The last inequality uses $P_{l,u}\le P_lS_l^{d_A}$,
$\ln(1+2^u/D)\le\iota_T$, and \eqref{eq:upper-first-epoch-sum}.
For the approximation factor, use \eqref{eq:upper-first-balanced-scale}.
For regression, $S_l/S_{l-1}\le2$ and \eqref{eq:initial-scale-schedule} give
$N_{l-1}\ge2^{-2\nu}T^{2e_*-1}S_l^{2\nu}$. By
Lemma~\ref{lem:upper-partition-counts}, we obtain
\[
P_l\le\min\{S_l^{d_S},C_\nu(d_S+1)T^{2-2e_*}S_l^{-2\nu}\}.
\]
At $l=0$, $S_0=1$ and $e_*\le1$ give
$P_0\le1=\min\{S_0^{d_S},(d_S+1)T^{2-2e_*}S_0^{-2\nu}\}$.
If $d_A\le2\nu$, we use $T_l\le T$ and the same intersection
\eqref{eq:upper-first-low-scale} to obtain
\[
\begin{aligned}
\sqrt{P_lS_l^{d_A}T_l}
&\le C_\nu\sqrt{d_S+1}\,T^{1/2}
\min\{S_l^{(d_S+d_A)/2},T^{1-e_B}S_l^{(d_A-2\nu)/2}\}\\
&\le C_\nu\sqrt{d_S+1}\,
T^{1/2+(d_S+d_A)/[2(2\nu+d_A+2d_S)]}
=C_\nu\sqrt{d_S+1}\,T^{e_B}.
\end{aligned}
\]
The first argument increases and the second does not.
If $d_A\ge2\nu$, then we have $e_*=e_A$ and $S_l\le T^{1/(2\nu+d_A)}$, from which we obtain
\[
\begin{aligned}
\sqrt{P_lS_l^{d_A}T_l}
&\le C_\nu\sqrt{d_S+1}\,T^{3/2-e_A}S_l^{(d_A-2\nu)/2}\\
&\le C_\nu\sqrt{d_S+1}\,
T^{3/2-e_A+(d_A-2\nu)/[2(2\nu+d_A)]}\\
&=C_\nu\sqrt{d_S+1}\,T^{e_A}.
\end{aligned}
\]
For the confidence radius, $B_{\phi}=\sqrt{3^\rho D}$ and $\rho=\lfloor\nu\rfloor$ give
\begin{equation}
D\le C_\nu(2+d)^\rho,\ B_{\phi}\le C_\nu(2+d)^{\rho/2}.
\label{eq:upper-polynomial-dimension}
\end{equation}
By \eqref{eq:continuation-entropy-bound}, we have
\begin{equation}
J\le C_{\nu,\kappa}(H+d^3\bar B)
\le C_{\nu,\kappa}\mathfrak B(2+d)^3.
\label{eq:upper-common-parameters}
\end{equation}
By Lemma~\ref{lem:upper-partition-counts} and
\eqref{eq:upper-first-confidence-log}, we have
\[
\ln\frac{8|\mathcal C_l|HT}{\delta_{l,u}}
\le d\ln S_{\max}+\ln(4HT)+\ln\frac2{\delta_{l,u}}
\le C_\nu d\iota_T.
\]
In \eqref{eq:regression-confidence-parameter}, we substitute $\lambda=B_{\phi}^2$:
\begin{equation}
\overline\beta(|\mathcal C_l|,\delta_{l,u})
\le C_\nu[H(B_{\phi}+\sqrt{(D+d)\iota_T})+J]
\le C_{\nu,\kappa}\mathfrak B(2+d)^{\max\{\rho/2,3\}}\iota_T^{1/2}.
\label{eq:upper-first-beta}
\end{equation}
By these bounds, \eqref{eq:upper-first-balanced-scale}, and
\eqref{eq:upper-first-rank-factors}, we get
\[
\begin{aligned}
\sum_{t\le q:l_t=l}\min\{H,2\bar b_t(z_t^1)\}
&\le C_{\nu,\kappa}\mathfrak B T^{e_*}
\left[(2+d)^{\rho/2+\max\{\rho/2,3\}+1/2}\iota_T
+(2+d)^{\rho/2+\nu/2+1/4}\iota_T^{\nu+1/2}\right]\\
&\le C_{\nu,\kappa}\mathfrak B(2+d)^{\rho+5\nu+5}
\iota_T^{1+2\nu}T^{e_*}.
\end{aligned}
\]
Multiplying by $L_{\max}+1\le C_\nu d\iota_T$ from
Lemma~\ref{lem:upper-partition-counts}, and using $\mathfrak B\le\mathfrak B^2$,
bounds the third term by the right-hand side of \eqref{eq:upper-first-step-total}.

Adding the three bounds in \eqref{eq:upper-first-three-term-decomposition}
and the grid error $T^{-1}$ proves \eqref{eq:upper-first-step-total}, since
$\mathfrak B,\iota_T,T\ge1$. By \eqref{eq:initial-confidence-failure-sum},
the total failure probability is at most
\[
\frac\delta8+\frac\delta{16}+\frac12\sum_{l,u\ge0}\delta_{l,u}
<\frac\delta8+\frac\delta{16}+\frac\delta{16}=\frac\delta4.
\]
The same event covers unfinished epochs and every admissible $q\le T$.
\end{proof}

\subsection{Proof of Lemma~\ref{lem:upper-historical-residual}}
\begin{proof}[Proof of Lemma~\ref{lem:upper-historical-residual}]
Fix $2\le h\le H$, $t\le T$, and $f\in\mathcal F_t$.
Suppress the stage superscript except on $\mathbb E_i^h$, $\Xi_t^h$,
and $V_f^{h+1}$. Our first goal is \eqref{eq:upper-historical-residual}:
\[
\sum_{i<t}\mathbb E_i^h[(|\Delta_f(z_i)|-\sigma_i(z_i))_+^2]
\le C\left[J^2+H^2\left(1+\ell_\delta+
\frac{B_{\rm res}^2}{\alpha_{\rm vs}^2}\right)\right].
\]
Recall $\Delta_f=f^h-\mathcal B^hV_f^{h+1}$ from
\eqref{eq:bellman-residual-definition}, $\sigma_i=\min\{H,2\tau_i\}$,
$\beta=\beta(|\mathcal C|,\delta/8)$, and $\ell_\delta=\ln(8HT/\delta)$.
By \eqref{eq:version-space}, we have $\Xi_t^h(f)\le J^2/H^2+1+\ell_\delta$.
To relate this accepted score to squared test weights, we evaluate
\eqref{eq:innovation-test-weights} at any state--action pair
$z\in\mathcal Z$:
\begin{equation}
a_i(W,z):=\operatorname{sgn}(\alpha_{\rm vs}W(g_i(z)))
\min\{H,(|\alpha_{\rm vs}W(g_i(z))|-\tau_i(z))_+\},\
a_i(W,z_i)=\chi_i^h(Z_i(W)).
\label{eq:vs-test-weight-function}
\end{equation}

Here $\mathbb E$ averages over the interaction,
$\mathbb E_{\mathsf P}$ uses the fixed Gaussian law
\eqref{eq:vs-gaussian-law}, and $\mathbb E_i^h$ denotes
\eqref{eq:historical-occupancy-integral}, with $f,W$ held fixed.
By \eqref{eq:reference-kernel-norm}, \eqref{eq:candidate-class}, and
\eqref{eq:bellman-backup-norm}, we have $\Delta_f\in\mathcal H_0$ and
\begin{equation}
\|\Delta_f\|_{\mathcal H_0}
\le\Lambda\bigl(\|f^h\|_{\mathcal H_k}
+\|\mathcal B^hV_f^{h+1}\|_{\mathcal H_k}\bigr)
\le2\Lambda(B_r+HB_p)\le B_{\rm res}.
\label{eq:upper-residual-rkhs-norm}
\end{equation}
Recall $W_c(z)=W(\Psi_c(z))$ from
\eqref{eq:vs-gaussian-law}, and define
\[
\mathcal E:=
\left\{W:\max_{c\in\mathcal C}
\sup_{z\in\overline c}|W_c(z)|\le q_0A_{\rm vs}\right\}.
\]
By \eqref{eq:vs-gaussian-supremum-bound}, Markov's inequality gives
\[
\mathsf P(\mathcal E)
\ge1-\frac{\mathbb E_{\mathsf P}\max_c\sup_z|W_c(z)|}{q_0A_{\rm vs}}
\ge1-\frac{q_0-1}{2q_0}\ge\frac12,
\]
which proves \eqref{eq:vs-gaussian-envelope-probability}.

To compare the Gaussian test with the residual, we use the event
\begin{equation}
\mathcal E_f:=\left\{W:\max_c\sup_{z\in\overline c}
\left|\alpha_{\rm vs}W(\Psi_c(z))
-\langle\Delta_f,\Psi_c(z)\rangle_{\mathcal H_0}\right|
\le\alpha_{\rm vs}q_0A_{\rm vs}\right\}.
\label{eq:upper-gaussian-residual-event}
\end{equation}
Equivalently, $W\in\mathcal E_f$ if and only if the family
$\widetilde W(v):=W(v)-\langle\Delta_f,v\rangle_{\mathcal H_0}/\alpha_{\rm vs}$,
$v\in\mathcal H_0$, belongs to $\mathcal E$.
Applying Lemma~\ref{lem:isonormal-shift} with shift
$-\Delta_f/\alpha_{\rm vs}$ to $\mathcal E$, then Jensen's inequality, we obtain
\begin{equation}
\begin{aligned}
\mathsf P(\mathcal E_f)
&=e^{-\|\Delta_f\|_{\mathcal H_0}^2/(2\alpha_{\rm vs}^2)}
\mathbb E_{\mathsf P}\!\left[\mathbf1_{\mathcal E}
 e^{-W(\Delta_f/\alpha_{\rm vs})}\right]\ge\mathsf P(\mathcal E)e^{-\|\Delta_f\|_{\mathcal H_0}^2/(2\alpha_{\rm vs}^2)}
\ge\tfrac12e^{-B_{\rm res}^2/(2\alpha_{\rm vs}^2)}.
\end{aligned}
\label{eq:upper-shifted-envelope-probability}
\end{equation}
Here $\mathsf P$ denotes the probability law of the Gaussian family $W$.
For the first inequality, symmetry under $W\mapsto-W$ implies
$\mathbb E_{\mathsf P}[W(\Delta_f/\alpha_{\rm vs})\mid\mathcal E]=0$;
the last uses \eqref{eq:upper-residual-rkhs-norm}.
For $W\in\mathcal E_f$ and $z\in c$,
we expand \eqref{eq:profiled-features}:
\begin{equation}
\begin{aligned}
|\alpha_{\rm vs}W(g_i(z))-\langle\Delta_f,g_i(z)\rangle|
&=\alpha_{\rm vs}\left|\widetilde W(\Psi_c(z))
-\sum_{j\in\mathcal T_i(c)}\phi_c(z)^\top G_i(c)^{-1}
\phi_c(z_j)\widetilde W(\Psi_c(z_j))\right|\\
&\le\alpha_{\rm vs}q_0A_{\rm vs}
\left(1+\sum_{j\in\mathcal T_i(c)}
|\phi_c(z)^\top G_i(c)^{-1}\phi_c(z_j)|\right)\\
&\le\alpha_{\rm vs}q_0A_{\rm vs}
(1+\sqrt{n_i(c)}\,w_i(z)).
\end{aligned}
\label{eq:upper-gaussian-adjustment-error}
\end{equation}
The equality uses the definition of $\widetilde W$ and linearity of $W$;
the inequalities use $\widetilde W\in\mathcal E$, equivalently
\eqref{eq:upper-gaussian-residual-event}, and
\eqref{eq:regression-coefficient-sum}, respectively.

On the event of Lemma~\ref{lem:local-ridge} with
$M=|\mathcal C|$ and with failure probability $\delta/8$,
\eqref{eq:coarse-residual-difference} gives, for
$f\in\mathcal F_t\subseteq\mathcal F_t^1$ and $i<t$,
\[
|\Delta_f(z)-\langle\Delta_f,g_i(z)\rangle_{\mathcal H_0}|
\le(B_{\phi}H+2\beta)w_i(z).
\]
Combining this with \eqref{eq:upper-gaussian-adjustment-error} and
\eqref{eq:innovation-threshold}, we obtain
\begin{equation}
|\alpha_{\rm vs}W(g_i(z))-\Delta_f(z)|
\le\alpha_{\rm vs}q_0A_{\rm vs}(1+\sqrt{n_i(c)}w_i(z))
+(B_{\phi}H+2\beta)w_i(z)=\tau_i(z).
\label{eq:upper-weight-approximation}
\end{equation}
Since $|\Delta_f|\le H$ by \eqref{eq:candidate-class} and
\eqref{eq:bellman-backup}, the weight definition
\eqref{eq:vs-test-weight-function} now gives
\begin{equation}
(|\Delta_f(z)|-\sigma_i(z))_+
=(|\Delta_f(z)|-2\tau_i(z))_+
\le\min\{H,(|\alpha_{\rm vs}W(g_i(z))|-\tau_i(z))_+\}
=|a_i(W,z)|.
\label{eq:upper-comparison-weight}
\end{equation}
The inequality uses $|\Delta_f|-2\tau_i\le
|\alpha_{\rm vs}W(g_i)|-\tau_i$ from \eqref{eq:upper-weight-approximation}.
Squaring and integrating against
\eqref{eq:historical-occupancy-integral}, we obtain, for every
$W\in\mathcal E_f$,
\begin{equation}
\sum_{i<t}\mathbb E_i^h[(|\Delta_f(z_i)|-\sigma_i(z_i))_+^2]
\le\sum_{i<t}\mathbb E_i^h a_i(W,z_i)^2.
\label{eq:upper-residual-to-weights}
\end{equation}

To control the right-hand side of \eqref{eq:upper-residual-to-weights},
we expand the accepted statistic \eqref{eq:innovation-test}. Specifically, we write $a_i=a_i(W,z_i)=\chi_i^h(Z_i(W))$ inside observed sums,
as in \eqref{eq:vs-test-weight-function}, and recall $c_i=c(z_i)$.
The numerator of the exponent in \eqref{eq:innovation-test} is
\[
S_f(W):=\sum_{i<t}\left[a_iU_i(f)-\beta w_i(z_i)|a_i|-\frac{a_i^2}{16}\right],\
\ln\mathbb E_{\mathsf P}e^{S_f(W)/(320H^2)}=\Xi_t^h(f).
\]
The equality follows from \eqref{eq:vs-gaussian-coordinate-realization}:
with the history fixed, $(Z_i(W))_{i<t}$ has law $\mathsf P_t^h$.

On $\mathcal E_f$,
\eqref{eq:upper-gaussian-adjustment-error} and
$|e_i(f)|\le(B_{\phi}H+\beta)w_i(z_i)$ from
\eqref{eq:bellman-residual-confidence} imply that
\[
\begin{aligned}
&|\langle\Delta_f,g_i(z_i)\rangle-\alpha_{\rm vs}W(g_i(z_i))+e_i(f)|
+\beta w_i(z_i)\\
&\le\alpha_{\rm vs}q_0A_{\rm vs}(1+\sqrt{n_i(c_i)}w_i(z_i))
+(B_{\phi}H+2\beta)w_i(z_i)\\
&=\tau_i(z_i).
\end{aligned}
\]
Substituting $U_i(f)=\langle\Delta_f,g_i(z_i)\rangle+e_i(f)
-\xi_i(V_f^{h+1})$ from \eqref{eq:innovation-decomposition}, we have
\[
\begin{aligned}
a_iU_i(f)-\beta w_i(z_i)|a_i|
&\ge a_i\alpha_{\rm vs}W(g_i(z_i))
-\tau_i(z_i)|a_i|-a_i\xi_i(V_f^{h+1})\\
&=|a_i|\bigl(|\alpha_{\rm vs}W(g_i(z_i))|-\tau_i(z_i)\bigr)
-a_i\xi_i(V_f^{h+1})
\ge a_i^2-a_i\xi_i(V_f^{h+1}).
\end{aligned}
\]
The equality uses the sign of $a_i$ in \eqref{eq:vs-test-weight-function};
when $a_i\ne0$, $|a_i|\le|\alpha_{\rm vs}W(g_i(z_i))|-\tau_i(z_i)$
proves the last inequality.
Subtracting $a_i^2/16$ and summing,
we obtain
\begin{equation}
S_f(W)\ge\frac{15}{16}\sum_{i<t}a_i^2
-\sum_{i<t}a_i\xi_i(V_f^{h+1}),\ W\in\mathcal E_f.
\label{eq:upper-weight-drift}
\end{equation}

We control the noise in \eqref{eq:upper-weight-drift} first.
For fixed $W$, $a_i(W,z)$ is determined before episode $i$ by
\eqref{eq:vs-test-weight-function}; its observed value is
$\mathcal I_i^h$-measurable, bounded by $H$, and common to all $V$.
Lemma~\ref{lem:innovation-offset-moment}, with the positive sign, therefore gives
\begin{equation}
\mathbb E\exp\left\{
\frac1{320H^2}\left[
\sup_{V\in\mathcal V_h}\sum_{i<t}a_i\xi_i(V)
-\frac1{16}\sum_{i<t}a_i^2\right]\right\}
\le e^{J^2/H^2}.
\label{eq:upper-fixed-weight-noise-moment}
\end{equation}
For the historical integrals, $0\le7a_i^2/(2560H^2)\le1$ and
$e^{-x}\le1-x/2$ on $[0,1]$ give
\[
\begin{aligned}
\mathbb E[e^{-7a_i^2/(2560H^2)}\mid\mathcal I_i^1]
&=\mathbb E_i^h e^{-7a_i^2/(2560H^2)}
\le1-\frac7{5120H^2}\mathbb E_i^h a_i^2
\le e^{-7\mathbb E_i^h a_i^2/(5120H^2)}.
\end{aligned}
\]
The equality uses \eqref{eq:historical-occupancy-integral} with $W$ fixed;
the last inequality is $1-y\le e^{-y}$.
Multiplying by $e^{7\mathbb E_i^h a_i^2/(5120H^2)}$ and iterating
conditional expectations yields
\begin{equation}
\mathbb E\exp\left\{
\frac7{2560H^2}\sum_{i<t}
\left[\frac12\mathbb E_i^h a_i^2-a_i^2\right]\right\}\le1.
\label{eq:upper-fixed-weight-occupancy-moment}
\end{equation}
The exponent below is half the sum of the exponents in
\eqref{eq:upper-fixed-weight-noise-moment} and
\eqref{eq:upper-fixed-weight-occupancy-moment}.
By Tonelli's theorem and the Cauchy--Schwarz inequality, we obtain
\begin{equation}
\begin{aligned}
&\mathbb E\mathbb E_{\mathsf P}
\exp\left\{\frac1{640H^2}\left[
\sup_{V\in\mathcal V_h}\sum_{i<t}a_i\xi_i(V)
+\frac7{16}\sum_{i<t}\mathbb E_i^h a_i^2
-\frac{15}{16}\sum_{i<t}a_i^2\right]\right\}\\
&\le\mathbb E_{\mathsf P}\Bigg[
\left(\mathbb Ee^{[\sup_V\sum_{i<t}a_i\xi_i(V)-\sum_{i<t}a_i^2/16]/(320H^2)}\right)^{1/2}
\left(\mathbb Ee^{\frac7{2560H^2}\sum_{i<t}[\mathbb E_i^ha_i^2/2-a_i^2]}\right)^{1/2}
\Bigg]\\
&\le\mathbb E_{\mathsf P}e^{J^2/(2H^2)}=e^{J^2/(2H^2)}.
\end{aligned}
\label{eq:upper-combined-log-moment}
\end{equation}
Tonelli's theorem applies to the nonnegative integrands.
The first inequality takes the Cauchy--Schwarz inequality over the
interaction with $W$ fixed; the second substitutes
\eqref{eq:upper-fixed-weight-noise-moment} and
\eqref{eq:upper-fixed-weight-occupancy-moment}.
Markov's inequality at $e^{J^2/(2H^2)+\ell_\delta}$ and a union bound
have total failure probability at most
$\sum_{h=2}^H\sum_{t=1}^Te^{-\ell_\delta}\le HTe^{-\ell_\delta}=\delta/8$.
On the complementary event, the logarithm of the inner Gaussian
expectation in \eqref{eq:upper-combined-log-moment} is at most
$J^2/(2H^2)+\ell_\delta$.

On $\mathcal E_f$, we combine
\eqref{eq:upper-residual-to-weights} and \eqref{eq:upper-weight-drift} to get
\[
\begin{aligned}
\frac7{16}\sum_{i<t}\mathbb E_i^h[(|\Delta_f(z_i)|-\sigma_i(z_i))_+^2]-S_f(W)&\le\frac7{16}\sum_{i<t}\mathbb E_i^h a_i^2
-\frac{15}{16}\sum_{i<t}a_i^2+\sum_{i<t}a_i\xi_i(V_f^{h+1})\\
&\le\sup_{V\in\mathcal V_h}\sum_{i<t}a_i\xi_i(V)
+\frac7{16}\sum_{i<t}\mathbb E_i^h a_i^2-\frac{15}{16}\sum_{i<t}a_i^2.
\end{aligned}
\]
The last inequality uses $V_f^{h+1}\in\mathcal V_h$.
Exponentiating at scale $640H^2$, integrating over $\mathcal E_f$ from \eqref{eq:upper-gaussian-residual-event},
and using the Markov bound from \eqref{eq:upper-combined-log-moment}, we get
\[
\mathbb E_{\mathsf P}\!\left[
\mathbf1_{\mathcal E_f}e^{-S_f(W)/(640H^2)}\right]
\le\exp\!\left\{\frac{J^2}{2H^2}+\ell_\delta-
\frac7{10240H^2}\sum_{i<t}\mathbb E_i^h[(|\Delta_f(z_i)|-\sigma_i(z_i))_+^2]\right\}.
\]
The residual sum is fixed under $\mathsf P$, and extending the
nonnegative integral to the whole space gives the stated upper bound.
By the H\"older inequality with exponents $3$ and $3/2$, we have
\[
\begin{aligned}
\mathsf P(\mathcal E_f)
&=\mathbb E_{\mathsf P}\!\left[\mathbf1_{\mathcal E_f}
 e^{S_f(W)/(960H^2)}e^{-S_f(W)/(960H^2)}\right]\\
&\le\left(\mathbb E_{\mathsf P}e^{S_f(W)/(320H^2)}\right)^{1/3}
\left(\mathbb E_{\mathsf P}\!\left[\mathbf1_{\mathcal E_f}
 e^{-S_f(W)/(640H^2)}\right]\right)^{2/3}\\
&\le\exp\!\left\{\frac{\Xi_t^h(f)}3+\frac{J^2}{3H^2}
+\frac{2\ell_\delta}3-\frac7{15360H^2}
\sum_{i<t}\mathbb E_i^h[(|\Delta_f(z_i)|-\sigma_i(z_i))_+^2]\right\}.
\end{aligned}
\]
In the first inequality, we extend the first factor's integral to the
whole space. The last substitutes \eqref{eq:innovation-test} and
the preceding bound.
Taking logarithms, using \eqref{eq:upper-shifted-envelope-probability} and rearranging, we obtain
\[
\begin{aligned}
\frac7{16}\sum_{i<t}\mathbb E_i^h[(|\Delta_f(z_i)|-\sigma_i(z_i))_+^2]
&\le320H^2\Xi_t^h(f)
+640H^2\left[\frac{J^2}{2H^2}+\ell_\delta\right]
+960H^2\left[\ln2+\frac{B_{\rm res}^2}{2\alpha_{\rm vs}^2}\right]\\
&\le640J^2+320H^2
+960H^2\left[\ell_\delta+\ln2+
\frac{B_{\rm res}^2}{2\alpha_{\rm vs}^2}\right].
\end{aligned}
\]
The last inequality substitutes $\Xi_t^h(f)\le J^2/H^2+1+\ell_\delta$
from \eqref{eq:version-space}. Multiplication by $16/7$ proves
\eqref{eq:upper-historical-residual}. The Markov event does not depend
on $f$, and the translation argument applies to every
fixed history and candidate, including candidates selected from the data.

For \eqref{eq:upper-slack-sum}, the function $\sigma_i$ is fixed
given $\mathcal I_i^1$ and lies in $[0,H]$. By
Lemma~\ref{lem:martingale-tools}(ii), $\sigma_i\le2\tau_i$,
and \eqref{eq:innovation-threshold}, we have
\begin{equation}
\begin{aligned}
\sum_{i=1}^T\mathbb E_i^h\sigma_i(z_i)
&\le2\sum_{i=1}^T\sigma_i(z_i)+2H\ell_\delta\\
&\le4\alpha_{\rm vs}q_0A_{\rm vs}
\left(T+\sum_{i=1}^T\sqrt{n_i(c_i)}w_i(z_i)\right)
+4(B_{\phi}H+2\beta)\sum_{i=1}^Tw_i(z_i)+2H\ell_\delta\\
&\le4\alpha_{\rm vs}q_0A_{\rm vs}(1+\sqrt{K_{\rm c}})T
+4(B_{\phi}H+2\beta)\sqrt{MTK_{\rm c}}+2H\ell_\delta.
\end{aligned}
\label{eq:upper-allowance-concentration}
\end{equation}
The last inequality applies \eqref{eq:regression-width-sum} and
\eqref{eq:regression-history-width-sum} with $N=T$, $M=|\mathcal C|$,
and $K_{\rm c}=2D\ln(1+T/D)$ from \eqref{eq:vs-residual-width-parameters}.
The first inequality holds
with failure probability at most $He^{-\ell_\delta}\le\delta/8$ over all
stages. Using $B_{\phi}H+2\beta\le2(B_{\phi}H+\beta)$ proves
\eqref{eq:upper-slack-sum}. The event of Lemma~\ref{lem:local-ridge} with
$M=|\mathcal C|$ and with failure probability $\delta/8$,
the Markov event from \eqref{eq:upper-combined-log-moment},
and the event of \eqref{eq:upper-allowance-concentration}
fail with total probability at most $3\delta/8\le\delta/2$,
proving both assertions  in the lemma.
\end{proof}

\subsection{Proof of Lemma~\ref{lem:upper-continuation-approximation}}
\begin{proof}[Proof of Lemma~\ref{lem:upper-continuation-approximation}]
Fix $g\in\mathcal H_0$ with $\|g\|_{\mathcal H_0}\le B$.
We construct $p_g$ and one state envelope $\varepsilon_g$ controlling
its error for all actions, then bound $\|\varepsilon_g\|_{L^2(\mu)}$ and the local dimension $r$.

Recall from Lemma~\ref{lem:rkhs-frequency-representation} that
$g(s,a)=\mathbb E[\theta_ge^{\mathrm i(\xi^\top s+\zeta^\top a)}]$ and
$\mathbb E|\theta_g|^2=\|g\|_{\mathcal H_0}^2\le B^2$.
Here $(\xi,\zeta)=\kappa Z/\sqrt W$, with state frequency
$\xi\in\mathbb R^{d_S}$, action frequency $\zeta\in\mathbb R^{d_A}$,
and independent $Z\sim\mathcal N(0,I_d)$, $W\sim\chi^2_{2\nu}$ as in
Lemma~\ref{lem:matern-frequency-representation}; all expectations below use this law.
For $S>0$, recall the truncation
\[
T_Sg(s,a)=\operatorname{Re}\mathbb E\left[
\theta_g\mathbf1\{\|\zeta\|_2>S\}e^{\mathrm i(\xi^\top s+\zeta^\top a)}\right].
\]
For the real polynomial $p_g$ in \eqref{eq:upper-action-polynomial},
with cutoff and degree specified below, \eqref{eq:upper-action-error-split} gives
\[
\sup_a|g(s,a)-p_g(s,a)|
\le\sup_a|T_Sg(s,a)|+\sup_a|(I-T_S)g(s,a)-p_g(s,a)|.
\]
We bound the first term in $L^2(\mu)$ and the second by \eqref{eq:upper-action-taylor-bound}.

\paragraph{Bounding the high-frequency error.}
Conditional on $(W,\zeta)$, $\xi$ has density
\[
(W/(2\pi\kappa^2))^{d_S/2}e^{-W\|\xi\|_2^2/(2\kappa^2)}.
\]
In the integrals below, $\theta_g$ is evaluated at
$Z=\sqrt W(\xi,\zeta)/\kappa$. For almost every $(W,\zeta)$,
$\mathbb E[|\theta_g|^2\mid W,\zeta]<\infty$, so the Cauchy--Schwarz
inequality gives
\[
\mathbb E[|\theta_g|\mid W,\zeta]
\le(\mathbb E[|\theta_g|^2\mid W,\zeta])^{1/2}<\infty.
\]
Also, we have
\[
\int_{\mathbb R^{d_S}}|\theta_g|^2
\left(\frac W{2\pi\kappa^2}\right)^{d_S}
 e^{-W\|\xi\|_2^2/\kappa^2}\,d\xi
\le\left(\frac W{2\pi\kappa^2}\right)^{d_S/2}
\mathbb E[|\theta_g|^2\mid W,\zeta]<\infty.
\]
Thus $\theta_g$ times the displayed density belongs to
$L^1(d\xi)\cap L^2(d\xi)$.
By Plancherel's identity \citep{tao2006fouriernotes}, we have
\begin{equation}
\begin{aligned}
&\int_{\mathbb R^{d_S}}\left|\mathbb E[\theta_ge^{\mathrm i\xi^\top s}\mid W,\zeta]\right|^2\,ds\\
&=\left(\frac{2\pi W}{\kappa^2}\right)^{d_S/2}
\mathbb E\left[|\theta_g|^2e^{-W\|\xi\|_2^2/(2\kappa^2)}\mid W,\zeta\right]\\
&\le\left(\frac{2\pi W}{\kappa^2}\right)^{d_S/2}
\mathbb E[|\theta_g|^2\mid W,\zeta].
\end{aligned}
\label{eq:upper-conditional-plancherel}
\end{equation}
The equality squares the density; the inequality uses $e^{-x}\le1$.
The conditional Cauchy--Schwarz inequality also gives
$|\mathbb E[\theta_ge^{\mathrm i\xi^\top s}\mid W,\zeta]|^2
\le\mathbb E[|\theta_g|^2\mid W,\zeta]$.

We integrate the conditional Cauchy--Schwarz bound over $\mathcal S$
and restrict \eqref{eq:upper-conditional-plancherel} to $\mathcal S$.
Since $\mu$ is Lebesgue measure and $\mu(\mathcal S)=1$, we obtain
\begin{equation}
\begin{aligned}
&\int_{\mathcal S}
\left|\mathbb E[\theta_ge^{\mathrm i\xi^\top s}\mid W,\zeta]\right|^2\,d\mu(s)\\
&\le\min\left\{1,\left(\frac{2\pi W}{\kappa^2}\right)^{d_S/2}\right\}
\mathbb E[|\theta_g|^2\mid W,\zeta].
\end{aligned}
\label{eq:upper-conditional-state-bound}
\end{equation}
To use this bound, we condition the definition of $T_Sg$
on $(W,\zeta)$ and take absolute values:
\begin{equation}
\begin{aligned}
\sup_a|T_Sg(s,a)|
&=\sup_a\left|\operatorname{Re}\mathbb E\left[
\mathbf1\{\|\zeta\|_2>S\}e^{\mathrm i\zeta^\top a}
\mathbb E[\theta_ge^{\mathrm i\xi^\top s}\mid W,\zeta]\right]\right|\\
&\le\sup_a\mathbb E\left[\mathbf1\{\|\zeta\|_2>S\}
|e^{\mathrm i\zeta^\top a}|
\left|\mathbb E[\theta_ge^{\mathrm i\xi^\top s}\mid W,\zeta]\right|\right]\\
&=\mathbb E\left[\mathbf1\{\|\zeta\|_2>S\}
\left|\mathbb E[\theta_ge^{\mathrm i\xi^\top s}\mid W,\zeta]\right|\right].
\end{aligned}
\label{eq:upper-tail-conditional-bound}
\end{equation}
The first equality uses the tower property: the indicator and
$e^{\mathrm i\zeta^\top a}$ are fixed given $(W,\zeta)$.
The inequality uses $|\operatorname{Re}x|\le|x|$ and the triangle inequality.
The last equality uses $|e^{\mathrm i\zeta^\top a}|=1$, leaving a bound independent of $a$.

Applying the integral Minkowski inequality
\citep[Proposition~4.1]{tao2006fouriernotes}, then the Cauchy--Schwarz inequality, we obtain
\begin{equation}
\begin{aligned}
\int_{\mathcal S}\sup_a|T_Sg(s,a)|^2\,d\mu(s)&\le\left(\mathbb E\left[\mathbf1\{\|\zeta\|_2>S\}
\min\left\{1,\left(\frac{2\pi W}{\kappa^2}\right)^{d_S/4}\right\}
\bigl(\mathbb E[|\theta_g|^2\mid W,\zeta]\bigr)^{1/2}\right]\right)^2\\
&\le B^2\mathbb E\left[\min\left\{1,\left(\frac{2\pi W}{\kappa^2}\right)^{d_S/2}\right\}
\mathbf1\{\|\zeta\|_2>S\}\right]\\
&\le\frac{\sqrt2 B^2}{2^\nu\Gamma(\nu)}\int_0^\infty
\min\left\{1,\left(\frac{2\pi w}{\kappa^2}\right)^{d_S/2}\right\}
w^{\nu-1}e^{-S^2w/(4d_A\kappa^2)}\,dw.
\end{aligned}
\label{eq:upper-mixed-minkowski}
\end{equation}
The first inequality uses the two conditional bounds above; the second
uses $\mathbb E|\theta_g|^2\le B^2$.
For the last, \eqref{eq:upper-action-frequency-tail} gives
$\Pr(\|\zeta\|_2>S\mid W)\le\sqrt2e^{-WS^2/(4d_A\kappa^2)}$;
we insert the density of $W$ and bound its factor $e^{-w/2}$ by one.

\paragraph{Case $d_A<4\nu$ or $d_S\ge d_A$.}
Recall that $F_{\rm vs}=a_{\rm mix}R$ and $\nu_S=\nu+d_S/2$ from \eqref{eq:vs-action-resolution}.
Use \eqref{eq:upper-action-polynomial} with cutoff $S=F_{\rm vs}$ and degree
$k-1$, where $k=\lceil8\mathrm e[d_A+\ln(2T)]\rceil$, and set
$\varepsilon_g(s):=\sup_a|T_{F_{\rm vs}}g(s,a)|+BT^{-2}$.

Each action block with center $a_U$ has sides of length at most $2/F_{\rm vs}$,
so we have $|\zeta^\top(a-a_U)|\le F_{\rm vs}\sqrt{d_A}/F_{\rm vs}\le2\sqrt{d_A}$
on $\{\|\zeta\|_2\le F_{\rm vs}\}$.
By \eqref{eq:upper-action-taylor-bound}, we have
\[
\sup_{s,a}|(I-T_{F_{\rm vs}})g(s,a)-p_g(s,a)|
\le B\frac{(2\sqrt{d_A})^k}{k!}
\le B\left(\frac{2\mathrm e\sqrt{d_A}}k\right)^k
\le B4^{-k}\le BT^{-2}.
\]
The last three inequalities use $k!\ge(k/\mathrm e)^k$,
$k\ge8\mathrm e d_A$, and $k\ln4\ge2\ln T$.
Hence $\varepsilon_g$ bounds the error in \eqref{eq:upper-action-error-split}.
Bounding the minimum in \eqref{eq:upper-mixed-minkowski} by its second entry,
we evaluate the Gamma integral with exponent $\nu_S=\nu+d_S/2$:
\begin{equation}
\begin{aligned}
&\int_{\mathcal S}\sup_a|T_{F_{\rm vs}}g(s,a)|^2\,d\mu(s)\\
&\le\frac{\sqrt2 B^2(2\pi/\kappa^2)^{d_S/2}\Gamma(\nu_S)}{2^\nu\Gamma(\nu)}
\left(\frac{4d_A\kappa^2}{F_{\rm vs}^2}\right)^{\nu_S}\\
&\le B^2(a_{\rm mix}/F_{\rm vs})^{2\nu_S}=B^2R^{-2\nu_S}.
\end{aligned}
\label{eq:upper-mixed-highpass}
\end{equation}
The second inequality substitutes \eqref{eq:mixed-action-scale}; the equality uses $F_{\rm vs}=a_{\rm mix}R$.
Thus we have $\|\varepsilon_g\|_{L^2(\mu)}\le B(R^{-\nu_S}+T^{-2})$.
By \citet[Appendix~C.3, Eq.~(C.12)]{foucartRauhut2013}, we have
$\Gamma(x)\le\sqrt{2\pi}\,x^{x-1/2}e^{-x+1/(12x)}$ for $x>0$.
Since $\nu_S=\nu+d_S/2\ge\nu$, we obtain
\[
\Gamma(\nu_S)^{1/(2\nu_S)}
\le\sqrt{\nu_S}\left(\frac{2\pi}{\nu_S}\right)^{1/(4\nu_S)}
 e^{-1/2+1/(24\nu_S^2)}\le C_\nu\sqrt{\nu_S}.
\]
Using this in \eqref{eq:mixed-action-scale}, with
$d_S/(4\nu_S)\le1/2$ and $\nu/\nu_S\le1$, we obtain
$a_{\rm mix}\le C_{\nu,\kappa}\sqrt{d_A\nu_S}\le C_{\nu,\kappa}(2+d)$.
For the local polynomial dimension $r$, we have
\[
r\le\binom{d_A+k-1}{d_A}
\le[\mathrm e(d_A+k-1)/d_A]^{d_A}\le[C\ln(\mathrm eT)]^{d_A},
\]
by $\binom nm\le(\mathrm en/m)^m$ and $k/d_A\le C\ln(\mathrm eT)$.

\paragraph{Case $d_A\ge4\nu$ and $d_S<d_A$.}
Recall $k_A=\lceil2\nu\rceil+1$ from \eqref{eq:vs-action-resolution}.
Use \eqref{eq:upper-action-polynomial} with cutoff $S=R$ and degree $k_A-1$, and set
$\varepsilon_g(s):=\sup_a|T_Rg(s,a)|
+C_{\nu,\kappa}B d_A^{\nu/2}F_{\rm vs}^{-k_A}R^{k_A-\nu}$.
By \eqref{eq:upper-action-taylor-bound} and \eqref{eq:upper-action-projected-moment},
with cutoff $R$ and resolution $F_{\rm vs}$, for $a\in U$ we have
\[
\begin{aligned}
|(I-T_R)g(s,a)-p_g(s,a)|
&\le\frac{B}{k_A!}
\left(\mathbb E[|\zeta^\top(a-a_U)|^{2k_A}\mathbf1\{\|\zeta\|_2\le R\}]\right)^{1/2}\\
&\le C_{\nu,\kappa}B d_A^{\nu/2}F_{\rm vs}^{-k_A}R^{k_A-\nu}.
\end{aligned}
\]
This bound is uniform over $s,U,a$, so \eqref{eq:upper-action-error-split} is at most $\varepsilon_g(s)$.
To obtain the exponent $\nu d/d_A$, we use $0<2\nu/d_A\le1/2$ and
$\min\{1,x\}\le x^{2\nu/d_A}$ in \eqref{eq:upper-mixed-minkowski}:
\[
\begin{aligned}
\int_{\mathcal S}\sup_a|T_Rg(s,a)|^2\,d\mu(s)
&\le\frac{\sqrt2 B^2(2\pi/\kappa^2)^{\nu d_S/d_A}}{2^\nu\Gamma(\nu)}
\int_0^\infty w^{\nu d/d_A-1}e^{-R^2w/(4d_A\kappa^2)}\,dw\\
&=\frac{\sqrt2 B^2(2\pi/\kappa^2)^{\nu d_S/d_A}\Gamma(\nu d/d_A)}{2^\nu\Gamma(\nu)}
\left(\frac{4d_A\kappa^2}{R^2}\right)^{\nu d/d_A}\\
&\le C_{\nu,\kappa}B^2(2+d)^{2\nu}R^{-2\nu d/d_A}.
\end{aligned}
\]
Here $\nu+\nu d_S/d_A=\nu d/d_A$ determines the Gamma integral.
Since $\nu\le\nu d/d_A<2\nu$, its Gamma factor is bounded by $C_\nu$,
$d_A^{\nu d/d_A}\le(2+d)^{2\nu}$, and the other factors are bounded by $C_{\nu,\kappa}$.
Taking square roots and adding the polynomial error, we obtain
\[
\|\varepsilon_g\|_{L^2(\mu)}
\le C_{\nu,\kappa}B\left[(2+d)^\nu R^{-\nu d/d_A}
+d_A^{\nu/2}F_{\rm vs}^{-k_A}R^{k_A-\nu}\right].
\]
By \eqref{eq:version-space-resolution} and \eqref{eq:vs-action-resolution}, we have
$R=T^{d_A/[d(2\nu+d_A)]}$ and
$F_{\rm vs}^{k_A}=T^{\nu/(2\nu+d_A)}R^{k_A-\nu}$, so we obtain
\[
R^{-\nu d/d_A}=F_{\rm vs}^{-k_A}R^{k_A-\nu}=T^{-\nu/(2\nu+d_A)}.
\]
Substitution and $(2+d)^\nu+d_A^{\nu/2}\le2(2+d)^{2\nu}$ prove the stated error bound.
By \eqref{eq:upper-action-fixed-rank}, we have
$r\le\binom{d_A+k_A-1}{k_A-1}\le C_\nu(2+d)^{\lceil2\nu\rceil}$.
In both cases, the coefficients are real and measurable by \eqref{eq:upper-action-polynomial}.
Dominated convergence, using $\mathbb E|\theta_g|\le B$, makes $T_Sg$ continuous;
the compactness of $\mathcal A$ then makes the envelopes measurable. So we have verified all the regularity conditions and finished the proof.
\end{proof}

\subsection{Proof of Lemma~\ref{lem:upper-residual-conversion}}

\begin{proof}[Proof of Lemma~\ref{lem:upper-residual-conversion}]
For $T=1$, the left side is at most $H$. Fix $T\ge2$ and the realized
policies and functions. By \eqref{eq:historical-occupancy-integral}, our target is
\[
\sum_{t=1}^T\mathbb E_t^h|g_t(z_t^h)|
=\sum_{t=1}^T\int_{\mathcal S}p_t(s)|g_t(s,a_t(s))|\,d\mu(s),
\]
where $p_t\le B_p$ is the state density, $a_t(s)$ is the selected action,
and $\mu$ is Lebesgue measure on $\mathcal S=[0,1]^{d_S}$.
We separate the contributions of $\sigma_t$ and $\varepsilon_t$ first,
then use the historical bound $b$ for the remaining residuals.

Recall $\Omega_T=\ln(\mathrm e+LB_pT)$. For $0<x\le H$, set
$\varepsilon_x=x/(8\sqrt{r\Omega_T})$ and index the action blocks by $j\in[L]$.
We retain residuals above $x$ only where both errors are at most $\varepsilon_x$:
\[
q_{t,j}(s):=p_t(s)\mathbf1\{a_t(s)\in j,\
|g_t(s,a_t(s))|\ge x,\ \sigma_t(s,a_t(s))\le\varepsilon_x,\ \varepsilon_t(s)\le\varepsilon_x\}.
\]
These are density weights, with $\sum_jq_{t,j}\le p_t$, and depend on $x$.
For $z=(s,a)$, the threshold integral is exactly
\[
\int_0^H\mathbf1\{8\sqrt{r\Omega_T}\max\{\sigma_t(z),\varepsilon_t(s)\}\le x\le|g_t(z)|\}\,dx
=\bigl(|g_t(z)|-8\sqrt{r\Omega_T}\max\{\sigma_t(z),\varepsilon_t(s)\}\bigr)_+.
\]
Moreover, the Cauchy--Schwarz inequality and $p_t\le B_p$ give
\[
\sum_t\int p_t\varepsilon_t\,d\mu
\le\sum_t\left(\int p_t^2\,d\mu\right)^{1/2}\|\varepsilon_t\|_{L^2(\mu)}
\le\sum_t\sqrt{B_p\int p_t\,d\mu}\,E=\sqrt{B_p}\,ET.
\]
Using $|g_t|\le8\sqrt{r\Omega_T}(\sigma_t+\varepsilon_t)
+(|g_t|-8\sqrt{r\Omega_T}\max\{\sigma_t,\varepsilon_t\})_+$, we obtain
\begin{equation}
\begin{aligned}
\sum_t\mathbb E_t^h|g_t(z_t^h)|
\le{}&8\sqrt{r\Omega_T}\sum_t\mathbb E_t^h\sigma_t(z_t^h)
+8\sqrt{B_pr\Omega_T}\,ET+H\\
&+\int_{H/T}^H\sum_{t,j}\int q_{t,j}\,d\mu\,dx.
\end{aligned}
\label{eq:upper-conversion-layer-cake}
\end{equation}
The omitted integral over $(0,H/T)$ is at most
$(H/T)\sum_t\int p_t\,d\mu=H$. We now fix $x$ and bound its remaining integrand.

By Lemma~\ref{lem:g-optimal-design}, choose a design $\pi_j$ on each
block closure. Its covariance is positive definite on the polynomial
coefficient space. In a basis $\psi_j$ orthonormal under $\pi_j$, we have
\[
\int\psi_j(a)\psi_j(a)^\top\,d\pi_j(a)=I,\
\sup_{a\in j}\|\psi_j(a)\|_2^2\le r.
\]
Here $I$ is the identity on this block's coefficient space, whose
dimension is $\operatorname{tr}I\le r$.
Write the assumed approximation as
$|g_t(s,a)-\theta_{t,j}(s)^\top\psi_j(a)|\le\varepsilon_t(s)$ for $a\in j$.
On $\{q_{t,j}>0\}$, we have $\varepsilon_t\le\varepsilon_x\le H/8$, and hence obtain
\begin{equation}
\|\theta_{t,j}(s)\|_2^2
=\int|\theta_{t,j}(s)^\top\psi_j(a)|^2\,d\pi_j(a)
\le(H+\varepsilon_t(s))^2\le4H^2.
\label{eq:upper-conversion-coefficient-norm}
\end{equation}
The polynomial is continuous, so its supremum on the block closure
equals its supremum on the block. In this same coefficient space, define
\[
A_{t,j}(s):=L^{-1}I+\sum_{i<t:q_{i,j}(s)>0}q_{i,j}(s)
\psi_j(a_i(s))\psi_j(a_i(s))^\top\succeq L^{-1}I,
\]
and $w_{t,j}(s)^2:=\psi_j(a_t(s))^\top A_{t,j}(s)^{-1}\psi_j(a_t(s))$
when $q_{t,j}>0$, with $w_{t,j}=0$ otherwise.
Writing $S_j(s)=\sum_tq_{t,j}(s)$, we apply
Lemma~\ref{lem:elliptical-potential} to the weighted features
$\sqrt{q_{t,j}}\psi_j(a_t)$, taken as zero when $q_{t,j}=0$:
\begin{equation}
\sum_t\min\{1,q_{t,j}w_{t,j}^2\}
\le2\ln\det(LA_{T+1,j})
\le2r\ln\left(1+\frac{\operatorname{tr}(LA_{T+1,j}-I)}r\right)
\le2r\ln(1+LS_j).
\label{eq:upper-pointwise-logdet}
\end{equation}
The second inequality uses the weighted arithmetic--geometric mean
inequality: each eigenvalue has weight $1/r$, and the number $1$ has
weight $1-\operatorname{tr}I/r\ge0$. These weights sum to one.
The last inequality uses
$\operatorname{tr}(LA_{T+1,j}-I)=L\sum_tq_{t,j}\|\psi_j(a_t)\|_2^2\le LrS_j$.
By Jensen's inequality and $\sum_j\int S_j\,d\mu\le T$, we have
\begin{equation}
\sum_{t,j}\int\min\{1,q_{t,j}w_{t,j}^2\}\,d\mu
\le2r\sum_j\int\ln(1+LS_j)\,d\mu
\le2Lr\ln\left(1+\sum_j\int S_j\,d\mu\right)
\le2Lr\ln(1+T).
\label{eq:upper-integrated-logdet}
\end{equation}
Also, $S_j\le B_pT$ implies $\ln(1+LS_j)\le\Omega_T$.

To use the historical assumption for the current $g_t$, fix $s,j$ with
$q_{t,j}(s)>0$. Whenever $i<t$ and $q_{i,j}(s)>0$, we have
\[
|\theta_{t,j}(s)^\top\psi_j(a_i(s))|
\le|g_t(s,a_i(s))|+\varepsilon_t(s)
\le(|g_t(s,a_i(s))|-\sigma_i(s,a_i(s)))_++2\varepsilon_x.
\]
Here the current weight ensures $\varepsilon_t\le\varepsilon_x$ and the historical
weight ensures $\sigma_i(s,a_i(s))\le\varepsilon_x$.
Consequently, we obtain
\begin{equation}
\begin{aligned}
|g_t(s,a_t(s))|
&\le\varepsilon_x+w_{t,j}
\left[\frac{\|\theta_{t,j}\|_2^2}{L}
+\sum_{i<t}q_{i,j}|\theta_{t,j}^\top\psi_j(a_i)|^2\right]^{1/2}\\
&\le\varepsilon_x+w_{t,j}
\left[\frac{4H^2}{L}
+\sum_{i<t}q_{i,j}(|g_t(s,a_i(s))|-\sigma_i(s,a_i(s)))_+^2\right]^{1/2}
+2\varepsilon_x\sqrt{S_j}\,w_{t,j}.
\end{aligned}
\label{eq:upper-matrix-gap}
\end{equation}
The first inequality uses the Cauchy--Schwarz inequality for $A_{t,j}$;
the second uses the Euclidean triangle inequality,
\eqref{eq:upper-conversion-coefficient-norm}, and $\sum_{i<t}q_{i,j}\le S_j$.
All zero-weight terms are omitted.
Since $q_{t,j}\le B_p$ and $\sum_jq_{i,j}\le p_i$, we have
\begin{equation}
\begin{aligned}
&\sum_{t,j}\int q_{t,j}
\left[\frac{4H^2}{L}
+\sum_{i<t}q_{i,j}(|g_t(s,a_i(s))|-\sigma_i(s,a_i(s)))_+^2\right]d\mu\\
&\le B_p\sum_t\left[4H^2+
\sum_{i<t}\mathbb E_i^h(|g_t(z_i^h)|-\sigma_i(z_i^h))_+^2\right]
\le B_pT(4H^2+b).
\end{aligned}
\label{eq:upper-conversion-weighted-history}
\end{equation}
The first inequality uses \eqref{eq:historical-occupancy-integral};
the second applies the assumed historical bound $b$.

We sum \eqref{eq:upper-matrix-gap} over $\{q_{t,j}w_{t,j}^2\le1\}$
and control its complement by \eqref{eq:upper-integrated-logdet}:
\[
x\sum_{t,j}\int q_{t,j}\mathbf1\{q_{t,j}w_{t,j}^2>1\}\,d\mu
\le B_px\sum_{t,j}\int\min\{1,q_{t,j}w_{t,j}^2\}\,d\mu
\le2B_pxLr\ln(1+T).
\]
On $\{q_{t,j}w_{t,j}^2\le1\}$, the Cauchy--Schwarz inequality,
\eqref{eq:upper-conversion-weighted-history}, and
\eqref{eq:upper-integrated-logdet} give
\begin{align*}
&\sum_{t,j}\int q_{t,j}w_{t,j}
\left[\frac{4H^2}{L}
+\sum_{i<t}q_{i,j}(|g_t(s,a_i(s))|-\sigma_i(s,a_i(s)))_+^2\right]^{1/2}
\mathbf1\{q_{t,j}w_{t,j}^2\le1\}\,d\mu\\
&\le\sqrt{B_pT(4H^2+b)}
\left(\sum_{t,j}\int\min\{1,q_{t,j}w_{t,j}^2\}\,d\mu\right)^{1/2}
\le\sqrt{2B_pLrT(4H^2+b)\ln(1+T)}.
\end{align*}
For the two terms containing $\varepsilon_x$, we use \eqref{eq:upper-pointwise-logdet}:
\[
\sum_{t:q_{t,j}w_{t,j}^2\le1}q_{t,j}\sqrt{S_j}\,w_{t,j}
\le\sqrt{S_j}\left(\sum_tq_{t,j}\right)^{1/2}
\left(\sum_t\min\{1,q_{t,j}w_{t,j}^2\}\right)^{1/2}
\le S_j\sqrt{2r\Omega_T}.
\]
Thus their sum is at most
$\varepsilon_x(1+2\sqrt{2r\Omega_T})\sum_{t,j}\int q_{t,j}\,d\mu$.
Our choice $\varepsilon_x=x/(8\sqrt{r\Omega_T})$ ensures
\[
\varepsilon_x(1+2\sqrt{2r\Omega_T})
=x\left(\frac1{8\sqrt{r\Omega_T}}+\frac{\sqrt2}{4}\right)
\le\frac{1+2\sqrt2}{8}x\le\frac x2,
\]
where $r\Omega_T\ge1$. Since $|g_t(s,a_t(s))|\ge x$ on $\{q_{t,j}>0\}$,
we combine the three estimates and subtract this contribution to obtain
\begin{equation}
\frac x2\sum_{t,j}\int q_{t,j}\,d\mu
\le2B_pxLr\ln(1+T)
+\sqrt{2B_pLrT(4H^2+b)\ln(1+T)}.
\label{eq:upper-conversion-threshold-bound}
\end{equation}
Finally, divide \eqref{eq:upper-conversion-threshold-bound} by $x/2$
and substitute into \eqref{eq:upper-conversion-layer-cake}.
Since $\int_{H/T}^H dx\le H$ and $\int_{H/T}^H x^{-1}\,dx=\ln T$, we obtain
\[
\begin{aligned}
\sum_t\mathbb E_t^h|g_t(z_t^h)|
\le{}&8\sqrt{r\Omega_T}\sum_t\mathbb E_t^h\sigma_t(z_t^h)
+8\sqrt{B_pr\Omega_T}\,ET+H\\
&+4B_pHLr\ln(1+T)
+2\ln T\sqrt{2B_pLrT(4H^2+b)\ln(1+T)}.
\end{aligned}
\]
Using $\ln T\le G_T=1+\lceil\log_2(2T)\rceil$ and
$4H^2+b\le4(H^2+b)$ proves \eqref{eq:upper-residual-conversion}.
All bounds hold for each realized sequence, including data-dependent functions.
\end{proof}
\subsection{Proof of Proposition~\ref{prop:upper-continuation-error}}

\begin{proof}[Proof of Proposition~\ref{prop:upper-continuation-error}]
The algorithm selects $f_t$ from a version space defined by historical
observations. We show that these selected candidates have small cumulative
Bellman residuals under the state--action distributions induced by the
policies executed in their respective episodes.
Since $H\le\mathfrak B$, it suffices to prove, for each $2\le h\le H$,
\begin{equation}
\sum_{t=1}^T\mathbb E_t^h|\Delta_{f_t}^h(z_t^h)|
\le C_{\nu,\kappa}\mathfrak B^2(2+d)^{\rho+7\nu+6}
\iota_T^{2+6\nu}T^{e_*}.
\label{eq:upper-continuation-one-stage}
\end{equation}
We work on the intersection of the event constructed in the proof of
Lemma~\ref{lem:upper-historical-residual} and the event of
Lemma~\ref{lem:uniform-confidence}; the total failure probability is at most
$\delta/2+\delta/8=5\delta/8$, simultaneously over all stages.
The former includes the regression event of Lemma~\ref{lem:local-ridge}
with failure probability $\delta/8$, so \eqref{eq:coarse-feasibility} gives $Q_*\in\mathcal F_t^1$.
By \eqref{eq:innovation-statistic-bound} with $\Delta_{Q_*}^h=0$ and
\eqref{eq:version-space}, we then have $Q_*\in\mathcal F_t$ for every $t$.

We first apply Lemma~\ref{lem:upper-historical-residual} to the selected
candidate $f_t\in\mathcal F_t$. Recall
$\Delta_{f_t}^h=f_t^h-\mathcal B^hV_{f_t}^{h+1}$ and
$\sigma_i^h=\min\{H,2\tau_i^h\}$, with $\tau_i^h$ defined in
\eqref{eq:innovation-threshold}.
Recall also $B_{\rm res}=1+2\Lambda(B_r+HB_p)$ from
\eqref{eq:vs-residual-width-parameters}, $\ell_\delta=\ln(8HT/\delta)$,
and the entropy-integral bound $J$ from Lemma~\ref{lem:continuation-entropy}.
Here $\mathbb E_i^h$ is the occupancy integral in
\eqref{eq:historical-occupancy-integral}, with $f_t$ held fixed when $i<t$.
By \eqref{eq:upper-historical-residual}, the same bound for every $t\le T$ yields
\[
\begin{aligned}
\left\{H^2+\sum_{i<t}\mathbb E_i^h
[(|\Delta_{f_t}^h(z_i^h)|-\sigma_i^h(z_i^h))_+^2]\right\}^{1/2}
&\le\left\{H^2+C\left[J^2+H^2\left(1+
\frac{B_{\rm res}^2}{\alpha_{\rm vs}^2}+\ell_\delta\right)\right]\right\}^{1/2}\\
&\le C\left[J+H\sqrt{1+\ell_\delta}
+\frac{HB_{\rm res}}{\alpha_{\rm vs}}\right].
\end{aligned}
\]
The second inequality separates the nonnegative terms under the square root.
The threshold sum $\sum_i\mathbb E_i^h\sigma_i^h(z_i^h)$ satisfies
\eqref{eq:upper-slack-sum} on the same event. Here
$M=|\mathcal C|\le R^d\le T$ for the partition after
\eqref{eq:version-space-resolution}, $\beta=\beta(M,\delta/8)$ from
\eqref{eq:regression-confidence-parameter}, and $K_{\rm c}=2D\ln(1+T/D)$.

Next, \eqref{eq:upper-residual-rkhs-norm} gives
$\|\Delta_{f_t}^h\|_{\mathcal H_0}\le2\Lambda(B_r+HB_p)$.
With this norm bound, we apply Lemma~\ref{lem:upper-continuation-approximation}
on the $L=L_{\rm vs}=|\mathscr U(F_{\rm vs})|$ blocks in
\eqref{eq:vs-action-resolution}. It gives fixed polynomial spaces of
dimension at most $r$ and measurable approximants satisfying
\[
\sup_{a\in\mathcal A}|\Delta_{f_t}^h(s,a)-p_{\Delta_{f_t}^h}(s,a)|
\le\varepsilon_{\Delta_{f_t}^h}(s),\
\|\varepsilon_{\Delta_{f_t}^h}\|_{L^2(\mu)}\le E,
\]
where $r$ and $E$ are the bounds in that lemma, uniform in $t$.

We now check the conditions in Lemma~\ref{lem:upper-residual-conversion}.
By Lemma~\ref{lem:state-occupancy-bound}, the stage-$h$ state densities
satisfy $p_i\le B_p$. By \eqref{eq:candidate-class} and
\eqref{eq:bellman-backup}, both $f_t^h$ and $\mathcal B^hV_{f_t}^{h+1}$
take values in $[0,H-h+1]$, so $|\Delta_{f_t}^h|\le H$; also $\sigma_i^h\ge0$.
All conditions now hold. We apply \eqref{eq:upper-residual-conversion}
with the uniform historical budget in \eqref{eq:upper-historical-residual}
and the threshold sum \eqref{eq:upper-slack-sum}. With its definitions
$G_T=1+\lceil\log_2(2T)\rceil$ and $\Omega_T=\ln(\mathrm e+L_{\rm vs}B_pT)$,
we obtain
\begin{equation}
\begin{aligned}
&\sum_t\mathbb E_t^h|\Delta_{f_t}^h(z_t^h)|\\
&\le C\bigg\{G_T\sqrt{B_pL_{\rm vs}rT\ln(1+T)}
\left[J+H\sqrt{1+\ell_\delta}+\frac{HB_{\rm res}}{\alpha_{\rm vs}}\right]\\
&\ +\sqrt{r\Omega_T}\left[(B_{\phi}H+\beta)\sqrt{MTK_{\rm c}}
+\alpha_{\rm vs}q_0A_{\rm vs}(1+\sqrt{K_{\rm c}})T+H\ell_\delta\right]\\
&\ +\sqrt{B_pr\Omega_T}\,ET+H\bigg\}.
\end{aligned}
\label{eq:upper-continuation-interface}
\end{equation}
Here $\sum_t\mathbb E_t^h|\Delta_{f_t}^h(z_t^h)|\le HT$ also holds.
Since the other terms in \eqref{eq:upper-residual-conversion} are nonnegative,
we absorb its linear term using
\[
\min\{HT,CB_pHL_{\rm vs}r\ln(1+T)\}
\le\sqrt C\,H\sqrt{B_pL_{\rm vs}rT\ln(1+T)}.
\]

We next substitute \eqref{eq:version-space-tuning} into the two terms
containing $\alpha_{\rm vs}$ in \eqref{eq:upper-continuation-interface}:
\[
\begin{aligned}
&G_T\sqrt{B_pL_{\rm vs}rT\ln(1+T)}\,
\frac{HB_{\rm res}}{\alpha_{\rm vs}}
+\sqrt{r\Omega_T}\,\alpha_{\rm vs}q_0A_{\rm vs}(1+\sqrt{K_{\rm c}})T\\
&=\sqrt r\left[G_T\sqrt{\ln(1+T)}+\sqrt{\Omega_T}\right]
\sqrt{HB_{\rm res}q_0A_{\rm vs}(1+\sqrt{K_{\rm c}})}
(B_pL_{\rm vs})^{1/4}T^{3/4}.
\end{aligned}
\]
Thus we need to bound $q_0A_{\rm vs}(1+\sqrt{K_{\rm c}})$.
Recall $\mathfrak B=1+H+B_p+\Lambda(B_r+HB_p)$ and
$\iota_T=\ln[\mathrm e(2+d)T\mathfrak B/\delta]$ from
\eqref{eq:regret-parameter-summary}, and $\rho=\lfloor\nu\rfloor$,
$D=\binom{d+\rho}{\rho}$, $B_{\phi}^2=3^\rho D$ from Section~\ref{sec:skovi}.

We first bound $A_{\rm vs}$ from \eqref{eq:vs-remainder-envelope}.
In the partition after \eqref{eq:version-space-resolution},
$m_R=\lfloor\log_2R\rfloor$ and each side is $2^{-m_R}$ or $2^{-m_R-1}$,
so every side lies in $[1/(2R),2/R]$.
By Lemma~\ref{lem:feature-remainder} with $a=2/R$ and
\eqref{eq:upper-polynomial-dimension}, we have
\[
\begin{aligned}
R^{-\nu}\le A_{\rm vs}
&=R^{-\nu}+\max_{c\in\mathcal C}\sup_{z\in\overline c}\|\Psi_c(z)\|\\
&\le R^{-\nu}+C_\nu(1+B_{\phi})(2\kappa\sqrt d/R)^\nu
\le C_{\nu,\kappa}(2+d)^{(\rho+\nu)/2}R^{-\nu}.
\end{aligned}
\]
For the explicit Gaussian envelope, recall
$\alpha,C_{\Psi},n_{\Psi}$ from \eqref{eq:vs-gaussian-cover-parameters}.
By \eqref{eq:upper-polynomial-dimension},
$B_{\phi}\le C_\nu(2+d)^{\rho/2}$, so
$C_{\Psi}\le C_{\nu,\kappa}R(2+d)^{(\rho+1)/2}$ and
$\ln n_{\Psi}\le C_{\nu,\kappa}\ln[\mathrm e(2+d)R]$.
Since $|\mathcal C|\le R^d$ and $R\le T^{1/\nu}$,
\[
\ln(1+2|\mathcal C|n_{\Psi}^d)
\le\ln3+d\ln R+d\ln n_{\Psi}
\le C_{\nu,\kappa}d\iota_T.
\]
Substituting this estimate into \eqref{eq:vs-gaussian-envelope} and
using the bound on $A_{\rm vs}$ above, we obtain
\begin{equation}
\begin{aligned}
q_0A_{\rm vs}
&=A_{\rm vs}\left[1+48\sqrt{\ln(1+2|\mathcal C|n_{\Psi}^d)}
+12\sqrt{\pi d/\alpha}\right]
\le C_{\nu,\kappa}A_{\rm vs}\sqrt{d\iota_T}\\
&\le C_{\nu,\kappa}(2+d)^{(\rho+\nu+1)/2}\iota_T^{1/2}R^{-\nu}.
\end{aligned}
\label{eq:upper-tail-gaussian-size}
\end{equation}
By \eqref{eq:vs-residual-width-parameters} and \eqref{eq:upper-polynomial-dimension}, we have
$1+\sqrt{K_{\rm c}}\le C_\nu(2+d)^{\rho/2}\iota_T^{1/2}$.
Since $H,B_p\le\mathfrak B$ and $B_{\rm res}\le2\mathfrak B$,
we have $(HB_{\rm res})^{1/2}B_p^{1/4}\le\sqrt2\,\mathfrak B^{5/4}$.
Substituting these bounds into the common factor above yields
\begin{equation}
\sqrt{HB_{\rm res}q_0A_{\rm vs}(1+\sqrt{K_{\rm c}})}
(B_pL_{\rm vs})^{1/4}T^{3/4}
\le C_{\nu,\kappa}\mathfrak B^{5/4}(2+d)^{(2\rho+\nu+1)/4}
\iota_T^{1/2}L_{\rm vs}^{1/4}R^{-\nu/2}T^{3/4}.
\label{eq:upper-alpha-substitution}
\end{equation}

For the remaining coefficients, \eqref{eq:upper-common-parameters} gives
$J\le C_{\nu,\kappa}\mathfrak B(2+d)^3$.
Using $\ln(64MHT/\delta)\le C\iota_T$ for $M\le T$ and
\eqref{eq:upper-polynomial-dimension} in \eqref{eq:regression-confidence-parameter},
we obtain
$B_{\phi}H+\beta\le C_{\nu,\kappa}\mathfrak B(2+d)^{\max\{\rho/2,3\}}\iota_T^{1/2}$.
Finally, \eqref{eq:vs-action-resolution} gives
$L_{\rm vs}\le[C_{\nu,\kappa}(2+d)]^{d_A}T$, so
$\Omega_T\le\ln(\mathrm e+[C_{\nu,\kappa}(2+d)]^{d_A}B_pT^2)
\le C_{\nu,\kappa}(2+d)\iota_T$, and $G_T\le C\iota_T$.
Substituting into \eqref{eq:upper-continuation-interface}, with $M\le R^d$, gives
\begin{equation}
\begin{aligned}
\sum_t\mathbb E_t^h|\Delta_{f_t}^h(z_t^h)|
\le{}&C_{\nu,\kappa}\mathfrak B^2(2+d)^{\rho+\nu+4}\iota_T^2\sqrt r\\
&\times\left[\sqrt{L_{\rm vs}T}+\sqrt{TR^d}
+L_{\rm vs}^{1/4}R^{-\nu/2}T^{3/4}+\frac{ET}{\mathfrak B}+1\right].
\end{aligned}
\label{eq:upper-continuation-raw}
\end{equation}
The four terms come from the historical budget, regression uncertainty,
the balance in \eqref{eq:upper-alpha-substitution}, and approximation error,
respectively. Their common coefficient is valid because
\[
\max\left\{3,\max\{\rho,3+\rho/2\}+\frac12,
\frac{2\rho+\nu+3}{4}\right\}\le\rho+\nu+4,
\]
the logarithmic powers are at most two, and $\sqrt{B_p}E\le\mathfrak B E$.
Recall $e_A=(\nu+d_A)/(2\nu+d_A)$,
$e_B=(\nu+d_A+3d_S/2)/(2\nu+d_A+2d_S)$, and $e_*=\max\{e_A,e_B\}$
from Section~\ref{sec:guarantees}.

\paragraph{The case $d_A<4\nu$ or $d_S\ge d_A$.}
Lemma~\ref{lem:upper-continuation-approximation} gives
$r\le(C\iota_T)^{d_A}$ and
$E\le2\Lambda(B_r+HB_p)(R^{-\nu_S}+T^{-2})$, where $\nu_S=\nu+d_S/2$.
The choice \eqref{eq:version-space-resolution} balances regression and approximation:
\[
\sqrt{TR^d}=TR^{-\nu_S}
\iff R^{d+2\nu_S}=T
\iff R=T^{1/(2\nu+2d_S+d_A)}.
\]
Both terms equal $T^{1-\nu_S/(2\nu+2d_S+d_A)}=T^{e_B}$.
By \eqref{eq:vs-action-resolution}, we have $L_{\rm vs}\le(a_{\rm mix}R)^{d_A}$,
where $1\le a_{\rm mix}\le C_{\nu,\kappa}(2+d)$ by
Lemma~\ref{lem:upper-continuation-approximation}. Thus we have
\[
\sqrt{L_{\rm vs}T}\le a_{\rm mix}^{d_A/2}\sqrt{TR^{d_A}}
\le a_{\rm mix}^{d_A/2}T^{e_B},
\]
since $d_A\le d$. The balanced term satisfies
\[
L_{\rm vs}^{1/4}R^{-\nu/2}T^{3/4}
\le a_{\rm mix}^{d_A/4}R^{(d_A-2\nu)/4}T^{3/4}
=a_{\rm mix}^{d_A/4}T^{e_B},
\]
where $3/4+(d_A-2\nu)/[4(2\nu+2d_S+d_A)]=e_B$.
Finally, the approximation term obeys
\[
\frac{ET}{\mathfrak B}\le2(TR^{-\nu_S}+T^{-1})\le4T^{e_B}.
\]
Since $a_{\rm mix},T^{e_B}\ge1$, the bracket in \eqref{eq:upper-continuation-raw}
is at most $8a_{\rm mix}^{d_A/2}T^{e_B}$.
For polynomial dimension dependence, we combine this bound with $HT$
using $\theta=(1-e_*)/(1-e_B)\in(0,1]$, so $\theta e_B+1-\theta=e_*$.
If $d_A<4\nu$, then we have $\theta d_A<4\nu$; if $d_S\ge d_A$, then we have
\[
1-e_B=\frac{\nu+d_S/2}{2\nu+2d_S+d_A}\ge\frac16,\
\theta d_A\le6d_A(1-e_A)=\frac{6\nu d_A}{2\nu+d_A}\le6\nu.
\]
Their weighted geometric mean has dimension factor
\[
(\sqrt r\,a_{\rm mix}^{d_A/2})^\theta
\le[C_{\nu,\kappa}(2+d)\iota_T]^{\theta d_A/2}
\le C_{\nu,\kappa}(2+d)^{3\nu}\iota_T^{3\nu}.
\]
Since $\mathfrak B^{2\theta}H^{1-\theta}\le\mathfrak B^2$,
$\theta(\rho+\nu+4)+3\nu\le\rho+4\nu+4$, and $2\theta+3\nu\le2+3\nu$,
we obtain \eqref{eq:upper-continuation-one-stage}.

\paragraph{The case $d_A\ge4\nu$ and $d_S<d_A$.}
Lemma~\ref{lem:upper-continuation-approximation}, with degree
$k_A-1=\lceil2\nu\rceil$, gives
$r\le C_\nu(2+d)^{\lceil2\nu\rceil}$ and
$ET/\mathfrak B\le C_{\nu,\kappa}(2+d)^{2\nu}T^{e_A}$.
To match this approximation rate, \eqref{eq:version-space-resolution} sets
\[
\sqrt{TR^d}=T^{e_A}
\iff R=T^{(2e_A-1)/d}=T^{d_A/[d(2\nu+d_A)]}.
\]
By \eqref{eq:vs-action-resolution}, we have
\[
F_{\rm vs}=(R^{\nu d/d_A}R^{k_A-\nu})^{1/k_A}
=R^{1+\nu d_S/(k_Ad_A)},\
L_{\rm vs}\le F_{\rm vs}^{d_A}=R^{d_A+\nu d_S/k_A}\le R^d,
\]
since $k_A>2\nu$. Hence we have $\sqrt{L_{\rm vs}T}\le\sqrt{TR^d}=T^{e_A}$.
For the balanced term, the same substitution yields
\[
L_{\rm vs}^{1/4}R^{-\nu/2}T^{3/4}
\le\sqrt{TR^d}\,R^{\frac{d_S}{4}(\nu/k_A+2\nu/d_A-1)}
\le\sqrt{TR^d}=T^{e_A},
\]
because $\nu/k_A<1/2$, $2\nu/d_A\le1/2$, and $R\ge1$.
The rank and approximation factors in \eqref{eq:upper-continuation-raw} satisfy
\[
(2+d)^{\rho+\nu+4}\sqrt r\,(2+d)^{2\nu}
\le C_\nu(2+d)^{\rho+3\nu+4+\lceil2\nu\rceil/2}
\le C_\nu(2+d)^{\rho+7\nu+6},
\]
where $\lceil2\nu\rceil/2\le\nu+1/2$. Combining this with the four bounds
and $e_A\le e_*$ proves \eqref{eq:upper-continuation-one-stage}.
Summing over $h=2,\ldots,H$ and using $H\le\mathfrak B$ establishes
\eqref{eq:upper-continuation-error}.
\end{proof}

\end{document}